%% file: arXiv.tex
\documentclass[twoside]{article}

\input{preamble_arXiv}

\usepackage{ifthen}

\newboolean{withSmall}
\setboolean{withSmall}{false}
\setboolean{withSmall}{true} 
\ifthenelse{\boolean{withSmall}}{   }{}

\newcommand\withMaybeSmall{\ifthenelse{\boolean{withSmall}}{ \small }{}}

\newboolean{arXivMode}
\setboolean{arXivMode}{true}

\usepackage{cleveref}

\begin{document}

\title{An extension of the angular synchronization problem to the heterogeneous setting\footnote{This work was supported by EPSRC grant EP/N510129/1}}

\author{Mihai~Cucuringu\footnotemark[1] \footnotemark[2]\;, 
Hemant Tyagi\footnotemark[3] \footnotemark[4]}

 \renewcommand{\thefootnote}{\fnsymbol{footnote}}
\footnotetext[1]{Department of Statistics and Mathematical Institute, University of Oxford, Oxford, UK. Email: mihai.cucuringu@stats.ox.ac.uk}
\footnotetext[2]{The Alan Turing Institute, London, UK}
\footnotetext[3]{Inria, Univ. Lille, CNRS, UMR 8524 - Laboratoire Paul Painlev\'{e}, F-59000.   Email: hemant.tyagi@inria.fr}
\footnotetext[4]{Authors are listed in alphabetical order.}

\renewcommand{\thefootnote}{\arabic{footnote}}

\maketitle

\begin{abstract}
Given an undirected measurement graph $G = ([n], E)$, 
the classical angular synchronization problem consists of recovering unknown angles $\theta_1,\dots,\theta_n$ from a collection of noisy pairwise measurements of the form $(\theta_i - \theta_j) \mod 2\pi$, for each $\set{i,j} \in E$. This problem arises in a variety of applications, including computer vision, time synchronization of distributed networks, and ranking from preference relationships. 
In this paper, we consider a generalization 
to the setting where there exist $k$ unknown groups of angles $\theta_{l,1}, \dots,\theta_{l,n}$, for $l=1,\dots,k$. For each $\setij \in E$, we are given noisy pairwise measurements of the form $\theta_{\ell,i}  - \theta_{\ell,j}$ for an \emph{unknown} $\ell \in \{1,2,\ldots,k\}$. This can be thought of as a natural extension of the angular synchronization problem to the heterogeneous setting of multiple groups of angles, where the measurement graph has an unknown edge-disjoint decomposition $G = G_1 \cup G_2 \ldots \cup G_k$, where the $G_i$'s denote the subgraphs of edges corresponding to each group.  
We propose a probabilistic generative model for this problem, along with a spectral algorithm for which we provide  a detailed theoretical analysis in terms of robustness against both sampling sparsity and noise. The theoretical findings are complemented by a comprehensive set of numerical experiments, showcasing the efficacy of our algorithm under various parameter regimes. Finally, we consider an application of bi-synchronization to the graph realization problem, and provide along the way an iterative graph disentangling procedure that uncovers the subgraphs $G_i$, $i=1,\ldots,k$ which is of independent interest, as it is shown to improve the final recovery accuracy across all the experiments considered. 
\end{abstract}


{
  \hypersetup{linkcolor=black}
  \tableofcontents
}

\section{Introduction}
\input{S_introduction}

\section{Bi-synchronization: Synchronization with two groups of angles} \label{sec:bi_sync} 
\input{S_BiSync}

\section{$k$-synchronization: Synchronization with $k$ groups of angles} \label{sec:k_sync}
\input{S_k_Sync}

\section{Numerical Experiments}   \label{sec:experimentsSync} 
\input{S_experiments}

\section{An application to the graph realization problem}  \label{sec:GRP} 
\input{S_GRP}

\FloatBarrier
\section{Conclusion and future directions} \label{sec:conclusion}
\input{S_conclusion}

 
\bibliographystyle{siam}
\bibliography{aa_bib_sync}

\appendix
\input{appendix}

\end{document}

%% file: preamble_arXiv.tex
\usepackage[margin=1in]{geometry}
\usepackage{framed} 
\usepackage{mathtools}
\usepackage[]{amsmath,amssymb,epsfig}
\usepackage{amsthm}
\usepackage{amsmath}
\usepackage{amssymb}
\usepackage{graphicx}
\usepackage{epstopdf}
\usepackage{comment}
\usepackage{array}
\usepackage{algorithm}
\usepackage{url}
\usepackage{ifthen}
\usepackage{wrapfig}
\usepackage{lscape}
\usepackage{algpseudocode}
\usepackage{setspace}
\usepackage{multicol}
\usepackage{multirow}
\usepackage{color}
\usepackage{colortbl}
\usepackage{xcolor}
\usepackage{rotating}
\usepackage{caption}
\usepackage{float}
\usepackage{ifthen}
\usepackage{placeins}
\usepackage{framed}

\usepackage[%
    font={small,sf},
    labelfont=bf,
    format=hang,    
    format=plain,
    margin=0pt,
    width=0.8\textwidth,
]{caption}

\usepackage[list=true]{subcaption}

\newtheorem{theorem}{Theorem}

\newtheorem{definition}{Definition}

\newtheorem{lemma}{Lemma}

\newtheorem{proposition}{Proposition}
\newtheorem{remark}{Remark}

\numberwithin{equation}{section}

\newcommand{\ER}{Erd\H{o}s-R\'enyi}   

\DeclareMathOperator{\tr}{tr}

\newcommand{\norm}[1]{\|{#1}\|}
\newcommand{\abs}[1]{|{#1}|}

\newcommand{\set}[1]{\left\{{#1}\right\}}
\newcommand{\dotprod}[2]{\langle#1,#2\rangle}

\newcommand{\expec}{\ensuremath{\mathbb{E}}}
\newcommand{\matR}{\ensuremath{\mathbb{R}}}

\newcommand{\prob}{\ensuremath{\mathbb{P}}}

\newcommand{\veca}{\mathbf a}

\newcommand{\real}{\text{Re}}
\newcommand{\imag}{\text{Im}}

\newcommand{\vtil}{\ensuremath{\widetilde{v}}}

\newcommand{\lamtil}{\ensuremath{\widetilde{\lambda}}}

\newcommand{\immat}{\ensuremath{W}}
\newcommand{\immattil}{\ensuremath{\widetilde{W}}}
\newcommand{\sigtil}{\ensuremath{\bar{\sigma}}}
\newcommand{\xbar}{\ensuremath{\bar{x}}}

\newcommand{\varep}{\ensuremath{\varepsilon}}
\newcommand{\ergen}{\ensuremath{\psi}}
\newcommand{\varepbar}{\ensuremath{\bar{\varepsilon}}}
\newcommand{\setij}{\ensuremath{\set{i,j}}}
\newcounter{ale}

\newenvironment{liste}{\begin{itemize}}{\end{itemize}}
\newcommand{\aliste}{\begin{liste} \setcounter{ale}{1}}
\newcommand{\zliste}{\end{liste}}

\usepackage[colorlinks=true,linkcolor=red,filecolor=green,citecolor=red]{hyperref}

\newcounter{noteMCctr} 

\newcounter{noteHTctr} 

 \usepackage[margin=1in]{geometry}

\usepackage{bbold}

%% file: S_introduction.tex
Finding group elements from noisy pairwise measurements of their ratios is known as the \textit{group synchronization} problem. For example, the synchronization problem over the special orthogonal group $SO(d)$ consists of estimating a set of $n$ unknown $d\times d$  rotation matrices $R_1,\ldots,R_n \in SO(d)$ from noisy measurements of a subset of the pairwise ratios $R_i R_j^{-1}$
\begin{align*}
	&  \underset{R_1,\ldots,R_n \in SO(d)}{\text{minimize}}  \sum_{\setij \in E} w_{ij} \|  R_i R_j^{-1} - R_{ij} \|_{F}^{2},
\end{align*}
where $||\cdot||_F$ denotes the Frobenius norm, and $w_{ij}$ are non-negative weights denoting the confidence in the noisy pairwise measurements $R_{ij}$. 
Here, $E$ is the set of pairs for which a ratio of group elements is available, and can be realized as the edge set of an  undirected graph $G=([n],E)$ with vertices corresponding to the group elements. 

Spectral and semidefinite programming (SDP) relaxations for solving an instance of the above synchronization problem were introduced and analyzed by Singer \cite{sync} in the context of angular synchronization, over the group SO(2) of planar rotations. Therein, one is asked to estimate $n$ unknown angles $\theta_1,\ldots,\theta_n \in [0,2\pi)$ given $m$ noisy measurements $\delta_{ij}$ of their offsets $\theta_i - \theta_j \mod 2\pi$.  The difficulty of the problem is amplified on one hand by the amount of noise in the offset measurements, and on the other hand by the fact that $m \ll {n \choose 2}$, i.e., only a very small subset of all possible pairwise offsets are measured. In general, one may consider other groups $\mathcal{G}$ ($\mathbb{Z}_2$, SO($d$), O($d$)) for which there are available noisy measurements $g_{ij}$ of ratios between  group elements
\begin{equation}
 g_{ij} = g_i g_j^{-1}; \;  g_i, g_j \in \mathcal{G}; \;  \set{i,j} \in E.
\end{equation}
Whenever the group $\mathcal{G}$ is compact and has a real or complex representation, one may construct a real or Hermitian matrix (which may also be construed  as a matrix of matrices) where the element in position $\set{i,j}$ is the matrix representation of the measurement $g_{ij}$ (which can also be a matrix of size $1 \times 1$, as is the case for synchronization over $\mathbb{Z}_2$), or the zero matrix if there is no direct measurement for the ratio of $g_i$ and $g_j$. For example, the rotation group SO(3) has a real representation using $3 \times 3$ rotation matrices, and the group SO(2) of planar rotations has a complex representation as points on the unit circle. 

\paragraph{Applications} 
Instances of the group synchronization problem have emerged in a very wide range of applications in recent years, ranking from structural biology to computer vision and ranking systems. To name a few specific examples, synchronization over SO(2) plays an important role in the framework for ranking \cite{syncRank},  image reconstruction from pairwise intensity differences  \cite{Stella_2012, Stella_2009}, and also engineering, in a specific divide-and-conquer algorithm for local-to-global sensor network localization \cite{asap2d}, also described in detail in Section \ref{sec:GRP}. 

In the context of ranking, the angular synchronization paradigm was leveraged in \cite{syncRank} for the purpose of ranking items (resp. players) given a sparse noisy  subset of pairwise preference relationships (resp. match outcomes). The approach in  \cite{syncRank} starts by compactifying the real line by wrapping it over the upper half of the unit circle,  rendering  the problem amenable to standard synchronization over the compact group SO($2$). The estimated solution allowed for the recovery of the player rankings, after a post-processing step of modding out the best circular permutation. Note that the proposed approach only focused on recovering the individual rankings, and not the magnitude (i.e., strength) of each player, as it was proposed in our recent paper \cite{SVDRank}. As detailed in Section  \ref{sec:setupMotivMain}, applications from the literature on ranking from heterogeneous data, provided part of the motivation for our present work. 
Concerning synchronization over SO($3$), specific applications include the \textit{structure-from-motion} problem in computer vision \cite{structFromMotion_Amit}, global alignment of 3D scans in computer graphics \cite{Tzeneva_Thesis}, structural biology for identifying the 3D structure molecules using NMR spectroscopy \cite{asap3d}, and cryo-electron microscopy \cite{shkolnisky2012viewing,singer2011viewing}.

\paragraph{Notation} We denote vectors and matrices in lower case and upper case letters, respectively. For a finite set $S$, we denote ${S \choose 2}$ to be the collection of all subsets of $S$ of size two. $X \sim U[a,b]$ denotes a random variable $X$ with uniform distribution over $[a,b]$.  
For a graph $G=(V,E)$, we denote its binary adjacency matrix by $A$, with $A_{ij}=1$ iff $\set{i,j} \in E$.  For a complex-valued vector $z$, we denote its complex conjugate transpose by $ z^*$. For $M \in \mathbb{C}^{m \times n}$, we denote its spectral norm by $\norm{M}_2$.

\paragraph{Paper outline}
Section \ref{sec:angSyncSo2} gives an overview of the angular synchronization problem and the various methodologies proposed for solving it.  
Section \ref{sec:setupMotivMain} introduces the $k$-synchronization problem, its motivation and summary of main results. 
Section \ref{sec:bi_sync} focuses on the  bi-synchronization problem (where $k=2$), and theoretically analyzes the performance of a spectral method under a probabilistic generative model assumption. 
Section \ref{sec:k_sync} extends these results to the general $k$-synchronization setup. 
Section   \ref{sec:experimentsSync}  details the outcomes of a variety of numerical experiments across different parameter regimes and algorithmic approaches, and also introduces an iterative graph disentangling procedure. 
Section \ref{sec:GRP} is an application of  bi-synchronization  to the setting of the graph realization problem. 
Finally, Section \ref{sec:conclusion} is a summary of results and future research outlook.

%
\section{Angular Synchronization over $SO(2)$} 
\label{sec:angSyncSo2}

Let $G = ([n],E)$ be an undirected graph and denote  $\theta_1,\dots,\theta_n \in [0,2\pi)$ to be a collection of $n$ unknown angles. 
The classical angular synchronization problem was introduced by 
Singer \cite{sync} and can be summarized as follows.
\begin{framed}
\textit{Synchronization over $SO(2)$.}
\begin{itemize}
\item \textbf{Input:} $\Theta_{ij} = \text{ noisy version of } (\theta_i - \theta_j) \text{ mod } 2\pi$; for each $\setij \in E$.
\item \textbf{Goal:} Recover the unknown ground truth angles: $\theta_1, \ldots, \theta_n  \in [0, 2 \pi)$.
\end{itemize}
\end{framed}
In order to be able to recover the angles, the graph $G$ needs to be connected. If there is no noise in the measurements, then one can easily recover the angles uniquely by fixing a root node,  sequentially traversing a spanning tree in $G$, and summing the offsets modulo $2 \pi$. Note that the angles will all be uniquely determined up to an additive phase given by the angle corresponding to the root node. Of course, once the measurement graph becomes disconnected, it is no longer possible to recover the angles and relate pairs of angles that belong to different disconnected components. 

The problem at hand becomes more interesting and challenging in the noisy and sparse setting -- one would ideally aim to have a method which is able to recover the original angles in a stable and robust manner. A sequential propagation approach that iteratively integrates the pairwise offsets across one, or potentially multiple, spanning trees is prone to fail quickly due to accumulation of errors. Instead, one would ideally put forth a method that integrates all the pairwise angle offsets in a globally consistent manner, and does so in a computationally efficient approach, while being robust to the noise and sparsity inherent in real data.

In an attempt to preserve the angle offsets as best as possible, Singer 
considered the following maximization problem. Consider first the $n \times n$  measurement matrix $H$ where for each $\setij \in {[n] \choose 2}$; we set 
\begin{equation}
{H}_{ij} = \begin{cases}
 e^{\imath \Theta_{ij}} & \text{if } \setij \in E \\
 0 & \text{if } \setij \notin E 
\end{cases}, 
\label{eq:mapToCircle} 
\end{equation}
with diagonal entries $H_{ii} = 1$. Note that the fact $\Theta_{ji} = (-\Theta_{ij}) \mod 2\pi$ renders $H$ to be a Hermitian matrix. Next, consider the following maximization problem 
\begin{equation}
\underset{  \theta_1,\ldots,\theta_n \in [0,2\pi)   }{\max}  \sum_{i,j=1}^{n} e^{-\iota \theta_i} H_{ij} e^{\iota \theta_j},
\label{eq:angularMaxObj}
\end{equation}
which gets incremented by $+1$ whenever an assignment of angles $\theta_i$ and $\theta_j$ perfectly satisfies the given edge constraint $  \Theta_{ij} = \theta_i - \theta_j \mod 2\pi$ (i.e., for a \textit{good} edge), while the contribution of an incorrect assignment (i.e., of a \textit{bad} edge) will be uniformly distributed over the unit circle in the complex plane. 
However, \eqref{eq:angularMaxObj} is non-convex and computationally expensive to solve in practice.

\paragraph{Spectral relaxation} 
Singer \cite{sync} considered solving instead
%
\begin{equation*}
\underset{  z_1,\ldots,z_n \in \mathbb{C};  \;\;\; \sum_{i=1}^n |z_i|^2 = n  }{\max} \;\; \sum_{i,j=1}^{n}  z_i^*  H_{ij}  z_j,
\end{equation*}
%
\noindent where the individual constraints $z_i = e^{\iota \theta_i}$ having unit magnitude are replaced by the much weaker single constraint $ \sum_{i=1}^n |z_i|^2 = n $. 
We thus arrive at the spectral relaxation of  \eqref{eq:angularMaxObj}, given by 
$ 
\underset{  || z ||_2 ^2 = n }{\max} \; z^* H z, 
\label{eq:finalRelaxAmitObj}
$ 
which can be solved via a simple eigenvector computation, by setting $z = v_1$, the top eigenvector of $H$. Finally, we obtain the estimated rotation angles $\hat{\theta}_1,...,\hat{\theta}_n$ 
and their corresponding elements in SO(2), $\hat{r}_1,...,\hat{r}_n$ as
\begin{equation} \label{est-r}
\hat{r}_i =  e^{\imath \hat{\theta}_i} = \frac{v_{1,i}}{|v_{1,i}|} , \quad i=1,2,\ldots, n.
\end{equation}
Note that the estimation of the rotation angles $\theta_1,\ldots,\theta_n$ is up to an additive phase since $e^{i\phi}v_1$ is also an eigenvector of $H$ for any $\phi \in \mathbb{R}$.

\paragraph{Semidefinite programming relaxation}
As an alternative to the spectral relaxation, \cite{sync}  also introduced a semidefinite programming (SDP) relaxation of \eqref{eq:angularMaxObj}. One can first write the objective function in \eqref{eq:angularMaxObj} as
\begin{equation*}
\sum_{i,j=1}^{n} e^{-\iota \theta_i} H_{ij} e^{\iota \theta_j} = \text{trace}(H \Upsilon),
\end{equation*}
where  $ \Upsilon $ is a  rank-$1$, $n \times n$ Hermitian matrix with
$ \Upsilon_{ij} = e^{\iota (\theta_i-\theta_j)},  $ 
$\forall i,j=1,2,\ldots,n$.
After dropping the rank-1 constraint, 
all the remaining constraints are convex and altogether define the following SDP relaxation for \eqref{eq:angularMaxObj}  
\begin{equation}
	\begin{aligned}
	& \underset{\Upsilon \in \mathbb{C}^{n \times n}}{\text{maximize}}
	& & \text{trace}( H \Upsilon) \\
	& \text{subject to}
	& & \Upsilon_{ii} = 1 & i=1,\ldots,n \\
	& & &  \Upsilon \succeq 0. 
	\end{aligned}
 \label{eq:SDP_program_SYNC}
\end{equation}
%
This program is similar to the seminal Goemans-Williamson SDP relaxation for the well-known MAX-CUT problem of finding the maximum cut in a weighted graph. The only difference in \eqref{eq:SDP_program_SYNC}  stems from the fact that we optimize over the cone of complex-valued Hermitian positive semidefinite matrices, not just real symmetric matrices.
Note that the recovered solution is not necessarily of rank-1, and the final estimate is obtained from the best rank-1 approximation. 

\paragraph{Generalized power method} In recent work  \cite{boumal2016nonconvex}, Boumal proposed a modified version of the power method, with a provable convergence to the global optimum. This approach bears the advantage of empirically converging faster compared to the previous convex relaxations. The proposed method, dubbed as \textit{generalized power method} (GPM), succeeds in the same noise regime as its predecessors, and enjoys the theoretical guarantees that, under a suitable noise regime, second-order necessary optimality conditions 
are also sufficient, despite its non-convexity.   
In a follow-up work \cite{zhong2018near}, Zhong and Boumal established near-optimal bounds for phase synchronization, by proving that the SDP relaxation is tight for a suitable noise regime under a Spiked Gaussian Wishart model for synchronization. The same line of work provided guarantees that GPM converges to its global optimum under a similar noise regime, and also established a linear convergence rate for GPM.

\paragraph{Message passing algorithms} 
Finally, we also point out the work of Perry et al. \cite{perry2018message}, who provide  message-passing algorithms for synchronization problems over compact groups such as $\mathbb{Z}_2$ and SO(2), using tools from representation theory and statistical physics. Under a Gaussian noise model ensemble, the authors identify regimes where the problem is computationally easy, computationally hard, and statistically impossible,  
thus providing evidence of a statistical-to-computational gap \cite{bandeira2018notes}.

\bigskip 
\paragraph{Probabilistic models for angular synchronization}
Before concluding this section, we mention that the performance of the spectral and SDP relaxation-based estimators was demonstrated by Singer under the following measurement graph and noise model. Consider $G$ to be generated via the {\ER} model,  where each edge of $G$ is present independently with probability $\lambda \in [0,1]$. For each $\setij \in E$, the measurements $\Theta_{ij}$ (with $i < j$ w.l.o.g) are assumed to be generated independently through the random model
\begin{equation}
\Theta_{ij} = \left\{
 \begin{array}{rll}
 (\theta_i - \theta_j) \ \text{mod} \ 2\pi & \quad \text{ with probability } p \\
 \sim U[0,2\pi)  & \quad \text{ with probability } (1-p)	\\ 
\end{array}.
   \right.
\label{AmitNoiseTheta}
\end{equation}
It was shown in \cite{sync} that the top eigenvector of $H$ has above random correlation with the ground truth vector, when $n$ is large enough relative to $p$. Additional noise models have been considered subsequently in  \cite{bandeira2017tightness,perry2018message,perry2016optimality}. 
For $z \in \mathbb{C}^n$ a vector with unit modulus complex entries, and $W \in \mathbb{C}^{n \times n}$ a Hermitian Gaussian Wigner matrix (with i.i.d. complex standard Gaussian entries above its diagonal), the matrix of pairwise measurements is modeled as  
\begin{equation*}
    C = z z^* + \sigma W,
\end{equation*}
and the task becomes to recover the vector $z$ given $C$.  
The Maximum Likelihood Estimator maximizes $x^* C x$ over the parameter space $|x_1| = |x_2| = \ldots = |x_n| = 1$. Bandeira et al. \cite{bandeira2017tightness} prove exact recovery under a similar noise model for $\mathbb{Z}_2$, while for the above complex case, they establish that, under suitable noise regimes, the SDP relaxation admits a unique solution of rank 1, revealing the global optimum of \eqref{eq:angularMaxObj}.

\section{$k$-synchronization: setup, motivation, \& main results}  
\label{sec:setupMotivMain}
Motivated by real-world applications, in particular instances of the graph realization problem (detailed later in Section \ref{sec:GRP}) and ranking from preference relationships, we now introduce a variation of the above synchronization formulation, dubbed as \textit{k-synchronization}, where there exist $k$ underlying latent sequences of angles. 
 
\paragraph{k-synchronization} Given the measurement graph $G = ([n],E)$, we consider the scenario where $G$ has an \textbf{unknown decomposition} into $k$ edge-disjoint  subgraphs $G_1 = ([n],E_1), \ldots, G_k=([n],E_k)$  such that $ E_l \cap E_{l'} = \emptyset, \mbox{ for } l \neq l' $, and $ \cup E_l = E$.

\begin{framed}
\textit{k-synchronization over $SO(2)$.}
\begin{itemize}
\item \textbf{Input:} $\Theta_{ij} = \text{ noisy version of } (\theta_{l,i} - \theta_{l,j}) \text{ mod } 2\pi$; for each $\setij \in E_l$, for a fixed $l \in \{1,2,\ldots,k\}$.  
\item \textbf{Goal:} Recover the unknown ground truth angles: $\theta_{l,1}, \ldots, \theta_{l,n}  \in [0, 2 \pi)$, for all $l \in \{1,2,\ldots,k\}$. 
\end{itemize}
\end{framed}

\paragraph{Bi-synchronization ($k=2$)}
For simplicity, we detail below the case $k=2$, which we refer to as \textit{bi-synchronization}, and illustrate it in Figure \ref{fig:biSyncTwoGraphs}.   
%
For ease of notation, we shall denote by $\alpha$ and $\beta$ the two collections of angles. Say there exists a collection of ground truth angles $\alpha_1, \ldots, \alpha_n$ (respectively $\beta_1, \ldots, \beta_n$) for the  synchronization problem over the graph $G_1$ (respectively $G_2$). 
For each $\set{i,j} \in E$, the user is given pairwise information $\Theta_{ij}$ where
\begin{equation*}
{\Theta}_{ij} = \begin{cases}
A_{ij} = \text{ noisy version of } (\alpha_i - \alpha_j) \mod 2\pi; & \text{if } \setij \in E_1 \\
B_{ij} = \text{ noisy version of } (\beta_i - \beta_j) \mod 2\pi; & \text{if } \setij \in E_2.
\end{cases}
\end{equation*} 
Given the combined set of noisy incomplete measurements $\Theta_{ij}$ for $\setij \in E$, 
one is asked to recover an estimate for 
$\alpha_1, \ldots, \alpha_n$ and  $\beta_1, \ldots, \beta_n$. The motivation for this problem arises in structural biology, where the distance measurements between the pairwise atoms may correspond to different configurations, as is the case of molecules that may have multiple conformations. In ranking systems, the two groups of measurements associated to $G_1$ and $G_2$ may correspond to two different judges, whose latent rankings we seek to recover. We detail these applications in Section \ref{sec:kSyncApp}. 
The following remarks are in order.
%
\begin{itemize} 
\item As for classical synchronization, we clearly require the sub-graphs $G_1, G_2$ to be 
respectively connected in order to be able to recover the $\alpha_i$'s and $\beta_i$'s.
\item In case $G_1, G_2$ were known, then the problem reduces to the classical synchronization 
problem over $G_1,G_2$ respectively. 
\end{itemize}

\subsection{Applications}  \label{sec:kSyncApp}
We refer the reader to Sections \ref{sec:experimentsSync} and \ref{sec:GRP}
for an application to the graph realization problem, and an algorithmic procedure for extracting the individual measurement subgraphs. In this section, we elaborate on several applications that have motivated this work. 

\begin{figure}[h!]
\vspace{-2mm}
\centering
\includegraphics[width=0.57\textwidth]{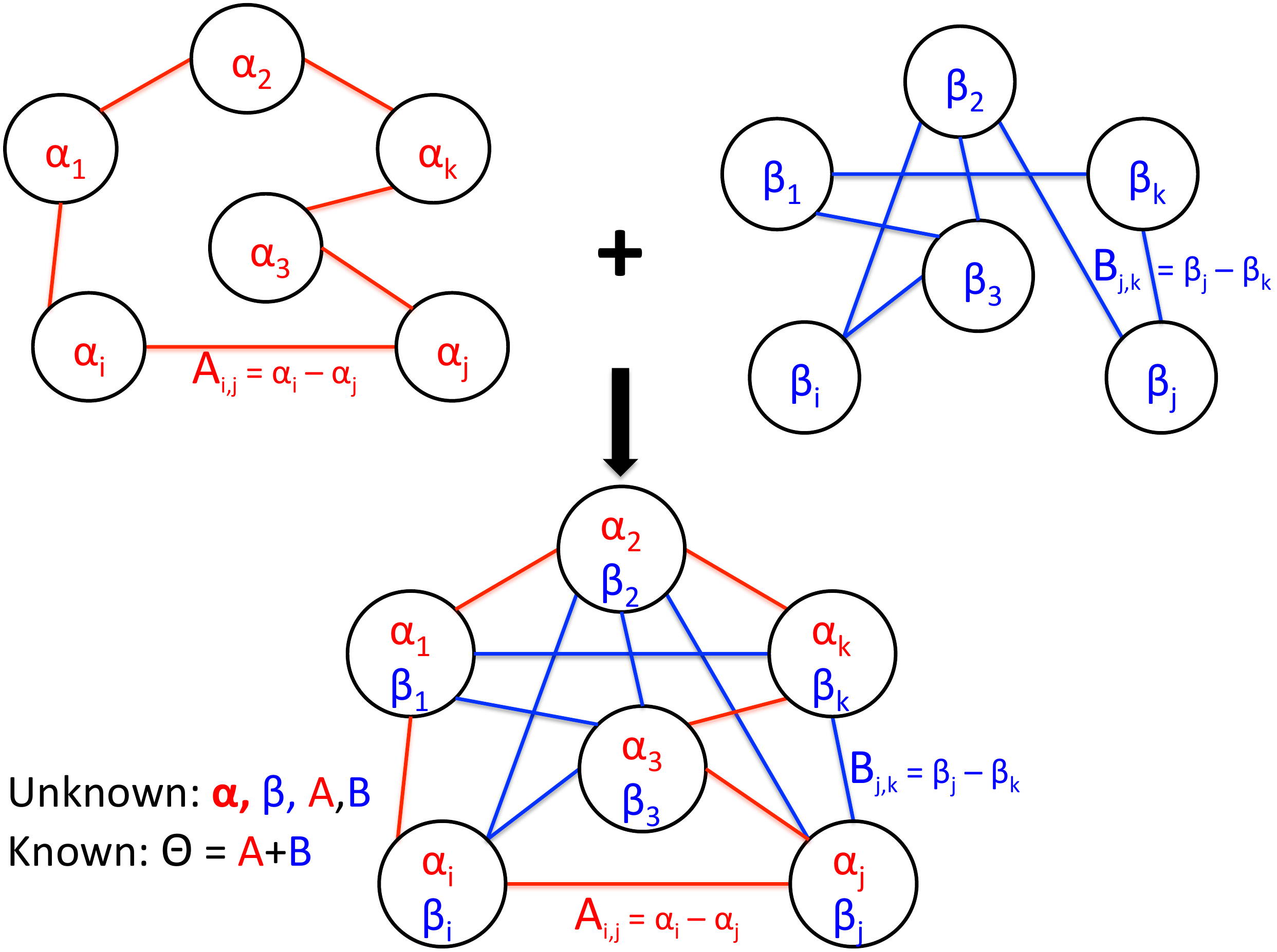}
\vspace{-2mm}
\captionsetup{width=1.00\linewidth}
\caption{An instance of the bi-synchronization problem. The measurement graph (bottom row) available to the user is a union of two edge disjoint graphs (top row) on the same vertex set. The top graphs each capture angle offsets for two different sets of angles: the top left graph encodes noise pairwise offsets between the angles $ \alpha_1, \ldots, \alpha_n$, while the top right graph contains similar information from the angles $\beta_1,  \ldots, \beta_n$.}
\label{fig:biSyncTwoGraphs}
\end{figure}
%

\paragraph{Graph realization problem and alignment of patch embeddings} 
In Section \ref{sec:GRP} we consider an application of $k$-synchronization to the popular \textit{graph realization problem}, where one is asked to recover a cloud of points in $\mathbb{R}^D$ from a small subset of noisy pairwise Euclidean distances. In applications of this problem to sensor network localization and structural biology, one may have readily available information on the embedding of a subset of nodes, and is left with the task of integrating the given embeddings into a globally consistent structure. We consider the scenario where the point cloud to be recovered may assume one of two possible configurations (denoted as   
type-$X$ and type-$Y$), and thus whenever a pair of subgraph embeddings are aligned (eg, via Procrustes analysis), the resulting alignment measurement may come from either configurations (i.e., the two subgraphs to be aligned are either both of type-$X$ or both of type-$Y$). We refer the reader to Section \ref{sec:GRP} for additional details and results on our application in this setting.

\paragraph{Extraction of latent rankings in multiple voters systems}
Yet another very closely related application arises in the ranking and recommendation systems literature.    
In many real-world scenarios, it is often the case that there exist multiple rating systems (eg, judges or voters) who provide incomplete and inconsistent rankings or pairwise comparisons between the same set of players or items. For example, in tournaments of $n$ players, where everyone plays against everybody else in a total of ${n \choose 2}$ matches, it may well be the case that, due to the large number of matches, a set of $k$ judges are distributed matches they are to referee. In other words, the set of all matches are randomly partitioned, and each set in the partition is attributed to a single judge, as shown pictorially in the example in Figure \ref{fig:biSyncRanking} for $k=2$. A similar setting can arise in applications well beyond sports and related competitions, including rating items in terms of preference relations between pairs of items. Such examples arise in many learning and social data analysis problems, including recommendation systems and collaborative-filtering for ranking movies for a user based on the movie rankings provided by other users, and information retrieval where one is interested in combining the results of different search engines.

The problem now becomes to disentangle the intrinsic ranking of all players that corresponds to each judge. When a judge is asked to evaluate a pair of players or items, she or he may base their decision on a specific latent feature which the said judge construes as most relevant for the ranking task. Different judges may use different intrinsic features on which they base their decision. The problem at hand now becomes to uncover the ranking provided by each judge. In light of the pipeline proposed in \cite{syncRank} by a subset of the authors, the $k$-synchronization problem and results covered in the present paper allow for the extraction of such latent rankings, given a single measurement matrix of pairwise ranking preferences.

\begin{figure}[h!]
\vspace{-2mm}
\hspace{40mm}  
\includegraphics[width=0.68\textwidth]{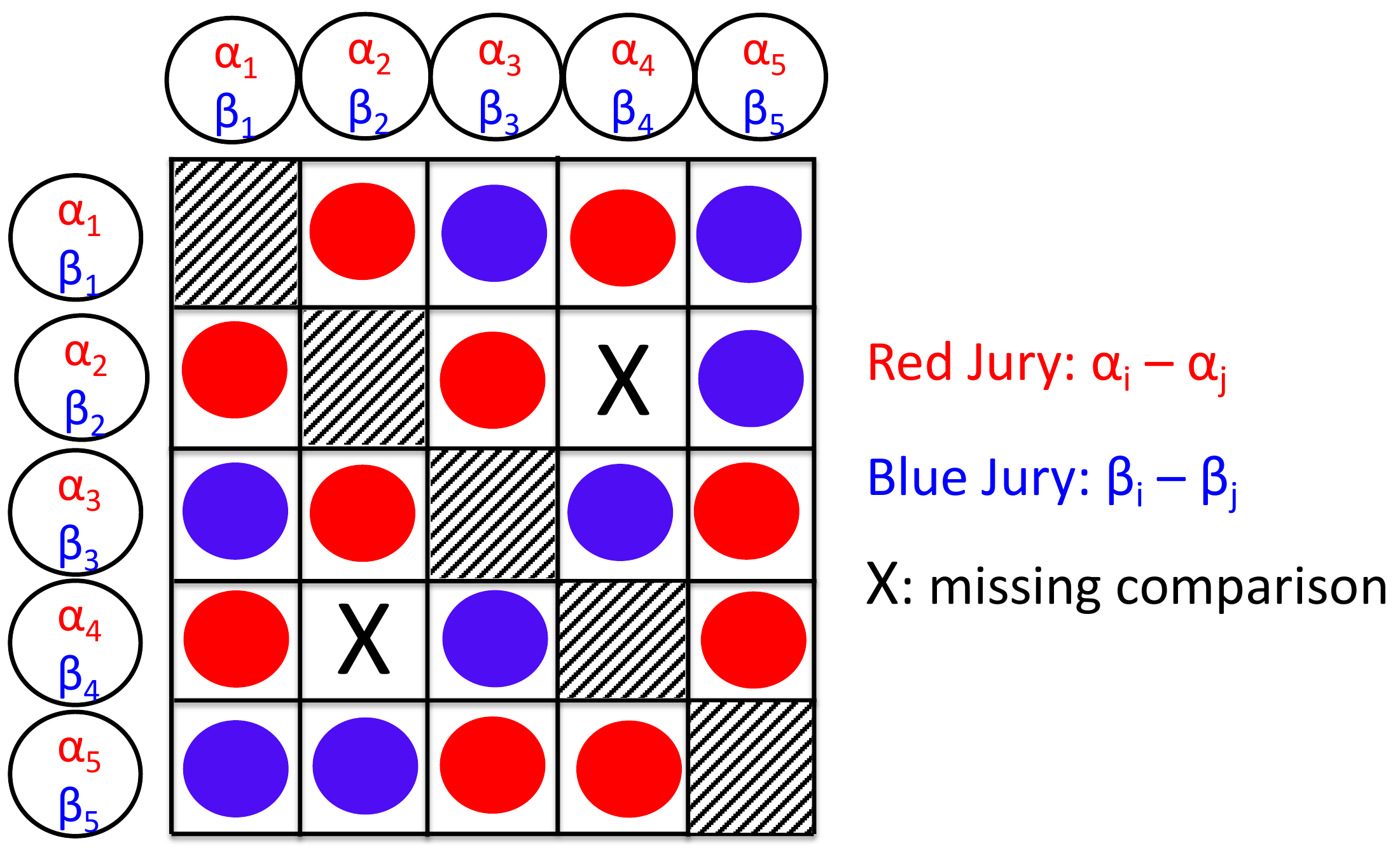}
\vspace{-2mm}
\captionsetup{width=1.00\linewidth}
\caption{An instance of the  \textbf{bi-synchronization}   problem in the context of ranking from pairwise comparisons with multiple voting systems. The red and blue judges each give a rating for pair of items $\setij$.  
$ (\alpha_1, \ldots, \alpha_n )$ denote the intrinsic strength/skill of the players  as construed by the red judge, while  $( \beta_1, \ldots, \beta_n)$ capture the strength/skill of the same set of $n$ players as perceived by the blue judge.} 
\label{fig:biSyncRanking}
\end{figure}

\begin{figure}[h!]
\vspace{-2mm}
\hspace{30mm}  
\includegraphics[width=0.68\textwidth]{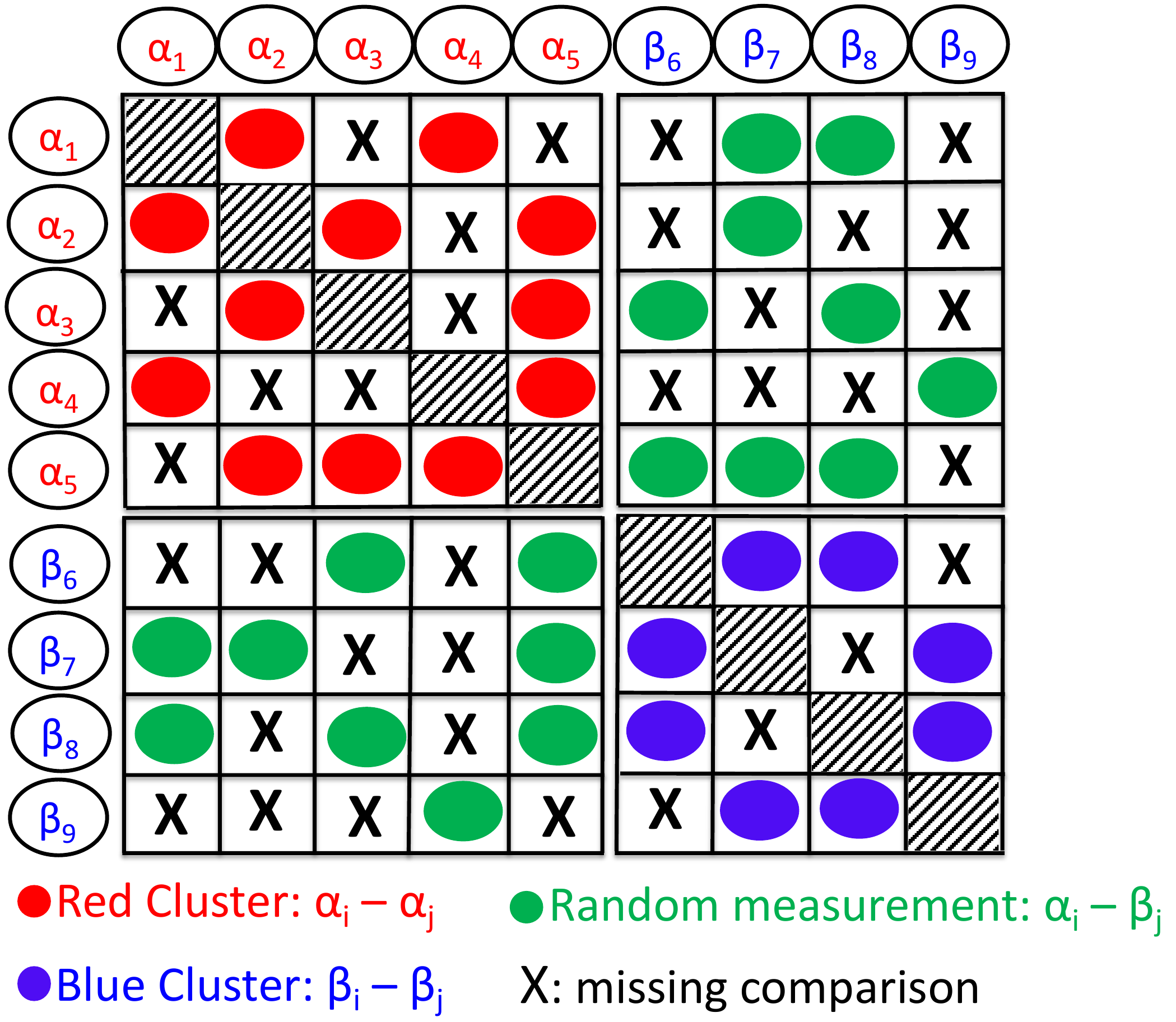}
\vspace{-2mm}
\captionsetup{width=1.00\linewidth}
\caption{An instance of the \textbf{cluster-synchronization}   problem, closely related to the  \textit{bi-synchronization} problem we study in this paper. The collection of angles is partitioned into two clusters 
$ \mathcal{C}_1 = (\alpha_1,\alpha_2,\ldots,\alpha_{n_1})$, 
$ \mathcal{C}_2 = (\beta_{n_1+1},\beta_{n+2},\ldots,\beta_{n_2})$, 
of potentially unequal sizes $n_1,n_2$ (as is the case in this example, with $n_1=5, n_2=4$), such that when a pair of angles is chosen from the same cluster, we obtain a (potentially noisy) offset $\alpha_i - \alpha_j$ or  $\beta_i - \beta_j$. However, when a pair $(\alpha_i,\beta_j)$ of angles from two different clusters is compared, the measured offset $\Theta_{ij}$ is completely random and contains no meaningful  information.}
\label{fig:clusterSync}
\end{figure}

\paragraph{Cluster-synchronization} Finally, we mention a closely related problem, coming from the ranking literature, of extracting partial rankings from pairwise comparison data. 
Similar to the setup pursued in this paper, there exist two collections of unknown angles grouped into two clusters 
$ \mathcal{C}_1 = (\alpha_1,\alpha_2,\ldots,\alpha_{n_1} )$, and 
$ \mathcal{C}_2 = (\beta_{n_1+1},\beta_{n+2},\ldots,\beta_{n_1 + n_2} )$ 
such that whenever a pair of angles from the same cluster is selected, the user has available a noisy proxy for the angle offset $\alpha_i - \alpha_j$ or  $\beta_i - \beta_j$. However, when a pair is selected such that one angle is from $ \mathcal{C}_1$ and the second angle from cluster  $ \mathcal{C}_2$, then the available measurement  $\Theta_{ij}$ is completely random, uniformly distributed in $[0,2\pi)$ (in particular,  $\Theta_{ij}$ is \textbf{not} a (potentially noisy)  proxy of $\alpha_i - \beta_j$). This setup is illustrated in Figure \ref{fig:clusterSync}, highlighting the cluster structure of the measurement matrix. As yet another instance of this problem, imagine for example two jigsaw puzzles being mixed together, and the goal is to disentangle the two sets of pieces. Or consider an archaeological site where two frescoes are identified, each of which is broken into many pieces which are now mixed together. 
We expect similar techniques as those studied in the present paper to be applicable in this setting as well, motivated by the fact that the expected measurement matrix has a low-rank structure. 

Finally, we stress  there is a fundamental difference between the setup in the bi-synchronization and cluster-synchronization problems. In both problems, there exist two groups of angles (of $\alpha$-type and $\beta$-type); however, in bi-synchronization the two groups must be of the same size $n_1=n_2$ and a measurement is either a noisy version between two $\alpha$-type angles or two $\beta$-type angles), while in cluster-synchronization we allow $n_1 \neq n_2$ and there also exist measurements across different types of angles.  Figures \ref{fig:biSyncRanking} and \ref{fig:clusterSync} illustrate pictorially the structural differences between these two problems.

\FloatBarrier
 

\subsection{Summary of our main contributions}
\begin{enumerate}
    \item We introduce the    $k$-synchronization problem and also propose a probabilistic generative model for it. This model is detailed in Section  \ref{sec:bi_sync} and Section \ref{sec:k_sync}. 
    
    \item We propose a spectral algorithm (Algorithm \ref{Algo:kSync}) which, under a certain approximate orthogonality assumption on the angles,  amounts to finding the top $k$ eigenvectors of the measurement matrix $H$ given in \eqref{eq:mapToCircle}. Furthermore, we also provide a detailed theoretical analysis of this algorithm for estimating the unknown angles in each group, under the aforementioned generative model. These results are stated as Theorem \ref{thm:bisync_main} for $k=2$ and Theorem \ref{thm:ksync_main} for the general case. 
    
    \item We provide a comprehensive set of numerical experiments in Section \ref{sec:experimentsSync} that illustrate the performance of our proposed spectral algorithm, under various parameter regimes (noise level, graph sparsity, and number of groups of angles $k$), and also compare it with the corresponding SDP relaxation, and also a normalized spectral version. 
    
    \item  We consider an application of bi-synchronization to the graph realization problem in Section \ref{sec:GRP}. Along the way, we also provide an iterative graph disentangling procedure, which can be of independent interest, and provide numerical experiments demonstrating that the proposed iterative scheme consistently improves the final recovery accuracy. 
\end{enumerate}

%% file: S_BiSync.tex

We begin by  analyzing the case where $k=2$ since it is relatively easier as compared to the general case. Recall that we are interested in recovering the two sequences of angles $ \alpha_1,\ldots,\alpha_n \in [0,2\pi) $ and $\beta_1,\ldots,\beta_n \in [0,2\pi)$.
Denoting $z_1,z_2 \in \mathbb{C}^n$ to be the entry-wise representations of $(\alpha_i)_i$ and $(\beta_i)_i$ on the unit circle where
\begin{equation}  \label{eq:bisync_circ_rep}
z_{1,i} = \frac{1}{\sqrt{n}} e^{\iota \alpha_i} \text{ and } z_{2,i} = \frac{1}{\sqrt{n}} e^{\iota \beta_i}; \quad  
i=1,\ldots,n,
\end{equation}
we will assume that the vectors $z_1,z_2 \in \mathbb{C}^n$ defined in \eqref{eq:bisync_circ_rep} 
are $\delta$-orthogonal for $\delta \in (0,1)$.
\begin{definition} \label{def:approx_orth}
A collection of vectors $x_1,x_2,\dots,x_m \in \mathbb{C}^n$ are said to be $\delta$-orthogonal 
for $\delta \in [0,1]$ if $$\abs{\dotprod{x_i}{x_j}} \leq \delta \norm{x_i}_2 \norm{x_j}_2, \quad \forall \ i \neq j.$$
\end{definition}
This is the case, for instance, when the angles $\alpha_i, \beta_i$ are generated 
uniformly at random in $[0,2\pi)$, and are i.i.d. 
Indeed, one can then show using standard large deviation inequalities 
that the vectors $z_1,z_2 \in \mathbb{C}^n$ in \eqref{eq:bisync_circ_rep} will be approximately 
orthogonal with high probability. This is stated formally in the form of the following Proposition.
\begin{proposition} \label{prop:bisync_approx_orth_whp}
If $\alpha_i, \beta_i \sim U[0,2\pi)$ i.i.d for $i=1,\dots,n$, it then follows 
for any given $\delta \in (0,1)$ that $\abs{\dotprod{z_1}{z_2}} \leq \delta$ holds 
with probability at least $1 - 2e\exp(-c \delta^2 n)$. Here $c > 0$ is a constant. 
\end{proposition}
Its proof is deferred to Appendix \ref{app:subsec_prop_approx_orth}. 
\subsection{Spectral method}
Our aim is to recover the $\alpha_i$'s and $\beta_i$'s in such a way that we preserve as best as possible the available input pairwise measurement. To this end, with the matrix $H$ as defined in \eqref{eq:mapToCircle}, one may consider the following objective function  

\begin{equation}
\underset{ \substack{ \alpha_1,\ldots,\alpha_n \in [0,2\pi) \\ \beta_1,\ldots,\beta_n \in [0,2\pi) }
}{\max}  
\;\;\; \sum_{ \set{i,j}  \in E_1 } e^{-\iota \alpha_i} H_{ij} e^{\iota \alpha_j} + 
\sum_{ \set{i,j} \in E_2}  e^{-\iota \beta_i} H_{ij} e^{\iota \beta_j}.
\label{eq:angularMaxObj_BiSync}
\end{equation}
Since the optimization problem \eqref{eq:angularMaxObj_BiSync} is non-convex and computationally expensive to solve in practice, we consider the following spectral relaxation, while keeping the definition \eqref{eq:bisync_circ_rep} in mind
\begin{equation*}
\underset{ \substack{|| w_1 ||_2 ^2 = || w_2 ||_2 ^2 = 1 }
}{\max}  
\;\; \sum_{\set{i,j} \in E_1}  w_{1,i}^*  H_{ij}  w_{1,j}  +  \sum_{\set{i,j} \in E_2}  w_{2,i}^*  H_{ij} w_{2,j}. 
\end{equation*}
Since we do not know the decomposition of the edge set into $E_1$ and $E_2$, and the ground truths $z_1,z_2$ are approximately orthogonal by assumption, we consider 
\begin{equation*}
\underset{ \substack{|| w_1 ||_2 ^2 = || w_2 ||_2 ^2 = 1  \\
w_1 \perp w_2 
}
}{\max}  
\;\; \sum_{ i,j=1}^{n} ( w_{1,i}^*  H_{ij} w_{1,j} + w_{2,i}^* H_{ij} w_{2,j}). 
\end{equation*}
Altogether, the spectral relaxation of the non-convex optimization problem in \eqref{eq:angularMaxObj_BiSync} is given by 
\begin{equation}
\underset{ \substack{|| w_1 ||_2 ^2 = || w_2 ||_2 ^2 = 1  \\
w_1 \perp w_2 
}
}{\max} \; w_1^* H w_1 + w_2^* H w_2, 
\label{eq:finalRelaxAmitObjBiSync}
\end{equation}
whose solution is obtained by setting $w_1 = v_1$ and $w_2 = v_2$, the top two eigenvectors of $H$. Finally, we can then obtain the estimates $\hat{\alpha}_{i}$,  $\hat{\beta}_{i}$ for $i=1,\dots,n$ as
\begin{equation*}
	e^{\iota \hat{\alpha}_{i}} =  \frac{v_{1,i}}{|v_{1,i}|}, \quad e^{\iota \hat{\beta}_{i}} =  \frac{v_{2,i}}{|v_{2,i}|}.
\end{equation*}	
\begin{remark} \label{rem:Hii_same}
The solution of \eqref{eq:finalRelaxAmitObjBiSync} is clearly unchanged if we replace $H$ with $H-c I$ for some scalar $c$. Therefore one can set the entries $H_{ii}$ to the same (arbitrarily chosen) real value.
\end{remark}

\subsection{Erd\H{o}s-Renyi measurement model} 
We are interested in the following random measurement model. 
\begin{itemize}
\item The graph $G$ is generated via the Erd\H{o}s-R\'{e}nyi random graph model $G(n,\lambda)$, where each edge of $G$ is present with probability $\lambda \in [0,1]$.

\item For each $\setij \in E$, we assign it to $E_1$ with prob.  $q$ and to $E_2$  with prob.  $1-q$. 
So the graphs $G_1$ and $G_2$ are clearly edge disjoint, i.e.,  $E_1 \cap E_2 = \emptyset$, 
and moreover, $E_1 \cup E_2 = E$.

\item Consider $\setij \in E$, and assume $i < j$ w.l.o.g. 
Let $N_{ij} \sim U[0,2\pi)$ i.i.d. If $\setij \in E_1$, we obtain $\Theta_{ij} = (\alpha_i-\alpha_j) \mod 2\pi$ 
with probability $\widetilde q_1$,  and $\Theta_{ij} = N_{ij}$ with probability 
$1-\widetilde q_1$, i.i.d. If $\setij \in E_2$, 
we obtain $\Theta_{ij} = (\beta_i-\beta_j) \mod 2\pi$ with probability $\widetilde q_2$,  
and $\Theta_{ij} = N_{ij}$ with probability $1-\widetilde q_2$, i.i.d. 
\end{itemize}
Let us denote $p_1 = q \widetilde q_1$ and $p_2 = (1-q) \widetilde q_2$. For each $\setij \in E$, 
with $i < j$, we can write the obtained measurement $\Theta_{ij}$ as
\begin{equation} \label{Def:MixtureOmega}
\Theta_{ij} = \left\{
	\begin{array}{rl}
   	 (\alpha_i - \alpha_j) \mod 2\pi & \text{ w.p } \ p_1,  \\
   	 (\beta_i - \beta_j) \mod 2\pi & \text{ w.p } \ p_2, \\
    N_{ij} \sim U[0,2\pi) & \text{ w.p } \ \eta = 1-p_1 - p_2,
	\end{array}
   \right.
\end{equation}
where $p_1$ (respectively $p_2)$ denotes the probability of a correct pairwise measurement 
in $G_1$( respectively $G_2$), and $1-p_1-p_2$ is the probability of getting an incorrect measurement. 
Moreover, we have $\Theta_{ji} = (-\Theta_{ij}) \mod 2\pi$.

The difficulty of the problem is amplified on one hand by the noise and sparsity of the 
measurements (dictated by $\eta$, $\lambda$ respectively), and on the other hand by the ambiguity 
that stems from the fact that the user does not know a-priori whether a given edge (i.e., 
the measurement that comes with it) corresponds to a (perhaps noisy) offset 
$\alpha_i - \alpha_j$ or an offset $\beta_i - \beta_j$. In other  words, the graphs $G_1$ 
and $G_2$, in particular their respective edge sets $E_1$ and $E_2$, are not known to the user a-priori. 
From now on, we will assume w.l.o.g\footnote{Clearly, some separation is needed in order to 
be able to distinguish the graphs.} that $p_1 > p_2$.

We start by mapping the points over the complex unit circle as in \eqref{eq:mapToCircle} to construct 
a Hermitian matrix $H$. 
Due to \eqref{Def:MixtureOmega}, we have for each $i < j$ that
\begin{equation}
H_{ij} = \left\{
	\begin{array}{rl}
   	e^{\iota (\alpha_i - \alpha_j)}  & \text{ w.p  } \  p_1 \lambda, 	\\
   	e^{\iota (\beta_i  - \beta_j)}   	&  \text{ w.p  }  \ p_2 \lambda,	\\
 	e^{\iota N_{ij}},  N_{ij} \sim U[0,2\pi) 		&  \text{ w.p  } \ (1-p_1-p_2) \lambda,	\\
		0		& \text{ w.p } \ 1-\lambda.	\\
     \end{array}
   \right.
\label{eq:bisync_HijDef}
\end{equation} 
In light of Remark \ref{rem:Hii_same}, we set $H_{ii} = \lambda(p_1+p_2)$ for convenience. 
Note that $H_{ji} =  H_{ij}^{*}$, hence 
we can write $\expec[H]$ as 
\begin{equation} 
	\expec[H] = np_1 \lambda z_1 z_1^* + np_2 \lambda z_2 z_2^*.
	\label{eq:bisync_defEH}
\end{equation}
$\mathbb{E} H$ is of rank at 
most 2,  unless
$z_1$ is a scalar multiple of $z_2$. Given the above, we can write the available 
measurement matrix $H$ as 
\begin{equation} \label{eq:bisync_Hrank2Decomp}
	H = \mathbb{E}(H) + R = n\lambda (p_1 z_1 z_1^* + p_2 z_2 z_2^*) + R, 
\end{equation}
where $R$ is a random Hermitian matrix with $R_{ii} = 0$, whose elements 
$R_{ij} = H_{ij} - \mathbb{E}[H_{ij}] $ are zero-mean independent random variables for $i \leq j$. 
In particular, for each $i < j$, we have
\begin{equation} \label{eq:bisync_noise_form}
R_{ij} = \left\{
	\begin{array}{rl}
   	e^{\iota (\alpha_i - \alpha_j)} - p_1\lambda e^{\iota (\alpha_i - \alpha_j)} 
		- p_2\lambda e^{\iota (\beta_i  - \beta_j)}  &  \text{ w.p }   p_1 \lambda, 	\\
   	e^{\iota (\beta_i  - \beta_j)} - p_1\lambda e^{\iota (\alpha_i - \alpha_j)} 
		- p_2\lambda e^{\iota (\beta_i  - \beta_j)} 	&  \text{ w.p }  p_2 \lambda,	\\
 	e^{\iota N_{ij}} - p_1\lambda e^{\iota (\alpha_i - \alpha_j)} 
		- p_2\lambda e^{\iota (\beta_i  - \beta_j)}, N_{ij} \sim U[0,2\pi)
			&  \text{ w.p } (1-p_1-p_2) \lambda,	\\
		- p_1\lambda e^{\iota (\alpha_i - \alpha_j)} 
		- p_2\lambda e^{\iota (\beta_i  - \beta_j)}	& \text{ w.p } 1-\lambda,	\\
     \end{array}
   \right.
\end{equation}
and $R_{ji} = R_{ij}^{*}$.

%
%
%
%

\subsection{Main result} \label{subsec:bisync_main_res}
Our main theorem for the bi-synchronization problem for the case of approximately orthogonal 
representations $z_1,z_2$ is outlined below.
%
\begin{theorem} \label{thm:bisync_main}
Let $z_1,z_2$ in \eqref{eq:bisync_circ_rep} be $\delta$-orthogonal for 
$\delta \in (0,1)$, and let $p_1 > p_2$. For constants $\varep \in (0,1)$ and $\mu \in [0,1/2]$, let
\begin{equation*} 
n \geq \frac{72(2+\varepsilon)^2 C(\lambda,p_1,p_2)}{\mu^2 \lambda^2 \min\set{(p_1-p_2)^2, p_2^2}}, 
\end{equation*}
where $C(\lambda,p_1,p_2) > 0$ depends only on $\lambda,p_1,p_2$ (see \eqref{eq:bisync_cterm_1}). 
Then with probability at least 
$1 - 3n\exp\left(- \frac{8 C(\lambda,p_1,p_2) n}{\sigtil(\lambda,p_1,p_2)^2 c_{\varepsilon}} \right)$, 
the following bounds hold.
\begin{align*}
\abs{\dotprod{z_1}{v_1}}^2 &\geq
1 - \left(\sqrt{\frac{2\mu}{1-\mu}} + \sqrt{1 - \frac{p_1 - p_2}{\sqrt{(p_1-p_2)^2 + 4p_1p_2\delta^2}}} \right)^2, \\
\abs{\dotprod{z_2}{v_2}}^2 &\geq
1 - \left(\sqrt{\frac{2\mu}{1-\mu}} + 
\sqrt{1 - \frac{p_1 (1-\delta^2) - p_2}{\sqrt{(p_1-p_2)^2 + 4p_1p_2\delta^2}}} \right)^2.
\end{align*}
Here, $c_{\varepsilon}, \sigtil(\lambda,p_1,p_2) > 0$ only depend on the indicated parameters 
(see \ref{eq:infty_norm_bd_1}).
\end{theorem}
The following remarks are in order regarding Theorem \ref{thm:bisync_main}.
\begin{enumerate}
    \item The above correlation bounds are essentially composed of two terms  -- one arising due to the perturbation $R$, and the other due to the approximate orthogonality of $z_1,z_2$. This second term vanishes in the special case $\delta = 0$. 
    
\item In order to understand the scaling of the terms involved, it is convenient to replace the terms $C(\lambda,p_1,p_2), \sigtil(\lambda,p_1,p_2)$ by their upper bounds $C(\lambda,p_1,p_2) \lesssim \lambda$ and $\sigtil(\lambda,p_1,p_2) \lesssim 1$. We then have that if 
\begin{equation*}
    \lambda \gtrsim \frac{\log n}{n}, \quad n \gtrsim \frac{1}{\mu^2 \lambda \min\set{(p_1-p_2)^2, p_2^2}},
\end{equation*}
then the stated correlation bounds hold with probability at least $1 - \frac{1}{n}$.
\end{enumerate}
\subsection{Proof of Theorem \ref{thm:bisync_main}} \label{bisync_main_proof}
Let us denote $\lamtil_1 \geq \lamtil_2$ (resp. $\vtil_1,\vtil_2$) to be the two largest  
eigenvalues (resp. corresponding eigenvectors) of $\expec[H]$. 
Since $z_1,z_2$ are approximately orthogonal and $p_1 \neq p_2$, we would expect, for small enough $\delta$, that $\lamtil_1 \approx n p_1 \lambda$, $\lamtil_2 \approx n p_2 \lambda$ and 
$\abs{\dotprod{\vtil_1}{z_1}} \approx 1$, $\abs{\dotprod{\vtil_2}{z_2}} \approx 1$ hold. 
This is stated precisely in the following Lemma.
%
%
\begin{lemma} \label{lem:eigperturn_approx_orth}
It holds that $\lamtil_2 \leq np_2 \lambda$ and $\lamtil_1 \geq np_1\lambda$. Moreover, 
for any given $\delta \in (0,1)$, if $z_1,z_2 \in \mathbb{C}^n$ as defined in \eqref{eq:bisync_circ_rep} 
are $\delta$-orthogonal, then the following bounds hold. 
\begin{align}
\lamtil_2 \geq \frac{(np_1\lambda + np_2\lambda) - \sqrt{(np_1\lambda - np_2\lambda)^2 + 4n^2p_1p_2\lambda^2\delta^2}}{2}, 
 \label{eq:bds_lamtil_2} \\ 
\lamtil_1 \leq 
\frac{(np_1\lambda + np_2\lambda) + \sqrt{(np_1\lambda - np_2\lambda)^2 + 4n^2p_1p_2\lambda^2\delta^2}}{2} 
, \label{eq:bds_lamtil_1}\\
\abs{\dotprod{z_1}{\vtil_1}}^2 
\geq \frac{p_1 - p_2}{\sqrt{(p_1-p_2)^2 + 4p_1p_2\delta^2}}, \qquad 
\abs{\dotprod{z_2}{\vtil_2}}^2 
\geq \frac{p_1 (1-\delta^2) - p_2}{\sqrt{(p_1-p_2)^2 + 4p_1p_2\delta^2}}. \label{eq:bds_zvtil}
\end{align}
\end{lemma}
\begin{proof}
See Appendix \ref{app:subsec_eigperturn_approx_orth}.
\end{proof}
Let us now denote $\lambda_1 \geq \lambda_2 \geq \cdots \geq \lambda_n$ (resp. $v_1,v_2,\dots,v_n$) 
to be the eigenvalues (resp. eigenvectors) of $H$. As a next step, we would like 
to quantify the deviation of $v_1, v_2$ from $\vtil_1, \vtil_2$ respectively; clearly 
this is dictated solely by the perturbation matrix $R$ arising due to noise. 
The following Lemma states this precisely. Note that the statement of the Lemma is deterministic --
it is conditioned on the event that $\norm{R}_2 \leq \triangle$ for a suitably small value of $\triangle$.
%
\begin{lemma} \label{lem:mat_pert_dk}
If $\norm{R}_2 \leq \triangle$ holds for $\triangle < \min\set{n(p_1-p_2)\lambda, np_2\lambda}$, 
we then have that
\begin{align}
\abs{\dotprod{v_1}{\vtil_1}}^2 
&\geq 
1 - \left(\frac{\triangle}{np_1\lambda - np_2\lambda - \triangle}\right)^2, \label{eq:vtilv_bds_1} \\
\abs{\dotprod{v_2}{\vtil_2}}^2 
&\geq 
1 - \left(\frac{\triangle}{\min\set{np_1\lambda - np_2\lambda - \triangle, np_2\lambda - \triangle}}\right)^2. \label{eq:vtilv_bds_2}
\end{align}
\end{lemma}
\begin{proof}
See Appendix \ref{app:subsec_mat_pert_dk}.
\end{proof}
We now need to bound the spectral norm of the perturbation matrix $R$. The following lemma shows that 
$\norm{R}_2 \lesssim \sqrt{n}$ holds with high probability. Before proceeding, 
let us write $R = \real(R) + \iota \imag(R)$ where $\real(R), \imag(R) \in \matR^{n \times n}$ 
are respectively symmetric and skew symmetric matrices.
\begin{lemma} \label{lem:rand_pert_bd}
Denote 
\begin{align} 
C(\lambda,p_1,p_2) = &2p_1\lambda[(1-p_1\lambda)^2 + (p_2\lambda)^2] + 2p_2\lambda[(1-p_2\lambda)^2 + (p_1\lambda)^2] \nonumber \\
&  \quad + (1-p_1-p_2)\lambda\bigl[\frac{1}{2} + (p_1\lambda + p_2\lambda)^2\bigr] 
+   (1-\lambda)(p_1\lambda + p_2\lambda)^2, \label{eq:bisync_cterm_1}
\end{align}
and let $\sigtil(\lambda,p_1,p_2) > 0$ be such that 
\begin{equation} \label{eq:infty_norm_bd_1}
\max_{i,j} \set{\norm{\real(R)_{ij}}_{\infty}, \norm{\imag(R)_{ij}}_{\infty}} \leq \sigtil(\lambda,p_1,p_2).
\end{equation}
Then for any $\varepsilon \geq 0$, there exists a universal constant $c_{\varepsilon} > 0$ such that 
\begin{equation} \label{eq:norm_randpert_bd}
   \norm{R}_2 \leq (2+\varepsilon) 6\sqrt{2 C(\lambda,p_1,p_2) n}, 
\end{equation}
with probability at least $1 - 3n\exp\left(- \frac{8 C(\lambda,p_1,p_2) n}{\sigtil(\lambda,p_1,p_2)^2 c_{\varepsilon}} \right)$.
\end{lemma}
\begin{proof}
See Appendix \ref{app:subsec_rand_pert_bd}.
\end{proof}
We now plug in $\triangle = (2+\varepsilon) 6\sqrt{2 C(\lambda,p_1,p_2) n}$ in 
Lemma \ref{lem:mat_pert_dk}. For a fixed $\mu \in [0,1/2]$, one can readily verify that  
\begin{equation} \label{eq:bisyn_cond_n_1}
n \geq \frac{72(2+\varepsilon)^2 C(\lambda,p_1,p_2)}{\mu^2 \lambda^2 \min\set{(p_1-p_2)^2, p_2^2}} 
\Longleftrightarrow \triangle \leq \mu\min\set{n\lambda(p_1-p_2), n\lambda p_2}.
\end{equation}
Moreover, provided \eqref{eq:bisyn_cond_n_1} holds, the bounds on 
$\abs{\dotprod{v_1}{\vtil_1}}^2,\abs{\dotprod{v_2}{\vtil_2}}^2$ in Lemma \ref{lem:mat_pert_dk} 
change to
\begin{equation} \label{eq:bd_v_vtil_1}
\abs{\dotprod{v_i}{\vtil_i}}^2 \geq 1 - \left(\frac{\mu}{1-\mu}\right)^2; \quad i=1,2.
\end{equation}
%
The statement of Theorem \ref{thm:bisync_main} now follows via the following Proposition which also completes the proof.
%
%
\begin{proposition} \label{prop:chain_bd_dotprod}
Consider $x,y,\xbar \in \mathbb{C}^n$ of unit $\ell_2$ norm. 
Let $\varep, \varepbar \in [0,1]$, and let
\begin{align*}
\abs{\dotprod{x}{y}}^2 \geq 1 - \varep, \quad 
\abs{\dotprod{\xbar}{y}}^2 \geq 1 - \varepbar. 
\end{align*}
It then follows that
\begin{align}
\abs{\dotprod{x}{\xbar}}^2 \geq 1 - (\sqrt{\varep} + \sqrt{2\varepbar})^2.
\end{align}
\end{proposition}
\begin{proof}
See Appendix \ref{app:subsec_chain_bd_dotprod}.
\end{proof}

%% file: S_k_Sync.tex
%
%
We now consider the general \emph{$k$-synchronization} problem involving $k$ groups  
of angles, namely $\theta_{l,1},\dots,\theta_{l,n}$ for $l = 1,\dots,k$. Let $z_l \in \mathbb{C}^n$ denote the entry-wise 
representations of $(\theta_{l,i})_i$  on the unit circle where
\begin{equation}  \label{eq:ksync_circ_rep}
z_{l,i} = \frac{1}{\sqrt{n}} e^{\iota \theta_{l,i}} ; \quad i=1,\ldots,n.
\end{equation} 

As before, we will assume that the representations $z_l \in \mathbb{C}^n$ 
in \eqref{eq:ksync_circ_rep} are $\delta$-orthogonal. 
If $k$ is small relative to $n$, then this is the case 
when the angles $\theta_{l,i}$, $1\leq l \leq k$, $1 \leq i \leq n$ are generated 
uniformly at random in $[0,2\pi)$ and are i.i.d. Indeed, using Proposition \ref{prop:bisync_approx_orth_whp} 
and applying the union bound, it follows that $z_1,\dots,z_k$ are $\delta$-orthogonal with probability at least
\begin{align*}
1 - 2{k \choose 2} e^{-c\delta^2 n} \geq 1 - e^{2\log k} e^{-c\delta^2 n} \geq 1 - e^{-c\delta^2 n/2}, 
\end{align*}
if $n \geq \frac{4\log k}{c \delta^2}$. 

For the general $k$-way synchronization, the objective function  analogous to \eqref{eq:finalRelaxAmitObjBiSync} is given by 
\begin{equation}
\underset{ \substack{ W \in \mathbb{C}^{n \times k} \\  W^*W = I_{k \times k } }  }{\max} \; W^* H W, 
\label{eq:finalRelaxAmitObjBiSync_kway}
\end{equation}
which is maximized when the columns of $W$ are taken to be the top $k$ eigenvectors of $H$, namely $v_1,\dots,v_k$. The estimates $\hat{\theta}_{l,i}$ are then recovered by normalizing the entries of $v_{l,i}$; the complete recovery procedure is outlined in Algorithm \ref{Algo:kSync}. Note that shifting the diagonal entries of $H$ by a scaled identity matrix does not change the solution of \eqref{eq:finalRelaxAmitObjBiSync_kway}, thus for convenience, we can set $H_{ii} = c$ for any arbitrarily chosen $c \in \matR$.

\ifthenelse{\boolean{arXivMode}}{ 
\begin{algorithm}[!ht]
\begin{algorithmic}[1]

\State \textbf{Input:}  Graph $G = (V,E)$ and measurements $\Theta_{ij}$ for $\set{i,j} \in E$.  

\State  Construct the Hermitian matrix $H \in \mathbb{C}^{n \times n}$ with off-diagonal entries given by \eqref{eq:mapToCircle} and $H_{ii} = 1$ for each $i=1,2,\ldots,n$.

\State Compute the top $k$ eigenvectors $v_1, v_2,\dots,v_k$ of $H$.

\State For each $l=1,\dots,k$, obtain estimates $\hat{\theta}_{l,i}$ for $i=1,\dots,n$ where
\begin{equation*}
	e^{\imath \hat{\theta}_{l,i}} =  \frac{v_{l,i}}{|v_{l,i}|}.
\end{equation*}
\end{algorithmic}
\caption{ Algorithm for the $k$-synchronization problem}
\label{Algo:kSync}
\end{algorithm}
}
{
\begin{algorithm}[!ht]
\begin{algorithmic}[1]

\STATE{\textbf{Input:}  Graph $G = (V,E)$ and measurements $\Theta_{ij}$ for $\set{i,j} \in E$.}  

\STATE{ Construct the Hermitian matrix $H \in \mathbb{C}^{n \times n}$ with off-diagonal entries given by \eqref{eq:mapToCircle} and $H_{ii} = 1$ for each $i=1,2,\ldots,n$.}

\STATE{Compute the top $k$ eigenvectors $v_1, v_2,\dots,v_k$ of $H$.}

\STATE{For each $l=1,\dots,k$, obtain estimates $\hat{\theta}_{l,i}$ for $i=1,\dots,n$ where
\begin{equation*}
	e^{\imath \hat{\theta}_{l,i}} =  \frac{v_{l,i}}{|v_{l,i}|}.
\end{equation*}}
\end{algorithmic}
\caption{ Algorithm for the $k$-synchronization problem}
\label{Algo:kSync}
\end{algorithm}
}

\subsection{Erd\H{o}s-Renyi measurement model} 
Analogous to the $k=2$ case, we consider the graph $G = (V,E)$ to be generated via the Erd\H{o}s-R\'{e}nyi random graph model $G(n,\lambda)$,  where each edge of $G$ is present 
with probability $\lambda \in [0,1]$. 
Given $G$, we then obtain for each $\set{i,j} \in E$ ($i < j$ w.l.o.g) the value $\Theta_{i,j}$ as
\begin{equation} \label{Def:MixtureOmega_gen}
\Theta_{ij} = \left\{
\begin{array}{rl}
   	(\theta_{l,i} - \theta_{l,j}) \mod 2\pi; & \text{ w.p } \;\; p_l \text{ for } l=1,\dots,k  \\
   	N_{ij} \sim U[0,2\pi); & \text{ w.p }  \;\;  \eta = 1 - \sum_{l=1}^{k} p_l, 
\end{array}
   \right.
\end{equation}
where $p_l$ denotes the probability of a correct pairwise measurement 
in $G_l$, and $\eta$ is the probability of getting an incorrect measurement. 
We will assume w.l.o.g that $p_1 > p_2 > \dots > p_k$.

We map the points over the complex unit circle as in \eqref{eq:mapToCircle} to construct 
a Hermitian matrix $H$. Then due to \eqref{Def:MixtureOmega_gen}, we have for 
each $i < j$ that
\begin{equation*} 
H_{ij} = \left\{
	\begin{array}{rl}
   	e^{\iota (\theta_{l,i} - \theta_{l,j})}  &  \text{with probability }   p_l \lambda \text{ for } l=1,\dots,k	\\
 	  e^{\iota N_{ij}},  N_{ij} \sim U[0,2\pi) &  \text{with probability } (1-\sum_{l=1}^{k} p_l) \lambda	\\
		0		& \text{with probability } 1-\lambda.	\\
     \end{array}
   \right.
\end{equation*}
Moreover, we set each diagonal entry $H_{ii} = \lambda(p_1+p_2+\dots+p_k)$ for convenience. 
We can then write $\expec[H]$ and $H$ as 
\begin{align*} 
	\expec[H] &= np_1 \lambda z_1 z_1^* + np_2 \lambda z_2 z_2^* + \dots 
	+ np\lambda z_k z_k^*, \\ 
		H &= \mathbb{E}(H) + R = n\lambda (p_1 z_1 z_1^* + \dots p_k z_k z_k^*) + R, 
\end{align*}
where $R$ is a random Hermitian matrix with $R_{ii} = 0$, and whose elements 
$R_{ij} = H_{ij} - \mathbb{E}[H_{ij}] $ are zero-mean independent random variables for $i \leq j$. 
In particular, for each $i < j$, we have
\begin{equation} \label{eq:ksync_noise_form}
R_{ij} = \left\{
	\begin{array}{rl}
   	e^{\iota (\theta_{l,i} - \theta_{l,j})} - \sum_{l'=1}^{k} p_{l'}\lambda e^{\iota (\theta_{l',i} - \theta_{l',j})} 
    &  \text{ w.p }   \;\; p_l \lambda, \ l=1,\dots,k 	\\
 	e^{\iota N_{ij}} - \sum_{l'=1}^{k} p_{l'}\lambda e^{\iota (\theta_{l',i} - \theta_{l',j})}, 
	N_{ij} \sim U[0,2\pi) &  \text{ w.p } \;\; (1-\sum_{l'=1}^{k} p_{l'}) \lambda	\\
		- \sum_{l'=1}^{k} p_{l'} \lambda e^{\iota (\theta_{l',i} - \theta_{l',j})} & \text{ w.p }  \;\;  1-\lambda,	\\
     \end{array}
   \right.
\end{equation}
and $R_{ji} = R_{ij}^{*}$. For $\delta$ sufficiently small, clearly 
the matrix $\expec[H]$ will be rank $k$. An easy application of Gerschgorin's disk theorem 
reveals that the choice $\delta < 1/(k-1)$ suffices.
%
%

%
%
\subsection{Main result} 
Before stating our main result, it will be helpful to define additional notation. For $1 \leq m \leq k$, denote $S_m = \sum_{j=1}^m p_j$. 
For any $\delta \in [0,1]$, let us define the numbers $\ergen_1(\delta),\dots,\ergen_k(\delta)$ as follows.
\begin{enumerate}
\item $\ergen_1(\delta) := \frac{p_2(k-1)}{p_1-p_2}\delta$.
\item $\ergen_{j}(\delta) := C_{j}\sqrt{\ergen_{j-1}(\delta)}$ for $2 \leq j \leq k$ where 
\begin{equation*}
C_j = \frac{p_{j+1}(k-1)\sqrt{2} + 4\sqrt{2}[2S_{j-1} + \sqrt{2}(j-1)(j-2)(p_1-p_{j}) 
+ \frac{(j-1)}{2}(p_2(k-1) - 2p_{j+1})]}{p_j - p_{j+1}}
\end{equation*}
with $p_{k+1} = 0$. 
\end{enumerate}
Moreover, with the convention $S_0 = 0$, let us also define 
\begin{align} 
E_j := 4\sqrt{2}S_{j-2} &+ 8(j-2)(j-3) (p_1 - p_{j-1}) + 4\sqrt{2}p_{2}(k-1)(j-2) \nonumber \\ 
&+ 4(j-1)(p_1 - p_{j+1}) + p_{j+1}(k-1); \quad 2 \leq j \leq k, \label{eq:E_j_def}
\end{align}
\begin{equation} \label{eq:Etilde_def}
  \widetilde{E} := 4\sqrt{2} S_{k-1} + 8(k-1)(k-2) (p_1 - p_{k}) + 4\sqrt{2} p_{2}(k-1)^2.
\end{equation}
Our main theorem for the $k$-synchronization problem for the case of approximately orthogonal 
representations $z_1,z_2,\dots,z_k$ is outlined below.
%
\begin{theorem} \label{thm:ksync_main}
Let $z_1,\dots,z_k$ in \eqref{eq:ksync_circ_rep} be $\delta$-orthogonal for $\delta \in (0,1)$. Assuming $p_1 > p_2 > \cdots > p_k$, under the notation defined previously, suppose that the following conditions are satisfied for a constant $\mu \in [0,1/2]$.
\begin{enumerate}
\item $\delta \leq \sqrt{2 \ergen_j(\delta)} \leq \frac{1}{2}$ for $1 \leq j \leq k-1$;
\item $\ergen_1(\delta) \leq \ergen_2(\delta) \leq \dots \leq \ergen_{k-1}(\delta)$; 
\item $\ergen_j(\delta) \leq \mu^2\left(\frac{p_j - p_{j+1}}{2E_{j+1}} \right)^2$ for $1 \leq j \leq k-2$;
\item $\ergen_{k-1}(\delta) \leq \mu^2\min\set{\left(\frac{p_k}{2\widetilde{E}} \right)^2, \left(\frac{p_{k-1} - p_k}{2E_{k}} \right)^2}$, 
\end{enumerate}
with $E_j, \widetilde{E}$ defined in \eqref{eq:E_j_def}, \eqref{eq:Etilde_def} respectively.
Moreover, for a constant $\varep \in (0,1/2)$, let $n$ satisfy 
\begin{equation*} 
n \geq \frac{288 (2+\varepsilon)^2 C(\lambda,p_1,\dots,p_k)}{\mu^2 \lambda^2 (\min_{1 \leq j \leq k}\set{p_j-p_{j+1}})^2}
\end{equation*}
where $C(\lambda,p_1,\dots,p_k) > 0$ depends only on $\lambda,p_1,\dots,p_k$ (see \eqref{eq:ksync_cterm_1}). 
Then with probability at least 
$1 - 3n\exp\left(-\frac{8 C(\lambda,p_1,\dots,p_k) n}{\sigtil(\lambda,p_1,\dots,p_k)^2 c_{\varepsilon}}\right)$, 
it holds that 
\begin{equation*}
\abs{\dotprod{v_j}{z_j}}^2 \geq 1 - \left(\sqrt{\ergen_j(\delta)} + \frac{\mu}{1-\mu}\right)^2, \quad \forall j=1,\dots,k.
\end{equation*}
Here, $c_{\varepsilon}, \sigtil(\lambda,p_1,\dots,p_k) > 0$ only depend on the indicated parameters (see \eqref{eq:infty_norm_grn_bd_1}).
\end{theorem}
The following remarks are in order regarding Theorem \ref{thm:ksync_main}. 
\begin{enumerate}
    \item A big part of the analysis revolves around deriving bounds on the correlation between the top $k$ eigenvectors of $\expec[H]$ and $z_j$'s for each $j =1,\dots,k$ (see Lemma \ref{lem:deflation_gen}). When $k=2$, this was relatively easier, as we had closed-form expressions for the top two eigenvalues of $\expec[H]$. This is not the case in general of course, and thus, we resort to a deflation argument that leads to upper and lower bounds on the top $k$ eigenvalues of $\expec[H]$, defined in a recursive manner.
    
    \item The conditions involving $\ergen_j(\delta)$ impose that $\delta$ be sufficiently small. In the special case $\delta = 0$, we obtain $\ergen_j = 0$ for each $j = 1,\dots,k$. We believe that the requirement imposed on $\delta$ in Theorem \ref{thm:ksync_main} is pessimistic, but obtaining less stringent conditions appears to be difficult.
    
    \item It is possible to replace the terms $C(\lambda,p_1,\dots,p_k)$ and  $\sigtil(\lambda,p_1,\dots,p_k)$ by their respective upper bounds $C(\lambda,p_1,\dots,p_k) \lesssim \lambda$ and $\sigtil(\lambda,p_1,\dots,p_k) \lesssim 1$. We then have for $\delta$ small enough that if 

\vspace{-7mm}
\begin{equation*}
    \lambda \gtrsim \frac{\log n}{n}, \quad n \gtrsim \frac{1}{\mu^2 \lambda (\min_{1 \leq j \leq k}\set{p_j-p_{j+1}})^2},
\end{equation*}
then the stated correlation bounds hold with probability at least $1 - \frac{1}{n}$.
\end{enumerate}
\subsection{Proof of Theorem \ref{thm:ksync_main}} \label{subsec:ksync_main_proof}
Let us denote $\lamtil_1 \geq \dots \geq \lamtil_k$ (resp. $\vtil_1,\dots,\vtil_k$) to be the $k$ largest  
eigenvalues (resp. corresponding eigenvectors) of $\expec[H]$. Since $z_i$'s are approximately 
orthogonal, and $p_i \neq p_j$, one would expect for small enough $\delta$ that $\lamtil_i \approx n p_i \lambda$ and $\abs{\dotprod{\vtil_i}{z_i}} \approx 1$ holds. 
This is stated precisely in the following Lemma.
%
\begin{lemma} \label{lem:deflation_gen}
For $1 \leq m \leq k$, denote $S_m = \sum_{j=1}^m p_j$. 
Let $\delta \in [0,1]$ additionally satisfy the following conditions.
\begin{enumerate}
\item $\delta \leq \sqrt{2 \ergen_j(\delta)} \leq \frac{1}{2}$ for $1 \leq j \leq k-1$.
\item $\ergen_1(\delta) \leq \ergen_2(\delta) \leq \dots \leq \ergen_{k-1}(\delta)$.
\end{enumerate}
It then follows for each $j=1,\dots,k$ that
\begin{equation} 
  np_j\lambda - l_j(\delta) \leq \lamtil_j \leq np_j\lambda + u_j(\delta) , \quad 
	\abs{\dotprod{\vtil_j}{z_j}}^2 \geq 1 - \ergen_j(\delta),
\end{equation}
where $l_j(\delta), u_j(\delta)$ are defined as follows.
\begin{align}
l_j(\delta) &:= 4\lambda \sqrt{2\ergen_{j-1}(\delta)}(nS_{j-1} + \sqrt{2} (j-1)(j-2) (np_1 - np_{j}) + np_{2}(k-1)(j-1)), \\
u_j(\delta) &:= 4\sum_{i=1}^{j-1} (np_i - np_{j+1})\lambda\sqrt{\ergen_i(\delta)} + np_{j+1}\lambda(k-1)\delta,
\end{align}
with $\ergen_0(\cdot), S_0 = 0$. In particular, $l_1(\delta) = 0$ and $u_1(\delta) = np_2\lambda(k-1)\delta$.
\end{lemma}
%
\begin{proof}
 See Appendix \ref{app:subsec_deflation_gen}.
\end{proof}
Our next step is to provide a generalization of 
Lemma \ref{lem:mat_pert_dk}. Recall that 
$\lambda_1 \geq \lambda_2 \geq \cdots \lambda_n$ (resp. $v_1,v_2,\dots,v_n$) 
are denoted to be the eigenvalues (resp. eigenvectors) of $H$. Our goal is 
to quantify the deviation of $v_i$ from $\vtil_i$ for each $i=1,\dots,k$. Before stating the lemma we need to define some additional notation.
%
%
\begin{lemma} \label{lem:mat_pert_dk_gen}
Under the notation and conditions in Lemma \ref{lem:deflation_gen}, let $\delta$ additionally satisfy the following conditions for $\mu \in [0,1/2]$.
\begin{enumerate}
\item $\ergen_j(\delta) \leq \mu^2\left(\frac{p_j - p_{j+1}}{2E_{j+1}} \right)^2$ for $1 \leq j \leq k-2$,

\item and $\ergen_{k-1}(\delta) \leq \mu^2\min\set{\left(\frac{p_k}{2\widetilde{E}} \right)^2, \left(\frac{p_{k-1} - p_k}{2E_{k}} \right)^2}$, 
\end{enumerate}
with $E_j, \widetilde{E}$ defined in \eqref{eq:E_j_def}, \eqref{eq:Etilde_def} respectively.
If $\norm{R}_2 \leq \triangle$ with 
$\triangle \leq \mu\frac{n\lambda}{2}\min_{1 \leq j \leq k}\set{p_j-p_{j+1}}$, 
it then follows for each $j = 1,\dots,k$ that $\abs{\dotprod{\vtil_j}{v_j}}^2 \geq 1 - \left(\frac{\mu}{2(1-\mu)}\right)^2$.
\end{lemma}
%
%
\begin{proof}
See Appendix \ref{app:subsec_mat_pert_dk_gen}.
\end{proof}
Next, we proceed to bound the spectral norm of the perturbation matrix $R$. 
This is shown in the following lemma,  the proof of which is similar to that 
of Lemma \ref{lem:rand_pert_bd}.

\begin{lemma} \label{lem:rand_pert_bd_gen}
Let $R = \real(R) + \iota \imag(R)$, where $\real(R), \imag(R) \in \matR^{n \times n}$ 
are respectively symmetric and skew-symmetric matrices. Denote 
\begin{align} 
C(\lambda,p_1,\dots,p_k) = \sum_{l=1}^k &2p_l\lambda\left[(1-p_l\lambda)^2 + \left(\sum_{l' \neq l} p_{l'} \lambda\right)^2 \right] 
+ (1-\sum_{l=1}^k p_l)\lambda\left[\frac{1}{2} + \left(\sum_{l'=1}^k p_{l'}\lambda \right)^2\right] \nonumber \\
&+ (1-\lambda)\left(\sum_{l'=1}^k p_{l'}\lambda \right)^2, \label{eq:ksync_cterm_1}
\end{align}
and let $\sigtil(\lambda,p_1,\dots,p_l) > 0$ be such that 
\begin{equation} \label{eq:infty_norm_grn_bd_1}
\max_{i,j} \set{\norm{\real(R)_{ij}}_{\infty}, \norm{\imag(R)_{ij}}_{\infty}} \leq \sigtil(\lambda,p_1,\dots,p_k).
\end{equation}
Then for any $\varepsilon \geq 0$, there exists a universal constant $c_{\varepsilon} > 0$ such that 
\begin{equation} \label{eq:norm_randpert_gen_bd}
   \norm{R}_2 \leq (2+\varepsilon) 6\sqrt{2 C(\lambda,p_1,\dots,p_k) n}, 
\end{equation}
with probability at least 
$1 - 3n\exp\left(- \frac{8 C(\lambda,p_1,\dots,p_k) n}{\sigtil(\lambda,p_1,\dots,p_k)^2 c_{\varepsilon}} \right)$.
\end{lemma}
%
%
\begin{proof}
See Appendix \ref{app:subsec_rand_pert_bd_gen}.
\end{proof}
Plugging $\triangle = (2+\varepsilon) 6\sqrt{2 C(\lambda,p_1,\dots,p_k) n}$ in 
Lemma \ref{lem:mat_pert_dk_gen}, we obtain 
\begin{equation}\label{eq:n_cond_gen}
\triangle \leq \mu\frac{n\lambda}{2}\min_{1 \leq j \leq k}\set{p_j-p_{j+1}} 
\Longleftrightarrow 
n \geq \frac{288 (2+\varepsilon)^2 C(\lambda,p_1,\dots,p_k)}{\mu^2 \lambda^2 (\min_{1 \leq j \leq k}\set{p_j-p_{j+1}})^2}.
\end{equation}
Hence if $n$ satisfies the condition in \eqref{eq:n_cond_gen} for $\mu \in [0,1/2]$, 
and $\delta$ satisfies the conditions 
in Lemmas \ref{lem:deflation_gen} and \ref{lem:mat_pert_dk_gen}, it holds with probability at least 
$1 - 3n\exp\left(-\frac{8 C(\lambda,p_1,\dots,p_k) n}{\sigtil(\lambda,p_1,\dots,p_k)^2 c_{\varepsilon}}\right)$ 
that 
\begin{equation*}
\abs{\dotprod{\vtil_j}{z_j}}^2 \geq 1 - \ergen_j(\delta), \quad
\abs{\dotprod{\vtil_j}{v_j}}^2 \geq 1 - \left(\frac{\mu}{2(1-\mu)}\right)^2; \quad j = 1,\dots,k.
\end{equation*}
Finally, by using Proposition \ref{prop:chain_bd_dotprod} for each $j = 1,\dots,k$ with 
$\varep_j = \ergen_j(\delta)$ and $\varepbar_j = \left(\frac{\mu}{2(1-\mu)}\right)^2$, we obtain 
\begin{equation*}
\abs{\dotprod{v_j}{z_j}}^2 \geq 1 - \left(\sqrt{\ergen_j(\delta)} + \frac{\mu}{1-\mu}\right)^2.
\end{equation*}
This completes the proof.

%% file: S_experiments.tex
This section details the outcomes of a variety of numerical experiments of our proposed  
\Cref{Algo:kSync} for the general setting of $k$-synchronization, showcasing its robustness to noise and sampling sparsity.  We evaluate performance in a number of settings, where we vary the number of nodes  $n$, the number of collections of angles $k$, and show goodness of recovery as a function of the sampling sparsity of the measurement graph $G$. Furthermore, we also experiment with a modified version of \Cref{Algo:kSync} which employs an additional normalization step, and also compare with the analogous SDP relaxation \cref{eq:SDP_program_SYNC} for $k$-synchronization, in a subset of the experiments.  
We measure performance by the correlation between the group of estimated and ground truth angles. More specifically, if $\theta_i$, resp. $\hat{\theta}_i$, denotes the ground truth, resp. estimated, angles for $i=1, \ldots, n$, we compute the correlation as
\begin{equation*} 
\mbox{Corr}(\theta, \hat{\theta}) = \abs{ \dotprod{z}{\hat{z}}}, \quad  \mbox{with}  \quad z_i =  \frac{1}{\sqrt{n}} e^{\imath \theta_i}, \;\; \mbox{and} \;\;  \hat{z}_i = \frac{1}{\sqrt{n}} e^{\imath \hat{\theta}_i}, \quad i = 1, \ldots, n.
\end{equation*}
%
%

\begin{remark}
For convenience, one may construe the given measurement graph $G$ as a union of $k+1$ edge-disjoint subgraphs, where $G_1, G_2, \ldots, G_k$ denote the subgraphs of \textit{good} measurements (with respective edge densities $p_1 \lambda, p_2\lambda, \ldots, p_k\lambda$), while the subgraph $W$ contains all the outlier \textit{bad} measurements (with edge density $ \eta \lambda $, where we recall that $\eta = 1-(p_1+p_2+\ldots+p_k)$). 
\end{remark}

\subsection{Setup I: correlation versus graph sparsity}
The plots in Figure \ref{fig:scanID_3_abcd_k2_500_1000}  pertain to the setting of bi-synchronization; the top, respectively bottom, row considers the case of $n=500$, respectively $n=1000$, angles for various sparsity and noise levels. As expected, the recovery of the collection of angles whose good measurement graph $G_1$ is denser  (i.e., higher $p_1$) is significantly better than that for which the measurement graph $G_2$ is sparser (i.e., $ p_1>p_2$). As the noise level increases, the performance gap between the two recovery error curves becomes wider, as expected. For sparse graphs and large levels of noise (as shown in Figure \ref{fig:scanID_3_abcd_k2_500_1000} (c) and (f)), the recovery is significantly worse, especially for the  angles corresponding to the sparser measurement graph $G_2$.




\begin{figure}[!ht]
\vspace{-2mm}
\centering
\raisebox{0.75in}{\rotatebox[origin=t]{90}{$k=2; \;\;  n=500$}}\hspace{1mm}
\subcaptionbox[]{ $(p_1, p_2) = (0.3, 0.2)$, $ \eta = 0.50$  
}[ 0.31\textwidth ]
{\includegraphics[width=0.31\textwidth, trim=0cm 0cm 0cm 0.7cm,clip] {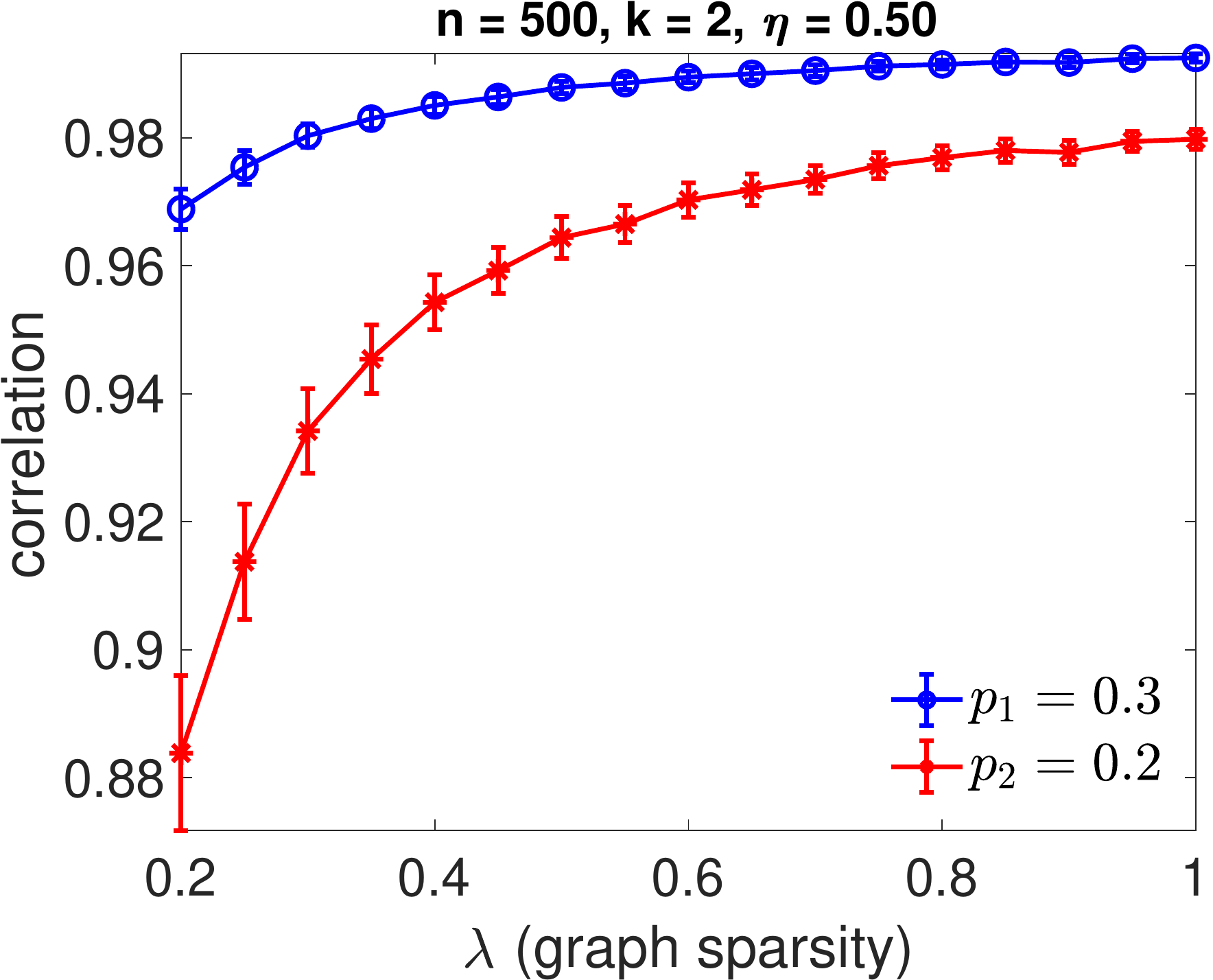} }
%
\subcaptionbox[]{ $(p_1, p_2) = (0.2,0.1)$, $ \eta = 0.70$   
}[ 0.31\textwidth ]
{\includegraphics[width=0.31\textwidth, trim=0cm 0cm 0cm 0.7cm,clip] {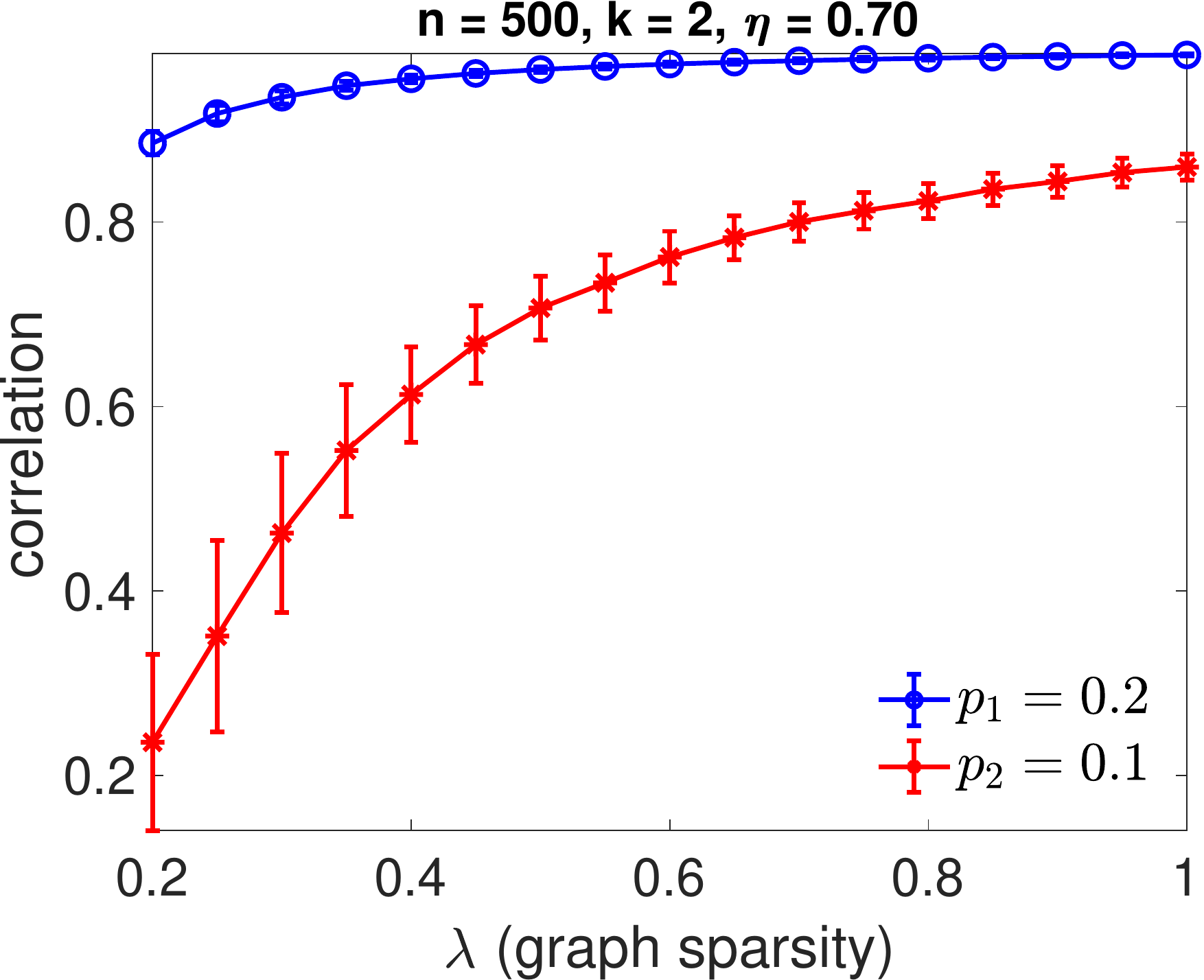} }
%
\subcaptionbox[]{ $(p_1, p_2) = (0.1,0.05)$, $\eta = 0.85$
}[ 0.31\textwidth ]
{\includegraphics[width=0.31\textwidth, trim=0cm 0cm 0cm 0.7cm,clip] {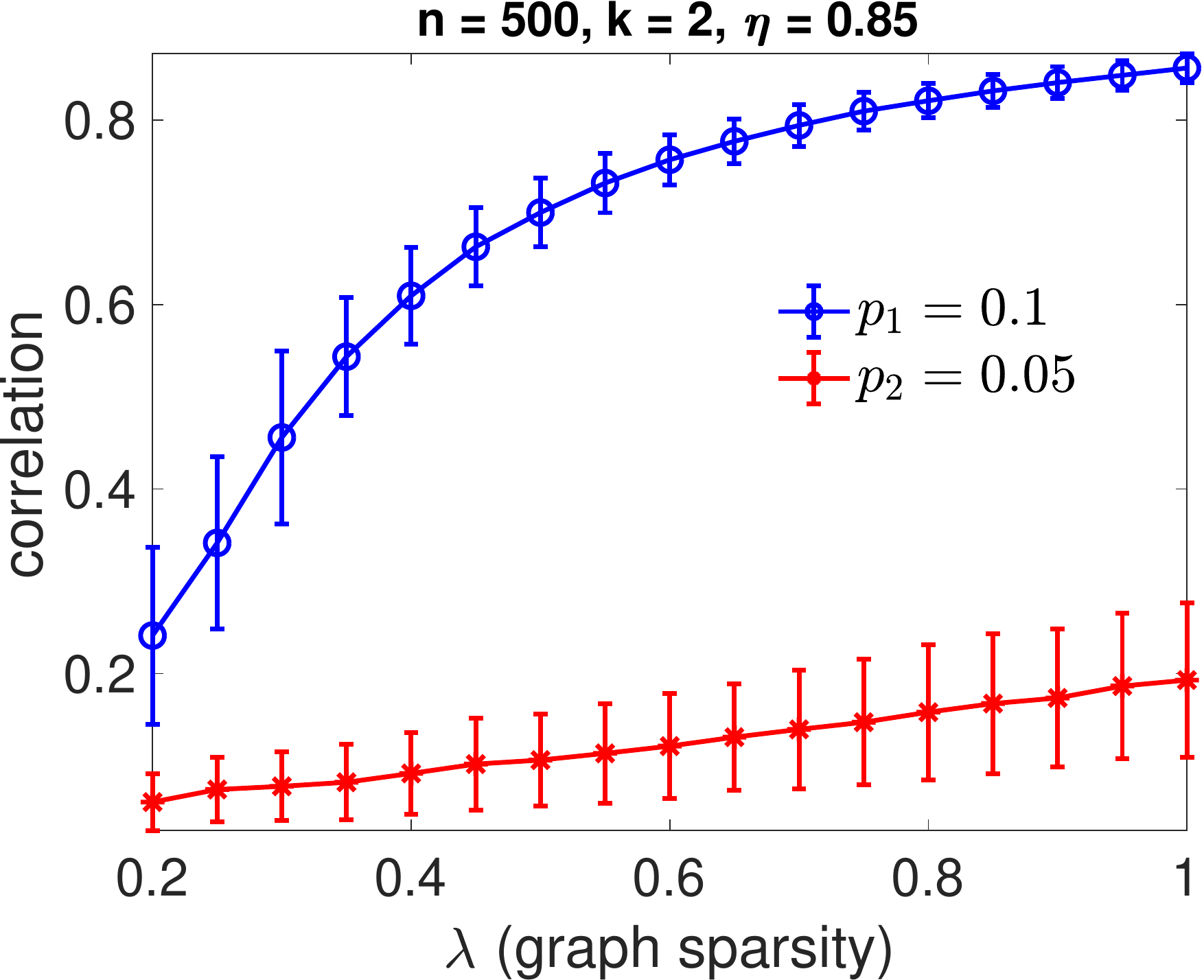} }
%
%

\raisebox{0.75in}{\rotatebox[origin=t]{90}{$k=2;  \;\; n=1000$}}\hspace{1mm}
\subcaptionbox[]{ $(p_1, p_2) = (0.3,0.2)$, $\eta = 0.50$
}[ 0.31\textwidth ]
{\includegraphics[width=0.31\textwidth, trim=0cm 0cm 0cm 0.7cm,clip] {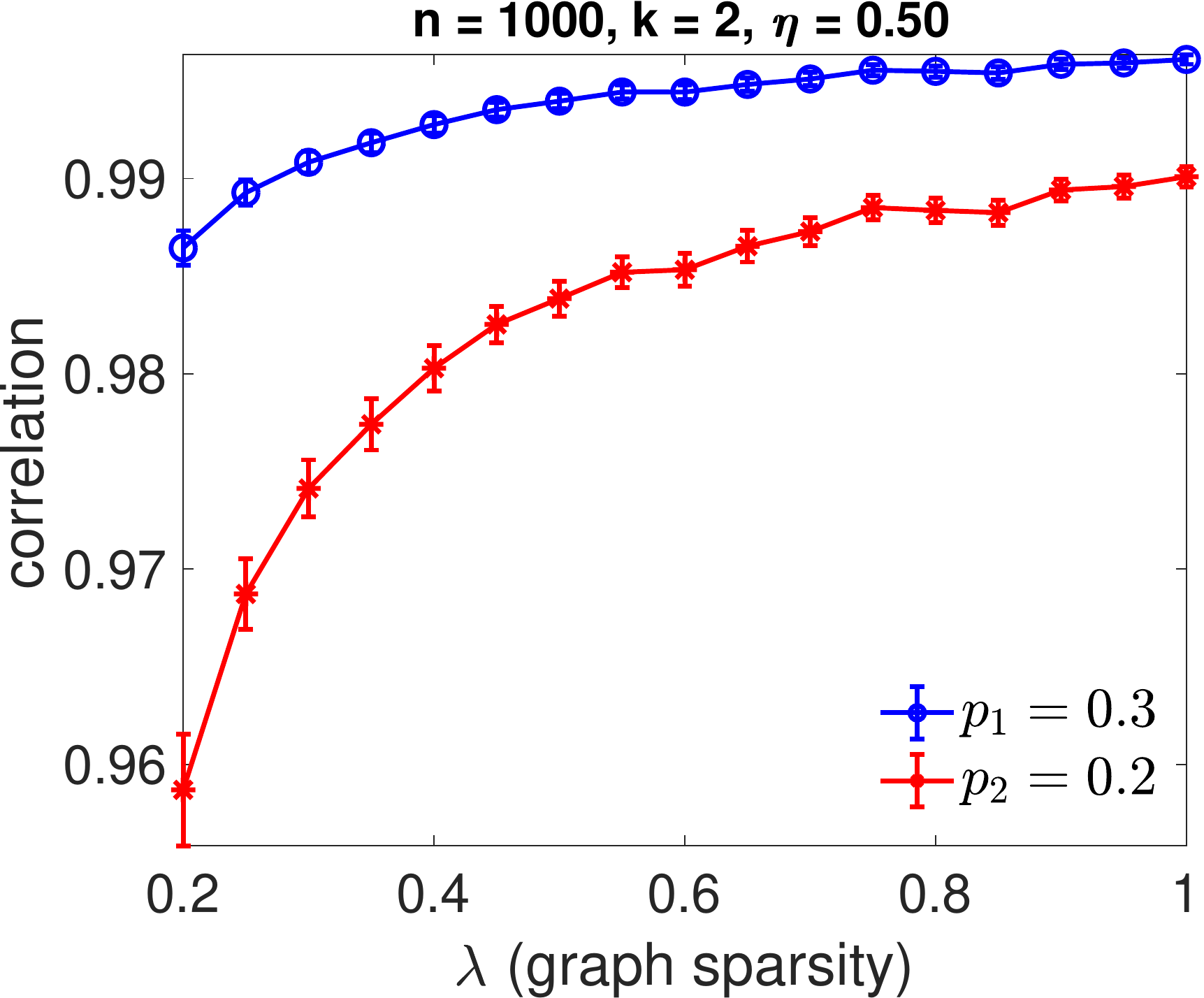} }
%
\subcaptionbox[]{ $(p_1, p_2) = (0.20,0.10)$, $\eta = 0.70$
}[ 0.31\textwidth ]
{\includegraphics[width=0.31\textwidth, trim=0cm 0cm 0cm 0.7cm,clip] {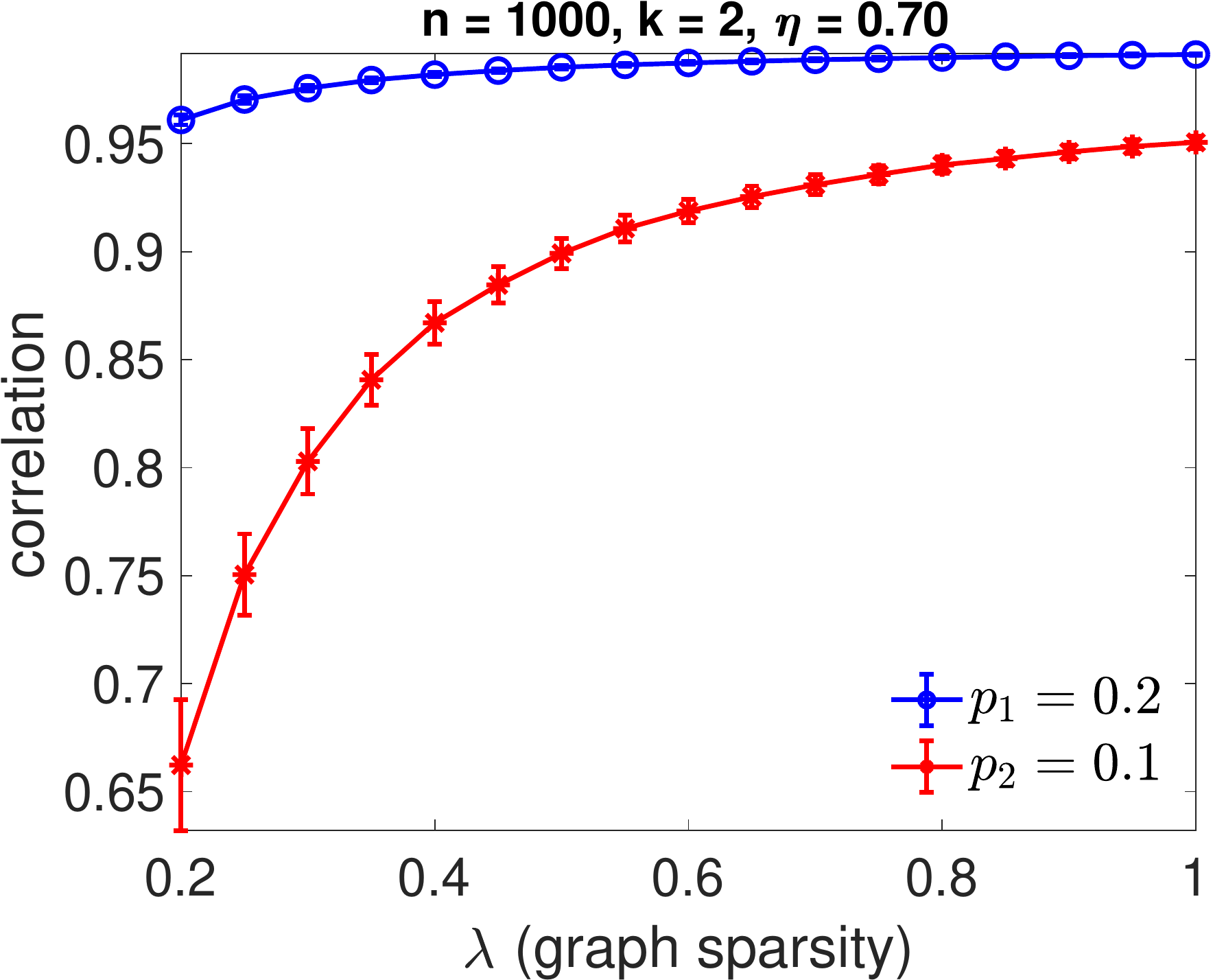} }
%
\subcaptionbox[]{ $(p_1, p_2) = (0.1,0.05)$, $\eta = 0.85$
}[ 0.31\textwidth ] 
{\includegraphics[width=0.31\textwidth, trim=0cm 0cm 0cm 0.7cm,clip] {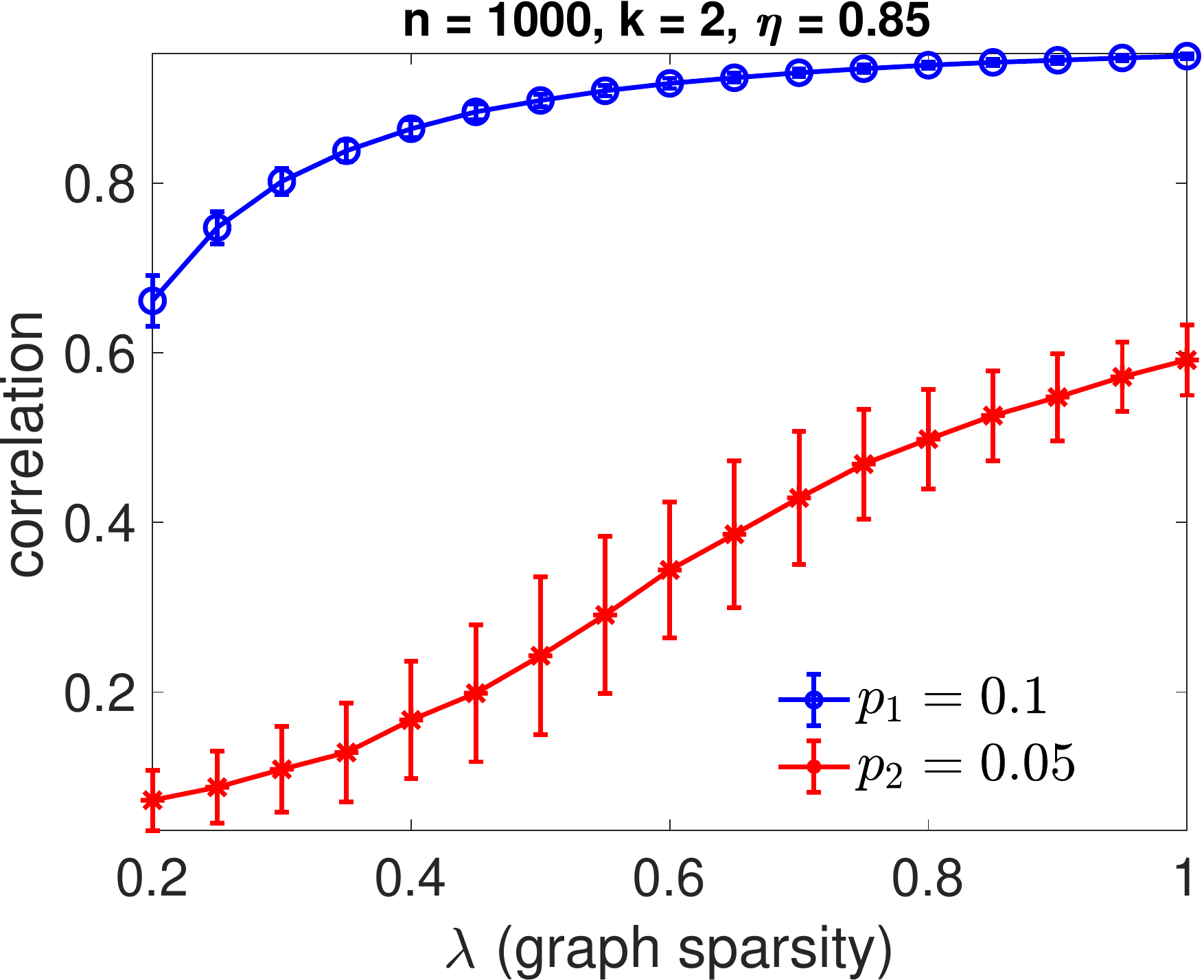} }
\vspace{-2mm}
\captionsetup{width=1.0\linewidth}
\caption[Short Caption]{Experimental Setup I: Correlation of recovered angles with the ground truth for $n=500$ (top), respectively $n=1000$ (bottom), with $k=2$ groups of angles, and various noise levels $\eta \in \{ 0.50, 0.70, 0.85\}$ as a function of $\lambda$ (sparsity of the measurement graph). Results are averaged over $20 \times 20 = 400$ runs (i.e, over 20 random instances of groups of angles, and for each such a group of angles, we consider 20 random instances of the sampling graph according to $p_1$ and $p_2$). 
}
\vspace{-3mm}
\label{fig:scanID_3_abcd_k2_500_1000}
\end{figure}

Figure \ref{fig:scanID_1_abcd_k34_1000} shows similar plots for the case of $k=3$-way (top  row) and $k=4$-way (bottom row) synchronization. When looking at the plots corresponding to $k=3$, as expected, the sparsest measurement graph $G_3$ attains the worst performance (i.e., the black curve corresponding to $p_3$). Furthermore, when the sparsity levels are close to each other, i.e., $p_1 - p_2$ and $p_2 - p_3$ are small, the overall performance  decreases, and the standard deviation error bars increase in magnitude. 
 %
The bottom-row plots of the same Figure \ref{fig:scanID_1_abcd_k34_1000} exhibit similar behaviour for the case of $k=4$-way synchronization. The bigger the gap between adjacent sparsity levels, the bigger the performance gap between the recovery levels of the corresponding collections of angles. For the sparsest setting and large levels of noise (as shown in Figure \ref{fig:scanID_1_abcd_k34_1000} (e) and (f)), we observe a significantly worse recovery for the angles corresponding to the sparser measurement graph $G_4$ (the magenta curve labeled $p_4$).



\begin{figure}[!ht]
\vspace{-2mm}
\captionsetup[subfigure]{justification=centering}
\centering
\raisebox{0.75in}{\rotatebox[origin=t]{90}{k=3}}\hspace{1mm}
\subcaptionbox[]{ 
$(p_1, p_2, p_3) = (0.30, 0.20, 0.10)$, and  $\eta = 0.40$
}[ 0.31\textwidth ]
{\includegraphics[width=0.31\textwidth, trim=0cm 0cm 0cm 0.7cm,clip] {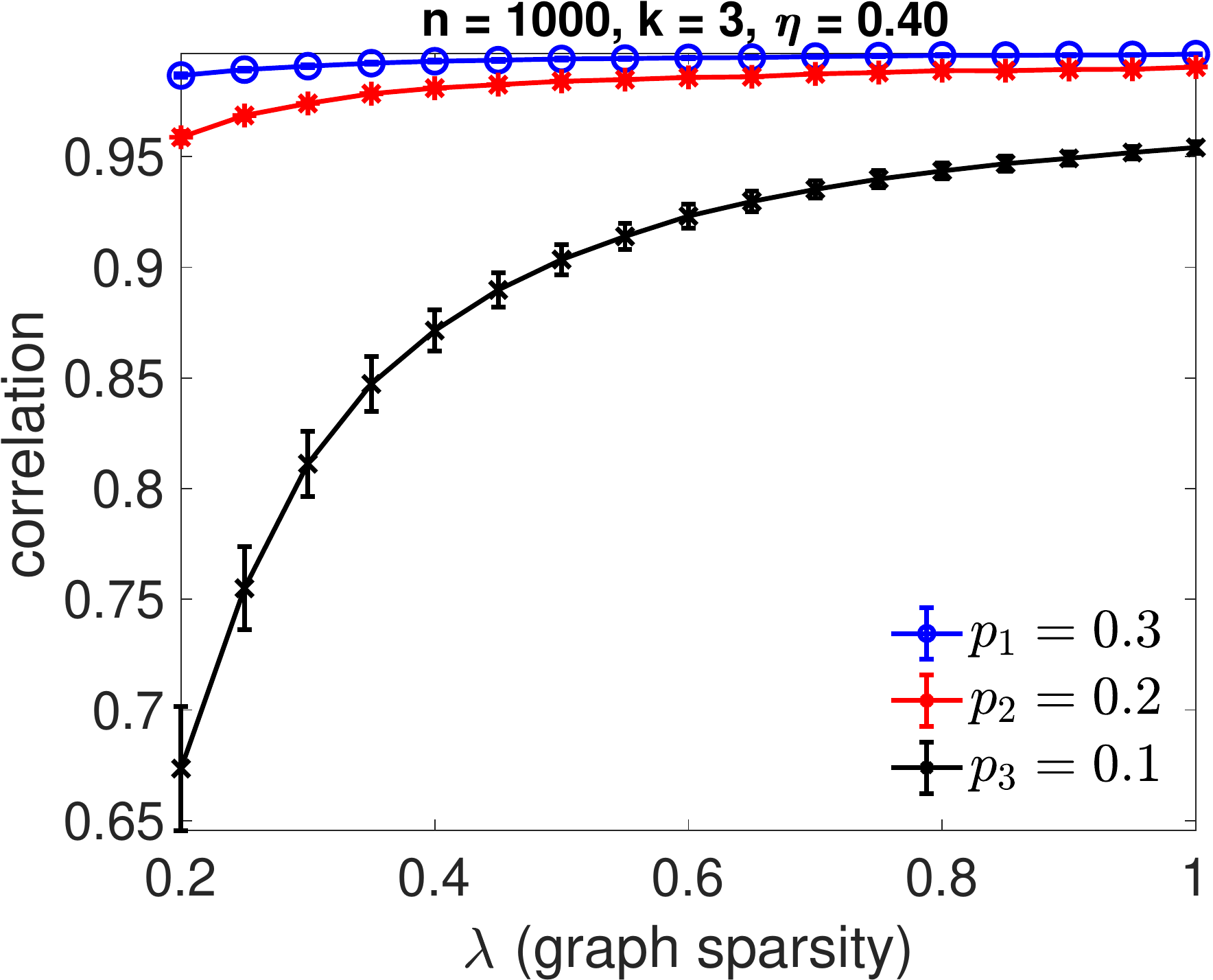} }
%
\subcaptionbox[]{ 
$(p_1, p_2, p_3) = (0.25, 0.20, 0.15)$,  and  $\eta = 0.40$
}[ 0.31\textwidth ]
{\includegraphics[width=0.31\textwidth, trim=0cm 0cm 0cm 0.7cm,clip] {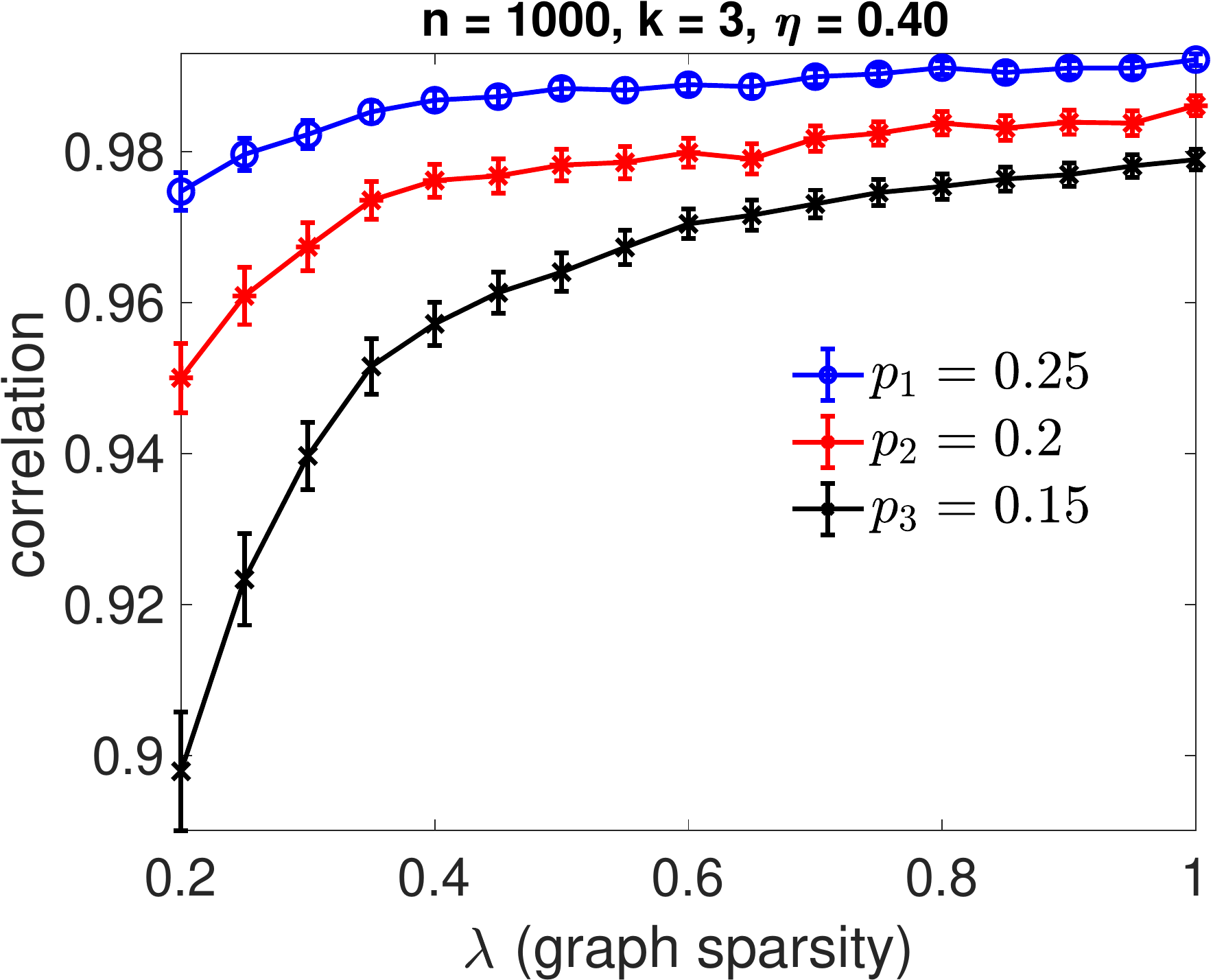} }
%
%
\subcaptionbox[]{ 
$(p_1, p_2, p_3) = (0.20, 0.15, 0.10)$,  and   $\eta = 0.55$
}[ 0.31\textwidth ]
{\includegraphics[width=0.31\textwidth, trim=0cm 0cm 0cm 0.7cm,clip] {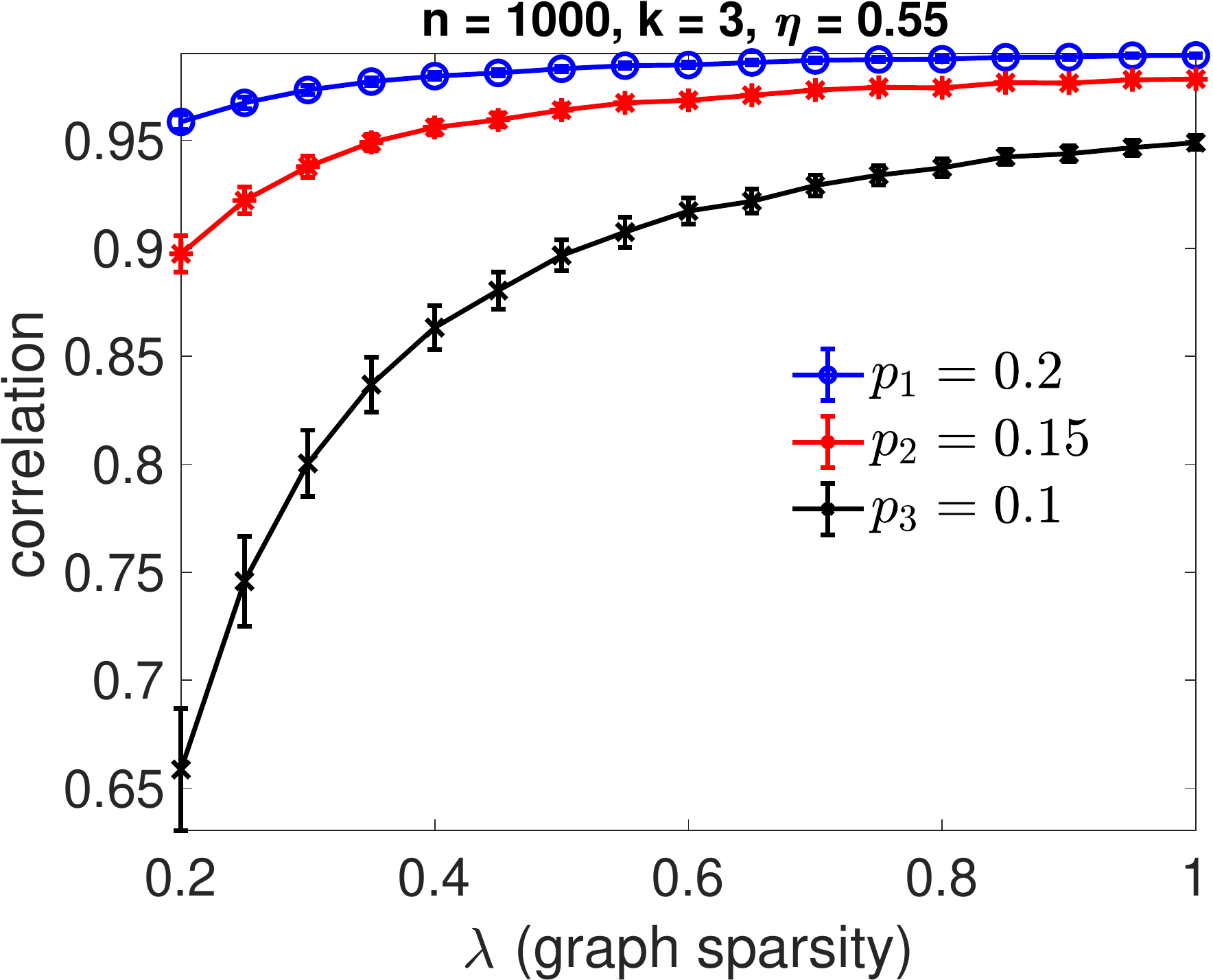} }
%
%
%
%
%
%

\raisebox{0.75in}{\rotatebox[origin=t]{90}{k=4}}\hspace{1mm}
\subcaptionbox[]{ $(p_1, p_2, p_3, p_4) = (0.25, 0.20, 0.15, 0.10)$, and  $\eta = 0.30$
}[ 0.31\textwidth ]
{\includegraphics[width=0.31\textwidth, trim=0cm 0cm 0cm 0.7cm,clip] {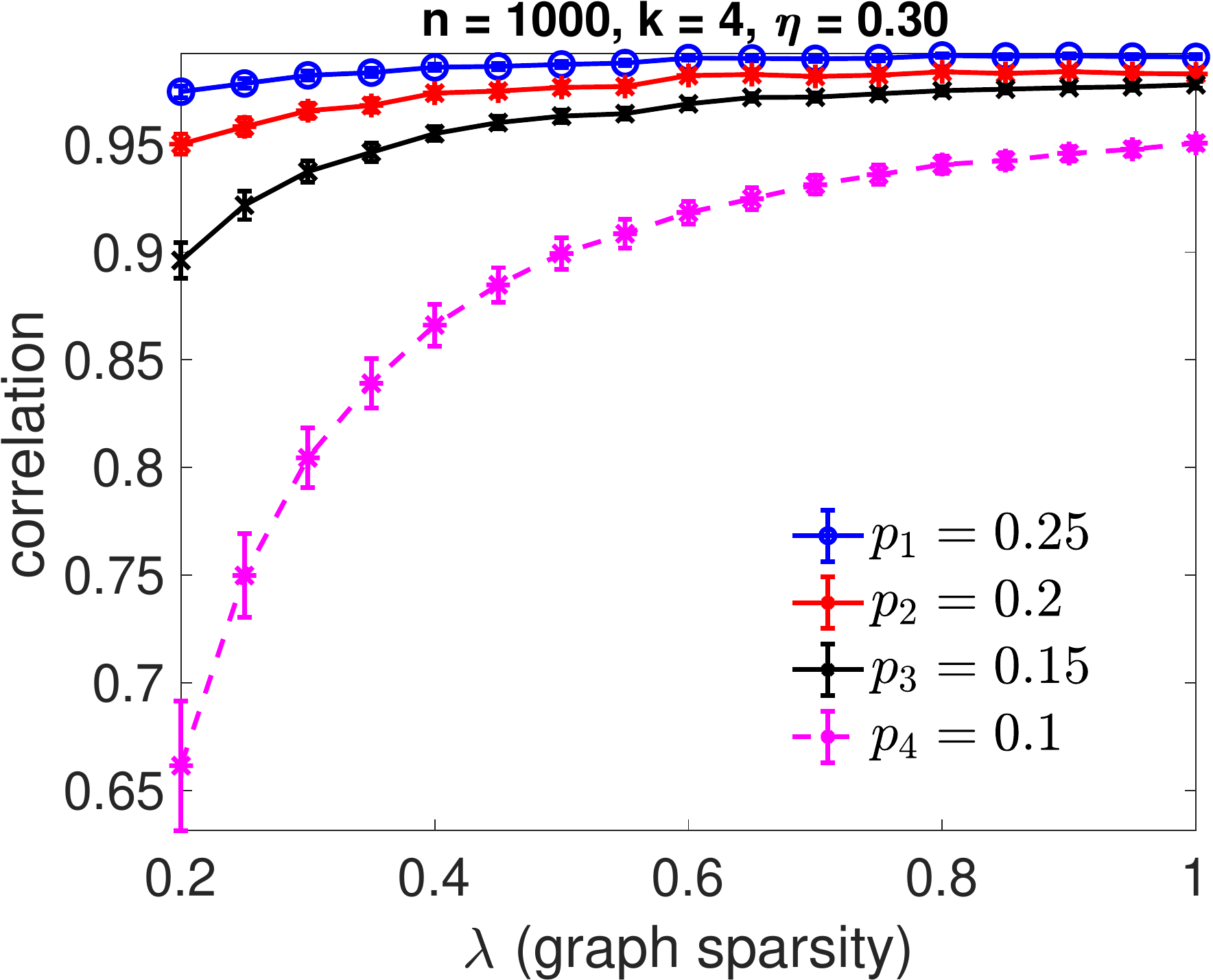} }
%
%
\subcaptionbox[]{ $(p_1, p_2, p_3, p_4) = (0.20, 0.15, 0.10, 0.05)$, and  $\eta = 0.50$
}[ 0.31\textwidth ]
{\includegraphics[width=0.31\textwidth, trim=0cm 0cm 0cm 0.7cm,clip] {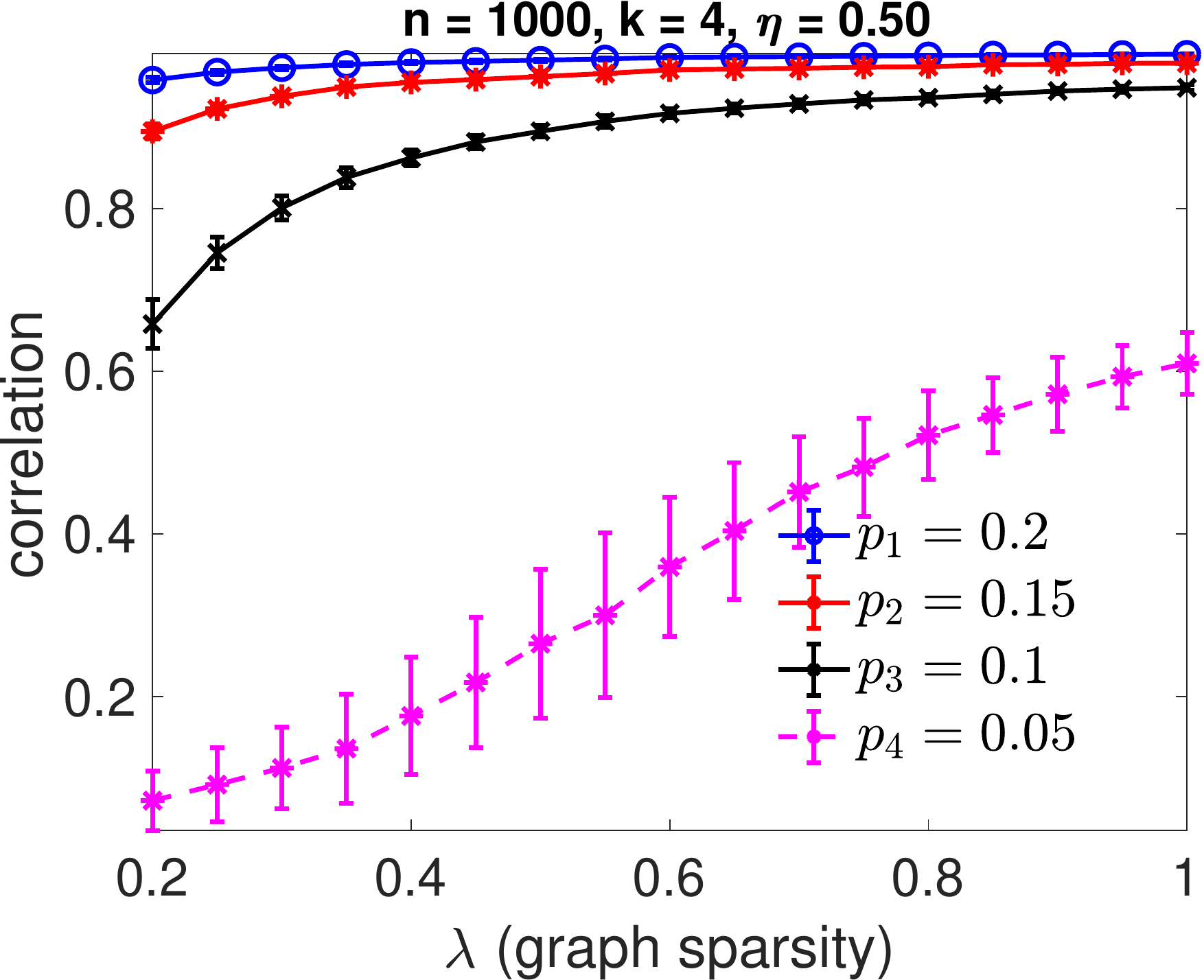} }
%
\subcaptionbox[]{ $(p_1, p_2, p_3, p_4) = (0.14, 0.11, 0.08, 0.05)$, and  $\eta = 0.62$
}[ 0.31\textwidth ]
{\includegraphics[width=0.31\textwidth, trim=0cm 0cm 0cm 0.7cm,clip] {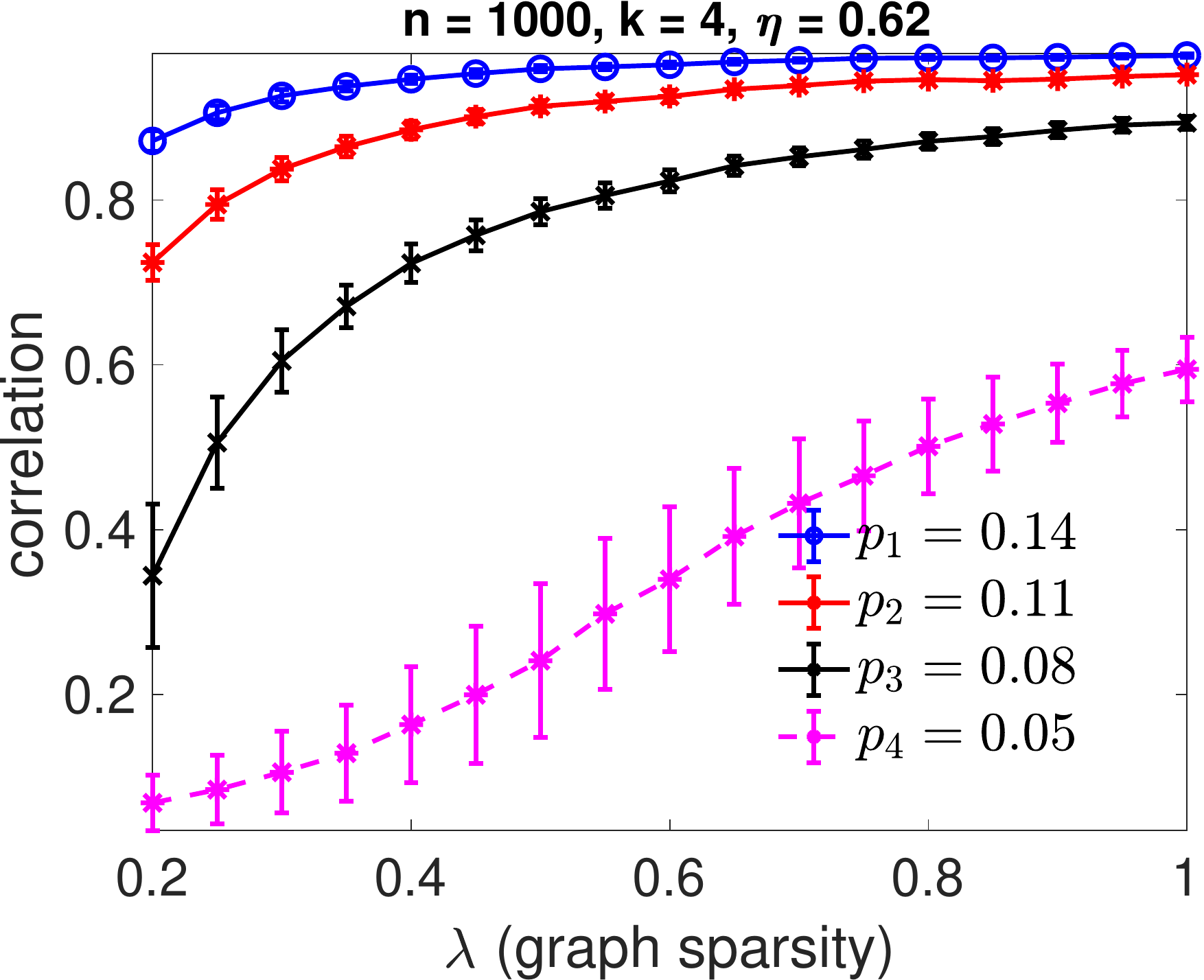} }
\vspace{-3mm}
\captionsetup{width=0.98\linewidth}
\caption[Short Caption]{Experimental Setup I: Correlation of recovered angles with the ground truth for $n=1000$, with $k=3$ (top) and $k=4$ (bottom) collections of angles, and various noise levels $\eta$,  as a function of $\lambda$ (sparsity of the measurement graph). Results are averaged over $20  \times 20 = 625$ runs.
}
\vspace{-4mm}
\label{fig:scanID_1_abcd_k34_1000}
\end{figure}

 

\subsection{Setup II: fixed gap experiments, correlation versus graph sparsity}

\input{S_Exp_FixedGap}

\subsection{Comparison with the Normalized Spectral and SDP relaxations} 

This section compares different spectral and SDP relaxations of the heterogeneous angular $k$-synchronization problem under the second experimental setup, where we fix the gap probability $\gamma$, and compare the accuracy of the various methods as we vary the noise level. 

\paragraph{Normalization of $H$}  
When dealing with measurement graphs exhibiting a skewed degree distribution, a suitable normalisation is typically employed in problems involving spectral methods, such as clustering (see for eg. \cite{kunegis2010spectral}).  To this end, instead of extracting the final estimated angles from the top eigenvectors of $H$, the authors of \cite{asap2d} considered the following normalization of $H$ by the diagonal matrix $D$ with diagonal elements given by $D_{ii} = \sum_{j=1}^N |H_{ij}|$, arriving at 
\begin{equation}
R  = D^{-1} H.
\label{eq:Rnormalization}
\end{equation}
Note that $R$ is similar to the Hermitian matrix $D^{-1/2} H D^{-1/2}$ via
$ R = D^{-1/2} (D^{-1/2}  H D^{-1/2}) D^{1/2}$, 
and thus $R$ has $n$ real eigenvalues. 
%
This normalization was subsequently employed with success in \cite{syncRank,SVDRank}, where it was shown to outperform the unnormalized version.  

Note that the operator $R$ was studied  in the context of angular synchronization and the graph realization problem \cite{asap2d}, and also in \cite{SingerVDM} which introduced Vector Diffusion Maps for nonlinear dimensionality reduction  and explored the interplay with the Connection-Laplacian operator for vector fields over manifolds.  We also remark that these Hermitian operators (construed as Hermitian analogues to the usual graph Laplacians) have been successfully used in the
directed clustering \cite{DirectedClustImbCuts}  and ranking \cite{syncRank}  literatures, where the net outcome of pairwise matches  between players can be encoded in a digraph with a skew-symmetric adjacency matrix. In particular, \cite{syncRank} formulated the ranking problem as an instance of the angular synchronization problem, considered an angular embedding and relied on the top eigenvector of the above matrix operator $R$ to recover a one-dimensional ordering of the players.
Very recently, we have considered the same normalization in \cite{SVDRank} in the context of ranking and synchronization over $ \mathbb{R} $ via SVD. In another recent work, \cite{FANUEL2017JACHA} proposed  a certain deformation of the combinatorial Laplacian, in particular, the \textit{Dilation Laplacian}, which is shown to perform well for ranking in directed networks of pairwise comparisons.

\input{S_Exp_SDP}

Figure \ref{fig:scanID_5c_k234_500_SDP} shows the correlation of recovered angles with the ground truth in the second experimental setup, comparing the performance of the two spectral relaxations and the SDP relaxation. The full set of methods we compare are  
\begin{itemize}
\item \textsc{EIG-H}: the relaxation that uses the top $k$ eigenvectors of matrix $H$, as in \Cref{Algo:kSync}; 
\item \textsc{EIG-R}: employs the normalized  
version of $H$ as detailed in \eqref{eq:Rnormalization},
\item  \textsc{SDP-BM}:  denotes the SDP relaxation \eqref{eq:SDP_program_SYNC}  solved via the Burer-Monteiro approach \cite{Barvinok1995, Burer2005}. Note that after solving the SDP program, we extract angles from the top $k$ eigenvectors of the computed solution $\Upsilon$ (an $n\times n$ complex-valued  Hermitian  positive-semidefinite matrix), following the same steps outlined in \Cref{Algo:kSync}. 
\end{itemize}

We fixed the number of angles  $n=500$, and kept constant the gap $\gamma = 0.05$ between consecutive sampling probabilities $ p_{l+1} - p_l = \gamma$, $l = 1, \ldots, k-1$, of the subgraphs $ G_{l}$. We varied $k \in \{2,3,4\}$ and the overall measurement graph sparsity $ \lambda \in \{ 0.2, 0.4, 0.8\}$, and plotted the correlation with ground truth, as we varied the noise level $ \eta $. The experiments reveal that the SDP relaxation typically outperforms the spectral relaxations on the angle recovery task, with the exception of the set of angles corresponding to the sparsest measurement subgraph $G_3$ for the case $k=3$, and $G_3$ and $G_4$ for the case $k=4$.

\paragraph{Numerical experiments for the Barab\'asi--Albert model}    
Finally, to highlight the benefits of the normalization step, we also compare the two spectral relaxations with the SDP relaxation in the setting where the underlying measurement graph (indicating the presence or absence of an edge) is generated from the \textit{Barab\'asi--Albert} (BA) model \cite{Barabasi99emergenceScaling}, which is known to lead to skewed degree distributions. We denote by $G_{BA}$ the resulting BA network, which can be used to build the relevant measurement matrix. In other words, the $\Theta$ matrix is given by the Hadamard product between the complete angle offset measurement graph (for e.g., as generated by the model in \eqref{eq:general_K_mixture}  with $\lambda = 1$), and the $G_{BA}$ graph. 
\begin{figure}[!ht]
\vspace{-2mm}
\centering
\subcaptionbox[]{  $ k = 2 $ 
}[ 0.324\textwidth ]
{\includegraphics[width=0.324\textwidth, trim=0cm 0cm 0cm 0.7cm,clip] {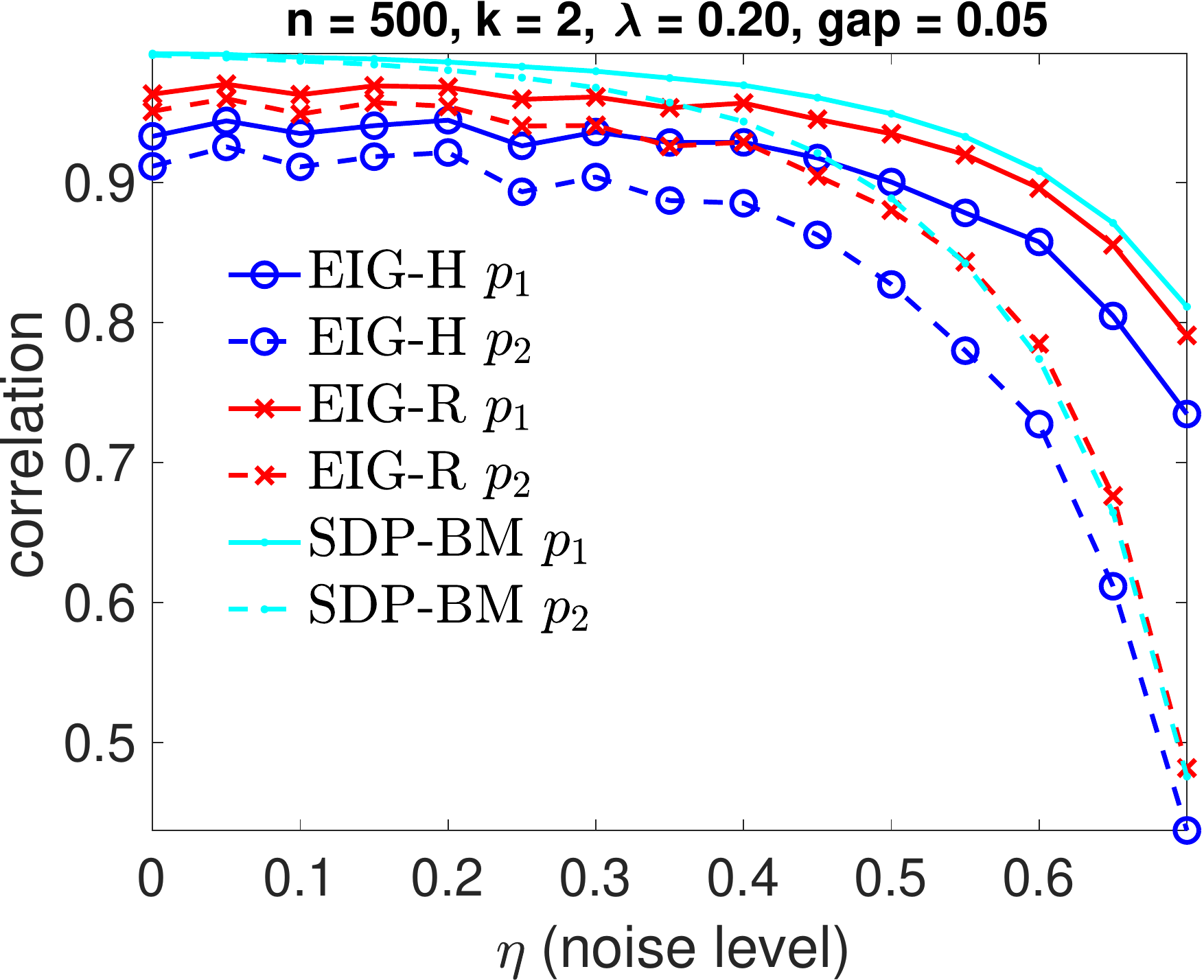} }
%
\subcaptionbox[]{ $ k = 3 $ 
}[ 0.324\textwidth ]
{\includegraphics[width=0.324\textwidth, trim=0cm 0cm 0cm 0.7cm,clip] {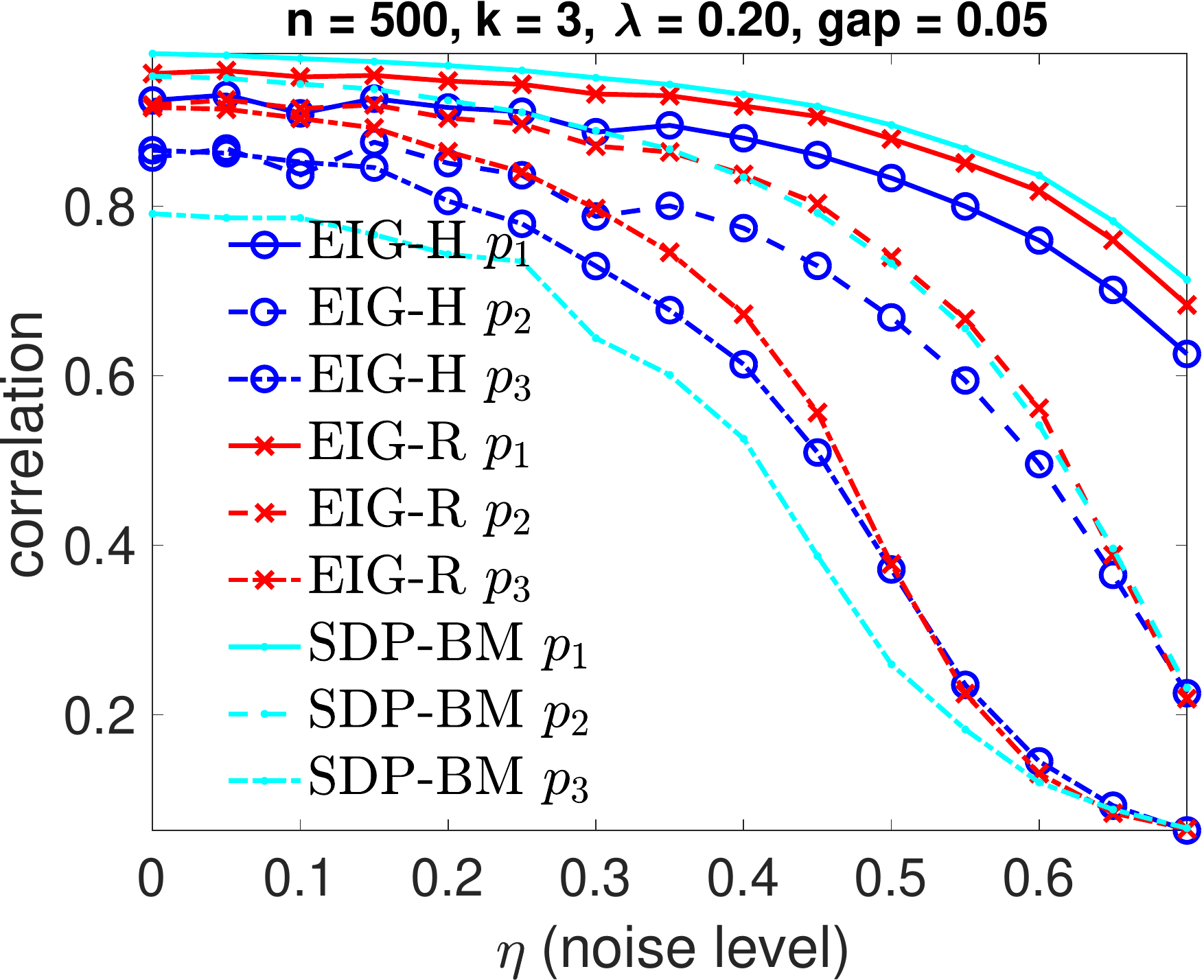} }
%
\subcaptionbox[]{ $ k = 4 $  
}[ 0.324\textwidth ]
{\includegraphics[width=0.324\textwidth, trim=0cm 0cm 0cm 0.7cm,clip] {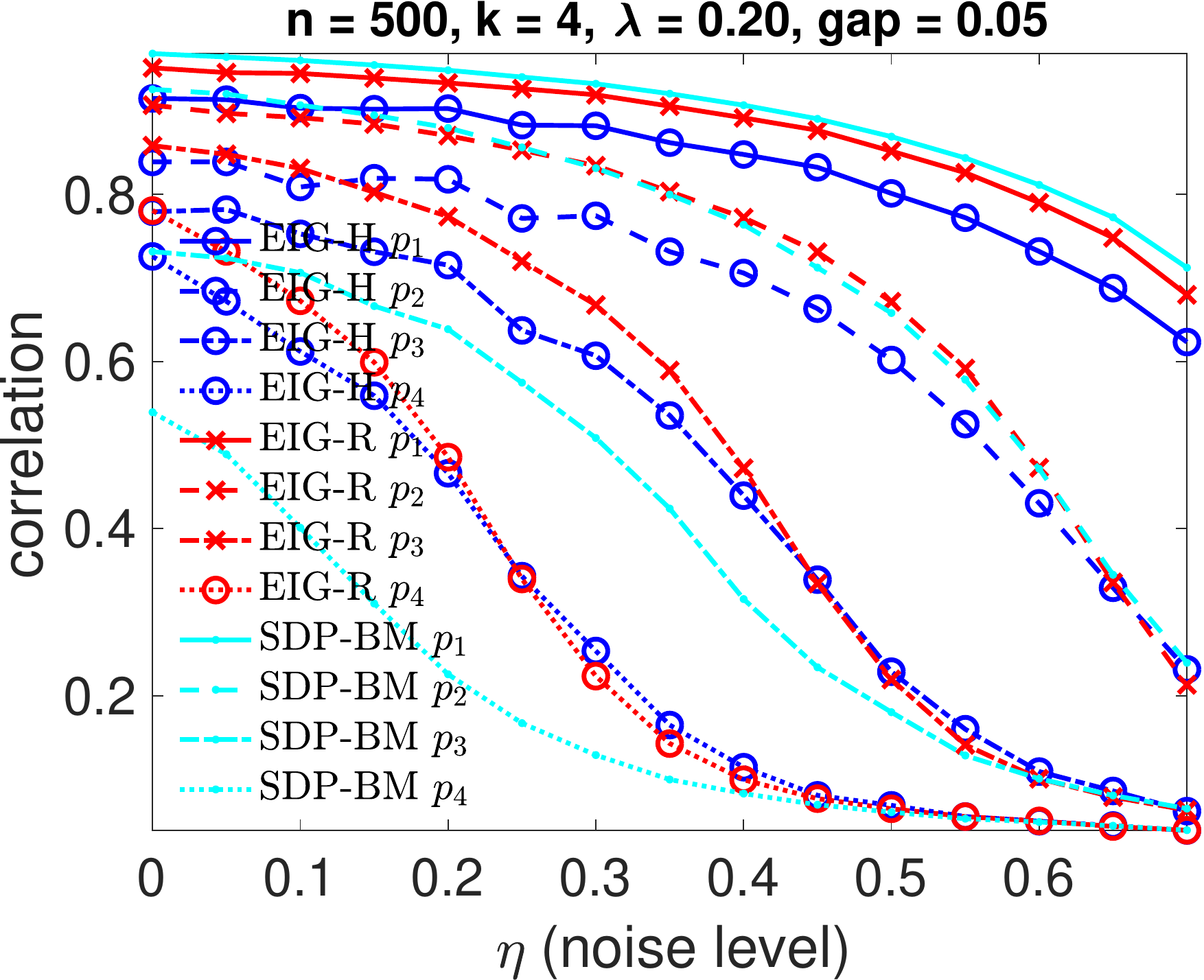} }
%
%
%
\captionsetup{width=1.0\linewidth}
\vspace{-2mm}
\caption[Short Caption]{Experimental Setup II: Performance comparison in the setting of a Barab\'asi--Albert model, for $n=500$, edge density $\lambda = 0.20$, gap $\gamma = 0.05$,  $k \in \{2,3,4\}$, as a function of the  noise levels $\eta$. Results are averaged over $20 \times  20$ runs.}
\vspace{-3mm}
\label{fig:scanID_6_abc_k234_500_SDP_BA_model}
\end{figure}


\subsection{Graph disentangling procedure}

In this section, we give a procedure for uncovering the latent measurement good graphs $G_1, G_2, \ldots, G_k$ corresponding to the different groups of angles. We refer to this process as ``\textit{graph  disentangling}", which is illustrated in \Cref{fig:graph_recovery_k2_eta20_n100}.  
%
The input to this pipeline is given by the initial estimates of the $k$ groups of angles, 
$\hat{\theta}_{l,i}$, for $l=1, \ldots, k$, and $i=1, \ldots, n$. 
In a nutshell, the algorithm we propose considers the \textbf{residual matrix} for each recovered solution with respect to the initial input matrix of pairwise measurements $\Theta$ 
\begin{equation}  \label{eq:residualMtx}  
\Psi_{l,ij} = \min \{ (\Theta_{ij} - \widehat{\Theta}_{l,ij})\mod 2\pi, \;\;  
(\widehat{\Theta}_{l,ij} - \Theta_{ij})\mod 2\pi  \}   ,  \;\;\;\;  l = 1, \ldots, k, 
\end{equation}   
where $ \widehat{\Theta}_{l, ij} = (\hat{\theta}_{l,i} - \hat{\theta}_{l,j}) \mod 2\pi$,  
and leverages this set of matrices $\Psi_{l}, l=1,\ldots,k$ for the purpose of attributing a given edge  $ \setij \in E$ to one of the subgraphs $G_1, G_2, \ldots, G_k$. More specifically, for each edge $\setij$ of the initial measurement graph $G$, we consider the set of residuals induced by each solution vector $ \hat{ \theta_l }, l = 1, \ldots, k $, and  attribute the smallest residual in the list to the  corresponding estimated graph $\widehat{G}_l$. In other words, we first compute the matrix of smallest residual errors
\begin{eqnarray} \label{eq:GammaDef}  
\Gamma_{ij} = \min_{l = 1, \ldots, k}   \Psi_{l, \; ij}, \ \setij \in E, 
\label{eq:gamma_ij}
\end{eqnarray}
and then update the residual matrix corresponding to each set of angles $ l= 1, \ldots, k$
\begin{equation}    \label{eq:UpdatedPhi}   
\tilde{\Psi}_{l, \; ij} = \begin{cases}
\Gamma_{ij} & \text{ if }  \Psi_{l, \; ij}  = \Gamma_{ij}    \\ 
0 & \text{if otherwise}
\end{cases}.
\end{equation} 
Note that we have decomposed 
$ \Gamma_{ij} =  \sum_{l=1}^{k}  \tilde{\Psi}_{l, \; ij} $, making the implicit assumption that there is a unique minimum in \eqref{eq:gamma_ij} (if this is not the case, we break ties arbitrarily and assign the residual to a single matrix $\tilde{\Psi}_{l}$).   
The top row of Figure \ref{fig:scanID_eee6_abc_k234_500_SDP} shows the distribution of the nonzero entries for each matrix $\tilde{\Psi}_{l}$, $l=1,2,3$, for an instance of the $k$-synchronization problem with $k=3$.
 
In the remaining part of this procedure, we use the support of the $k$ matrices  $\tilde{\Psi}^{(1)}, \tilde{\Psi}^{(2)}, \ldots, \tilde{\Psi}^{(k)} $ to recover the $k+1$ underlying graphs given by
\begin{itemize}
\item the $k$ measurement \textit{good} graphs $G_1,G_2, \ldots, G_k$, and with adjacency matrices given by $A_1, A_2, \ldots, A_k$, henceforth denoted as the list $(G_l, A_l), \; l=1,\ldots,k$;
\item the \textit{bad} graph $W$ (with adjacency matrix $B$) corresponding to the outlier entries. 
\end{itemize}
Denote by $(\tilde{G}_l,\tilde{A}_l)$ the support graph corresponding to the matrix of residuals $\tilde{\Psi}^{(l)}, l=1 \ldots k $. Let  
\begin{equation*}
\Theta^{(l)} = \Theta  \odot  \tilde{A}_l,    
\label{eq:ThetaL}
\end{equation*}
where $\odot$  denotes the entry-wise Hadamard product of two matrices, and note that this also induces the following decomposition of the initial measurement matrix $\Theta$
\begin{equation*}
    \Theta = \sum_{l=1}^k   \Theta^{(l)}. 
\end{equation*}
Recall the ground truth decomposition of the adjacency matrix $A$ of the 
measurement graph $G$
\begin{equation*}
A = \sum_{l=1}^{k} A_l  \; + B, 
\end{equation*}
and the fact that $G$  is comprised of both \textit{good} and \textit{bad} edges. 
Each good edge $ \setij \in E_l$ will ideally be placed in its corresponding graph $\widetilde{G}_l$  as it should achieve the smallest residual in \eqref{eq:GammaDef}. However, each (\textit{bad}) edge in $W$ will also be randomly attributed to one of the graphs $ \widetilde{G}_l $.  To this end, we perform a second synchronization stage, where we further synchronize individually each $\Theta^{(l)}$ matrix, recover the solution angles $ \hat{\theta}_{l,i}^{(1)} $ (the superscript is to indicate that this is the first iteration of our entire pipeline described thus far), and then rebuild  each denoised pairwise comparison matrix $\widehat{\Theta}^{(l)}$ as 
\begin{equation*}
\widehat{\Theta}^{(l)}_{ij} = (\widehat{\theta}_{l,i}^{(1)}  -  \hat{\theta}_{l,j}^{(1)})  \; \mod 2\pi .  
\end{equation*}
Similar to \eqref{eq:residualMtx},  we first build the matrix of residuals between  $\widehat{\Theta}^{(l)}$ and $\Theta^{(l)}$, restricted to the support of $\widetilde{G}_l$ (which only contain edges present in the original measurement graph $G$), and finally classify each edge as being  \textit{good} or \textit{bad} depending on the magnitude of its residual.  
In doing so, for simplicity, we assume knowledge of the noise level $1 - p_l$,  and classify as \textit{bad} edges the  $1 - p_l$  percent of the edges with the largest residual, as depicted in the histograms in Figure \ref{fig:scanID_eee6_abc_k234_500_SDP}. 
This altogether renders the following decomposition of $(\widetilde{G}_l, \widetilde{A}_l)$  into the \textit{bad} and \textit{bad} subgraphs with adjacency matrices satisfying 
\begin{equation*}
	 \widetilde{A}_{l}   = \widetilde{A}_{l}^{good} + \widetilde{A}_{l}^{bad}. 
\end{equation*}
Finally, we obtain edge-disjoint estimates $\widehat{A}_1,\ldots,\widehat{A}_k$ for the set of graph adjacency matrices  $ A_1, \ldots, A_k$ as
$$  \widehat{A}_l  := \widetilde{A}_{l}^{good},   $$
while for the bad graph $W$, we pool together all the bad edges inferred from each of the $k$ solutions, and estimate the adjacency matrix of $W$ as 
\begin{equation*}
\widehat{B} := \sum_{l=1}^k   \widetilde{A}_{l}^{bad}. 
\end{equation*}
Figure \ref{fig:graph_recovery_k2_eta20_n100}  
is a pictorial representation of the above pipeline for $k=2$, that depicts the extraction of the underlying subgraph structures. We consider the case of a complete measurement graph (thus $\lambda =1$), with $k=2$ subgraphs $G_1$ and $G_2$ corresponding to two sets of measurement angles $\theta_{1}$ and $\theta_{2}$, with good edge probability $p_1 = 0.45$, respectively $p_2 = 0.35$, for the correct pairwise offsets. Recall that in the bi-synchronization setup, we denoted two such sequences of angles by $\theta_{1} = (\alpha_1, \ldots, \alpha_n)$ and $\theta_{2}=(\beta_1, \ldots, \beta_n)$.    
The subgraph $W$ of outlier measurements has $\eta = 1 - (p_1+p_2) = 0.2$.
These graphs are depicted in the top row of Figure \ref{fig:graph_recovery_k2_eta20_n100}.  The second row of this figure illustrates the combined graph $G = G_1 \cup  G_2 \cup  W$ (with adjacency matrices $ A = A_1 + A_2 + B $), which is available to the user. The bottom row shows the disentangled graph estimates as recovered by the above procedure.

\paragraph{An iterative approach for synchronization} The entire pipeline described thus far can be iterated $M$ times, and the entire process is summarized in Algorithm \ref{algo:IterSyncAlgo}. 
The bottom plot of the Figure \ref{fig:scanID_eee6_abc_k234_500_SDP} shows a histogram of the residuals for each of the three sets of estimated angles, after $M=20$ iterations of the above process. Note that it is noticeable to the naked eye that the residuals become significantly smaller after the iterative process, for all three sets of angles, hinting on the effectiveness of the proposed procedure. To this end, we furthermore plot in Figure \ref{fig:iterations_detangle_k234} the correlation between the estimated solutions and the ground truth, for each of the first $M=20$ iterations of Algorithm  \ref{algo:IterSyncAlgo}. We consider three different problem instances
$k=2$ (with $p_1 = 0.23, p_2=0.15$), 
$k=3$ (with $p_1 = 0.18, p_2=0.15$, $p_3=0.12$), 
$k=4$ (with $p_1 = 0.2, p_2=0.17$, $p_3=0.13$, $p_4=0.10$), and remark that in each problem instance and for each collection of angles, the correlations increase for subsequent iterations, showcasing the efficacy of the iterative procedure. 

\begin{remark}  \label{rem:iterativeAlgo}
Note that the proposed iterative scheme is of independent interest, and could also be considered in the context of the classical angular synchronization problem, where $k=1$ and the goal is uncover the subgraph $G_1$ of good edges and the subgraph $W$ of bad edges. Furthermore, re-weighting schemes could also be explored, where a confidence score (at iteration $r$) on the available measurement can be used and based on the residual from the previous iteration $r-1$. To the best of our knowledge, such an iterative re-weighting scheme has not been previously considered in the synchronization literature, and it would be an interesting research direction to analyze how it compares to the standard methods detailed in Section \ref{sec:angSyncSo2}.
\end{remark}

\ifthenelse{\boolean{arXivMode}}{ 
\begin{algorithm}[!ht]
\begin{algorithmic}[1]
\State \textbf{Input:} Measurement graph $G = ([n], E)$,  pairwise offset measurements $\Theta_{ij}$ for $\set{i,j} \in E$, number of groups of angles $k$, number of iterations $M$.

\textsc{// Stage 1: Compute initial estimates for $\hat{\theta}_{l,i}^{(0)}$}  
\State Estimate angles $\hat{\theta}_{l,i}^{(0)}$ from the top $k$ eigenvectors of the matrix $H$ as detailed in Algorithm \ref{Algo:kSync}.

\textsc{// Stage 2: Iterative procedure }  
\For{\texttt{ $r = 1,2, \ldots , M$ }}

\State Build the residual matrices $ \Psi_{l}, l=1,\ldots,k$ as in \eqref{eq:residualMtx}.

\State Use the residual matrices to decompose the input measurement matrix $\Theta$  
into  $  \Theta = \sum_{l=1}^k   \Theta^{(l)}.$

\State Synchronize each individual matrix $\Theta^{(l)}$, obtain a denoised estimate $\widehat{\Theta}^{(l)}$ matrix, and the groups of angles $ \hat{\theta}_{l,i}^{(r)} $, for  $l=1,\ldots,k$.

\State Consider the residual matrix between $\widehat{\Theta}^{(l)}$ and $\Theta^{(l)}$, and classify each existing edge as \textit{good} or \textit{bad}.

\State Produce the set of edge-disjoint  \textit{good} subgraphs estimates $\widehat{G}_1^{(r)}, \ldots, \widehat{G}_k^{(r)}$, and the  \textit{bad} subgraph $\widehat{W}^{(r)}$.
\EndFor

\State \textbf{Output:} Estimates $\hat{\theta}_{l,i}^{(r)}$, for $l=1, \ldots, k$, and $i=1, \ldots, n$, and subgraphs of \textit{good} edges $\widehat{G}_1^{(r)}, \ldots, \widehat{G}_k^{(r)}$ and subgraph of \textit{bad} edges $\widehat{W}^{(r)}$.
\end{algorithmic}
\caption{\textsc{Iterative synchronization and graph disentangling algorithm}}
\label{algo:IterSyncAlgo} 
\end{algorithm}
}
{
\begin{algorithm}[!ht]
\begin{algorithmic}[1]
\STATE \textbf{Input:} Measurement graph $G = ([n], E)$,  pairwise offset measurements $\Theta_{ij}$ for $\set{i,j} \in E$, number of groups of angles $k$, number of iterations $M$.

\textsc{// Stage 1: Compute initial estimates for $\hat{\theta}_{l,i}^{(0)}$}  
\STATE{Estimate angles $\hat{\theta}_{l,i}^{(0)}$ from the top $k$ eigenvectors of the matrix $H$ as detailed in Algorithm \ref{Algo:kSync}.}

\textsc{// Stage 2: Iterative procedure }  
\FOR{\texttt{ $r = 1,2, \ldots , M$ }}

\STATE{Build the residual matrices $ \Psi_{l}, l=1,\ldots,k$ as in \eqref{eq:residualMtx}.}

\STATE{Use the residual matrices to decompose the input measurement matrix $\Theta$  
into  $  \Theta = \sum_{l=1}^k   \Theta^{(l)}.$}

\STATE{Synchronize each individual matrix $\Theta^{(l)}$, obtain a denoised estimate $\widehat{\Theta}^{(l)}$ matrix, and the groups of angles $ \hat{\theta}_{l,i}^{(r)} $, for  $l=1,\ldots,k$.}

\STATE{Consider the residual matrix between $\widehat{\Theta}^{(l)}$ and $\Theta^{(l)}$, and classify each existing edge as \textit{good} or \textit{bad}.}

\STATE{Produce the set of edge-disjoint  \textit{good} subgraphs estimates $\widehat{G}_1^{(r)}, \ldots, \widehat{G}_k^{(r)}$, and the  \textit{bad} subgraph $\widehat{W}^{(r)}$.}
\ENDFOR

\STATE \textbf{Output:} Estimates $\hat{\theta}_{l,i}^{(r)}$, for $l=1, \ldots, k$, and $i=1, \ldots, n$, and subgraphs of \textit{good} edges $\widehat{G}_1^{(r)}, \ldots, \widehat{G}_k^{(r)}$ and subgraph of \textit{bad} edges $\widehat{W}^{(r)}$.
\end{algorithmic}
\caption{\textsc{Iterative synchronization and graph disentangling algorithm}}
\label{algo:IterSyncAlgo} 
\end{algorithm}
}

\begin{figure}[!ht]
\vspace{-2mm}
\centering
\includegraphics[width=0.90824\textwidth] {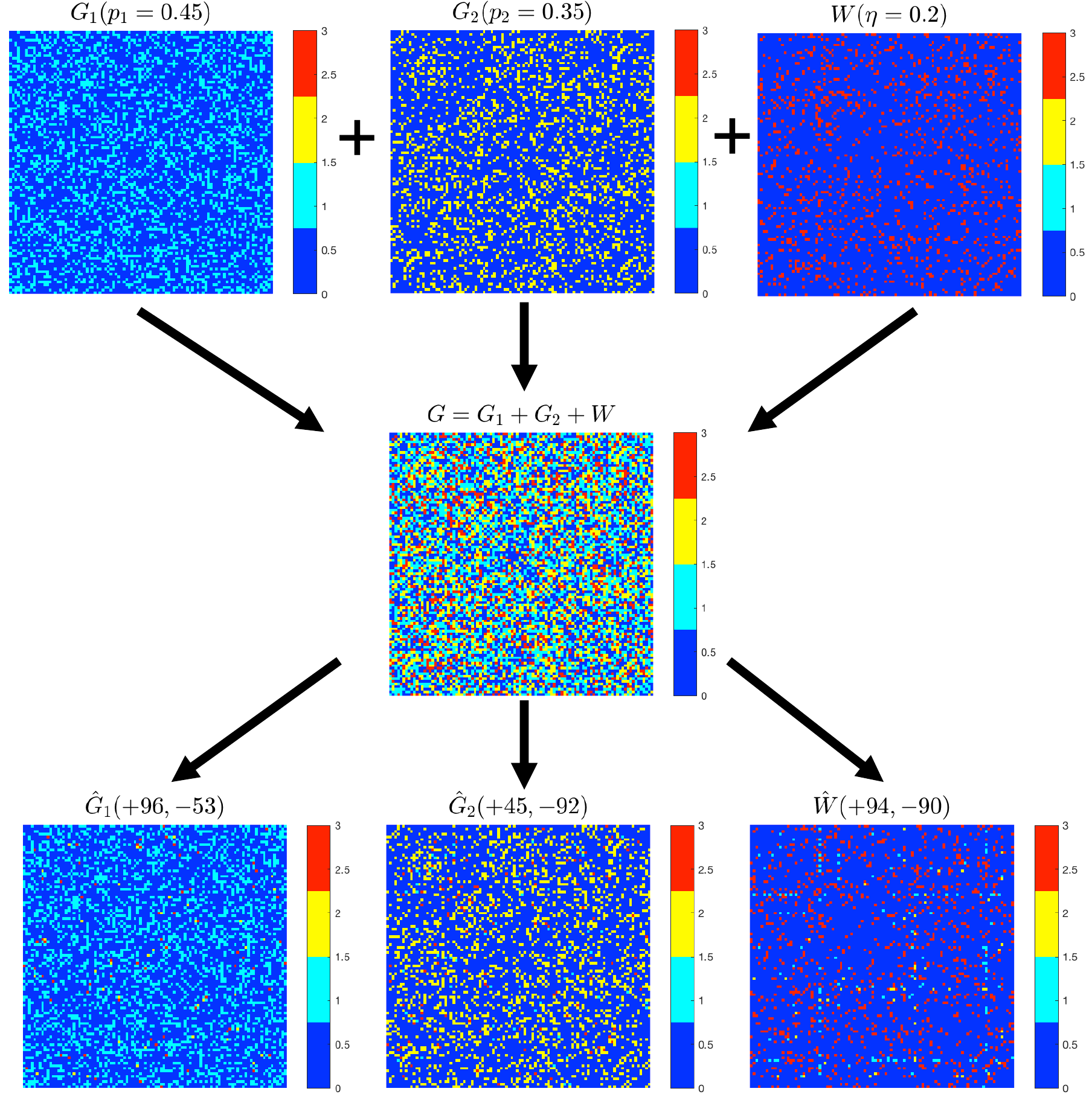} 
\vspace{-3mm}
\captionsetup{width=0.98\linewidth}
\caption[Short Caption]{Top: ground truth \textit{good} graph $G_1, G_2$ and the \textit{bad} graph $W$, all edge disjoint, of size $n=100$, with corresponding probabilities $p_1=0.45$, $p_2 = 0.35$ and $\eta=0.20$. Middle: the user observes an angle offset matrix $\Theta$ whose underlying entangled measurement graph $G=G_1 \cup G_2 \cup W$ is depicted. Bottom: final estimates  of $\widehat{G}_1, \widehat{G}_2,  \widehat{W}$ from our recovery pipeline, with a visualization of the classification errors; the titles contain the two types of errors being made: the number of extra edges (+) and the number of missing edges (-) for each individual estimated graph.  
}
\vspace{-2mm}
\label{fig:graph_recovery_k2_eta20_n100}
\end{figure}

\begin{figure}[!ht]
\centering
\includegraphics[width=0.9824\textwidth] {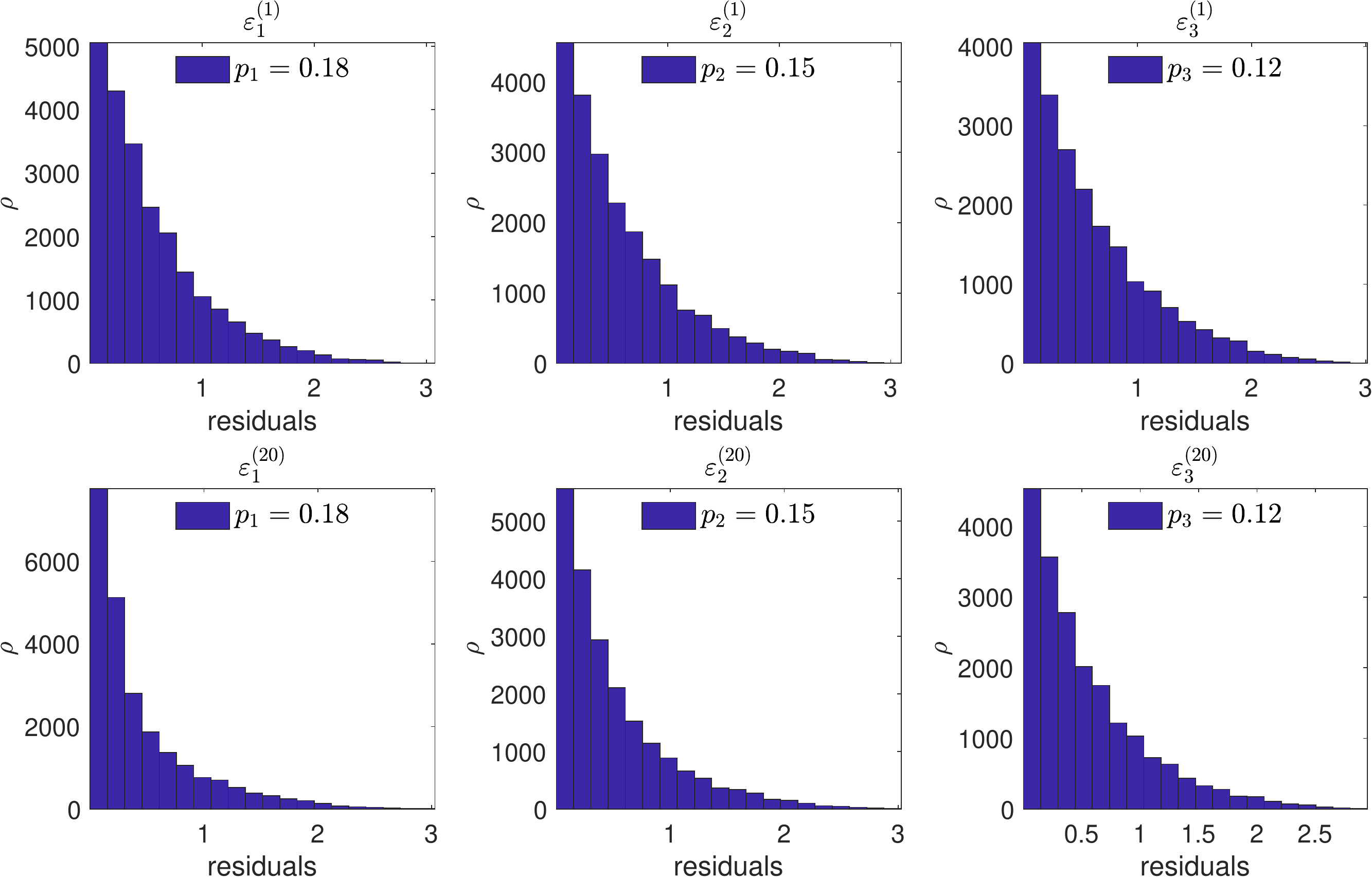}
\vspace{-2mm}
\captionsetup{width=0.98\linewidth}
\caption[Short Caption]{Histogram of synchronization residuals given by the non-zero entries of the matrix $ \tilde{\Psi}_{l}, \; l=\{1,2,3\}$, as defined in \eqref{eq:UpdatedPhi}, where the  $k=3$ underlying measurement graphs have edges densities $(p_1,p_2,p_3)= (0.18, 0.15, 0.12)$, noise level  $ \eta = 0.55$, and sparsity $ \lambda = 0.3$. The top (resp. bottom) row plots the non-zero residual entries after one iteration (resp., 20 iterations of  Algorithm \ref{algo:IterSyncAlgo}), with columns indexing the three groups of underlying angles.  
Note that for each column (i.e., group of angles), the residuals become significantly smaller after 20 iterations, for all three groups of angles, showcasing the effectiveness of the iterative procedure.} 
\label{fig:scanID_eee6_abc_k234_500_SDP}
\end{figure}

\begin{figure}[!ht]
\vspace{-3mm} 
\centering
\subcaptionbox[]{  
$k=2$; $\eta=0.62$; $\lambda=0.3$
}[ 0.31\textwidth ]
{\includegraphics[width=0.31\textwidth, trim=0cm 0cm 0cm 0.6cm,clip] {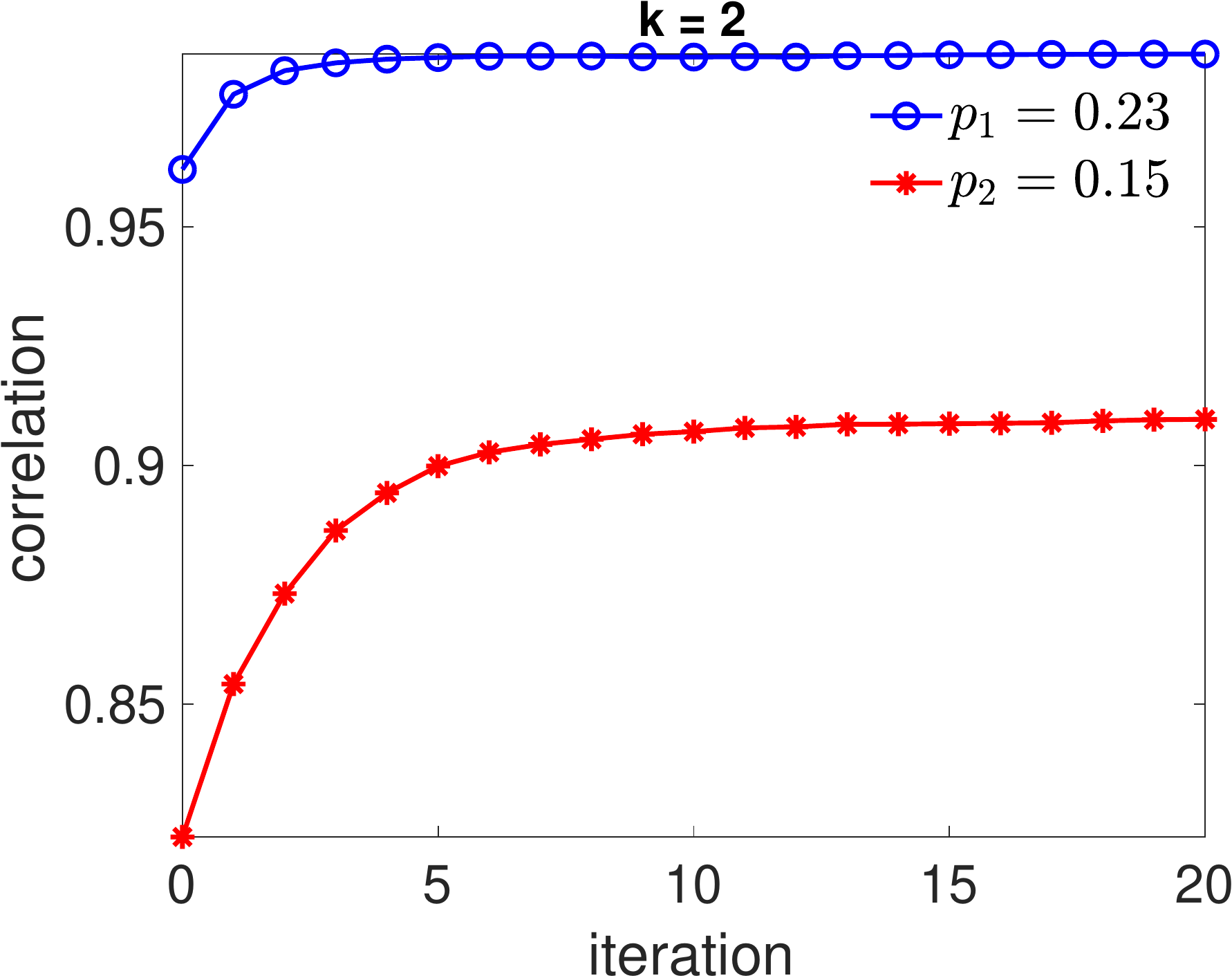} }
%
\subcaptionbox[]{ 
$k=3$; $\eta=0.55$; $\lambda=0.3$
}[ 0.31\textwidth ]
{\includegraphics[width=0.31\textwidth, trim=0cm 0cm 0cm 0.6cm,clip] {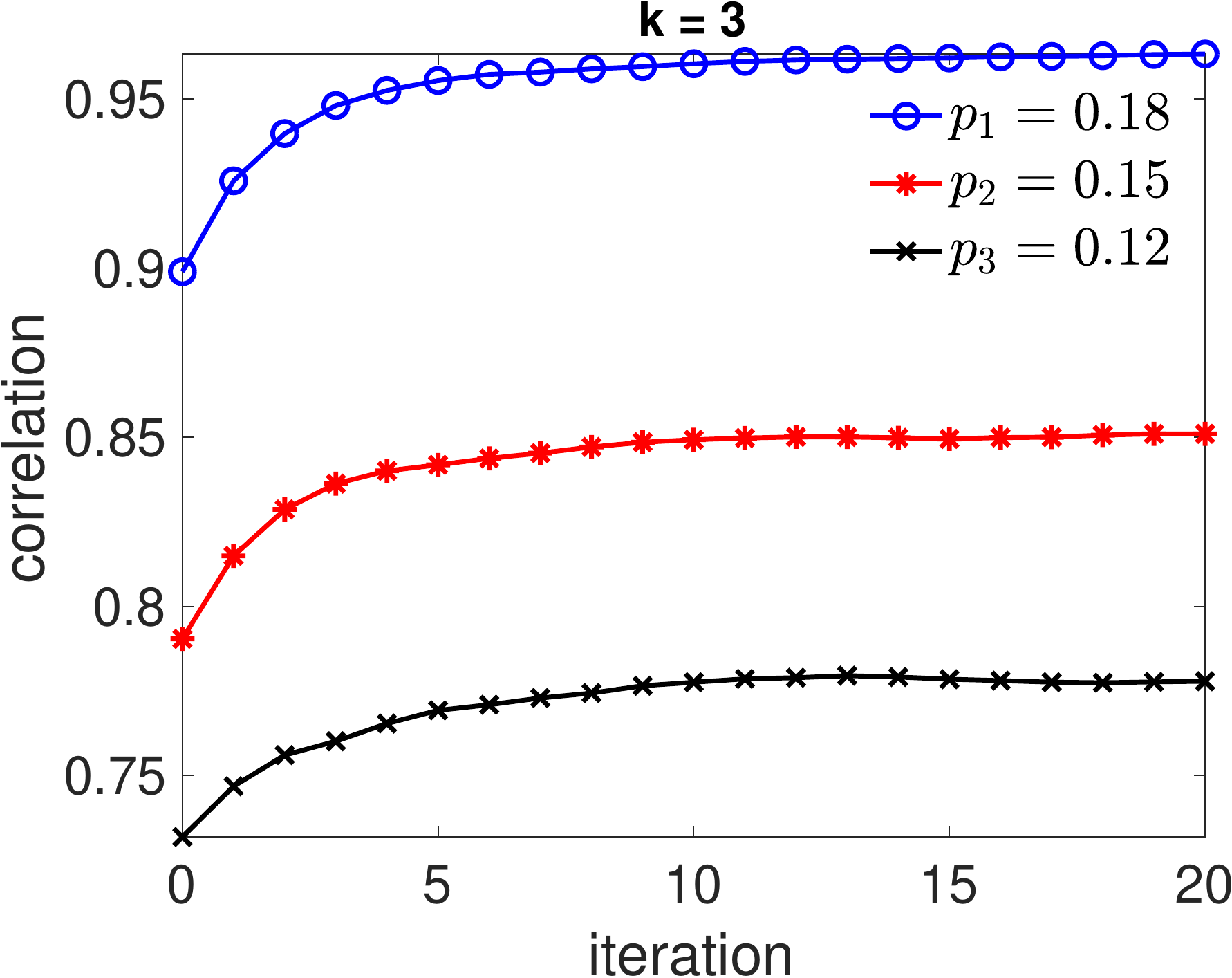} }
%
\subcaptionbox[]{ 
$k=4$; $\eta=0.40$; $\lambda=0.5$
}[ 0.31\textwidth ]
{\includegraphics[width=0.31\textwidth, trim=0cm 0cm 0cm 0.6cm,clip] {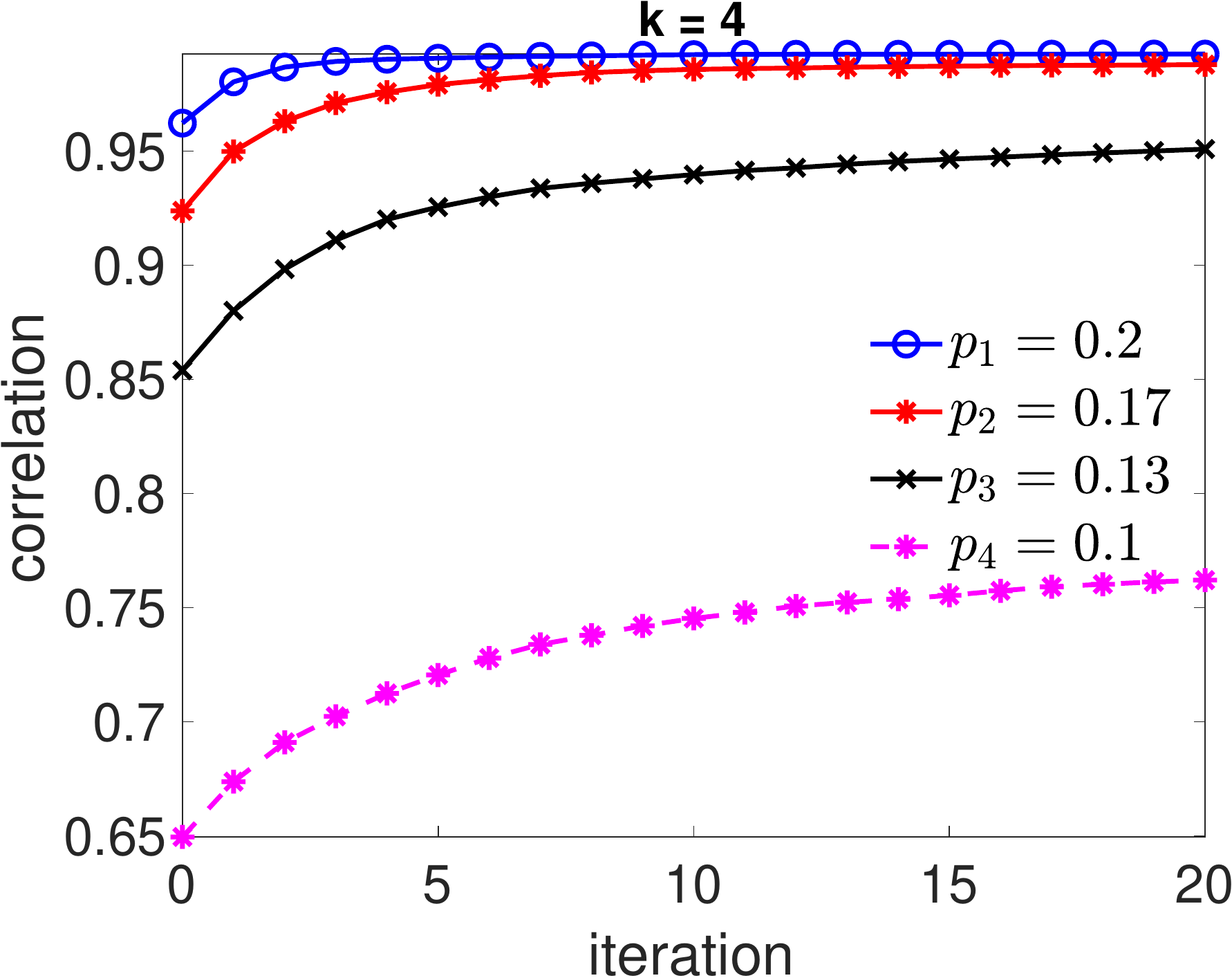} }
%
\vspace{-2mm}
\captionsetup{width=0.98\linewidth}
\caption[Short Caption]{Correlation with ground truth for the first 20 iterations of Algorithm \ref{algo:IterSyncAlgo}, for three problem instances with $k=2,3,4$ and $n=500$, at various noise levels $\eta$ and sparsity levels $\lambda$. In all three problem instances, and across all subsets of angles within each instance, the iterative procedure leads to an increase in the correlation with ground truth, especially for the groups of angles with the sparsest measurement graph.} 
\label{fig:iterations_detangle_k234}
\vspace{-3mm}
\end{figure}


%% file: S_Exp_FixedGap.tex
We also consider a second experimental setup, where we fix the gap between consecutive values of the $p_i$ sampling parameter corresponding to the fraction of \textit{good} measurements from $G_i$. 
We recall the mixture model considered in 
\eqref{Def:MixtureOmega} for $k=2$, and the more general setting from 
\eqref{Def:MixtureOmega_gen} for any $k \geq 2$, 
where $p_l \lambda$ denotes the probability of a correct pairwise measurement 
from subgraph $G_l$, $\eta$ is the probability of getting an incorrect measurement, and $ \lambda $ the overall sparsity of the measurement graph. Recall that we have assumed  w.l.o.g that $p_1 > p_2 > \dots > p_k$.
In the following experiments, for a given noise level $ \eta $, we choose $p_1, \ldots, p_k$ such that $ \sum_{l=1}^{k} p_l = 1-\eta$ and  $ p_{l+1} - p_{l} = \gamma, \forall l = 1,\dots, k-1$. 
For example, under the complete measurement graph scenario $\lambda = 1$, if $k=4$, $ \eta = 0.20$, and the gap parameter $ \gamma = 0.05$, the resulting good-edge probabilities are given by  $p_1 = 0.275, \; p_2 = 0.225, \;     p_3 = 0.175, \;  p_4 = 0.125$. The final mixture model is then 

\begin{equation} \label{eq:general_K_mixture}
\Theta_{ij} = \left\{ 
\begin{array}{rl} 
\theta_{l,i} - \theta_{l,j} &  \text{with probability } p_l \lambda \text{ for } l=1,\dots,k	\\
\Theta_{ij} \sim U[0,2\pi) &  \text{with probability } (1-\sum_{l=1}^{k} p_l) \lambda \\
0& \text{with probability } 1-\lambda.	\\
\end{array}
\right.
\end{equation} 

Figure \ref{fig:scanID_2_c_k234_500_gap} reports the outcomes of numerical experiments, for three different values of $k \in \{ 2,3,4 \}$ (indexing the rows), and three values of the graph density parameter $ \lambda \in \{ 0.2, 0.4, 0.8\}$.

%

\begin{figure}[!ht]
\vspace{-2mm}

\hspace{2.8cm} $ \lambda = 0.20 $ \hspace{3.6cm}   $ \lambda = 0.40 $   \hspace{3.6cm}  $ \lambda = 0.80 $

\vspace{4mm}
\centering
\raisebox{0.75in}{\rotatebox[origin=t]{90}{k=2}}\hspace{1mm}  
\subcaptionbox[]{ $  $
}[ 0.31\textwidth ]
{\includegraphics[width=0.31\textwidth, trim=0cm 0cm 0cm 0.7cm,clip] {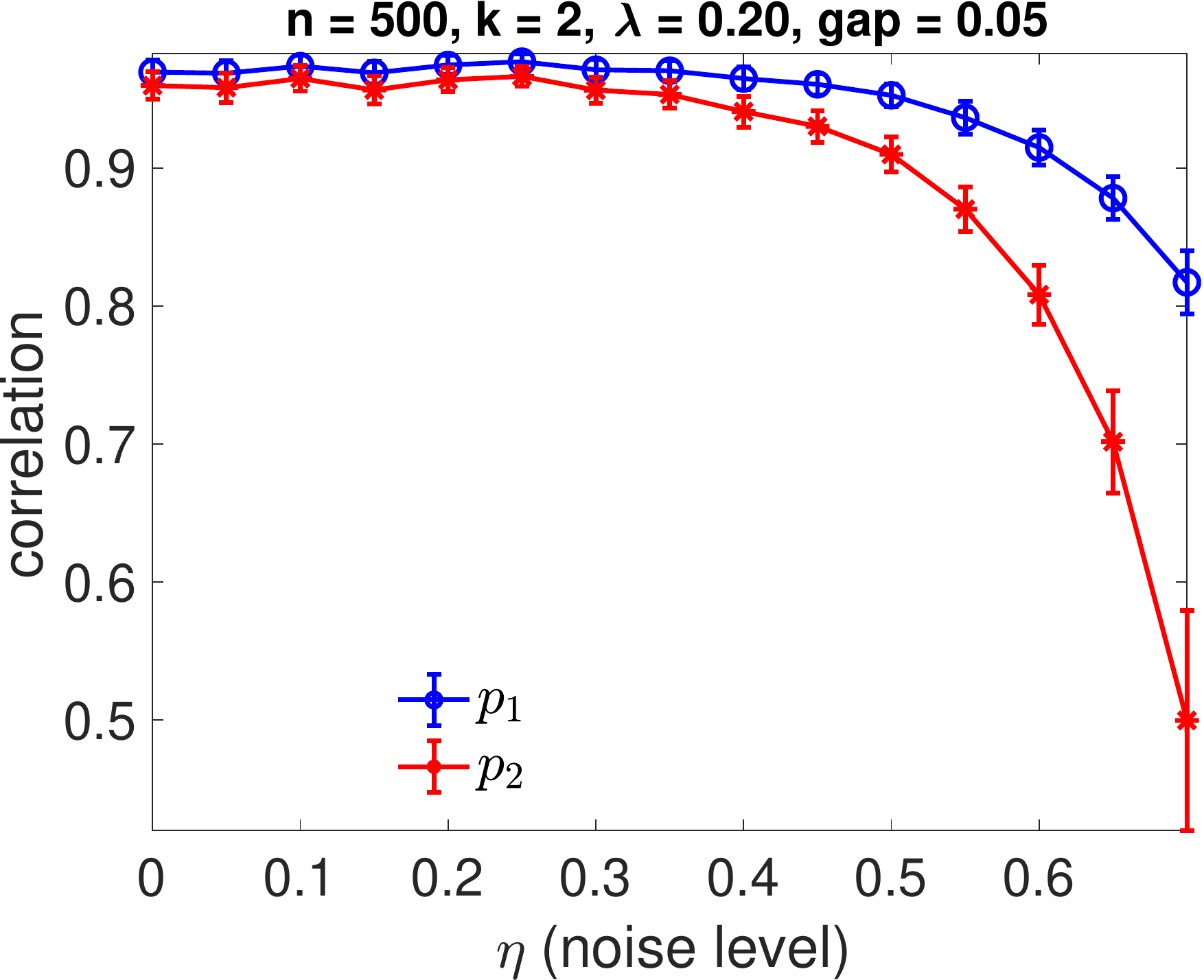} }
%
\subcaptionbox[]{ 
}[ 0.31\textwidth ]
{\includegraphics[width=0.31\textwidth, trim=0cm 0cm 0cm 0.7cm,clip] {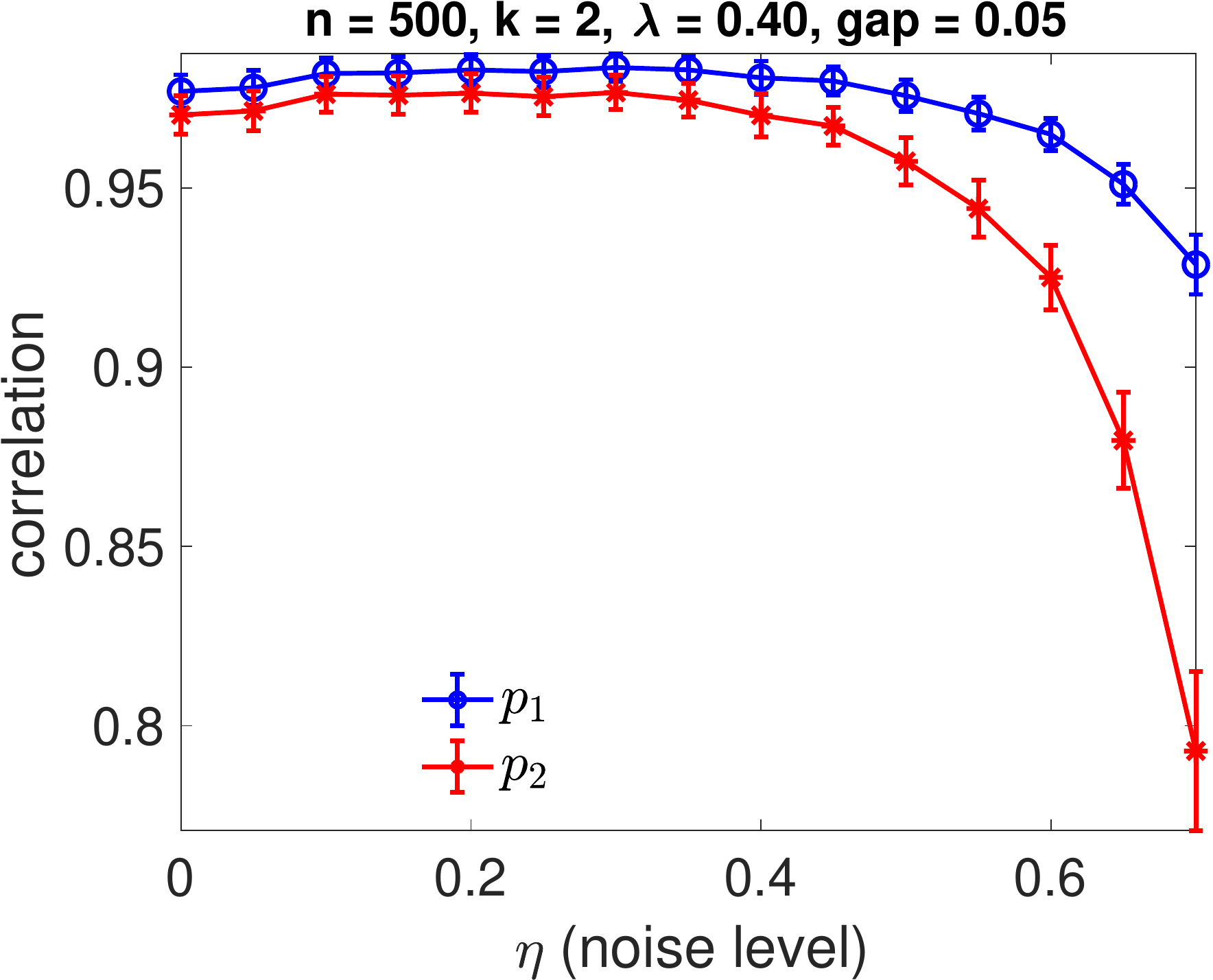} }
%
\subcaptionbox[]{ 
}[ 0.31\textwidth ]
{\includegraphics[width=0.31\textwidth, trim=0cm 0cm 0cm 0.7cm,clip] {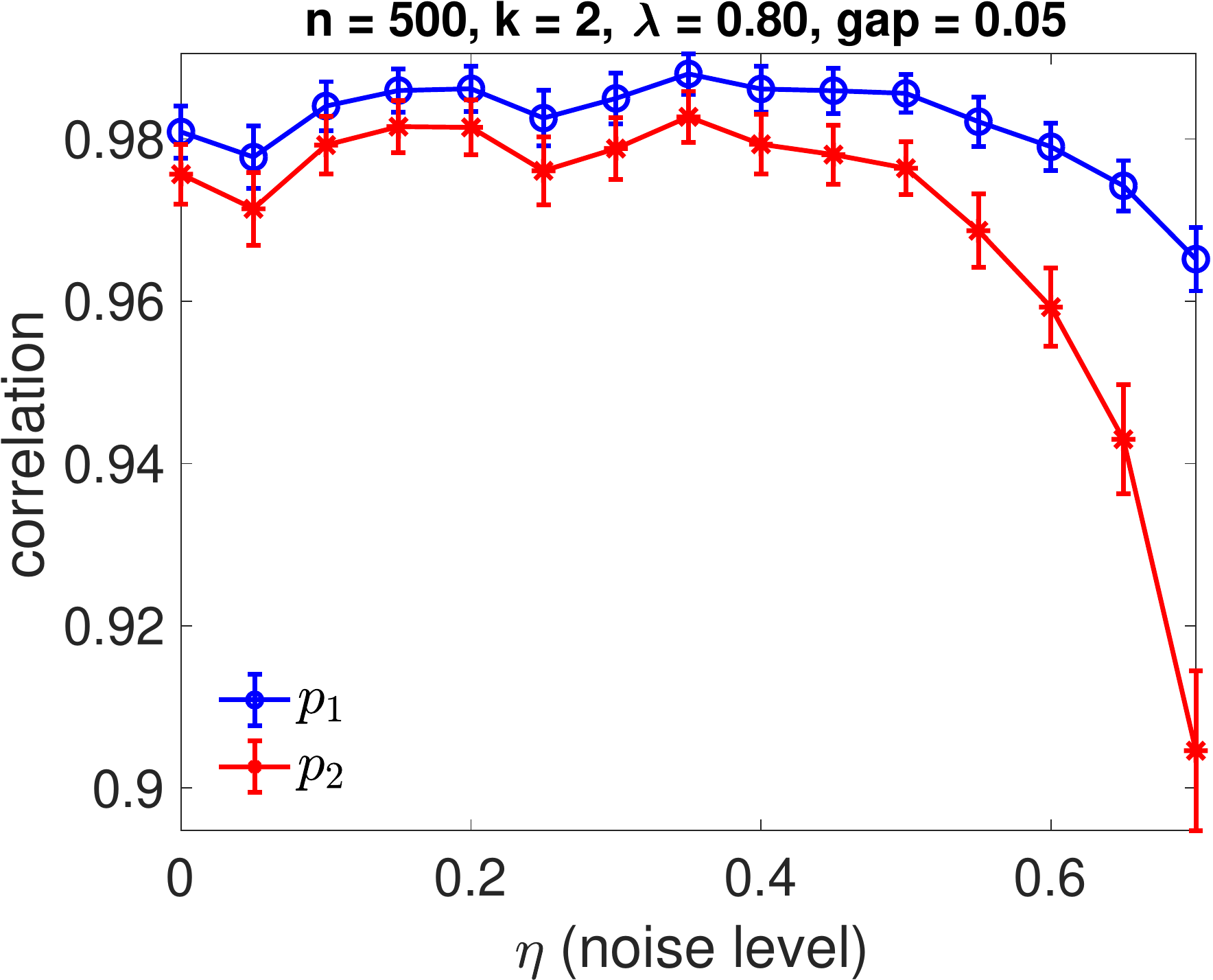} }
%
%
 %

\raisebox{0.75in}{\rotatebox[origin=t]{90}{k=3}}\hspace{1mm}
\subcaptionbox[]{ 
}[ 0.31\textwidth ]
{\includegraphics[width=0.31\textwidth, trim=0cm 0cm 0cm 0.7cm,clip] {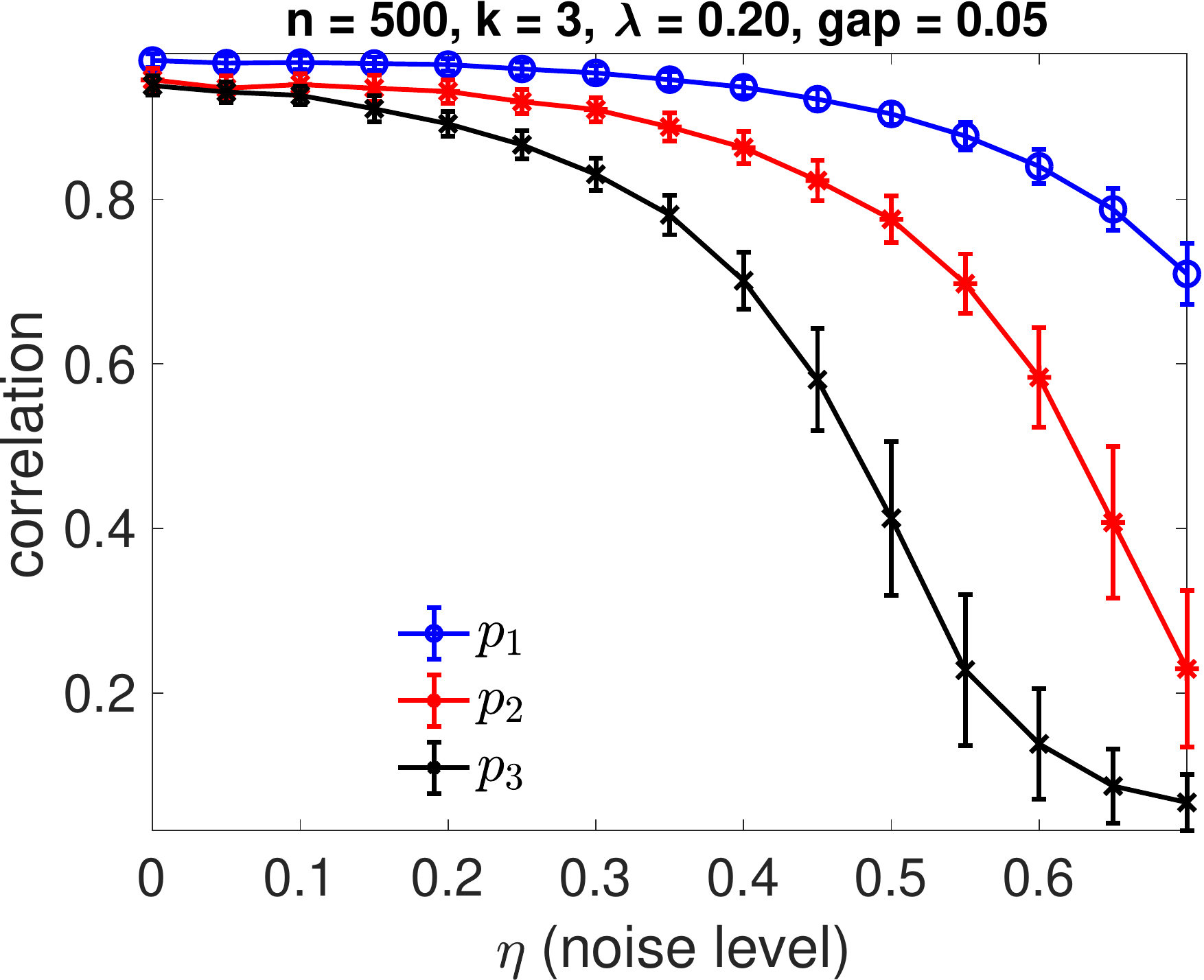} }
%
\subcaptionbox[]{ 
}[ 0.31\textwidth ]
{\includegraphics[width=0.31\textwidth, trim=0cm 0cm 0cm 0.7cm,clip] {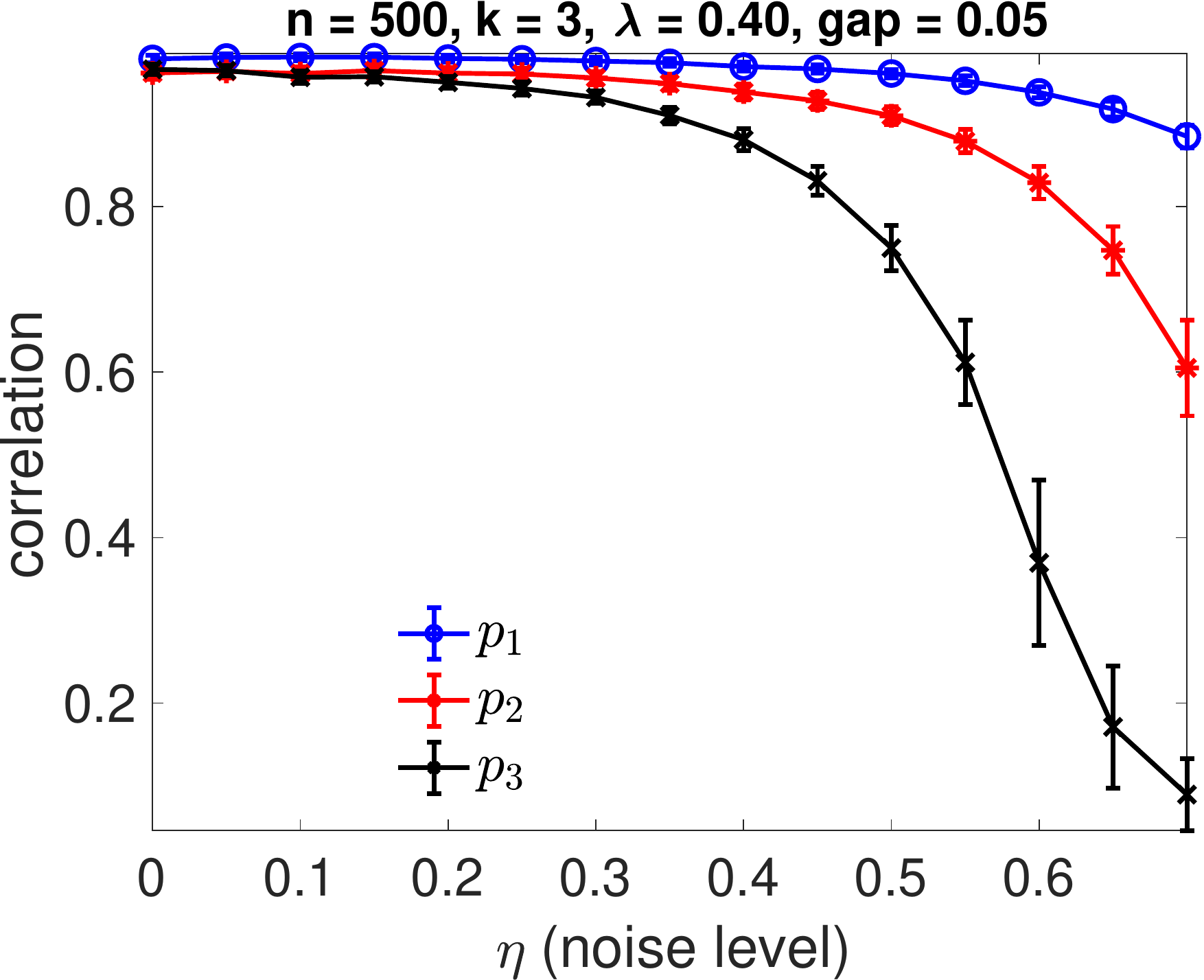} }
%
\subcaptionbox[]{ 
}[ 0.31\textwidth ]
{\includegraphics[width=0.31\textwidth, trim=0cm 0cm 0cm 0.7cm,clip] {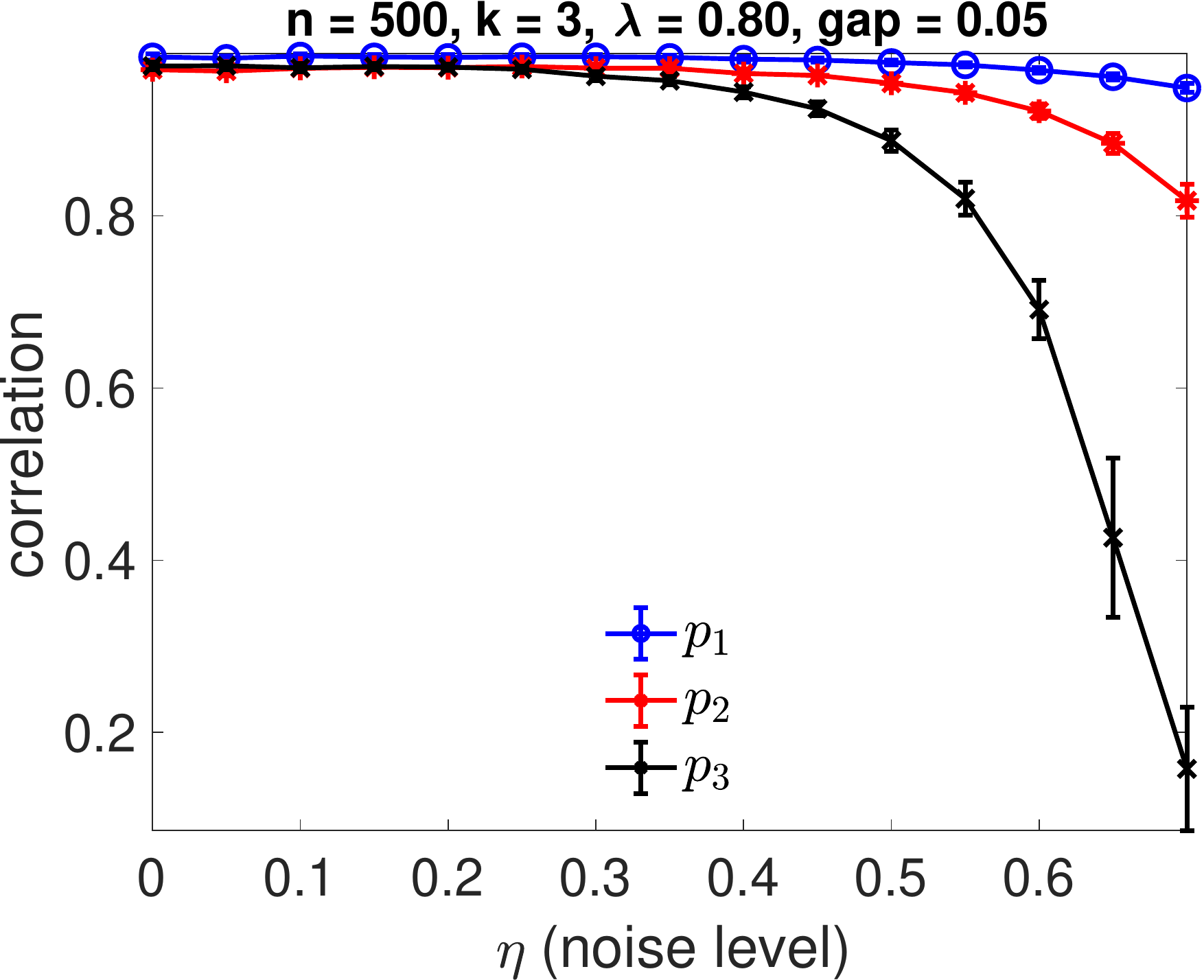} }
%

\raisebox{0.75in}{\rotatebox[origin=t]{90}{k=4}}\hspace{1mm}
\subcaptionbox[]{ 
}[ 0.31\textwidth ]
{\includegraphics[width=0.31\textwidth, trim=0cm 0cm 0cm 0.7cm,clip] {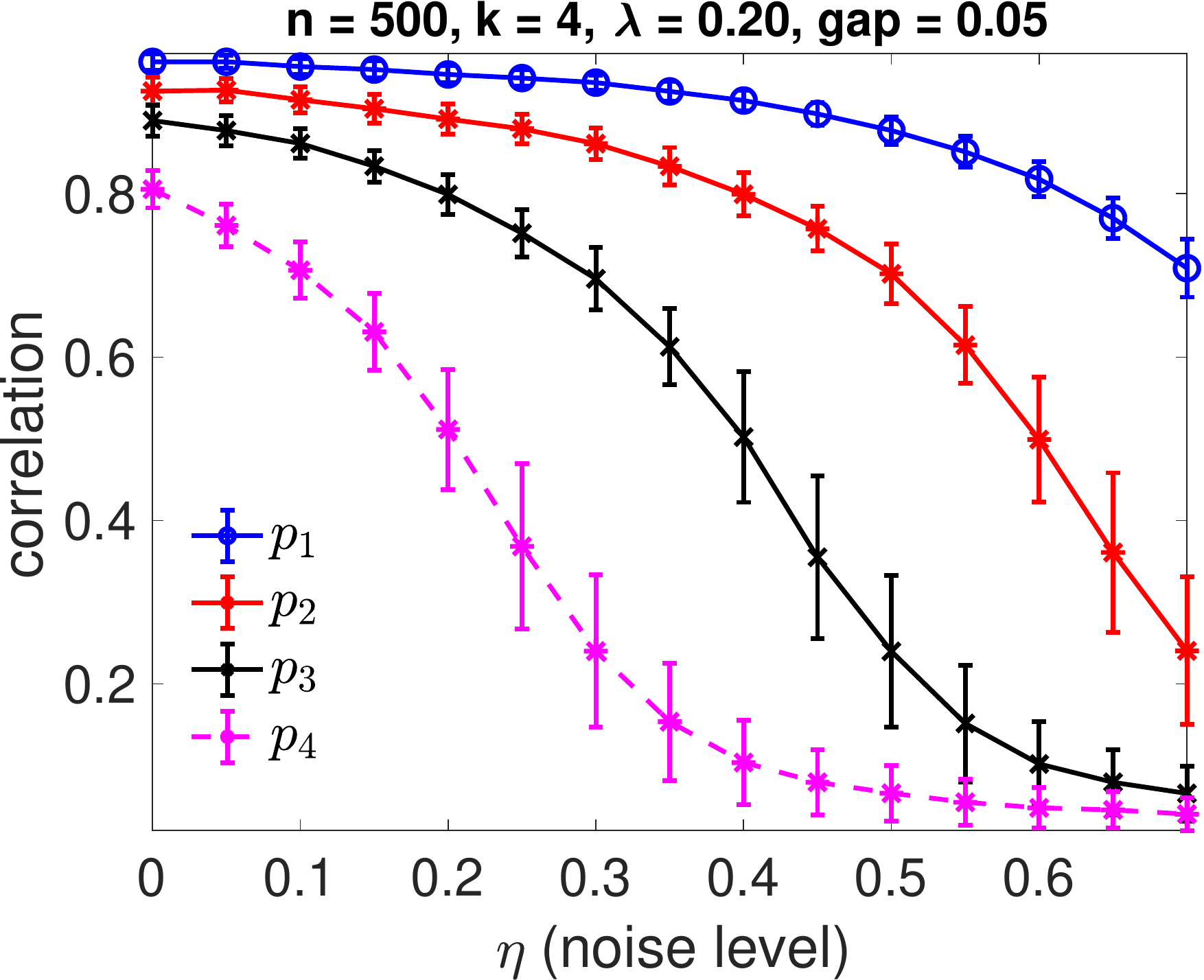} }
%
\subcaptionbox[]{ 
}[ 0.31\textwidth ]
{\includegraphics[width=0.31\textwidth, trim=0cm 0cm 0cm 0.7cm,clip] {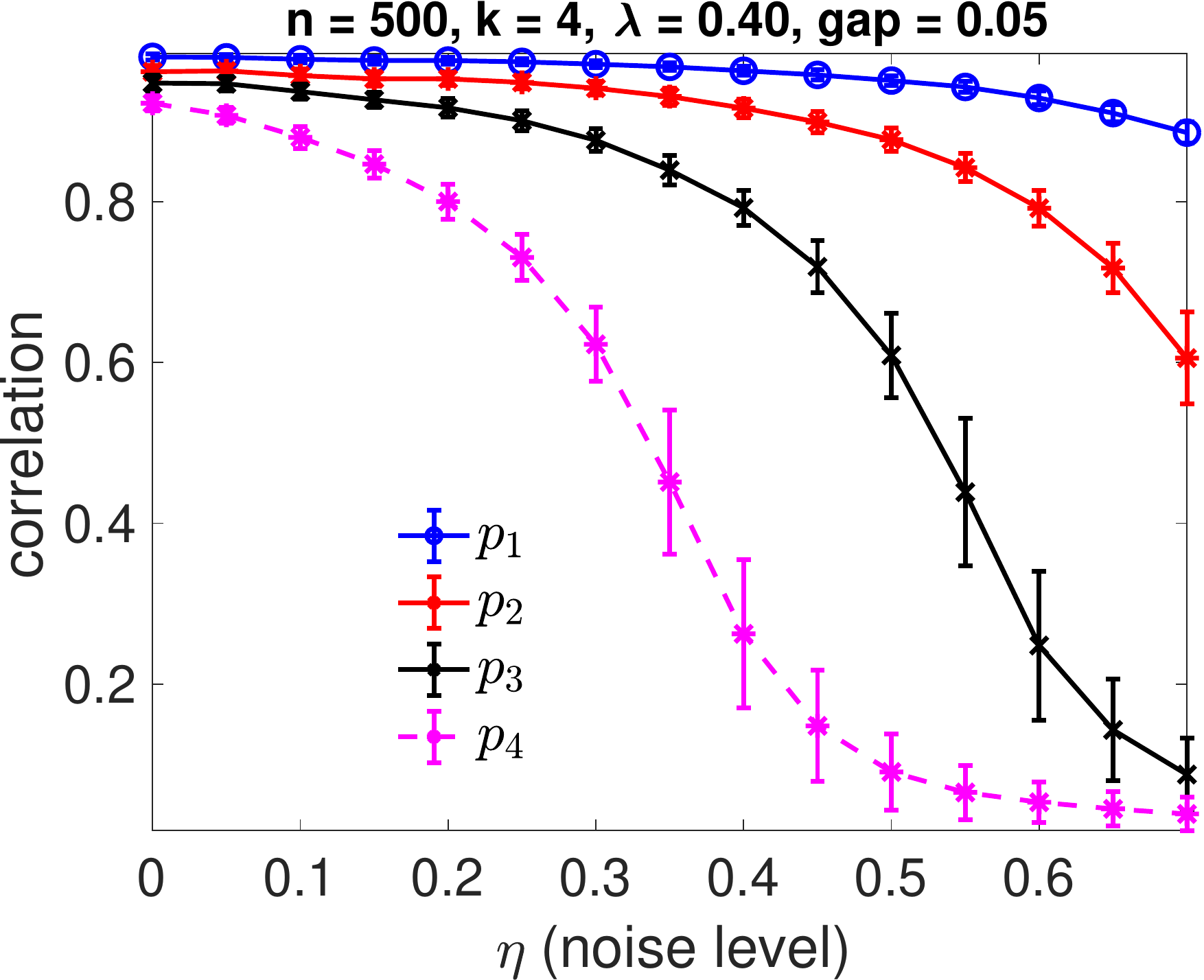} }
%
\subcaptionbox[]{ 
}[ 0.31\textwidth ]
{\includegraphics[width=0.31\textwidth, trim=0cm 0cm 0cm 0.7cm,clip] {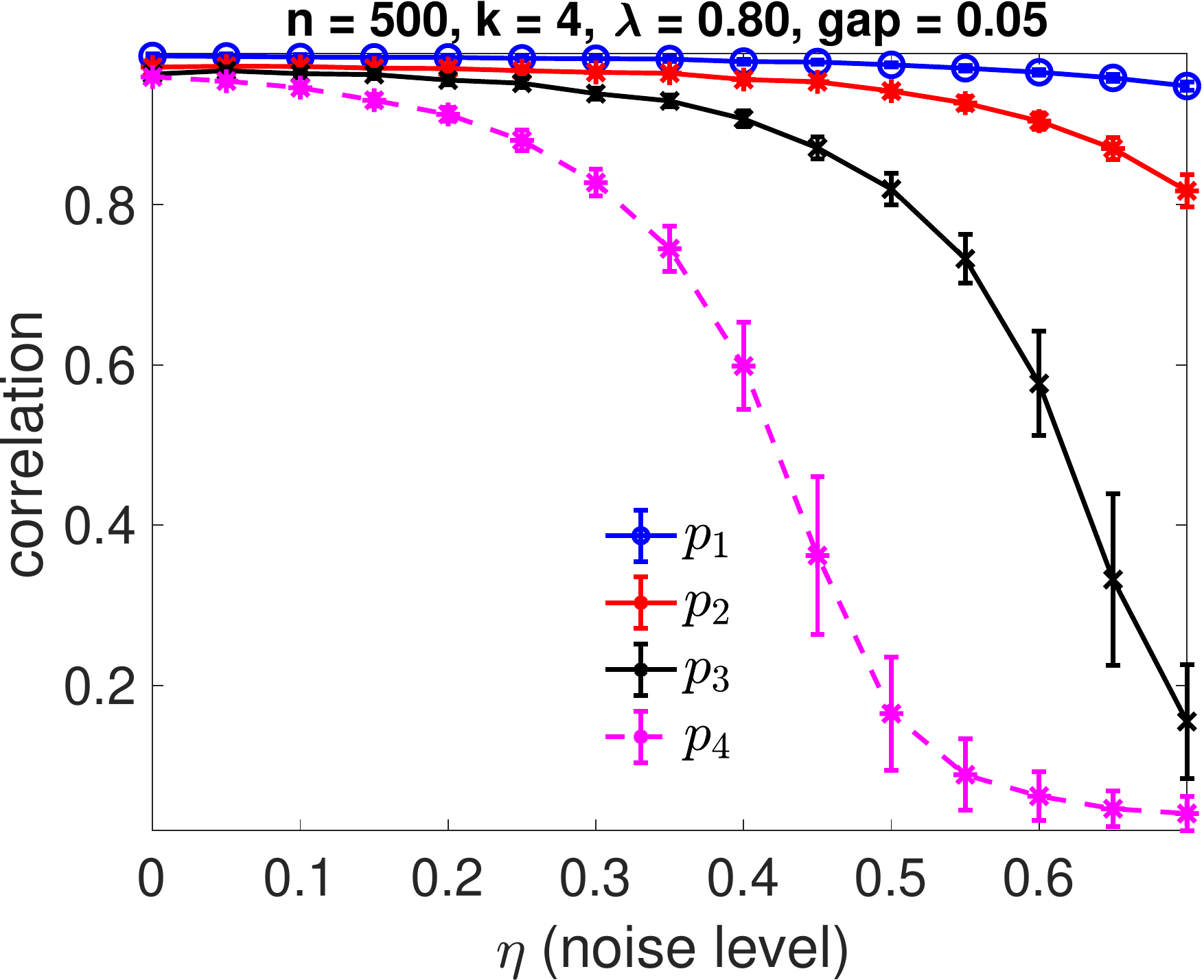} }
%
\captionsetup{width=0.98\linewidth}
\caption[Short Caption]{   
Experimental Setup II:  Correlation of recovered angles with the ground truth, where we keep constant the gap $\gamma = 0.05$ between consecutive sampling probabilities $ p_{l+1} - p_l = \gamma, l = 1, \ldots, k-1$, of the subgraphs $ G_{l}$. We also fix the number of angles $n=500$, and  vary  $k \in \{2,3,4\}$ (indexing the rows), and the overall measurement graph sparsity $ \lambda \in \{ 0.2, 0.4, 0.8\}$ (indexing the columns). Each plot shows the correlation with the ground truth, as we vary the noise level $ \eta $. Results are averaged over $20  \times 20 = 400$ runs.}
\vspace{-3mm}
\label{fig:scanID_2_c_k234_500_gap}
\end{figure}

%% file: S_Exp_SDP.tex
\begin{figure}[!ht]
\vspace{-2mm}  
\hspace{2.8cm} $ \lambda = 0.20 $ \hspace{3.6cm}   $ \lambda = 0.40 $   \hspace{3.6cm}  $ \lambda = 0.80 $

\vspace{4mm}
\centering
\raisebox{0.75in}{\rotatebox[origin=t]{90}{k=2}}\hspace{1mm}  
%
\subcaptionbox[]{ %
}[ 0.31\textwidth ]
{\includegraphics[width=0.31\textwidth] {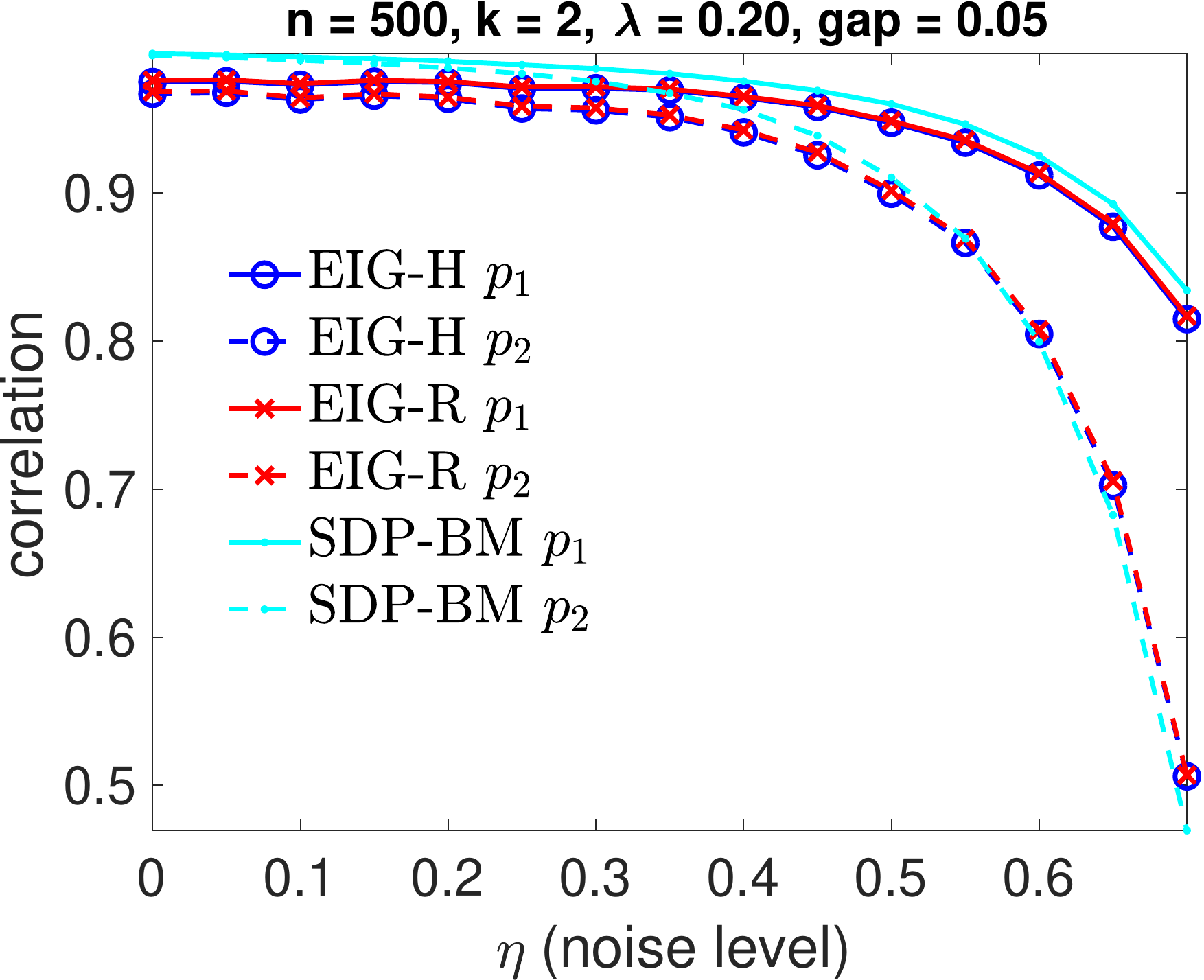} }
%
\subcaptionbox[]{  %
}[ 0.31\textwidth ]
{\includegraphics[width=0.31\textwidth] {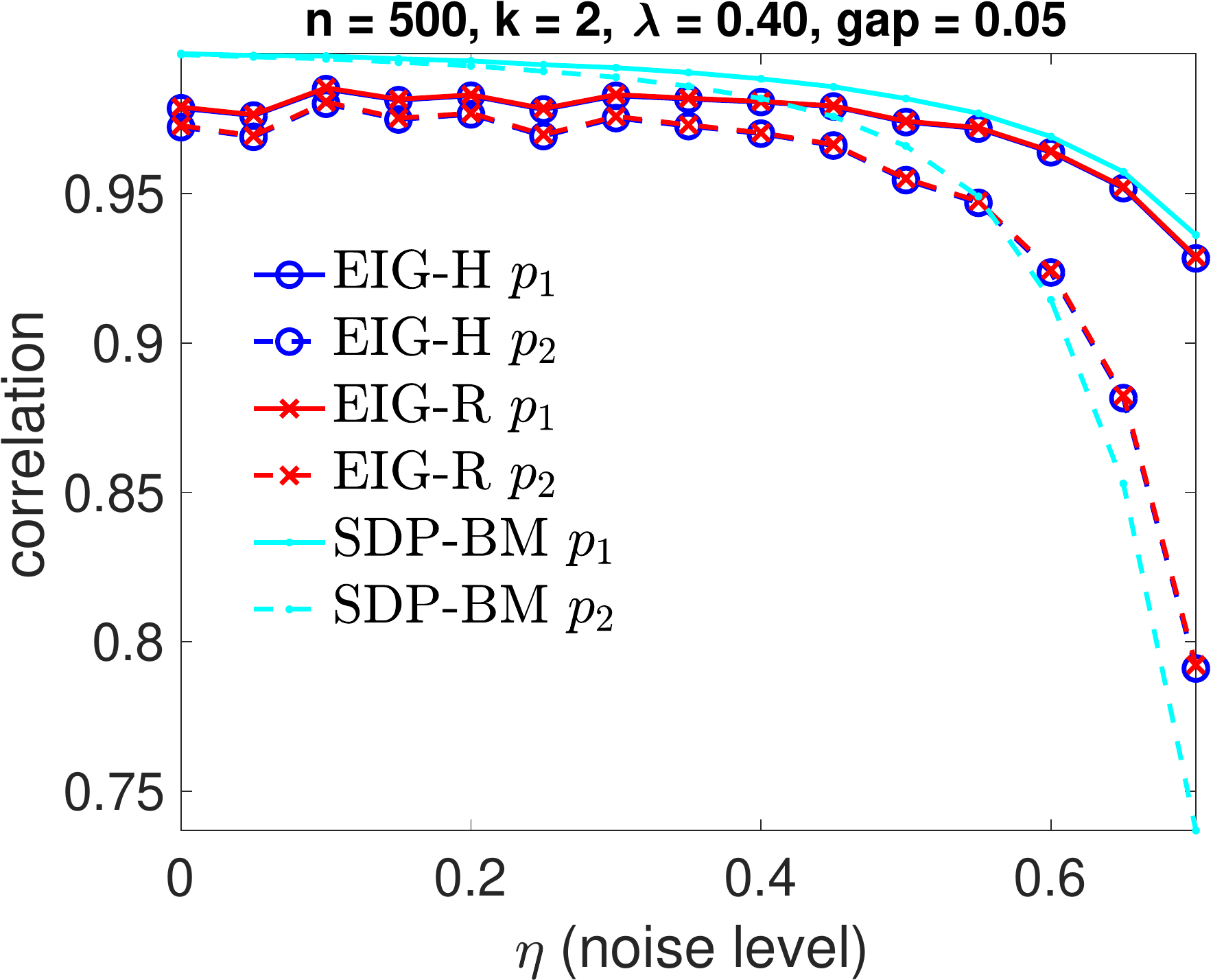} }
%
\subcaptionbox[]{ %
}[ 0.31\textwidth ]
{\includegraphics[width=0.31\textwidth] {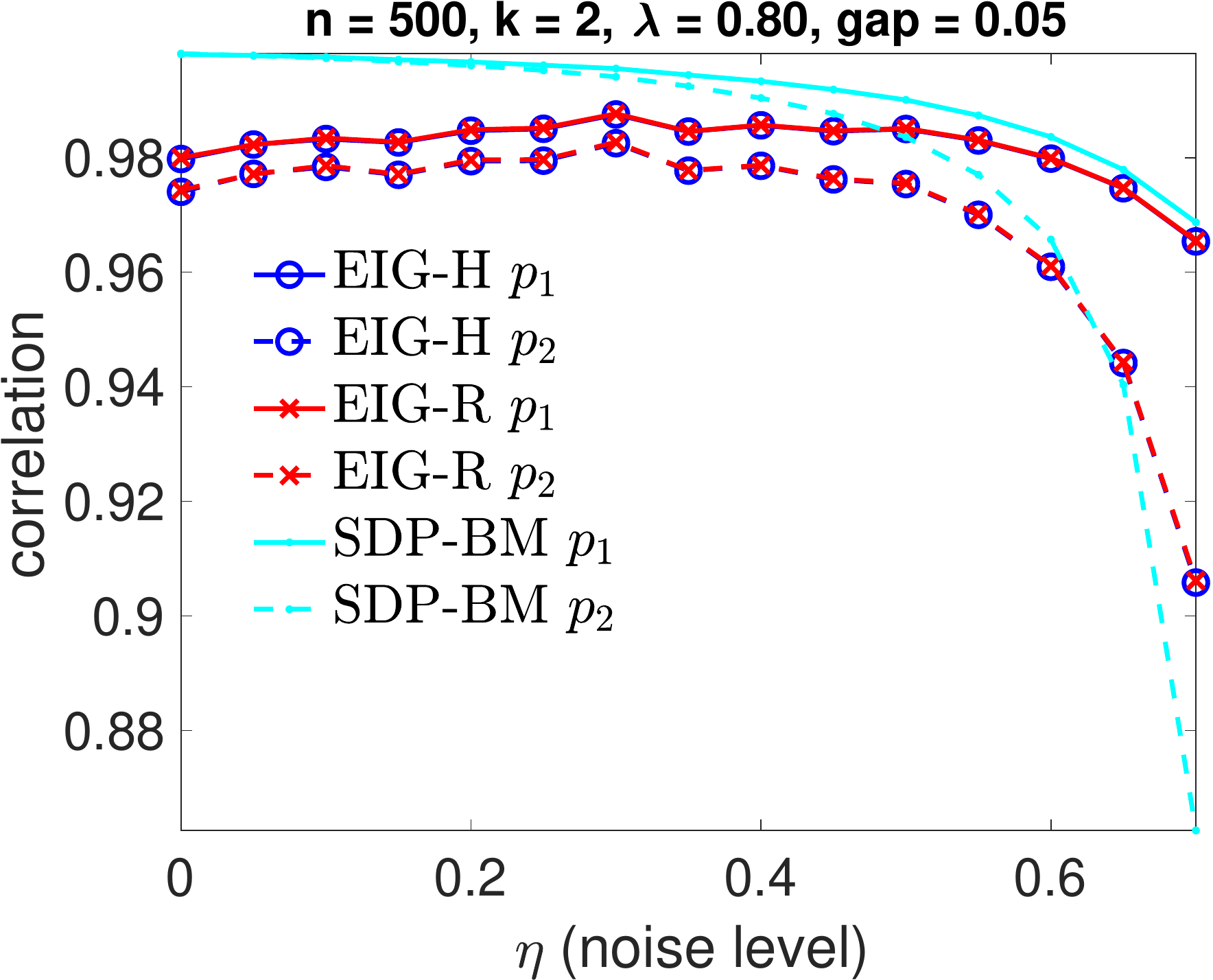} }
%

\raisebox{0.75in}{\rotatebox[origin=t]{90}{k=3}}\hspace{1mm}
\subcaptionbox[]{ %
}[ 0.31\textwidth ]
{\includegraphics[width=0.31\textwidth] {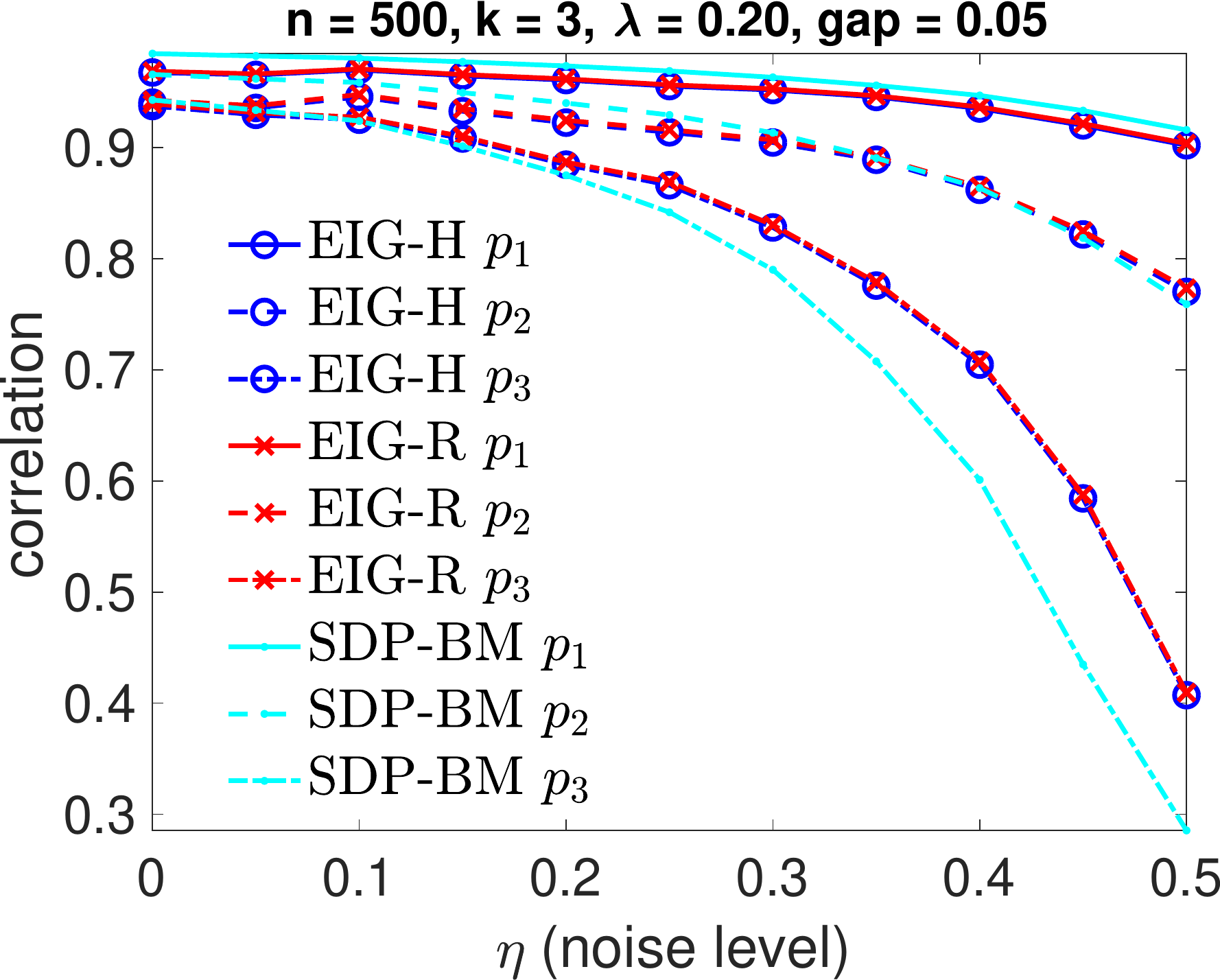} }
%
\subcaptionbox[]{ %
}[ 0.31\textwidth ]
{\includegraphics[width=0.31\textwidth] {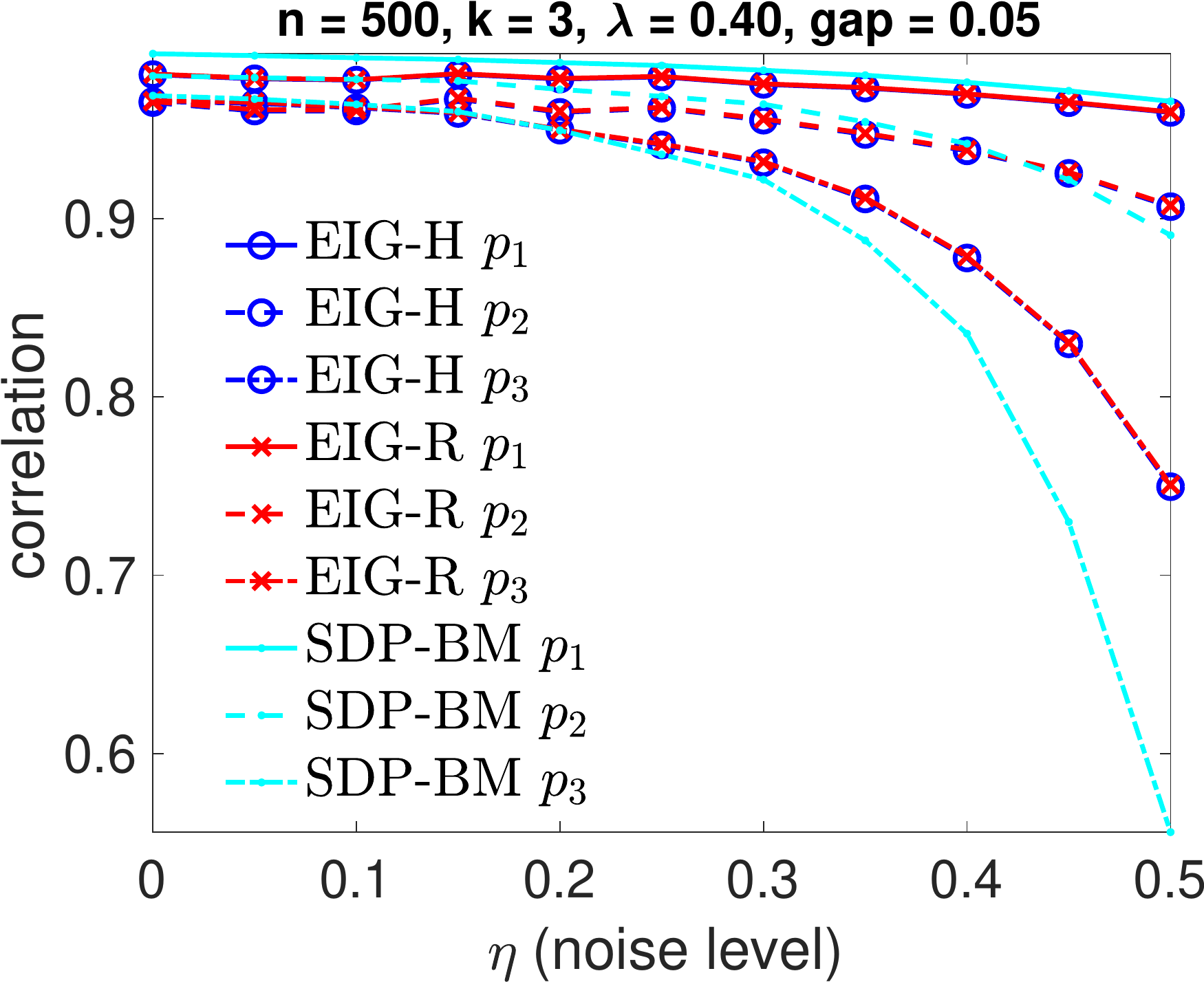} }
%
\subcaptionbox[]{ %
}[ 0.31\textwidth ]
{\includegraphics[width=0.31\textwidth] {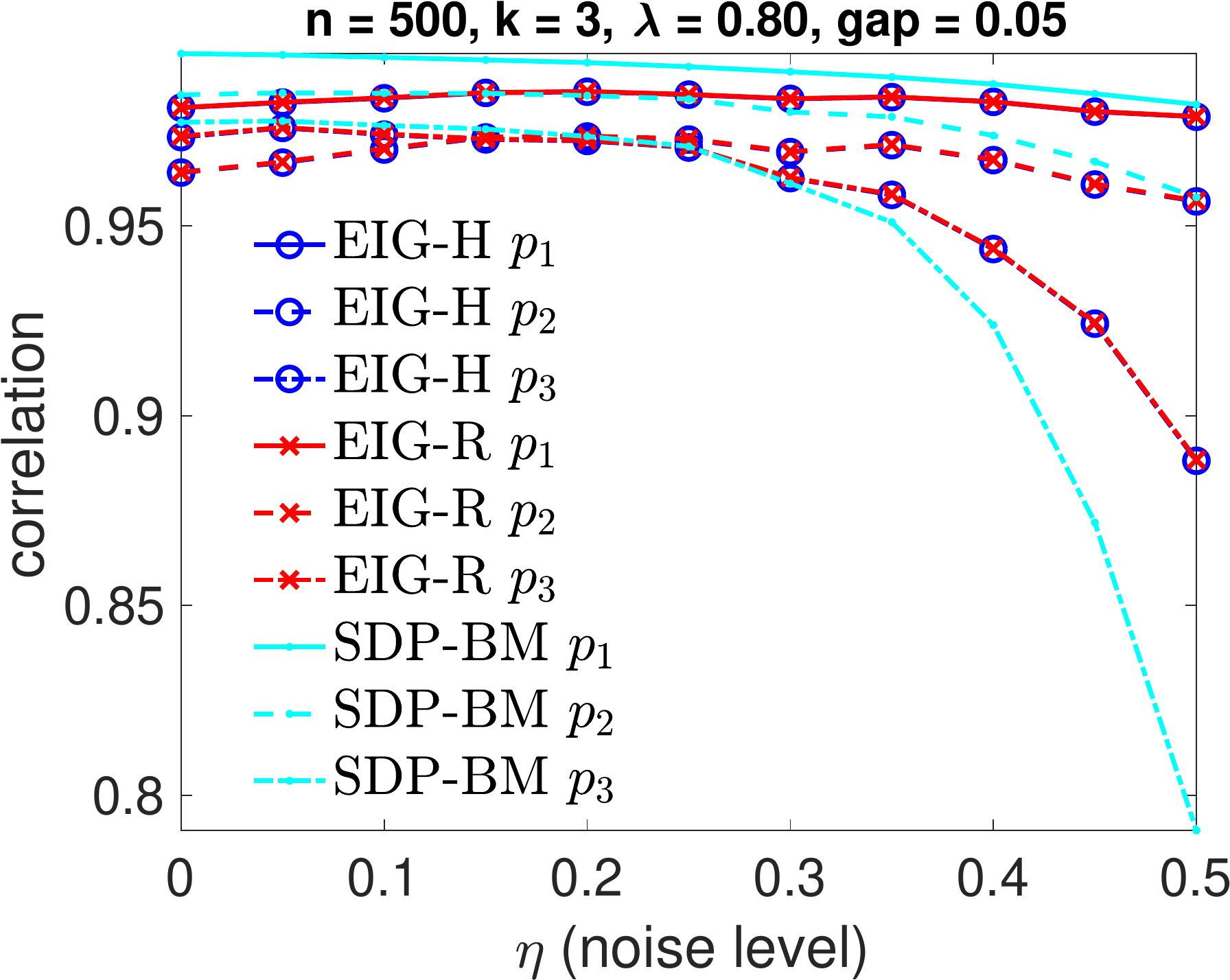} }
%

\raisebox{0.75in}{\rotatebox[origin=t]{90}{k=4}}\hspace{1mm}
\subcaptionbox[]{ 
}[ 0.31\textwidth ]
{\includegraphics[width=0.31\textwidth] {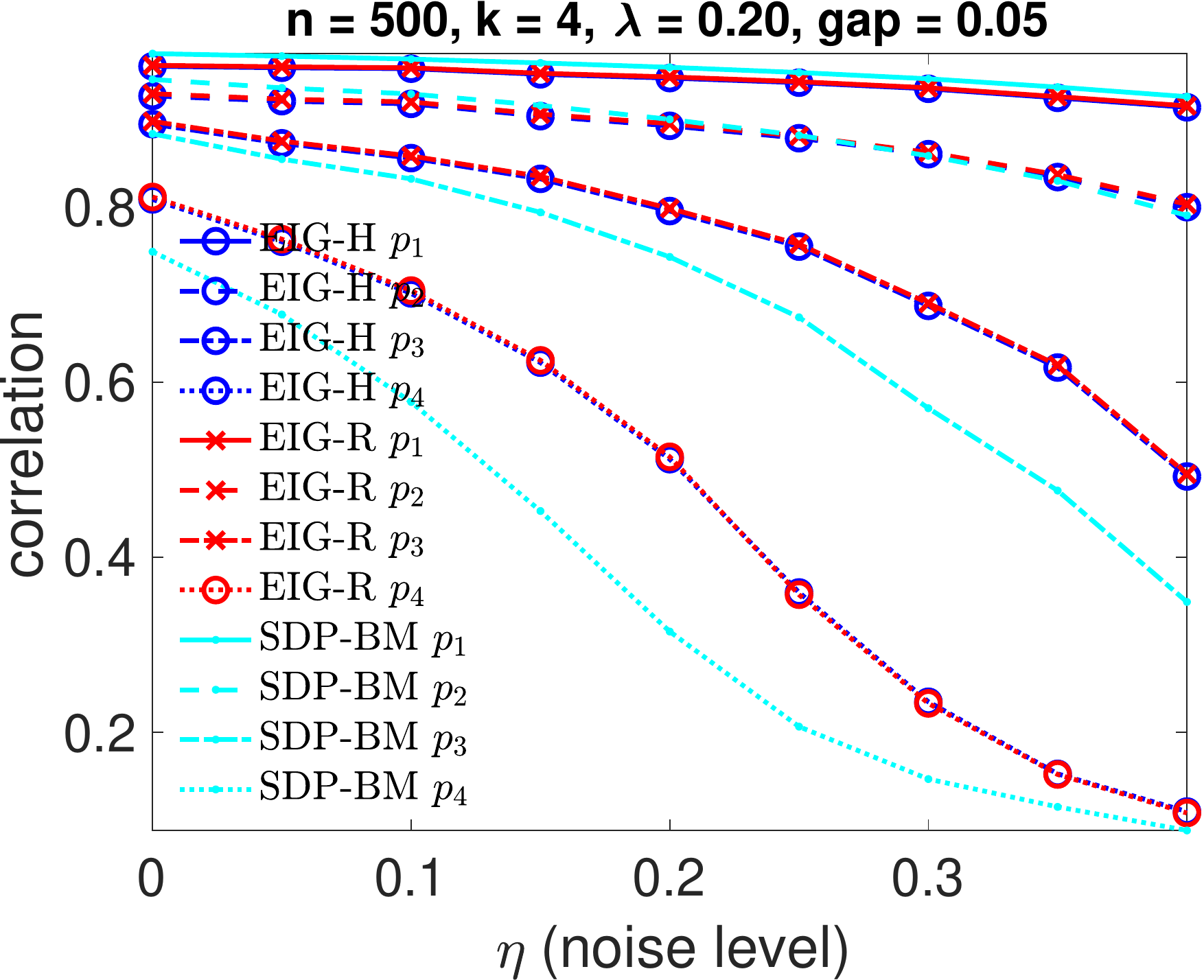} }
%
\subcaptionbox[]{ %
}[ 0.31\textwidth ]
{\includegraphics[width=0.31\textwidth] {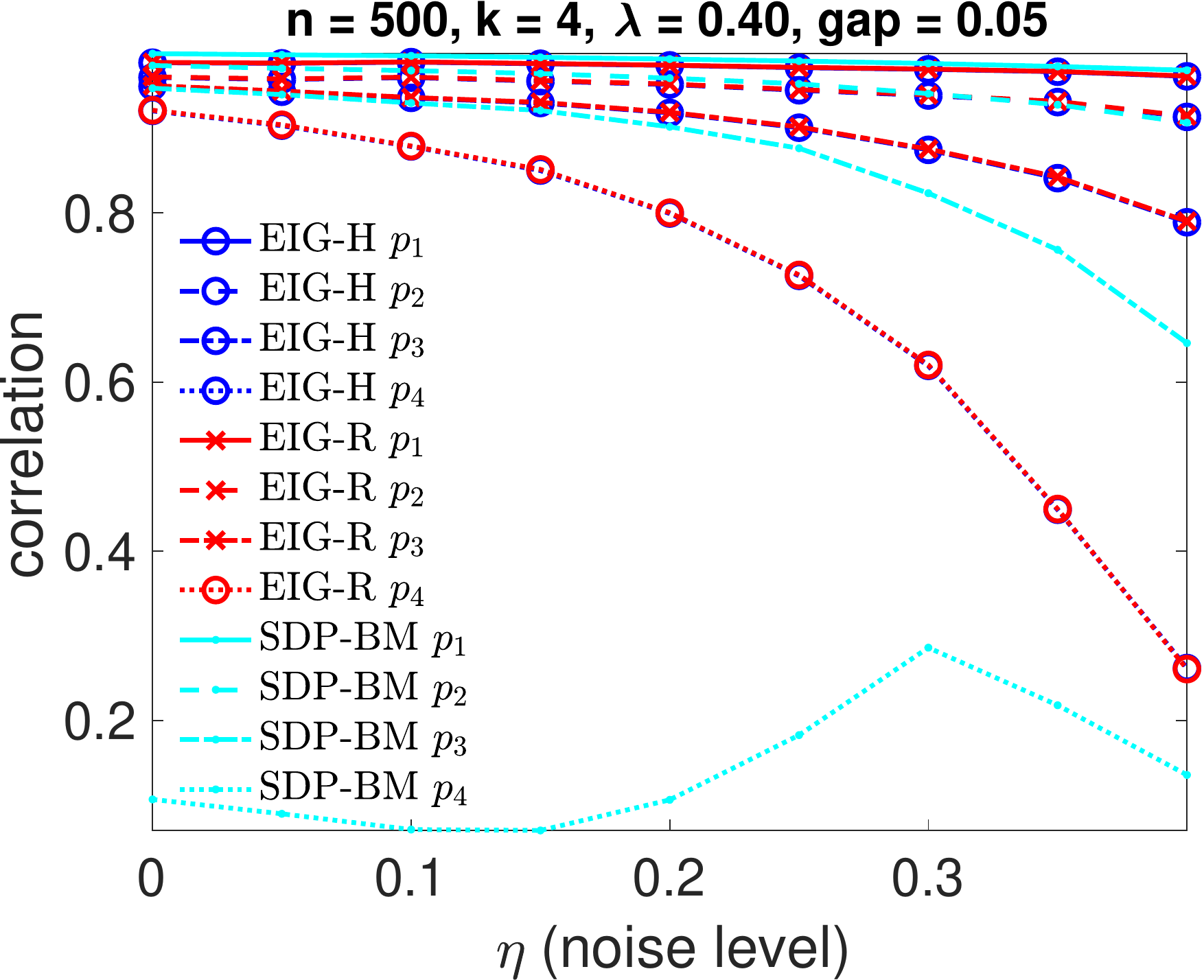} }
%
\subcaptionbox[]{ %
}[ 0.31\textwidth ]
{\includegraphics[width=0.31\textwidth] {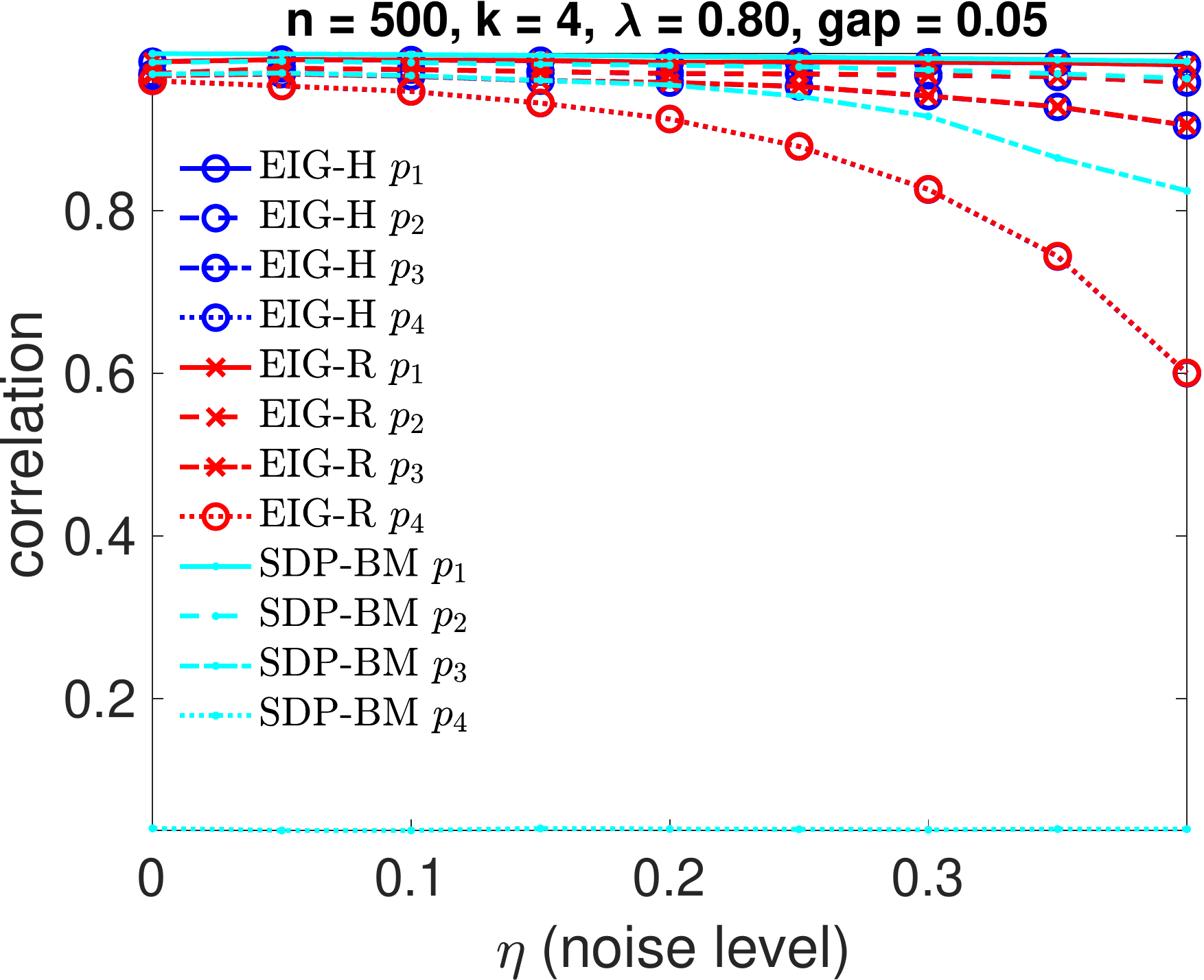} }
%
\vspace{-2mm}   
\captionsetup{width=0.98\linewidth}
\caption[Short Caption]{  
Experimental Setup II: Correlation of recovered angles with the ground truth in the second experimental setup
where we compared the performance of the two spectral relaxations and the SDP relaxation. \textsc{EIG-H} denotes the relaxation that uses the top eigenvector of matrix $H$ from \eqref{eq:bisync_HijDef}; \textsc{EIG-R} employs the normalized matrix 
\cref{eq:Rnormalization}, 
and  \textsc{SDP-BM} denotes the SDP relaxation \cref{eq:SDP_program_SYNC}  solved via the Burer-Monteiro approach. We keep constant the gap $\gamma = 0.05$ between consecutive sampling probabilities $ p_{l+1} - p_l = \gamma, l = 1, \ldots, k-1$, of the subgraphs $ G_{l}$. The number of angles is kept fixed at $n=500$, and we vary  $k \in \{2,3,4\}$ (indexing the rows), and the overall measurement graph sparsity $ \lambda \in \{ 0.2, 0.4, 0.8\}$ (indexing the columns). Each plot shows the correlation with ground truth, as we vary the noise level $ \eta $. Results are averaged over $20  \times 20 = 400$ runs.}
\vspace{-3mm}
\label{fig:scanID_5c_k234_500_SDP}
\end{figure}

%% file: S_GRP.tex

We now turn our attention to an application of bi-synchronization to a particular instance of the graph realization problem, where the graph disentangling procedure will prove useful.

\subsection{The graph realization problem}
In the \textit{graph realization problem} (GRP), one is given a graph $G=([n],E)$ together with a non-negative distance measurement $d_{ij}$ associated with each edge, and is asked to compute a realization of $G$ in $\mathbb{R}^d$. In other words, for any pair of adjacent nodes $i$ and $j$, the distance $d_{ij} = d_{ji}$ is available, and the goal is to find a $d$-dimensional embedding $p_1, p_2, \ldots, p_n \in \mathbb{R}^d$ such that $\|p_i-p_j \| =d_{ij}, \text{ for all } \set{i,j} \in E$. Due to its practical significance, the GRP has attracted a lot of attention in recent years, across many communities such as wireless sensor networks \cite{biswas_stress_sdp,overview}, structural biology \cite{molecule_problem}, dimensionality reduction, Euclidean ball packing and multidimensional scaling (MDS) \cite{cox}. 

In many real world applications, such as sensor networks and structural biology, the given distances $d_{ij}$ between adjacent nodes are not accurate, $d_{ij} = \|p_i - p_j \| + \varepsilon_{ij}$ where $\varepsilon_{ij}$ represents the added noise, and the goal is to find an embedding that realizes all known distances $d_{ij}$ as best as possible.
The 2D-ASAP algorithm proposed in \cite{asap2d} belongs to the group of algorithms that integrate local distance information into a global structure determination. For every sensor, one first identifies  globally rigid\footnote{
A framework is a finite graph $G$ together with a finite set of vectors in $d$-space, where each $p_i$ corresponds to a node of $G$, and the edges of $G$ correspond to fixed length bars connecting the nodes. Two frameworks with the same graph $G$ are equivalent if the lengths of their edges are the same.
A framework $G(p)$ is {\em globally rigid} in $\mathbb{R}^d$ if all frameworks $G(q)$ in $\mathbb{R}^d$ which are $G(p)$-equivalent (have all bars the same length as $G(p)$) are congruent to $G(p)$ (that is, they are related by a rigid transformation).
A graph $G$ is {\em generically globally rigid} in $\mathbb{R}^d$ if $G(p)$ is globally rigid at
all generic configurations $p$ \cite{Connelly3,Connelly}. Only recently it was demonstrated that global rigidity is a generic property in this sense, for graphs in each dimension \cite{Connelly,Harvard1}.} 
subgraphs of its $1$-hop neighborhood, denoted as patches. Each patch is then separately localized in a coordinate system of its own, using for example either a stress minimization approach \cite{gotsman}, or semidefinite programming (SDP).

In the noise-free case, the coordinates of the nodes in each patch must agree with their global positioning up to some unknown rigid motion, that is, up to translation, rotation and possibly reflection. To every patch there corresponds an element of the Euclidean group Euc(d), and the goal is to estimate the group elements that will properly align all the patches in a globally consistent manner. By finding the optimal alignment of all pairs of patches whose intersection is large enough, we obtain measurements for the ratios of the unknown group elements, which is thus an instance of the group synchronization problem over the Euclidean group Euc(d). 

Unlike the generative models analyzed in the earlier sections, the measurement graph $G$   arising in the graph realization problem is a disc graph, as opposed to an {\ER}  graph.  Consider a patch $P_i$ to be given by a central node $i$ along with all its neighbors within a radius $r$. In order for a pair of patches $P_i$ and  $P_j$ to be aligned, they need to have enough points in common (typically at least $d+1$, where $d$ is the ambient dimension), which in turn means that the two patches $P_{i}$ and $P_{j}$ have their center nodes at most $2r$ distance apart. Thus, the measurement graph $G$ is best modeled by a disc graph.

The specific problem instance we consider here is one where there exist two underlying non-congruent embeddings that we seek to recover, $X$ and $Y \in  \mathbb{R}^{n \times 2}$. For each such embedding, we have available information on the embedding of patches (subgraph embeddings), but whenever we have available a pairwise alignment between two patches, we do not know a-apriori whether this measurement pertains to structure $X$ or $Y$. We further detail this setup later in the section, and next introduce an existing algorithm from the literature on the graph realization problem.

\subsection{The ASAP algorithm}
This section details the steps of the ASAP (As-Synchronized-As-Possible)  algorithm, a divide-and-conquer pipeline  proposed in \cite{asap2d} for the problem of scalable recoverings of point clouds from a sparse noisy set of pairwise distance information.  

The ASAP approach starts by decomposing the given graph $G$ into overlapping subgraphs (referred to as \textit{patches}), which are then embedded via the method of choice. To every local patch embedding, there corresponds a scaling and an element of the Euclidean group Euc($d$) of $d$-dimensional rigid transformations, and our goal is to estimate the group elements that will properly align all the patches in a globally consistent framework. The local optimal alignments between pairs of overlapping patches (whose intersection should be large enough - preferably tens of even hundreds of nodes) yield noisy measurements for the ratios of the above unknown group elements. Finding group elements from noisy measurements of their ratios is nothing but the group synchronization problem, the central topic of our present paper. 
For completeness, Table~\ref{overview} gives an overview of the approach, and Figure \ref{fig:pipeline} shows a schematic view of the pipeline we consider, as detailed in \cite{asap2d}.

\begin{algorithm*}
  \begin{center}
	\small
    \begin{tabular}{| p{0.15\textwidth} | p{0.78\textwidth} |} 
      \hline
      INPUT & $G=(V,E), \; |V|=n, \; |E|=m, \; \; d$\\
      \hline
      \hline
      Choose Patches &
      1. Break $G$ into $N$ overlapping patches $P_1,\ldots,P_N$.\\
      Embed Patches &
      2. Embed each patch $P_i$ separately via the method of choice (for eg, cMDS).
      \\
      \hline
      Step 1 & 1. Align all pairs of patches $(P_i, P_j)$ that have enough nodes in common. \\
      Rotate \& Reflect & 
      2. Estimate their relative rotation and possibly reflection $H_{ij} \in O(d) \subset \mathbb{R}^{d\times d}$.\\
      & 3. Build a sparse $dN\times dN$ symmetric matrix $H=(H_{ij})$ where entry $ij$ is itself a matrix in $O(d)$.\\
      & 4. Define $\mathcal{H} = D^{-1} H$, where $D$ is a diagonal matrix with \newline
      $D_{1+d(i-1),1+d(i-1)}=\ldots=D_{di,di} = deg(i)$, $i = 1,\ldots,N$, where $deg(i)$ is the node degree of patch $P_i$. \\
      & 5. Compute the top $d$ eigenvectors $v_i^\mathcal{H}$ of $\mathcal{H}$  satisfying  $ \mathcal{H} v_i^\mathcal{H}= \lambda_i^\mathcal{H} v_i^\mathcal{H}, i=1,\ldots,d$. \\
      & 6. Estimate the global reflection and rotation of patch $P_i$ by the orthogonal matrix $\hat{h}_i $ that is closest to $\widetilde{H}_i $ in Frobenius norm, where $\widetilde{H}_i $ is the submatrix corresponding to the i$^\text{th}$ patch in the $dN \times d$ matrix formed by the top $d$ eigenvectors $[v_1^\mathcal{H} \ldots v_d^\mathcal{H}]$. \\
      & 7. Update the embedding of patch $P_i$ by applying the  orthogonal transformation $\hat{h}_i $. \\ 
      \hline
      Step 2 Translate &  Solve $m\times n$ overdetermined system of linear equations for optimal translation in each dimension.\\
      \hline
      \hline
      OUTPUT & Estimated coordinates $\hat{x}_1,\ldots,\hat{x}_n$ \\
      \hline
    \end{tabular}
  \end{center}
  \caption{The ASAP algorithm  \cite{asap2d}.}
  \label{overview}
  \end{algorithm*}

\normalsize

\begin{figure}[ht]
\vspace{-3mm}
\begin{center}
\includegraphics[width=0.95\columnwidth]{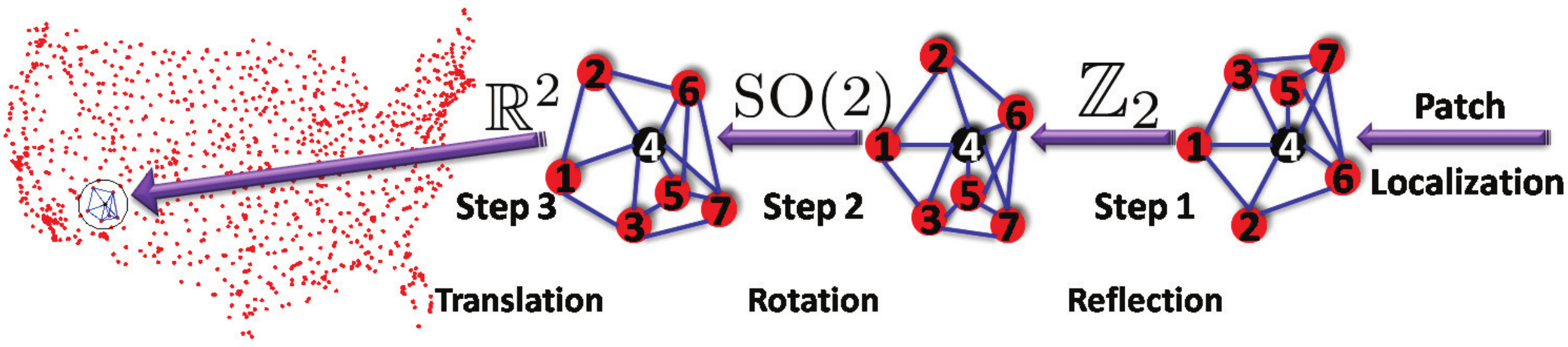}
\end{center}
\vspace{-2mm}
\captionsetup{width=0.99\linewidth}
\caption{The ASAP recovery process in $d=2$ dimensions, considered in \cite{asap2d},  for a patch in the US graph.} \label{fig:pipeline}
\vspace{-3mm}
\end{figure}


\begin{figure}[!ht]
\vspace{-2mm}
\centering
\subcaptionbox[]{  
}[ 0.41\textwidth ]
{\includegraphics[width=0.32\textwidth]{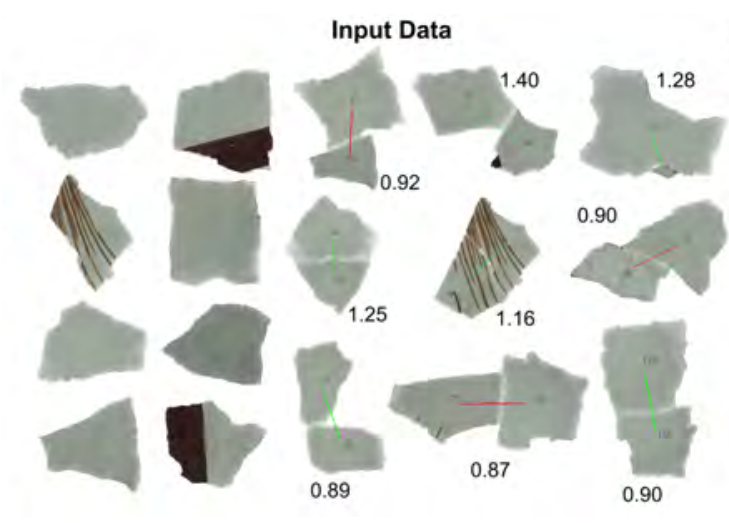} }
%
\subcaptionbox[]{   
}[ 0.45\textwidth ]
{\includegraphics[width=0.28\textwidth]{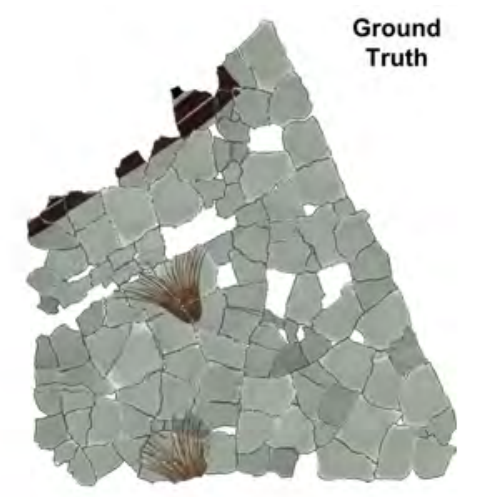} }
\vspace{-3mm} 
\captionsetup{width=0.98\linewidth}
\caption[Short Caption]{Alignment of frescos, as explored in \cite{syncFrescoes}. The alignment of the pieces can be construed as an instance of the synchronization problem over the Euclidean group of rigid motions.}  
\vspace{-2mm}
\label{fig:syncFrescoes_AC}
\end{figure}

\subsection{The multi-graph realization problem}
The patch alignment problem and the ASAP algorithm detailed above could be construed as aligning pieces of a jigsaw puzzles. Patches could be interpreted as pieces of the puzzle, and the task becomes to assign to each individual piece an element of the Euclidean Group in such a way that when we apply the specific transformation to each individual patch, all the pieces of the puzzle synchronize together.  
For the  case of a jigsaw puzzle, only a rotation and translation are needed; a reflection is not required since this is trivially handled for the pieces of the jigsaw puzzle, by simply identifying the front-back of each piece. 

Finally, we recall here the \textit{cluster-synchronization} problem detailed in Section \ref{sec:kSyncApp}, which can be solved and analyzed with similar tools as those used in the present article. Imagine a instance where one combines two different sets of puzzles, and aims to disentangle the two sets of pieces. In a real-world application, this could arise, for instance, at an archaeological site where two broken froscoes are uncovered, both shattered into pieces which got mixed up, and the goal is to separate out the two froscoes. When one considers a pair of pieces, it is not known whether they both come from same fresco, or from two different ones. A probabilistic model for this  setting is one where there are two groups of angles 
$ \mathcal{C}_1 = (\alpha_1,\alpha_2,\ldots,\alpha_{n_1} )$, and 
$ \mathcal{C}_2 = (\beta_{n_1+1},\beta_{n+2},\ldots,\beta_{n_1 + n_2} )$,  
and an available measurement between two angles $ \alpha_i$ and $\beta_j$ from two different clusters is completely uniformly distributed in the unit circle, while a measurement between a pair of angles from the same cluster is a noisy version of the offset $\alpha_i-\alpha_j$ or $\beta_i-\beta_j$. This model is reminiscent of the standard stochastic block model, well studied in the clustering and network analysis literatures.  

\medskip 

\paragraph{Setup}
We consider the following instance of the graph realization problem, where one aims to recover the coordinates of two clouds of points, $X$ and $Y \in \mathbb{R}^{n \times 2}$, indexed by the same set of nodes. More specifically, the goal is to identify two clouds of points,
$p^{(X)}_1, \ldots, p^{(X)}_n \in \mathbb{R}^d$
and 
$p^{(Y)}_1, \ldots, p^{(Y)}_n \in \mathbb{R}^2$, given information on subgraph embeddings from either confirmation. In other words, to each node we associate a patch
(eg, its one-hop neighborhood), which may assume one of two possible embeddings $\mathcal{P}_i^{(X)}$ and $\mathcal{P}_i^{(Y)}$.
Whenever a pair of patches $\mathcal{P}_i$ and $\mathcal{P}_j$ are compared and aligned, one of three things may happen. 

\begin{itemize}
\item (with probability $p_1$) Both subgraph embeddings are of type-$X$, and we align $\mathcal{P}_i^{(X)}$ and $\mathcal{P}_j^{(X)}$, with $ \setij \in E_1$; 
\item (with probability $p_2$)  both subgraph embeddings are of type-$Y$, and we align  $\mathcal{P}_i^{(Y)}$ and $\mathcal{P}_j^{(Y)}$
with $ \setij \in E_2$;  
\item (with probability $\eta = 1-p_1-p_2$)  one embedding is of type-$X$ and the other one of type-$Y$, for example 
$\mathcal{P}_i^{(X)}$ and $\mathcal{P}_j^{(Y)}$ (with $ \setij \in E_W$, the edge set of the bad subgraph) are aligned, and the resulting outlier measurement is a random number uniformly distributed in the unit circle, and thus completely uninformative. 
\end{itemize}
The above is nothing but an instance of the bi-synchronization problem considered in this paper. The task then becomes to identify the two subgraphs $G_1$ and $G_2$, which can ultimately facilitate the recovery of the two global embeddings $X$ and $Y$.

In our experiments, we let the point cloud $Y$ be to a non-rigid transformation of point cloud $X$, and the goal is to recover both embeddings. 
Figure  \ref{fig:US_clean}  shows the two embeddings $X$ (a) and $Y$ (b), while (c) shows the Procrustes alignment of two point clouds, together with the displacement error bars, clearly showing the two embeddings are not congruent. 
Both sets of patches take their 2D coordinates from their corresponding ground truth embedding. For example, in the noiseless case, the 2D embedding of patch $P_{i}^{(X)}$ is given by the coordinates within $X$ of the nodes contained in  $P_{i}^{(X)}$. When noise is added, indexed by $\sigma$, we perturb both 2-D coordinates by additive Gaussian noise with mean $0$ and variance $\sigma^2$.

\begin{figure}[!ht]
\vspace{-3mm}
\centering
\subcaptionbox[]{  
}[ 0.31\textwidth ]
{\includegraphics[width=0.31\textwidth, trim=0.8cm 0.4cm 0.2cm 0.4cm,clip] {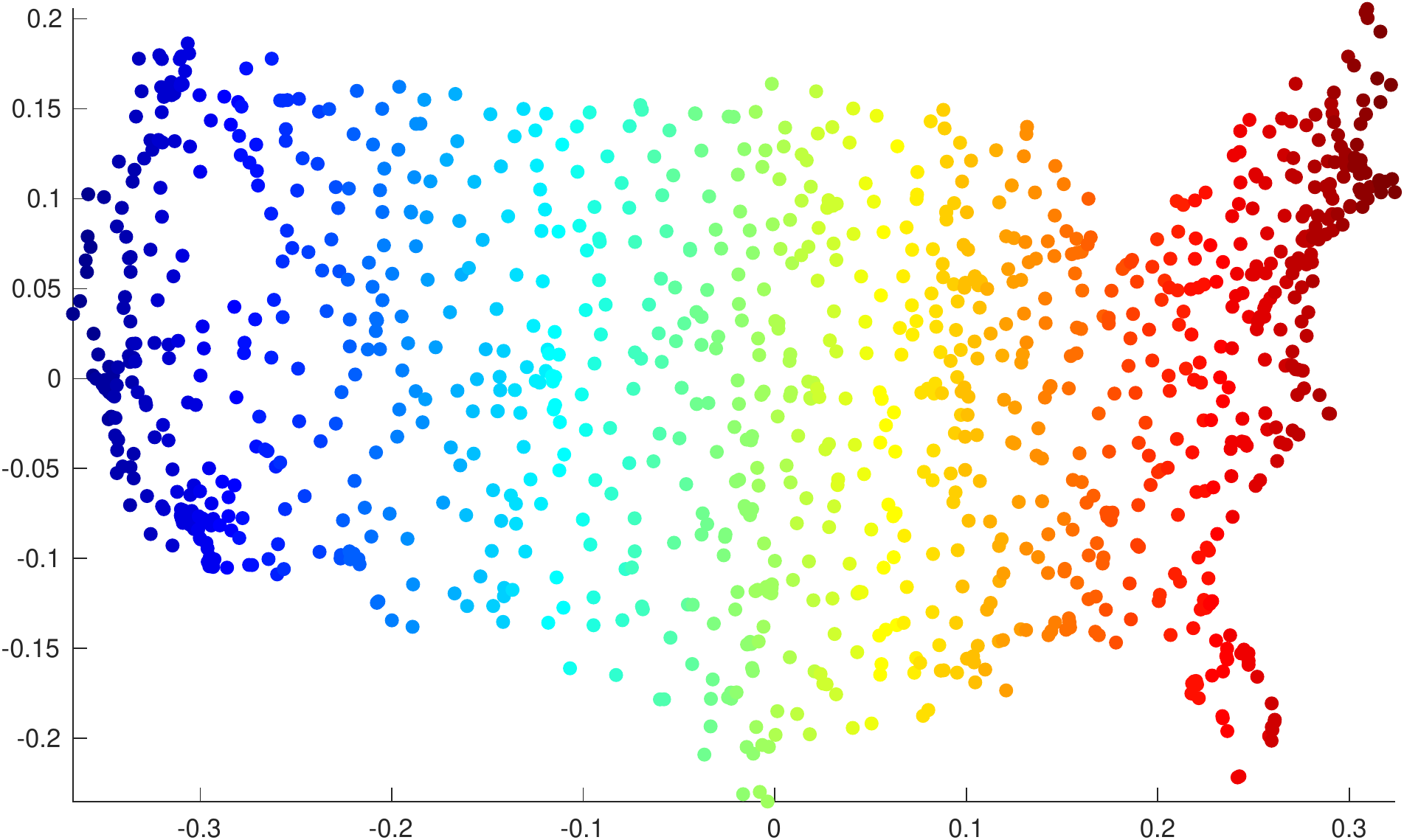} }
\hspace{2mm}
\subcaptionbox[]{  
}[ 0.31\textwidth ]
{\includegraphics[width=0.31\textwidth, trim=0.9cm 0.4cm 0.2cm 0.4cm,clip] {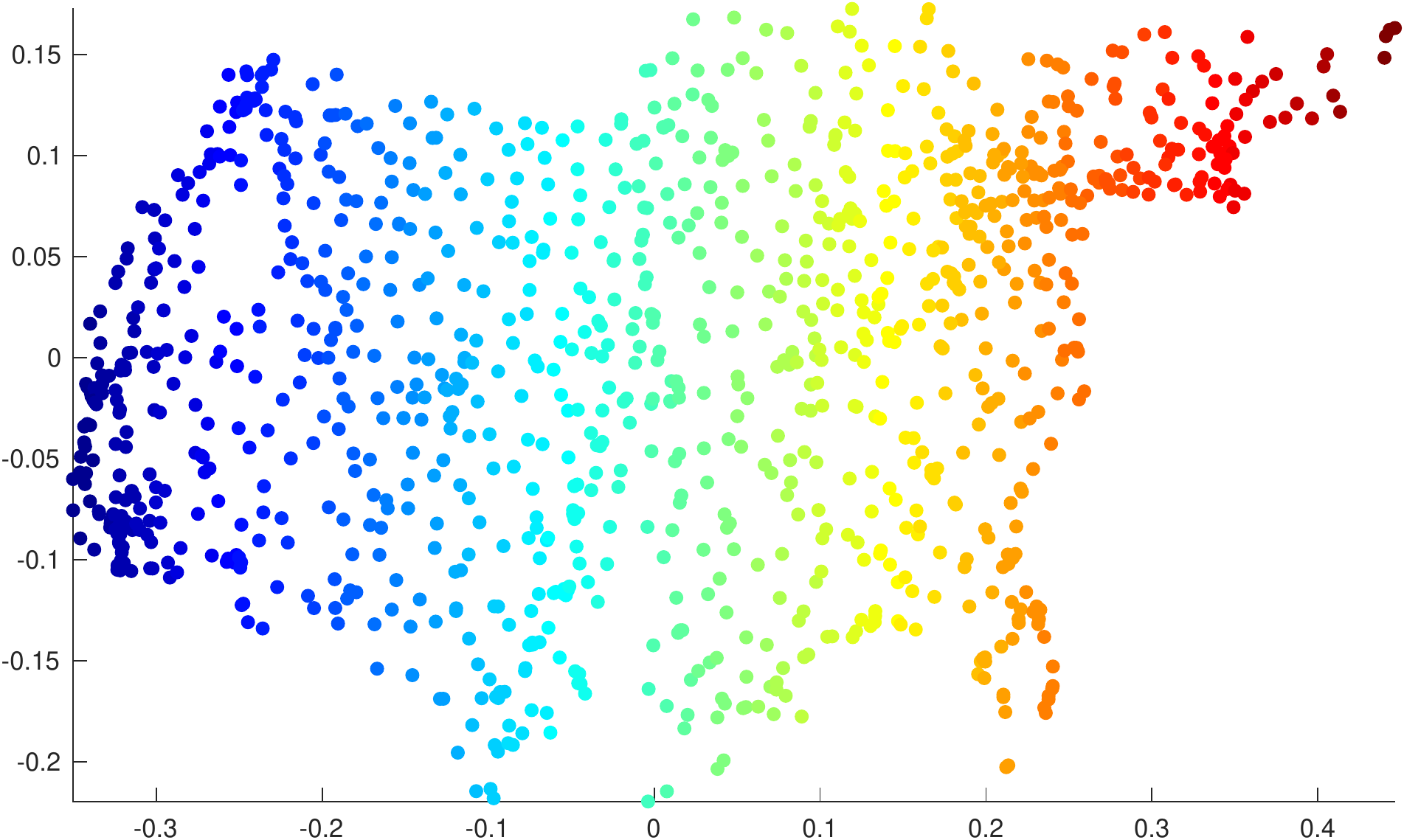} }
%
\subcaptionbox[]{   
}[ 0.33\textwidth ]
{\includegraphics[width=0.33\textwidth, trim=3cm 1.3cm 1.5cm 0.4cm,clip] {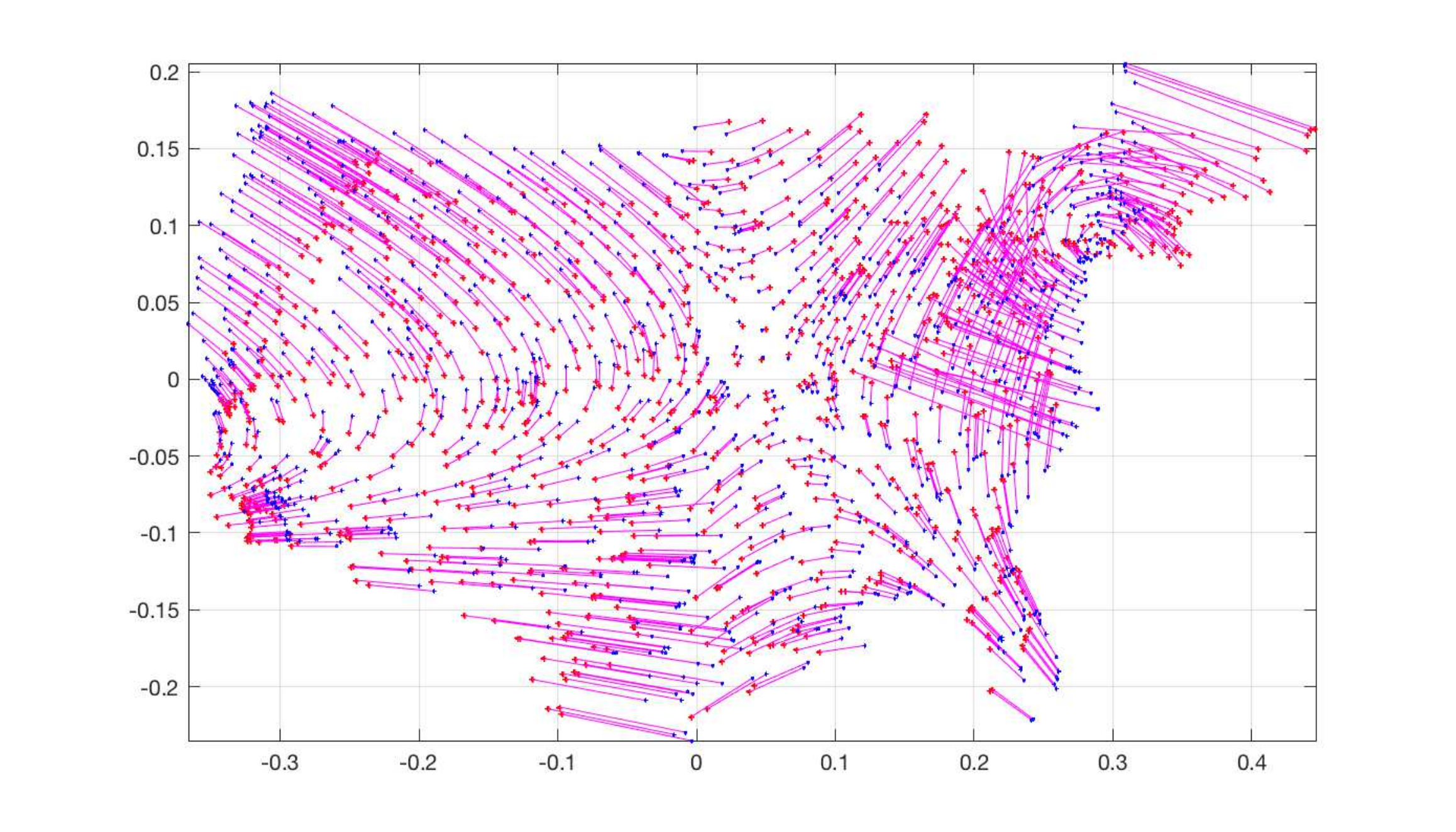} }
%
\vspace{-2mm}
\captionsetup{width=0.98\linewidth}
\caption[Short Caption]{Clean data, two non-congruent embeddings of the US graph: (a) Original US map (b) Non-rigid transformation of the original data (shear transformation, followed by a rotation and scaling of the Midwest and East Coast). (c) Procrustes alignment of (a) and (b) showing the displacements between the two embeddings. 
}  
\vspace{-2mm}
\label{fig:US_clean}
\end{figure}

\begin{figure}[!ht]
\vspace{-2mm}
\centering
\subcaptionbox[ ]{ $G =G_1 \cup G_2$
}[ 0.26\textwidth ]
{\includegraphics[width=0.26\textwidth] {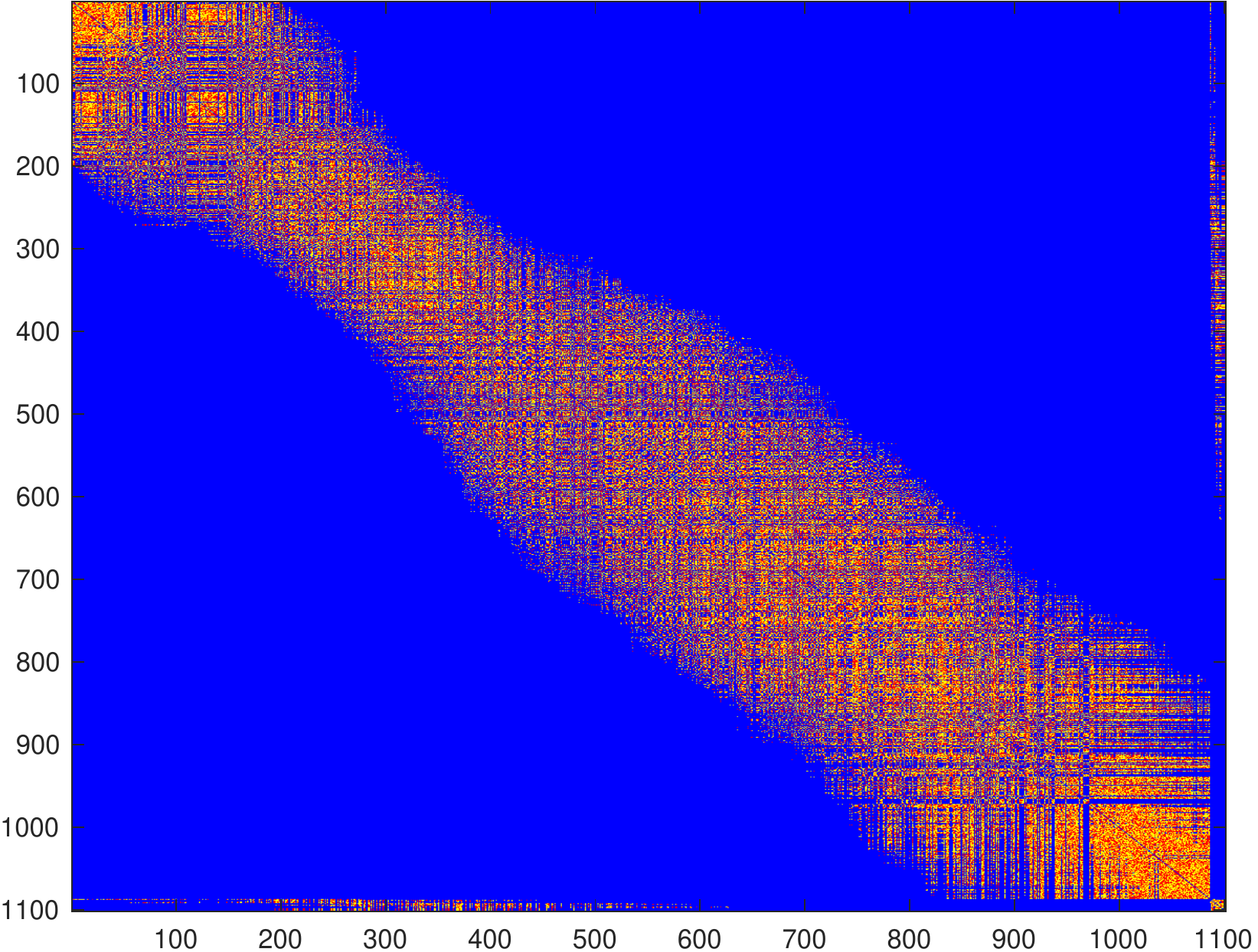}}
\subcaptionbox[ ]{ Degree distribution $G_1$
}[ 0.26\textwidth ]
{\includegraphics[width=0.26\textwidth] {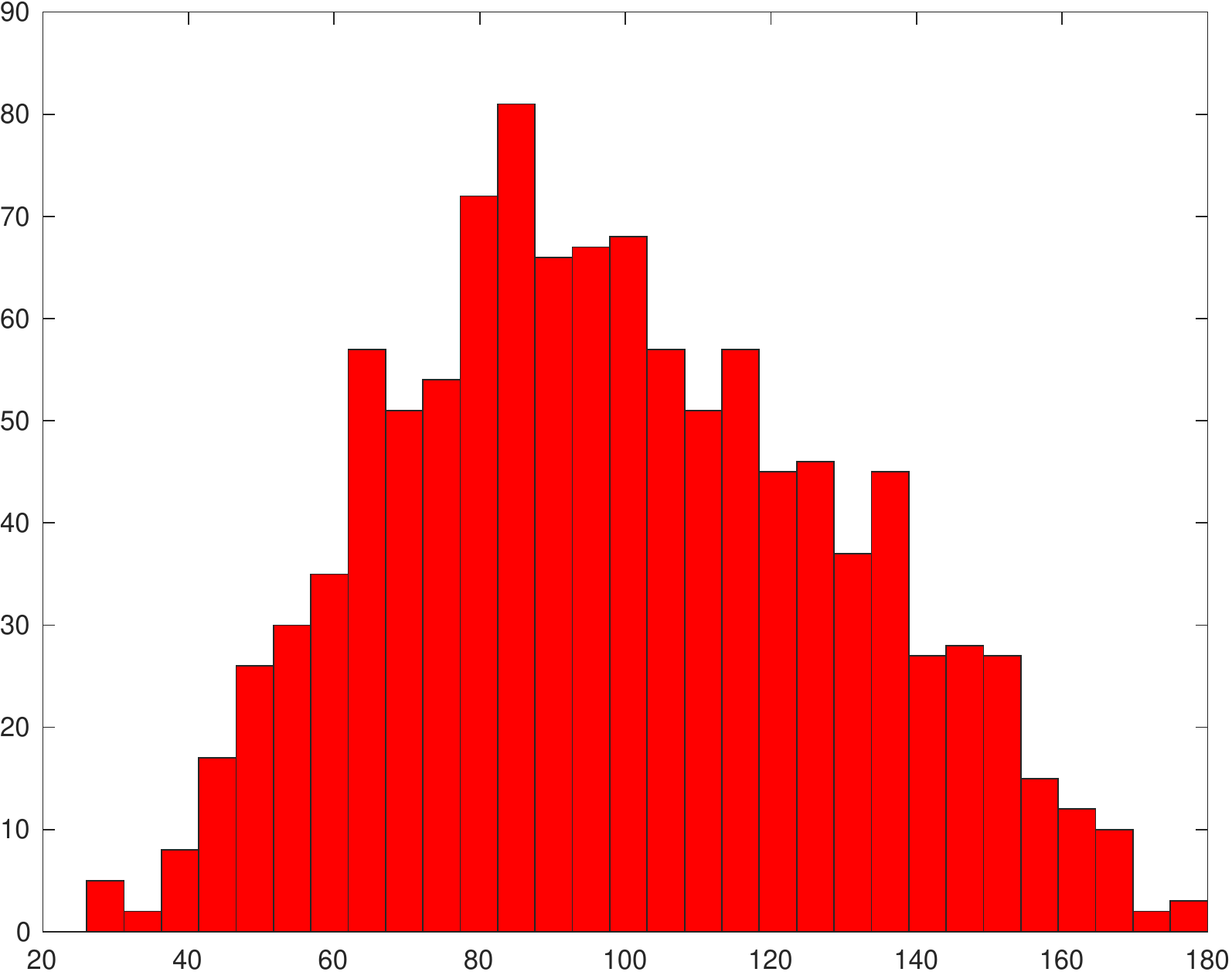}}
\subcaptionbox[ ]{ Degree distribution $G_2$
}[ 0.26\textwidth ]
{\includegraphics[width=0.26\textwidth] {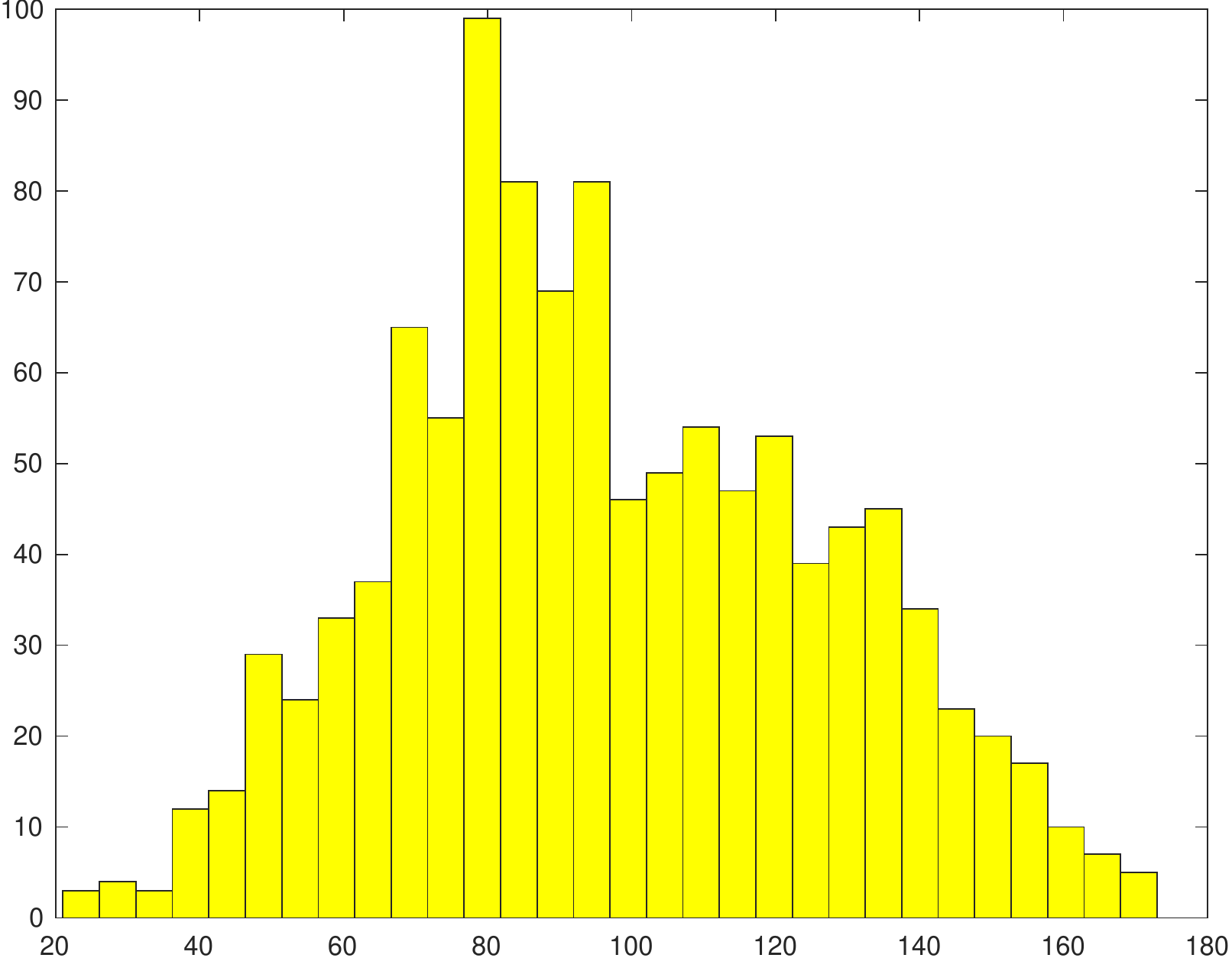}}
%
\subcaptionbox[ ]{ Heatmap of degrees in $G_1$ 
}[ 0.26\textwidth ]
{\includegraphics[width=0.26\textwidth] {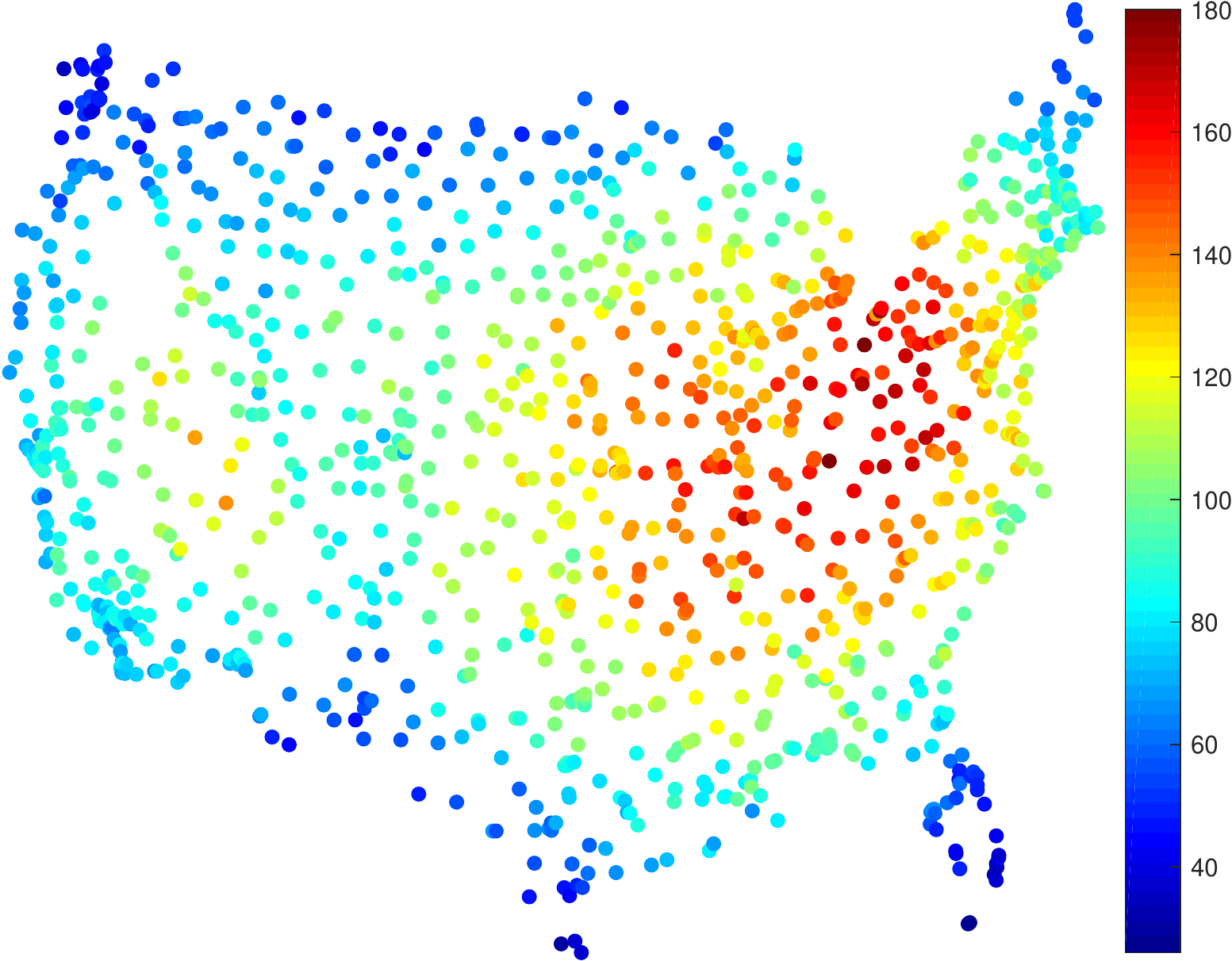} }
%
\subcaptionbox[ ]{ Heatmap of degrees in $G_2$ 
}[ 0.26\textwidth ]
{\includegraphics[width=0.26\textwidth] {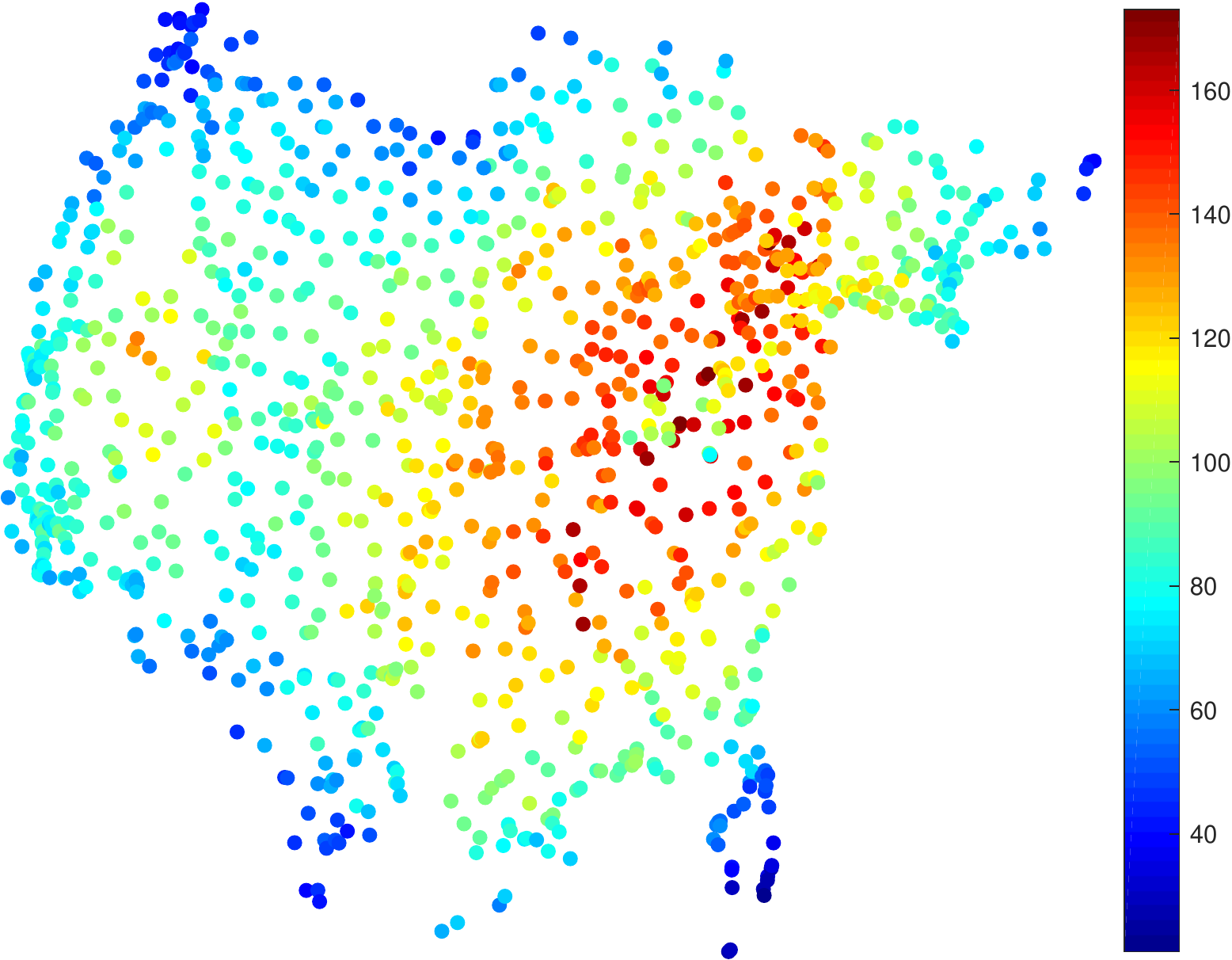} }
%
\captionsetup{width=0.99\textwidth}
\caption[Short Caption]{(a) Adjacency matrix of the measurement graph $G$, highlighting the decomposition $G =G_1 $ (red) $+ G_2$ (yellow). Plots (b) and (c) show the histogram of degrees in $G_1$ and $G_2$. The embeddings of $X$ and $Y$ are colored by their corresponding degree, i.e., each node $i$ in the embedding (that lies at the center of patch $P_i$) is colored by the number of other patches it overlaps with. 
}
\label{fig:US__G1_G2}
\vspace{-2mm}
\end{figure}

The top left plot of Figure \ref{fig:US__G1_G2} shows the adjacency matrix of the measurement graph $G$, highlighting the decomposition $G =G_1 \cup G_2$, with edges of $G_1$, respectively $G_2$, are colored in red, respectively yellow. The rightmost plots in the top row of the same Figure \ref{fig:US__G1_G2} show the histogram of degrees in the patch graphs $G_1$ and $G_2$ motivating the normalized version $D^{-1} H$.

\begin{figure}[!ht]
\vspace{-3mm}
\hspace{1.6cm} $\hat{X}$  \hspace{3cm}   $\hat{X} \;\; vs. \;\; X$   \hspace{2.6cm} $\hat{Y}$  \hspace{3.2cm}   $\hat{Y} \;\; vs. \;\; Y$ 

\centering
%
%
\subcaptionbox[ ]{$ \sigma = 0 $
}[ 1\textwidth ]
{\includegraphics[width=0.21\textwidth, trim=0.2cm 0.4cm 0.4cm 1.8cm,clip] {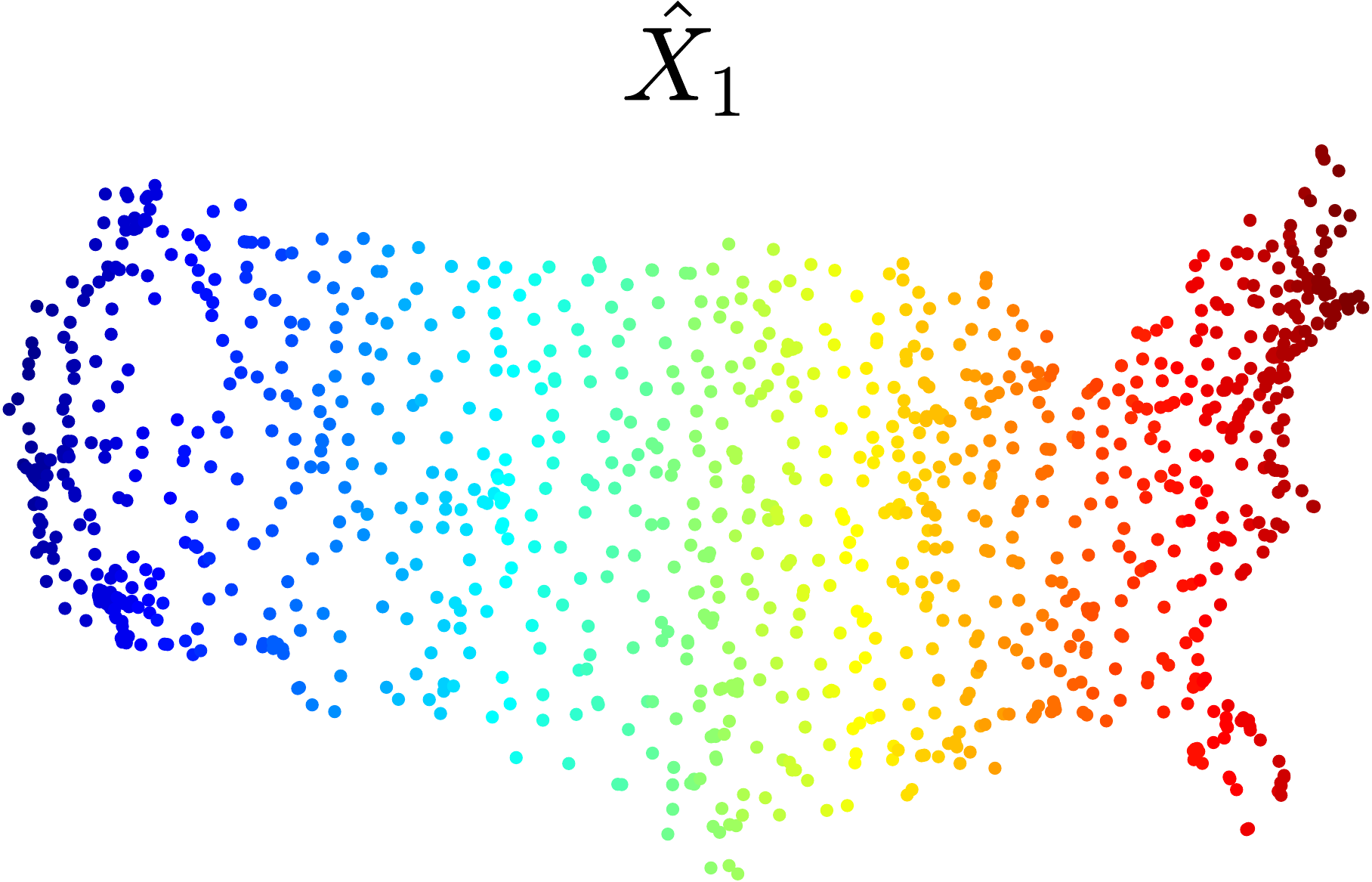}
\includegraphics[width=0.27\textwidth, trim=0.2cm 1cm 0.4cm 1.8cm,clip] {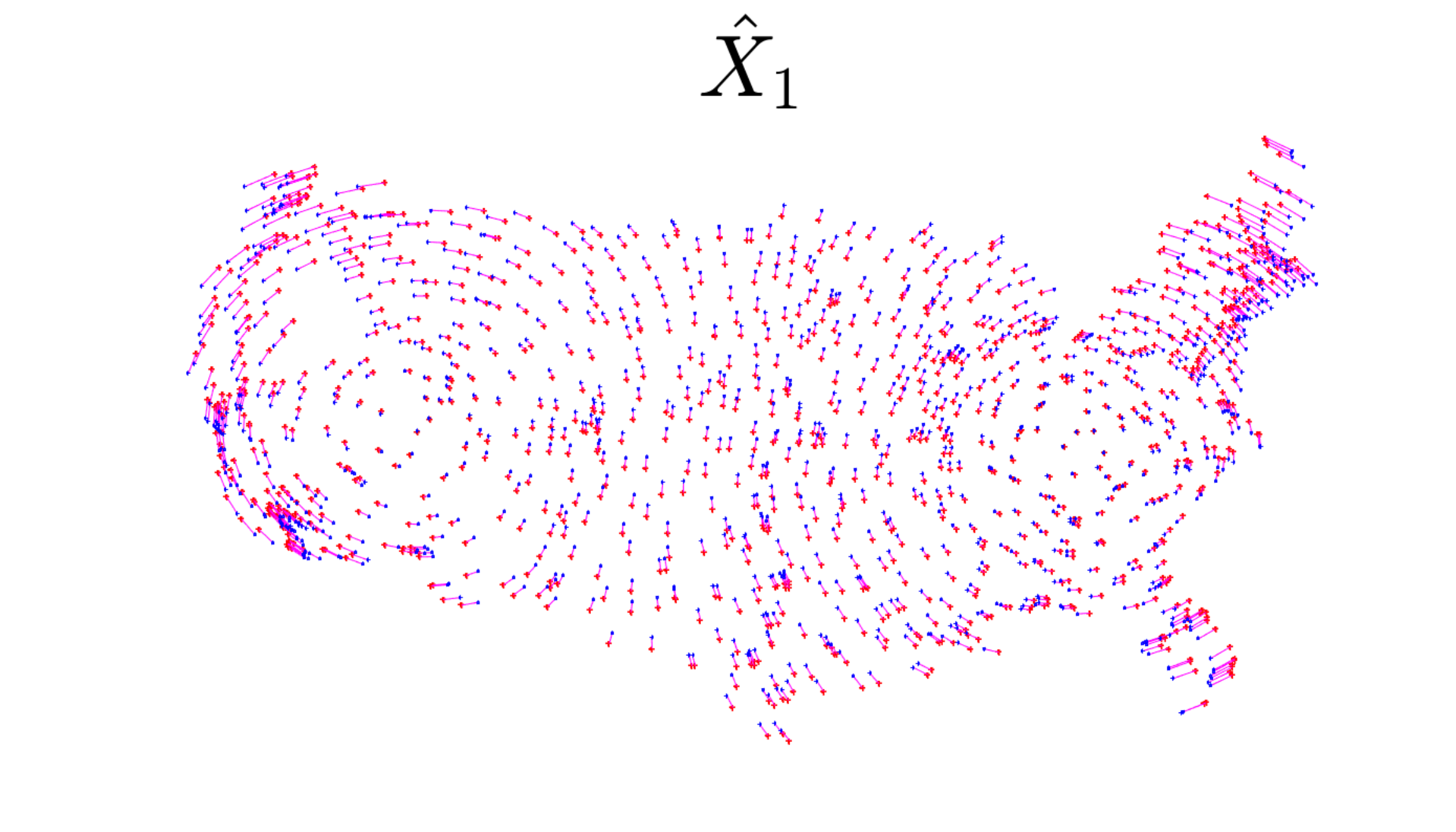}
\includegraphics[width=0.21\textwidth, trim=0.2cm 0.4cm 0.4cm 1.8cm,clip] {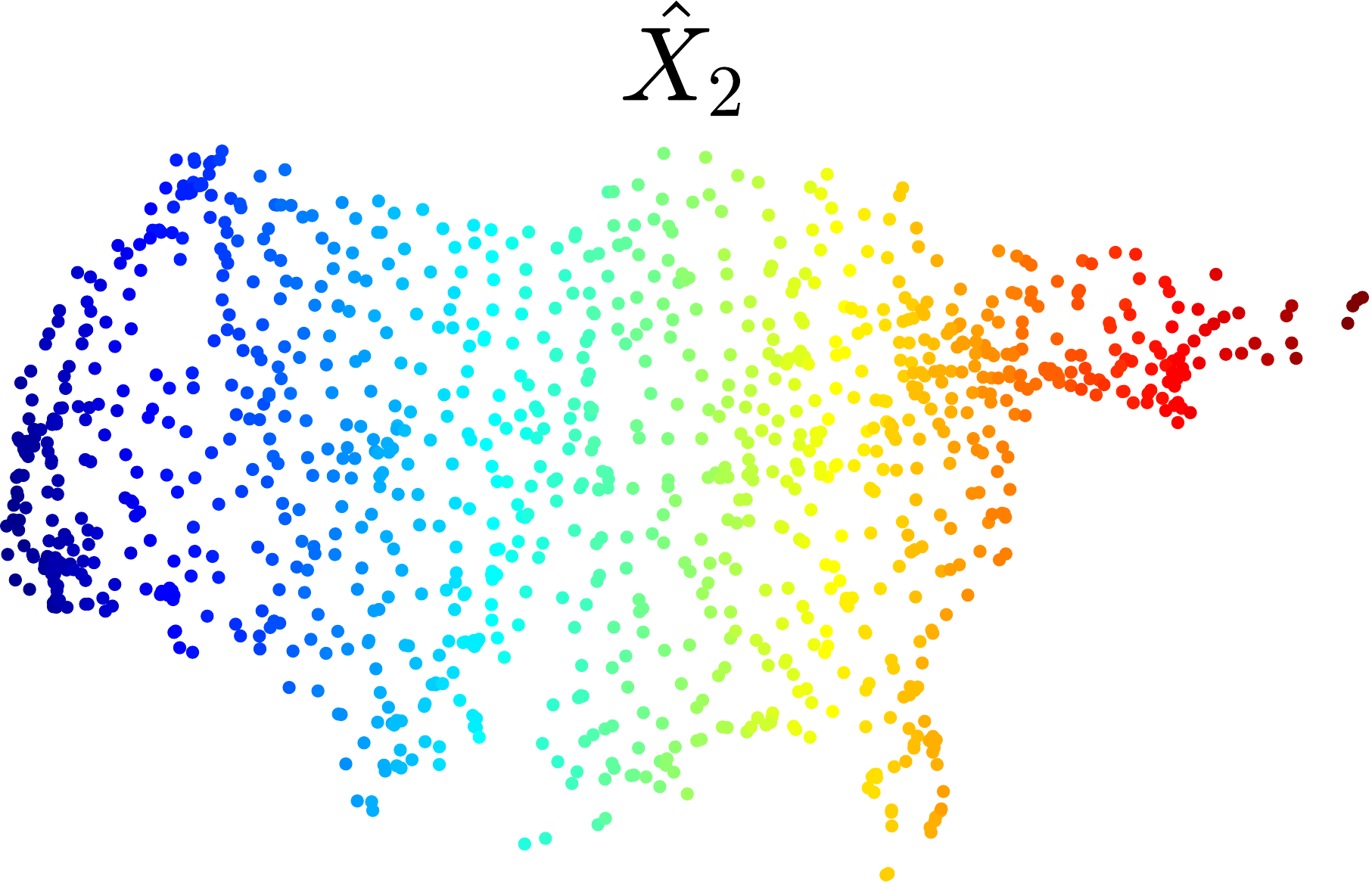}
\includegraphics[width=0.27\textwidth, trim=0.2cm 1cm 0.4cm 1.8cm,clip] {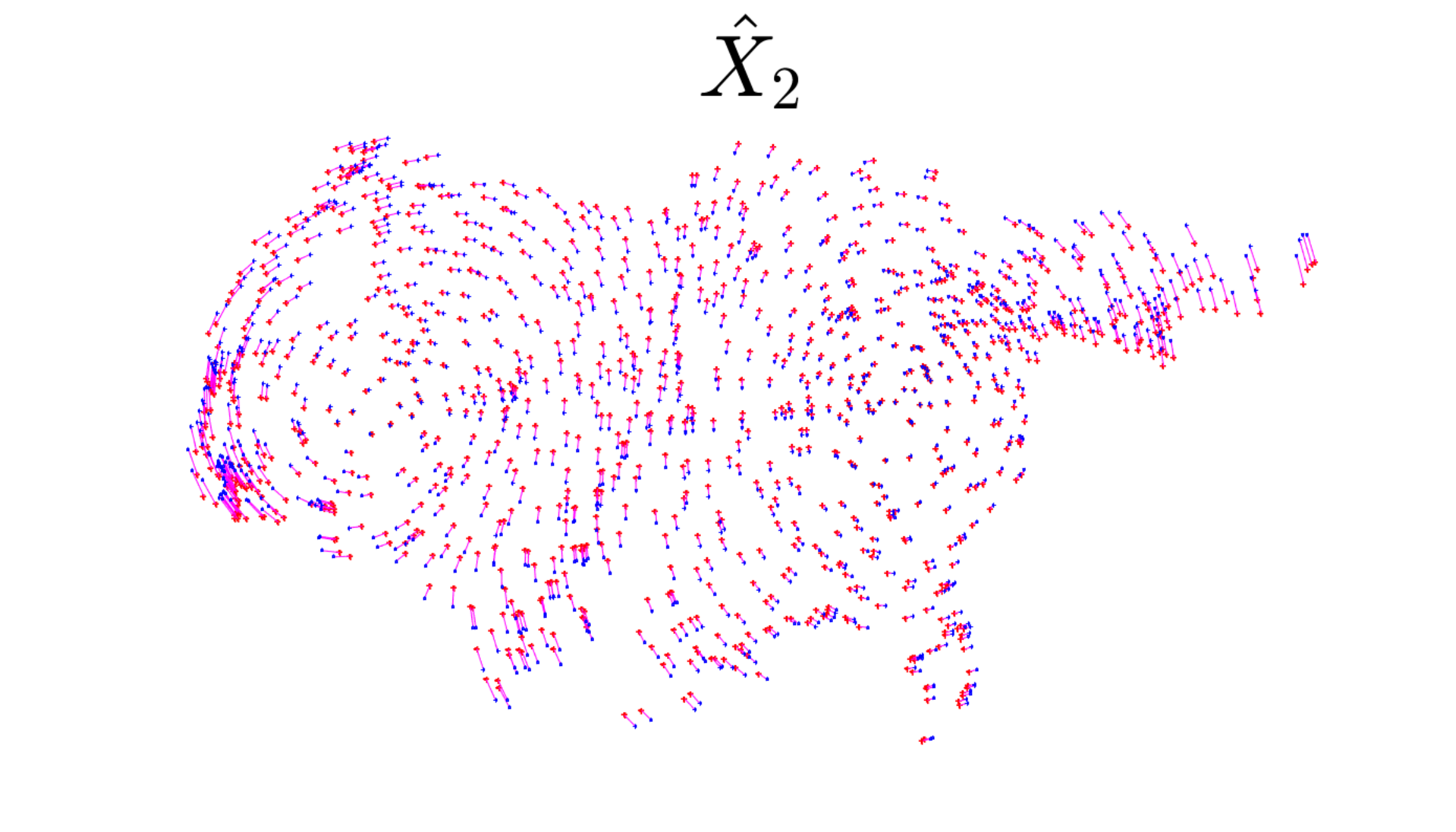}  
}  
%
%
%
\subcaptionbox[ ]{  $ \sigma = 0.20 $
}[ 1\textwidth ]
{\includegraphics[width=0.21\textwidth, trim=0.2cm 0.4cm 0.4cm 1.8cm,clip] {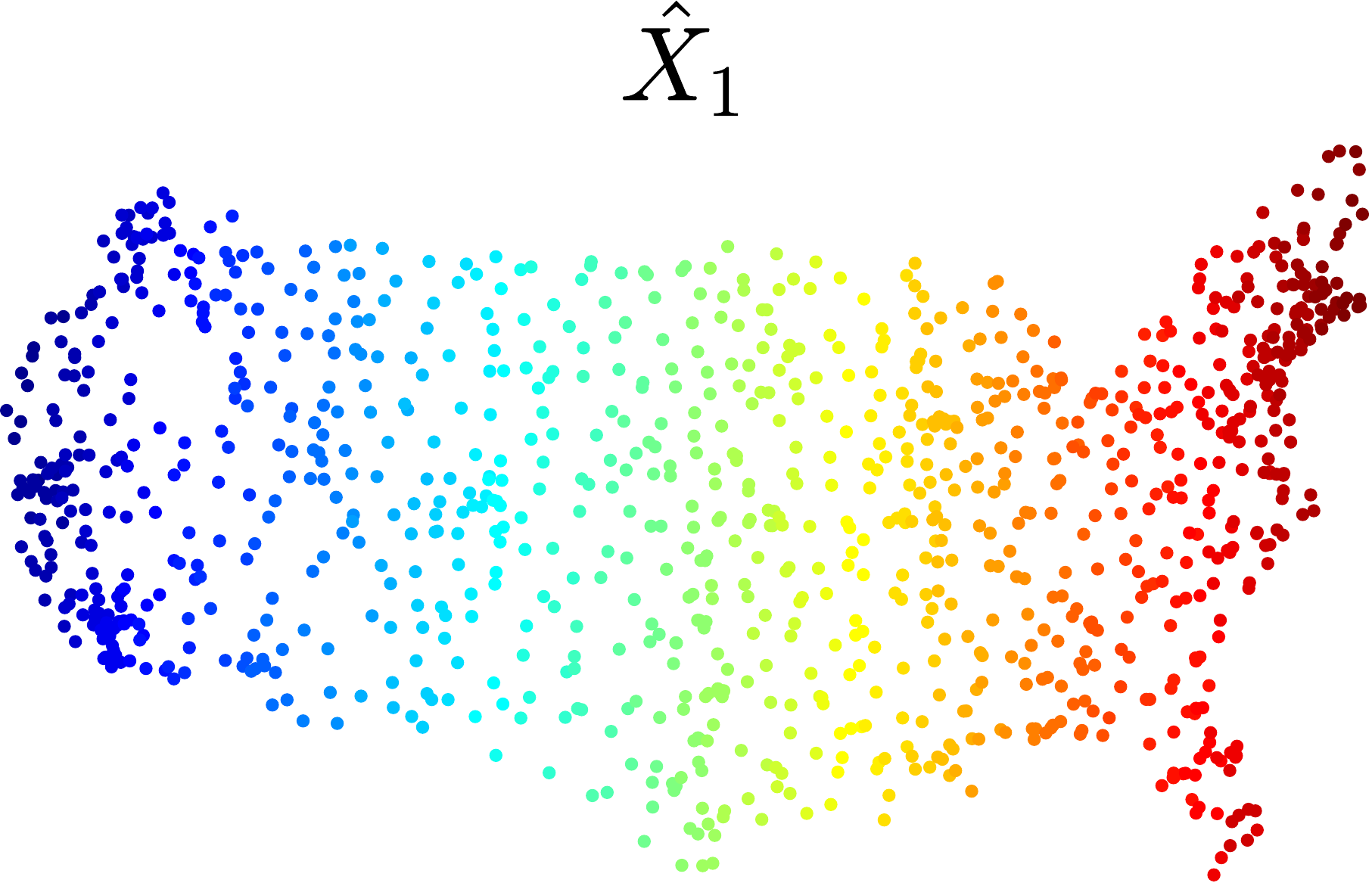}
\includegraphics[width=0.27\textwidth, trim=0.2cm 1cm 0.4cm 1.8cm,clip] {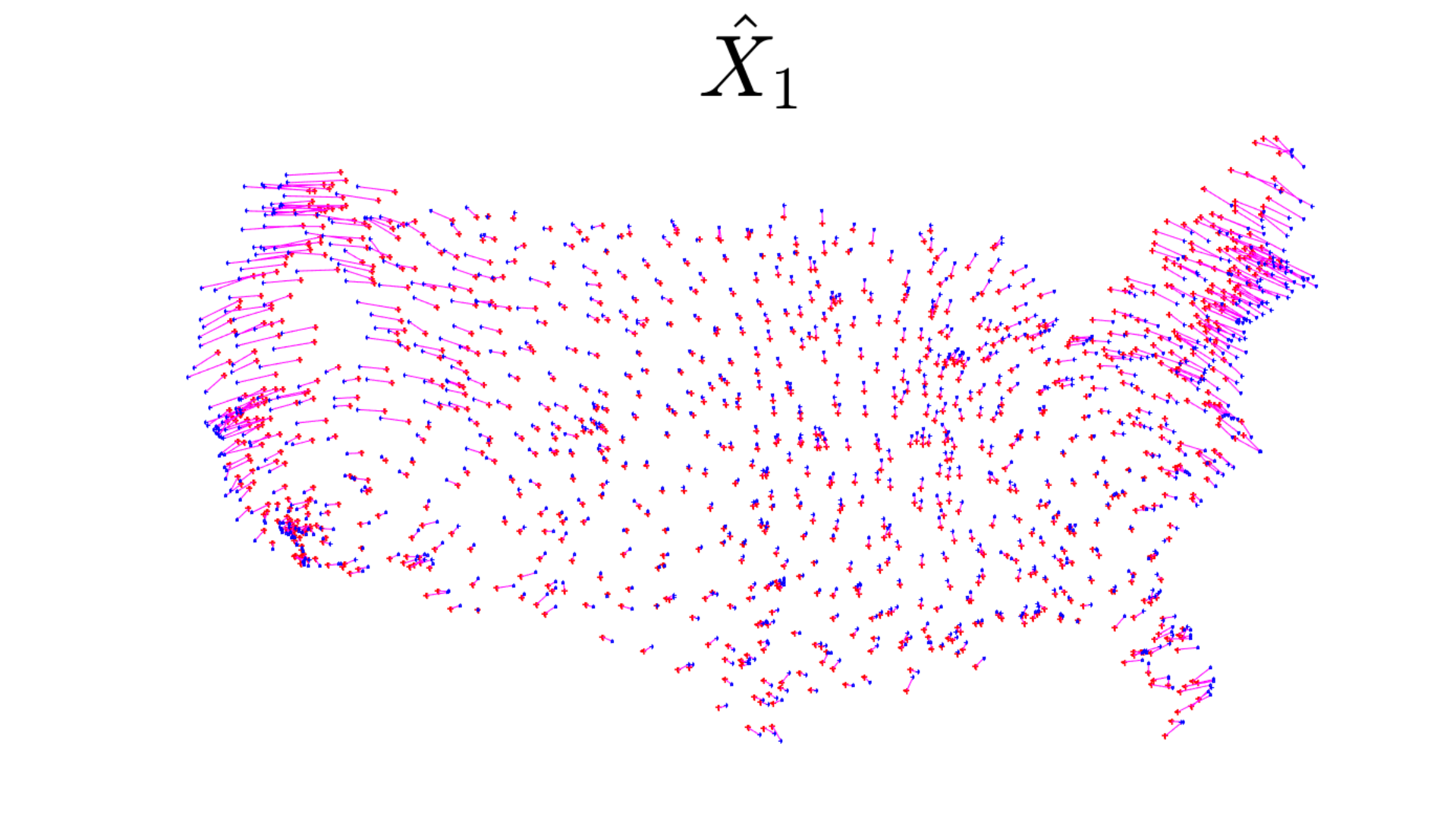}
\includegraphics[width=0.21\textwidth, trim=0.2cm 0.4cm 0.4cm 1.8cm,clip] {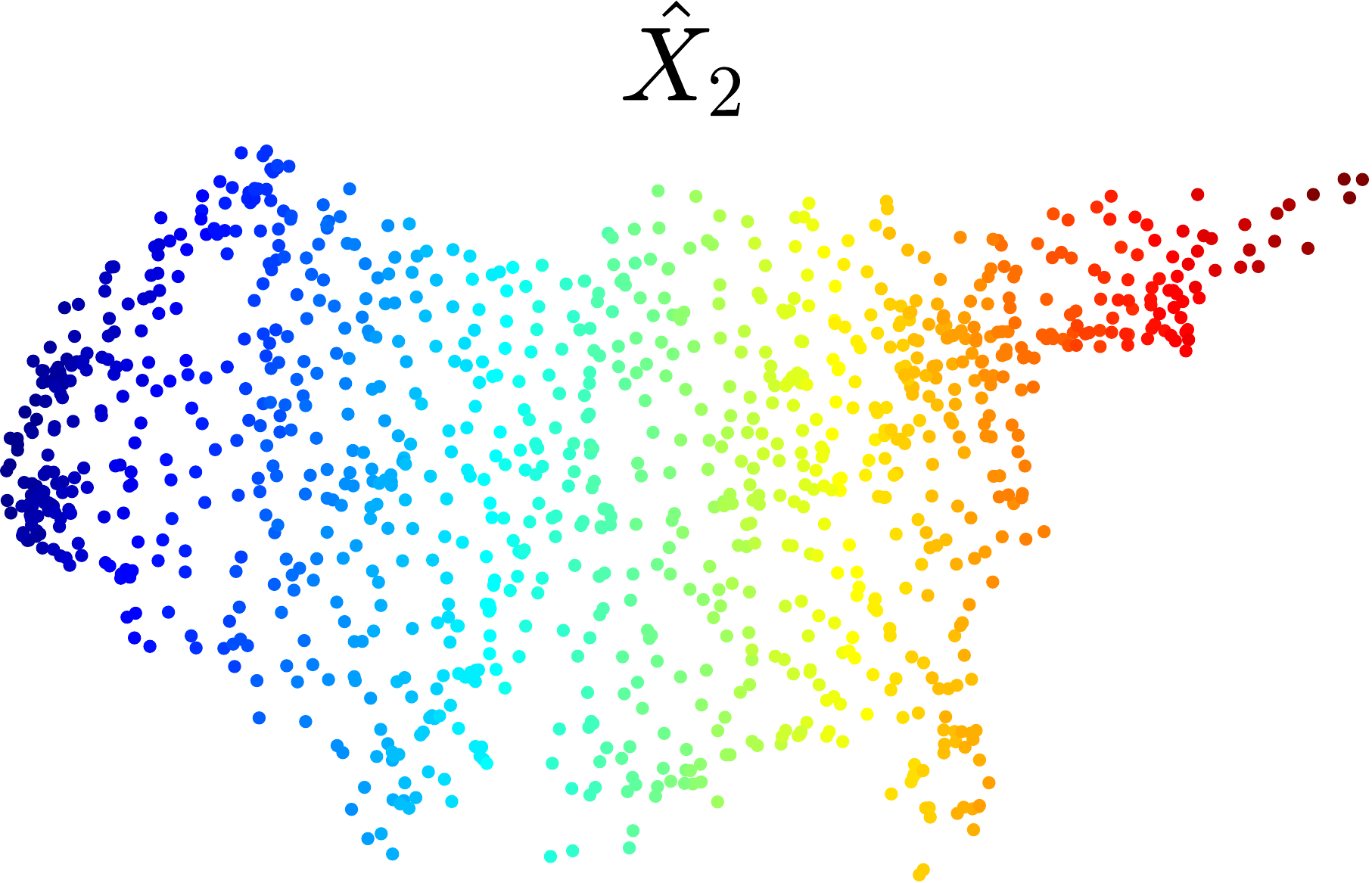}
\includegraphics[width=0.27\textwidth, trim=0.2cm 1cm 0.4cm 1.8cm,clip] {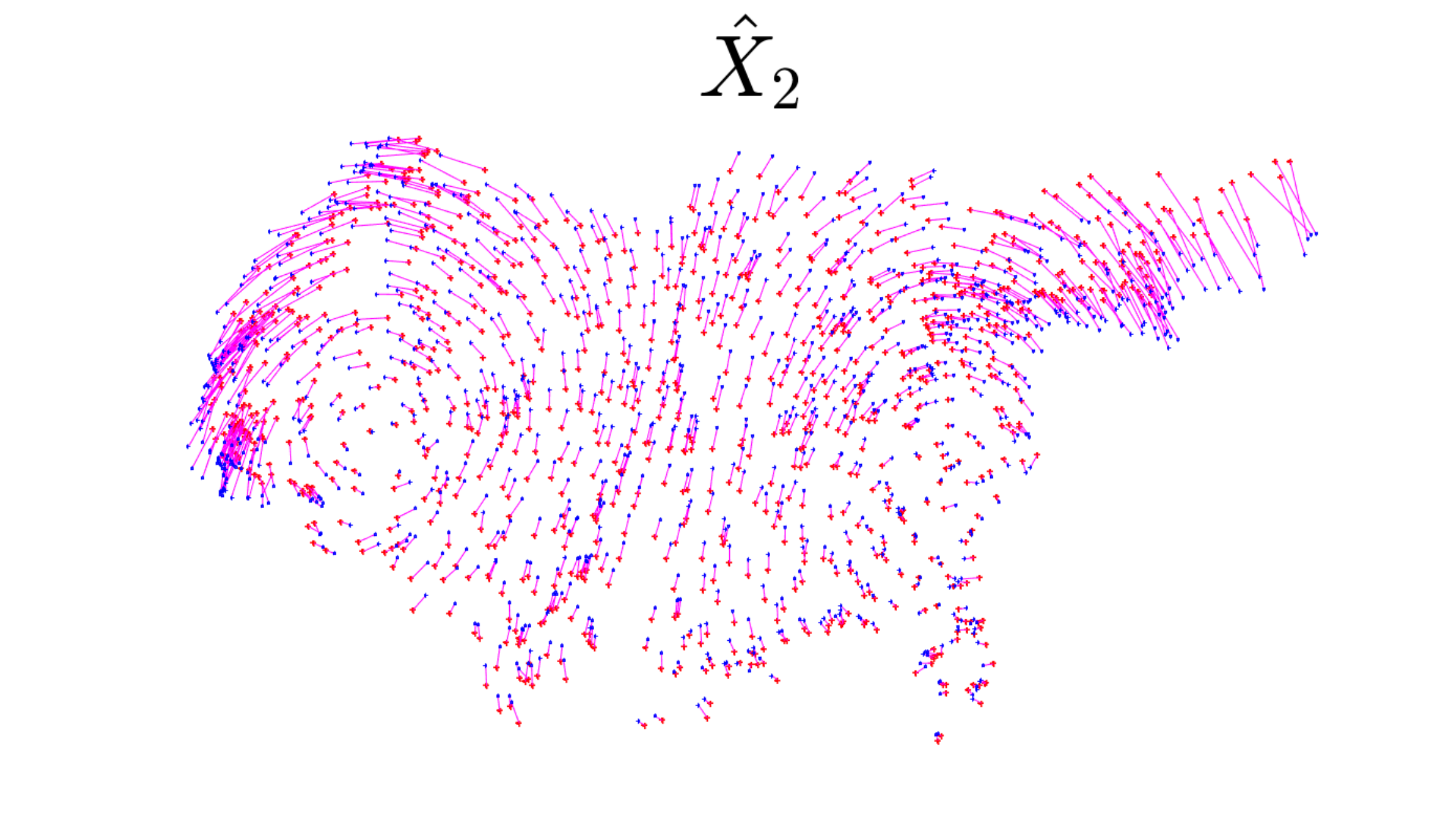}  
}  
%
%
\subcaptionbox[ ]{ $ \sigma = 0.40 $
}[ 1\textwidth ]
{\includegraphics[width=0.21\textwidth, trim=0.2cm 0.4cm 0.4cm 1.8cm,clip] {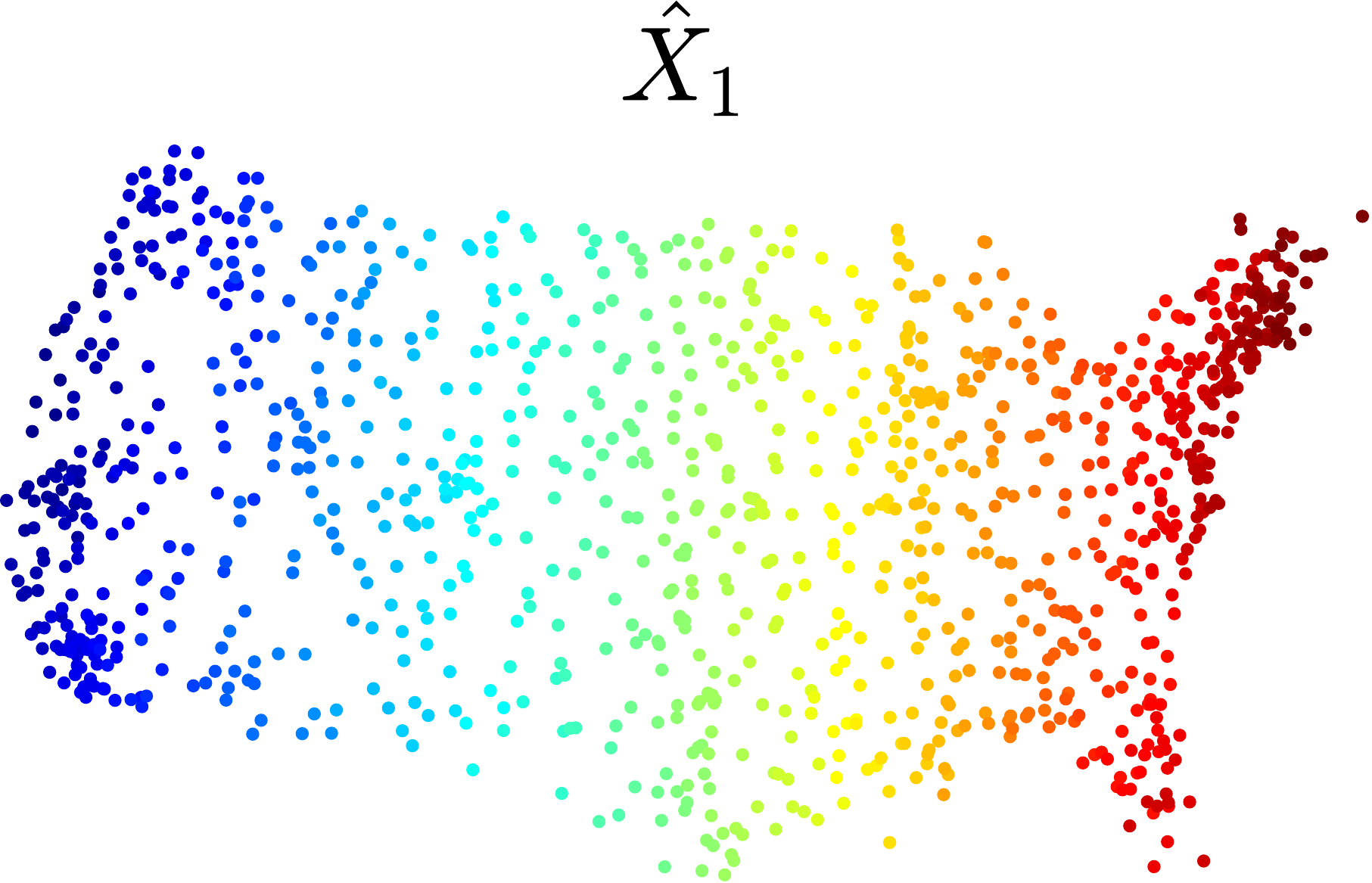}
\includegraphics[width=0.27\textwidth, trim=0.2cm 1cm 0.4cm 1.8cm,clip] {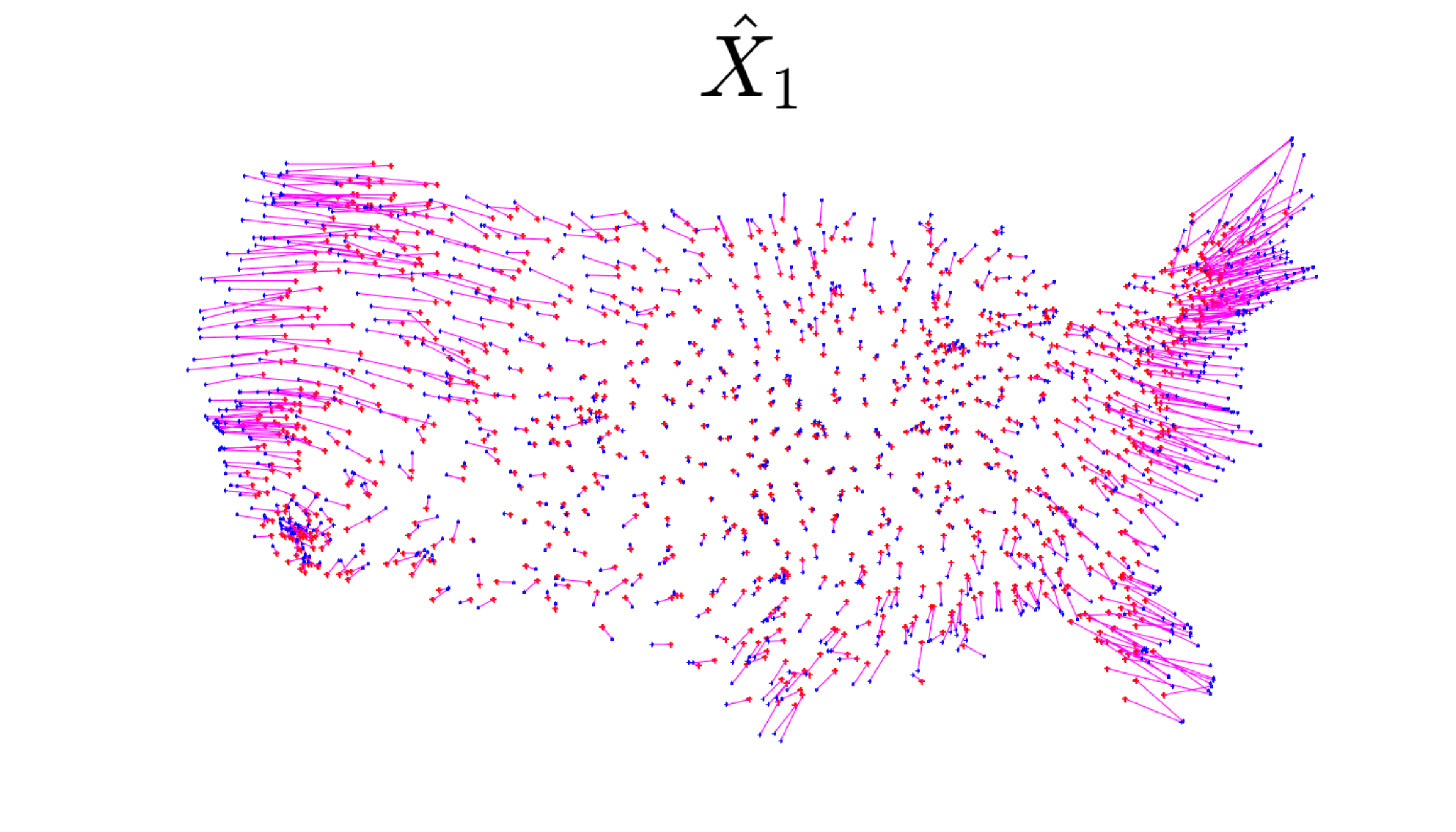}
\includegraphics[width=0.21\textwidth, trim=0.2cm 0.4cm 0.4cm 1.8cm,clip] {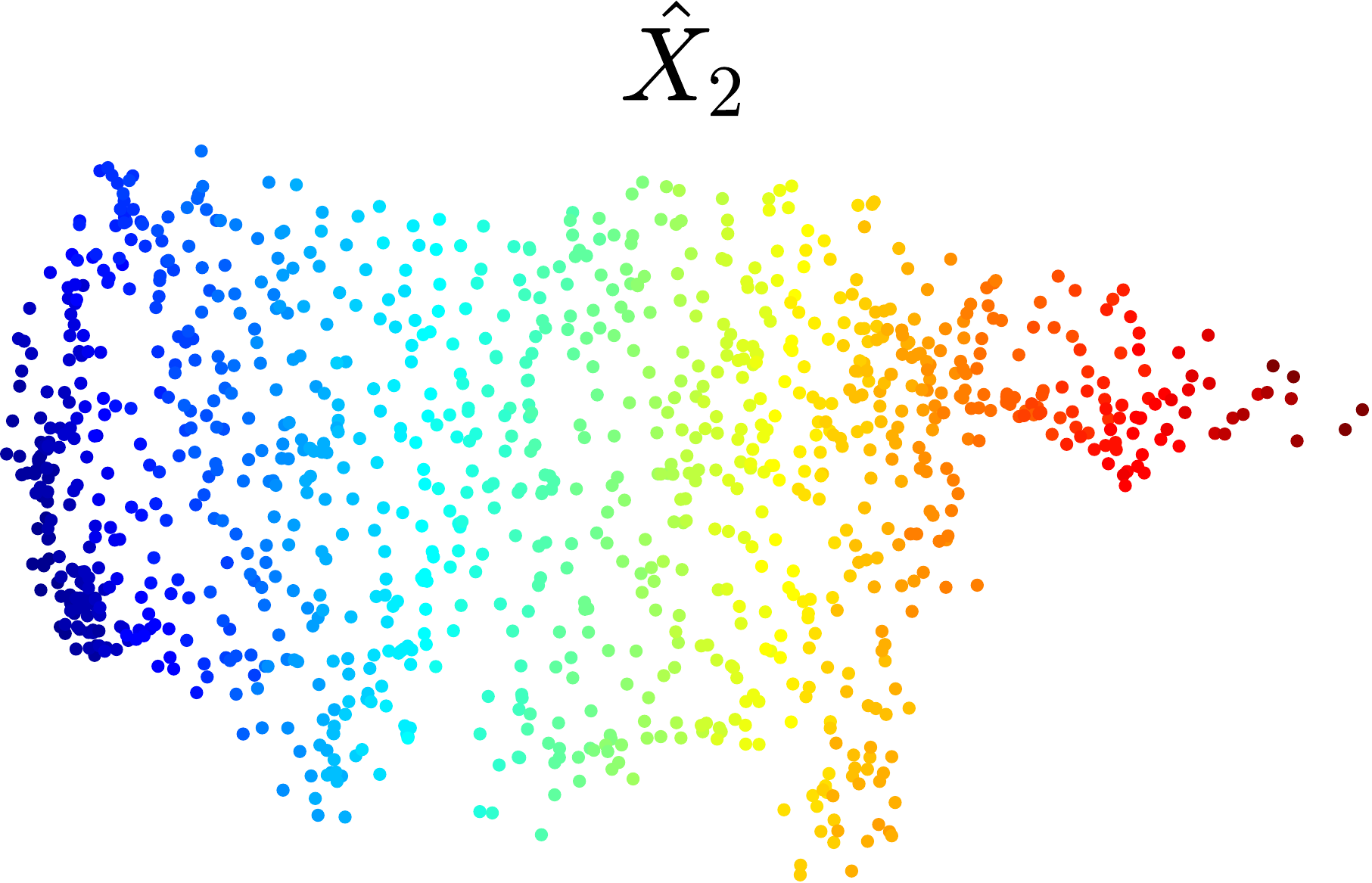}
\includegraphics[width=0.27\textwidth, trim=0.2cm 1cm 0.4cm 1.8cm,clip] {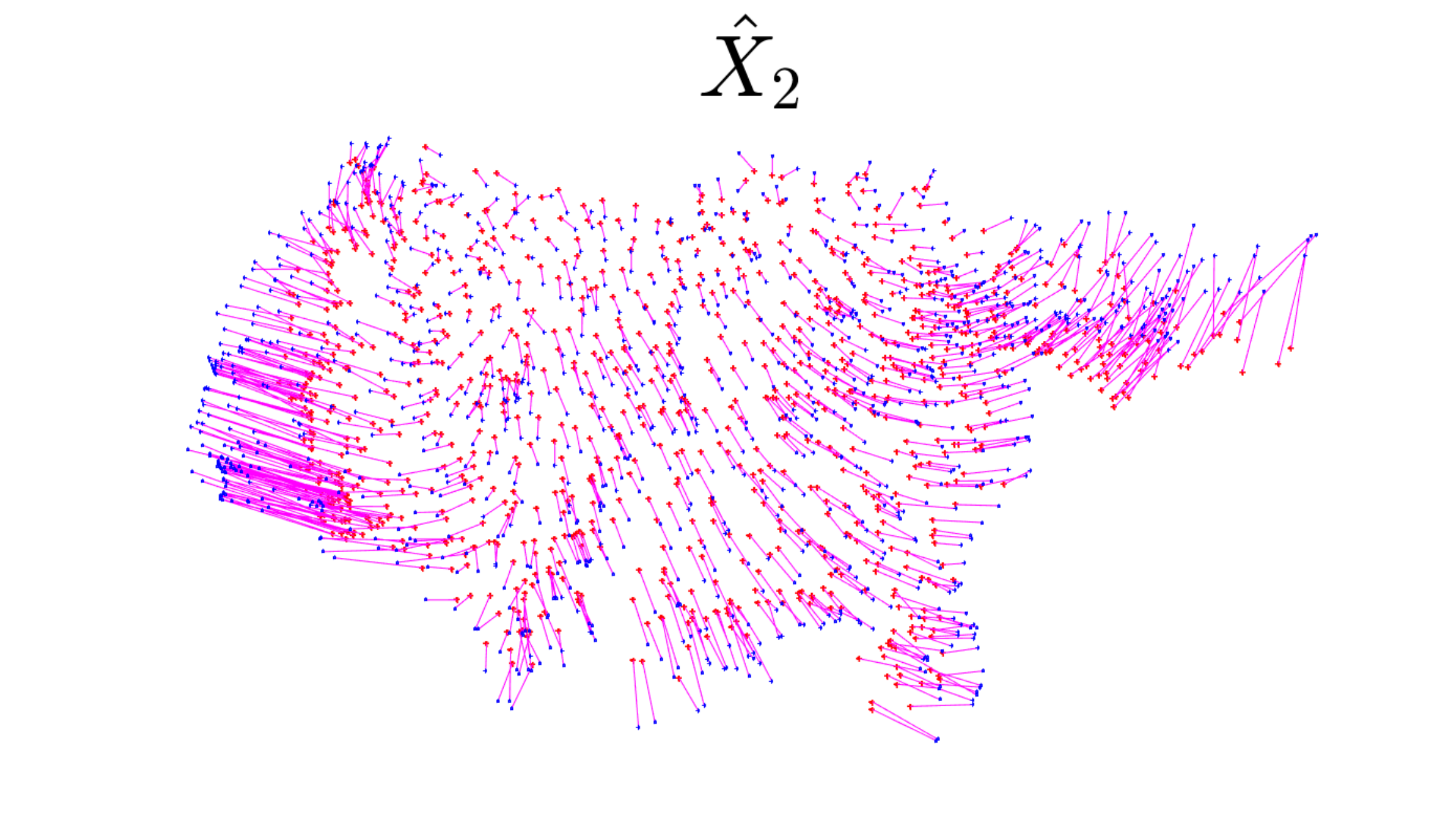}  
}  
%
%
\subcaptionbox[ ]{ $ \sigma = 0.60 $
}[ 1\textwidth ]
{\includegraphics[width=0.21\textwidth, trim=0.2cm 0.4cm 0.4cm 1.8cm,clip] {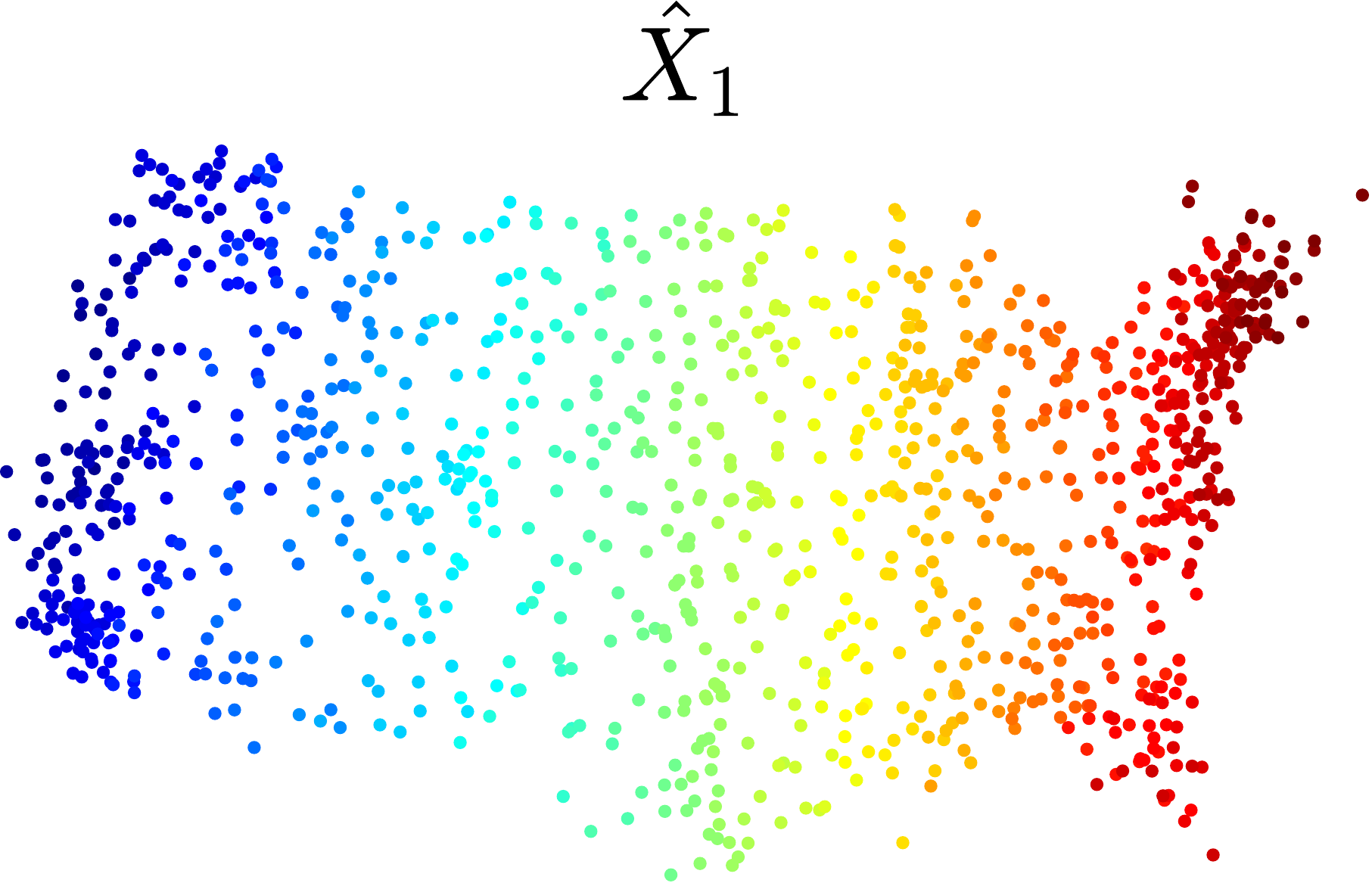}
\includegraphics[width=0.27\textwidth, trim=0.2cm 1cm 0.4cm 1.8cm,clip] {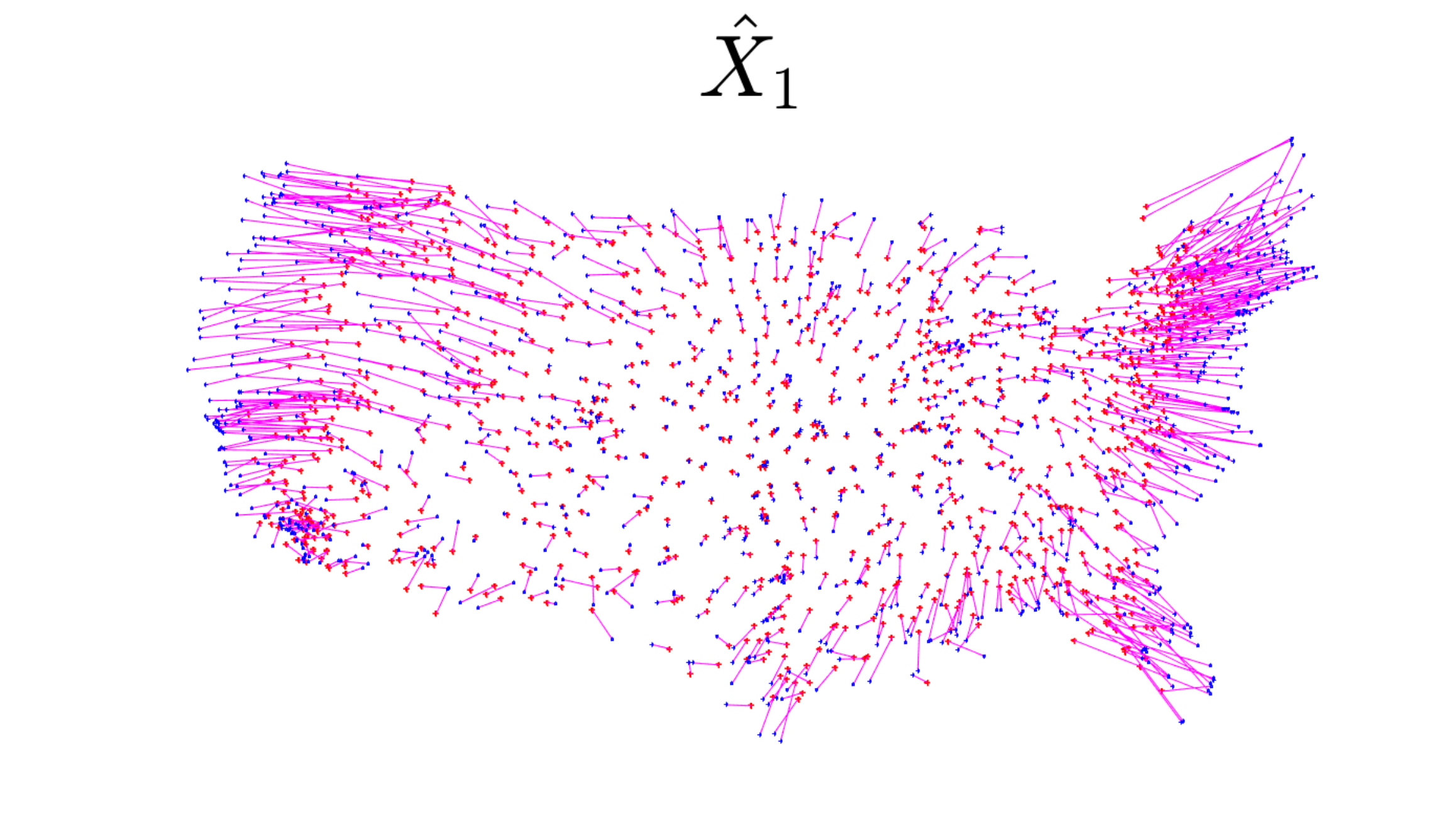}
\includegraphics[width=0.21\textwidth, trim=0.2cm 0.4cm 0.4cm 1.8cm,clip] {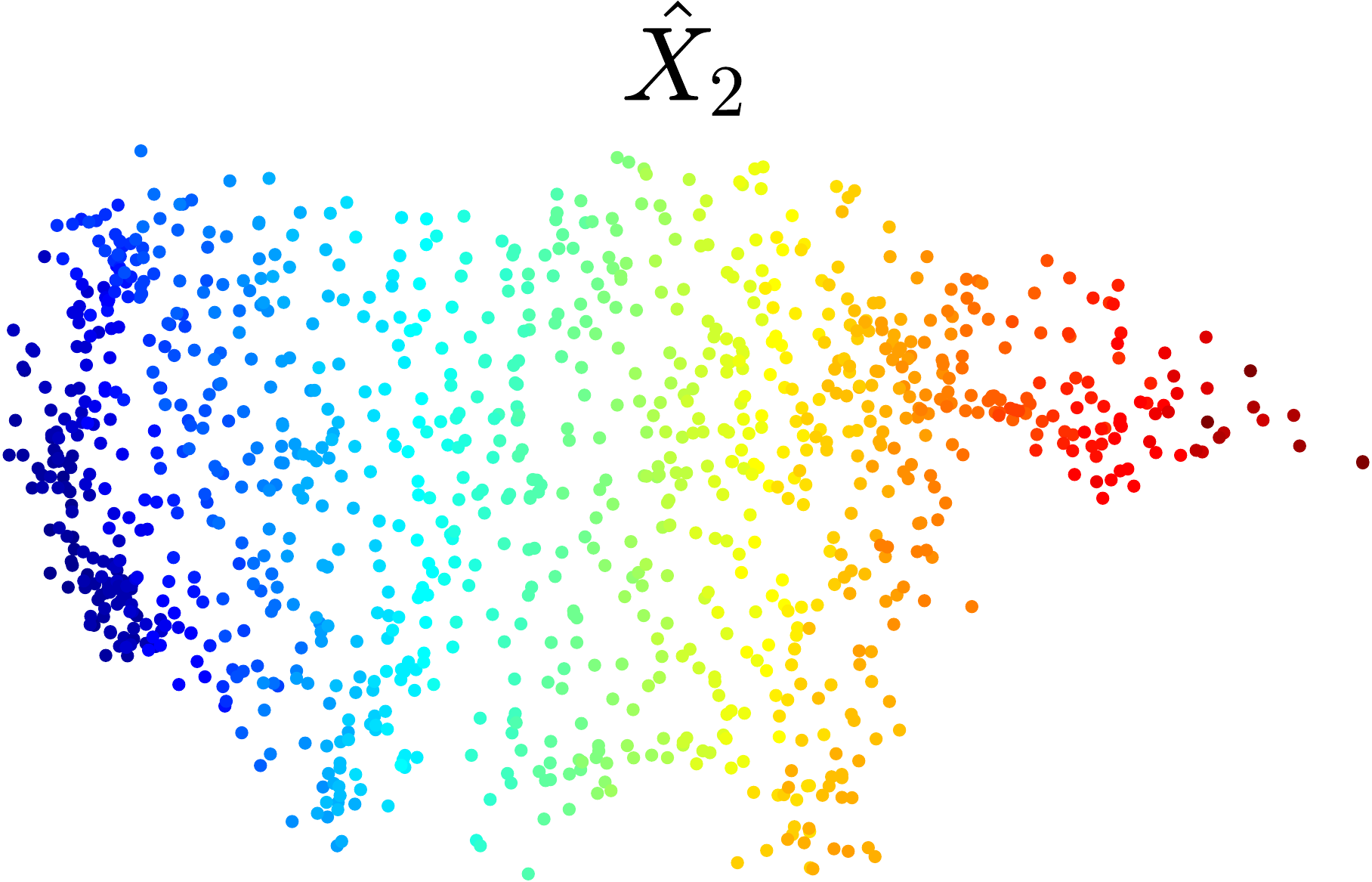}
\includegraphics[width=0.27\textwidth, trim=0.2cm 1cm 0.4cm 1.8cm,clip] {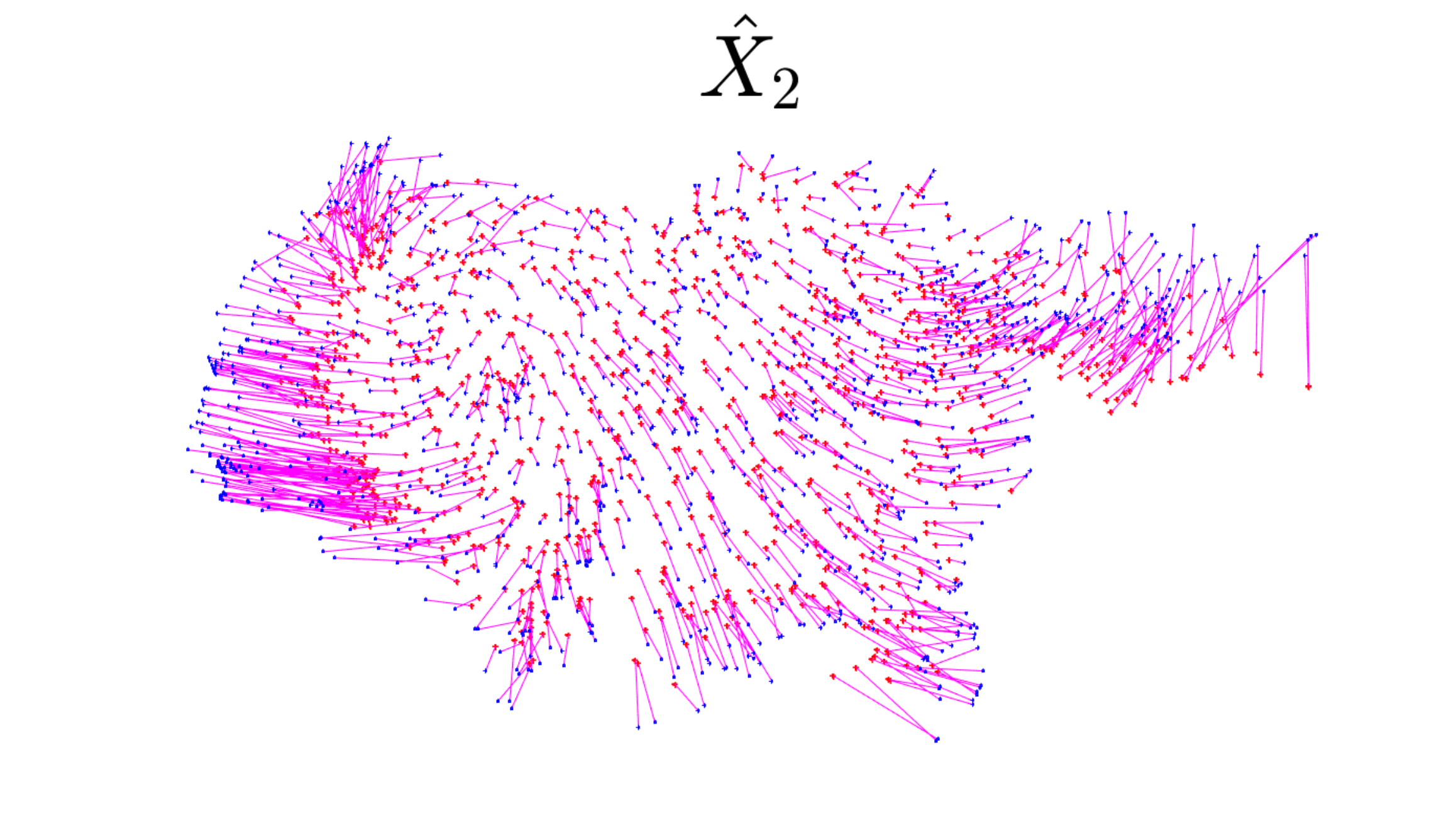}  
}  
\subcaptionbox[ ]{ $ \sigma = 0.80 $
}[ 1\textwidth ]
{\includegraphics[width=0.21\textwidth, trim=0.2cm 0.4cm 0.4cm 1.8cm,clip] {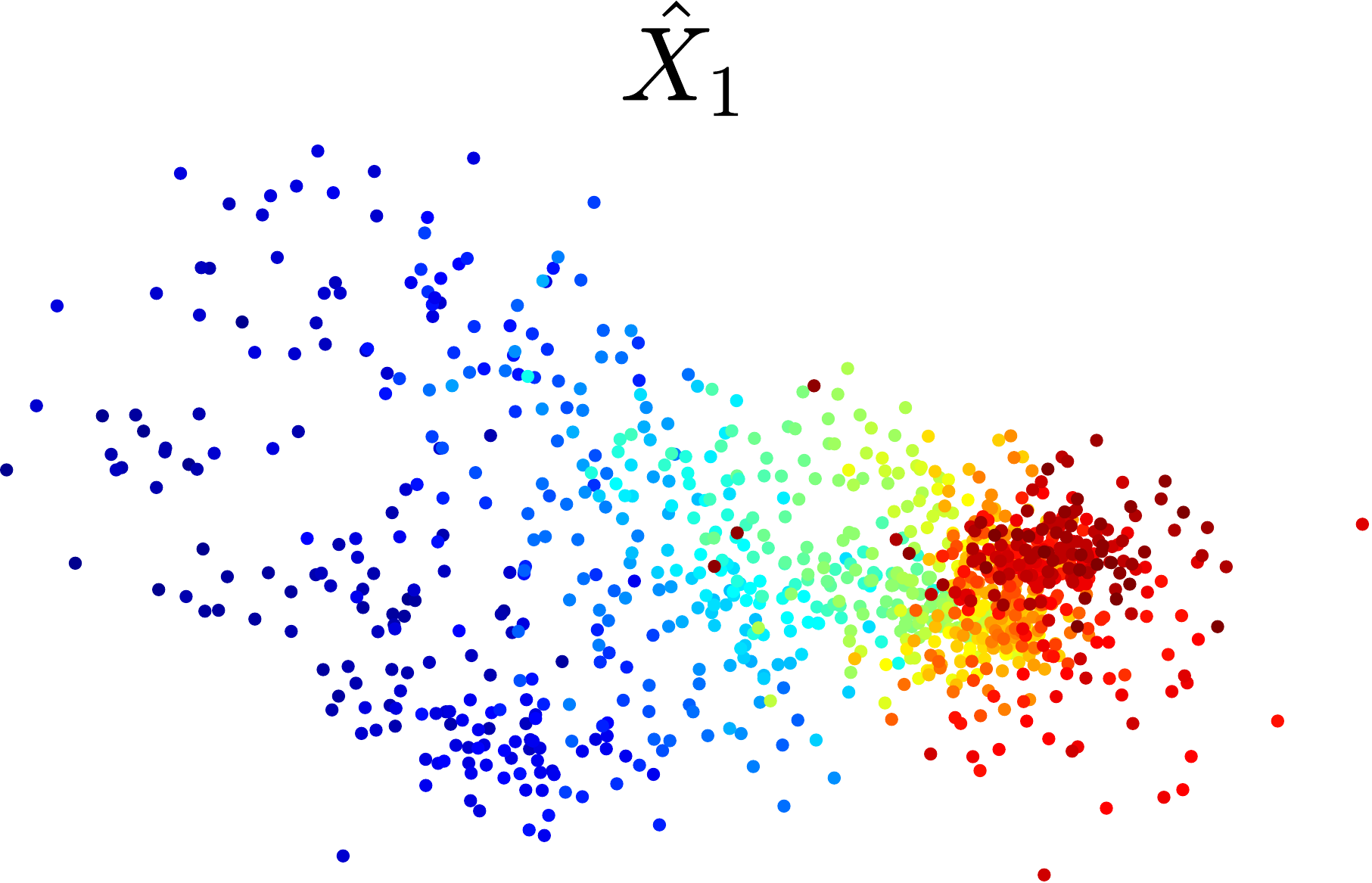}
\includegraphics[width=0.27\textwidth, trim=0.2cm 1cm 0.4cm 1.8cm,clip] {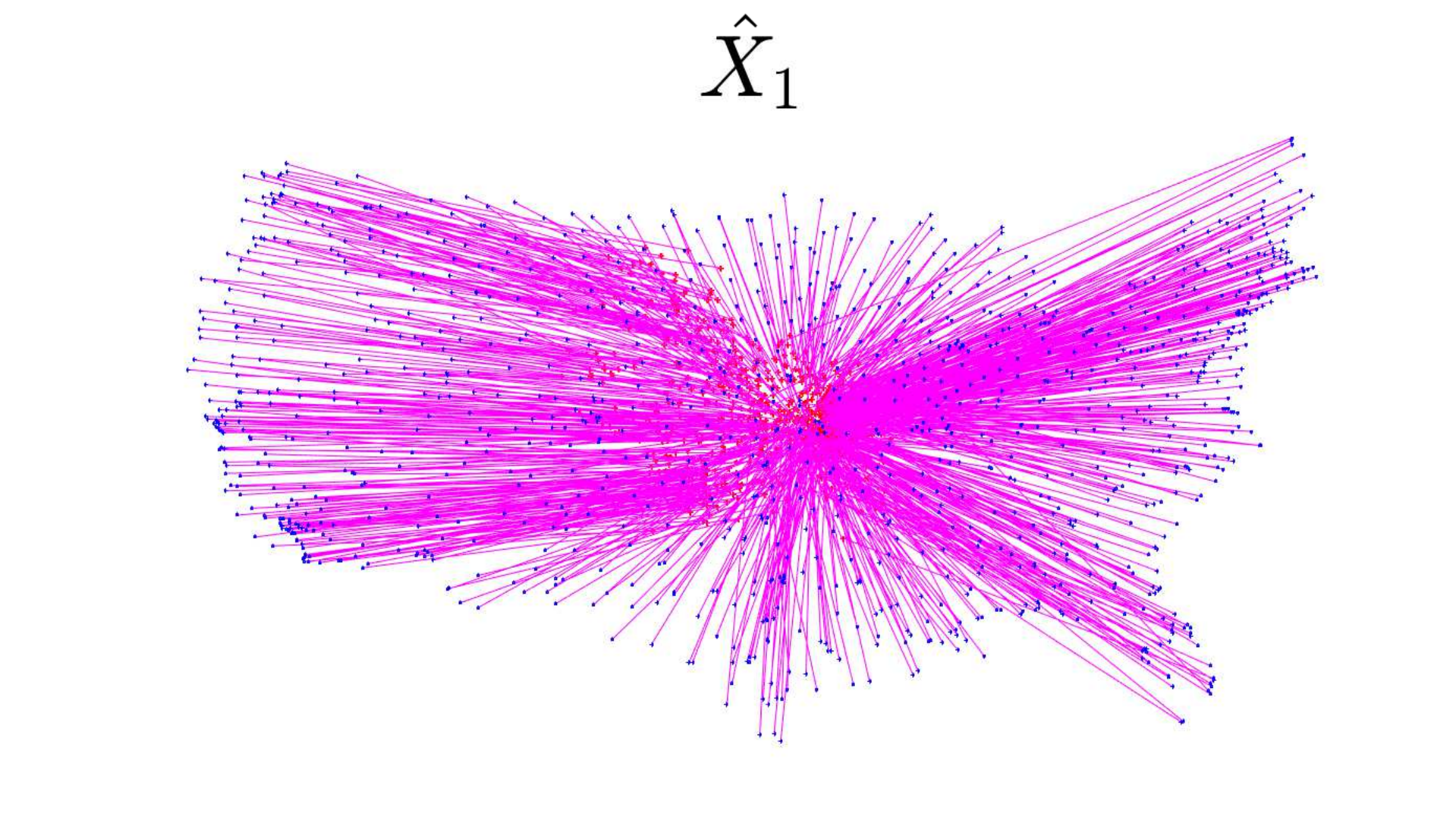}
\includegraphics[width=0.21\textwidth, trim=0.2cm 0.4cm 0.4cm 1.8cm,clip] {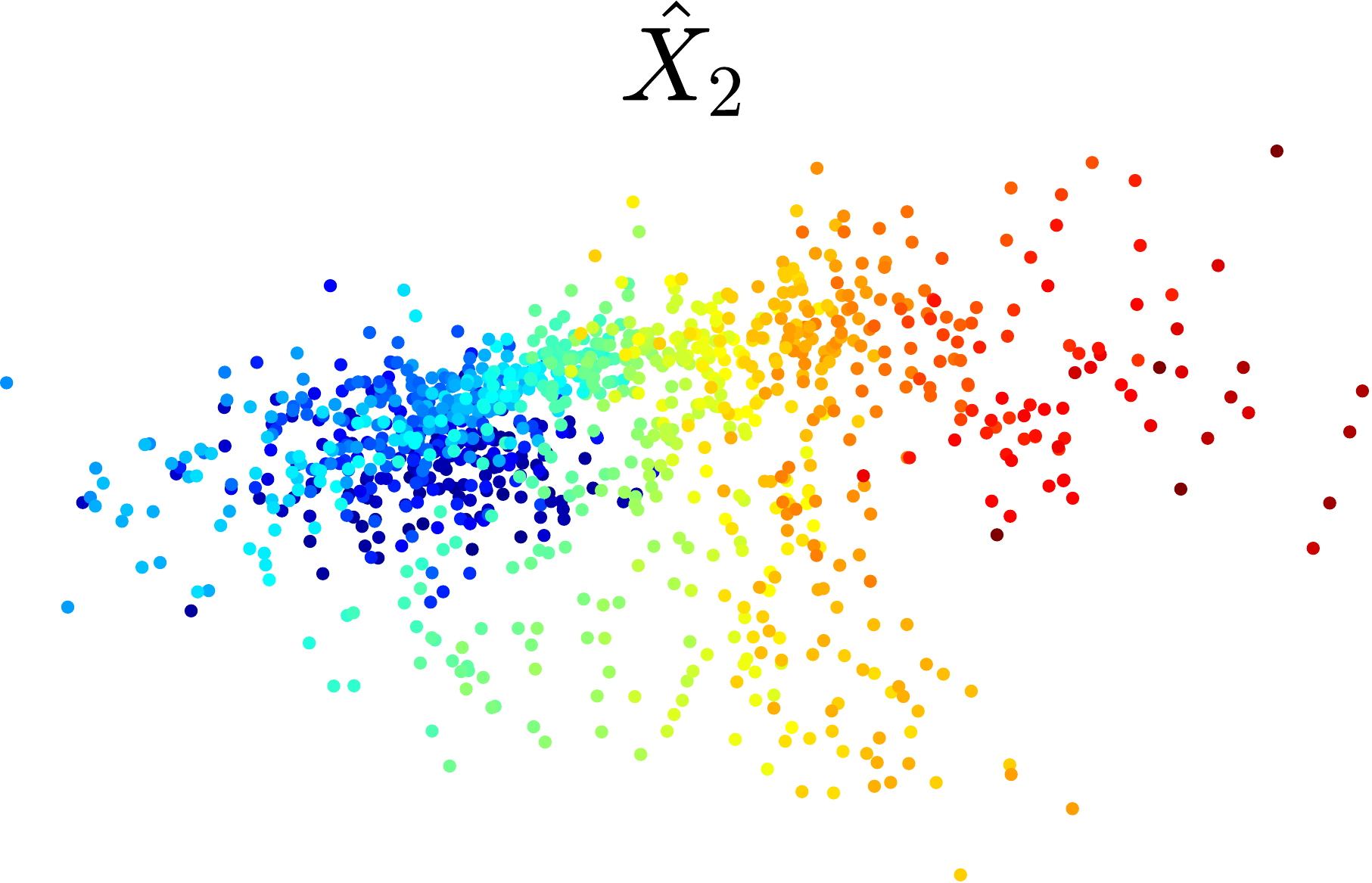}
\includegraphics[width=0.27\textwidth, trim=0.2cm 1cm 0.4cm 1.8cm,clip] {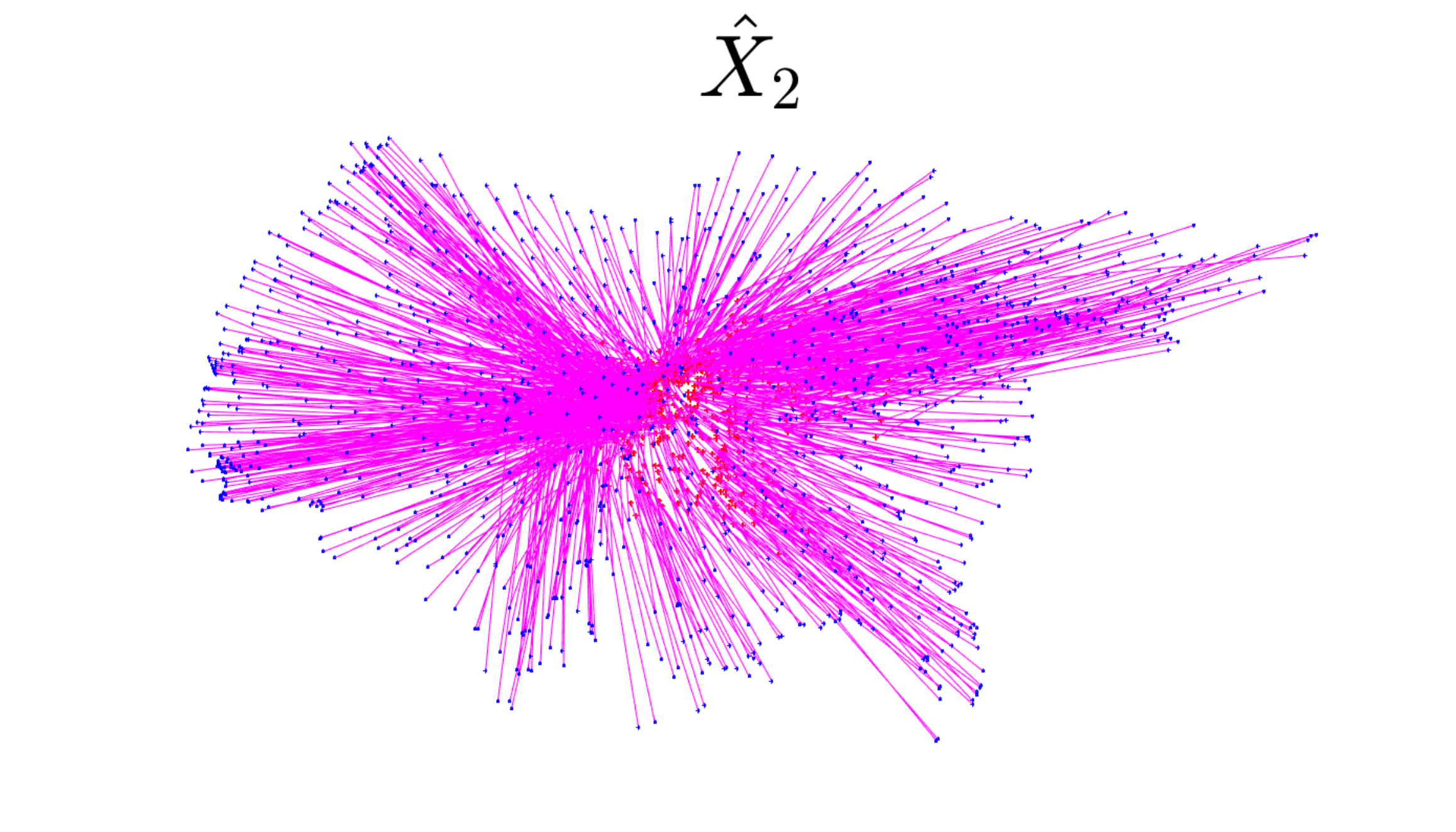}  
}  
%
\vspace{-2mm}
\captionsetup{width=0.98\linewidth}
\caption[Short Caption]{Recovery of the two estimated  embeddings $\hat{X}$ and $\hat{Y}$ shown in the first and  third columns, where the coloring is given by the longitude.  The second column shows the Procrustes alignment between  the recovered embedding $\hat{X}$ (red) and the ground truth $X$ (blue), together with the displacement error bars. The forth column shows a similar visualization for $\hat{Y}$ and the ground truth $Y$. We vary the noise level $\sigma$ used in the perturbations of the local patch embeddings in both ensembles, across the following set of values $\sigma \in \{ 0, 0.2, 0.4, 0.6, 0.8 \}$. 
}  
\label{fig:US__noises_X1_X2}
\end{figure}

%% file: S_conclusion.tex

This paper considers an extension of the classical synchronization problem (of recovering a group of angles given a subset of noisy pairwise angle offsets) to the setting when \textit{multiple} latent groups of angles exist. We propose probabilistic generative models for $k$-synchronization, along with spectral algorithms which we theoretically analyze, under a suitably defined notion of $\delta$-orthogonality. 
There are various further extensions one could pursue. 

\vspace{0mm}
\paragraph{List-synchronization} An interesting direction, ongoing work similar in spirit to the setup explored in the present paper, is that of \textit{list-synchronization}, where there exists a single latent group of angles, $\theta_1, \ldots, \theta_n$, and for each available edge $\set{i,j}$ in the measurement graph $G$, one has available not only a single measurement proxy for the offset $\theta_i - \theta_j$, but rather a list of $k$ possible offsets, of which only a single one is correct (or approximately correct), while the remaining $k-1$ elements of the edge list are outliers, as depicted in Figure \ref{fig:listSync}. This setting arises for instance in the GRP problem, where in the high-noise regime, one may consider multiple candidates for the pairwise angle offset, when aligning pairs of overlapping patches via the method of choice (see for example Section 6 within \cite{asap2d}). This problem is also related to the 
\textsc{Max-2-Lin} $\text{mod}$ $L$ problem \cite{andersson2001new,nonUniqueGamesSync}, for maximizing the number of satisfied linear equations $ \text{mod} L$  with exactly two variables in each equation, also known to be an instance of unique games \cite{FeigeLovaszTwoProver}.

\vspace{0mm}
\begin{figure}[h!]
\vspace{-3mm}
\centering
\includegraphics[width=0.55\textwidth]{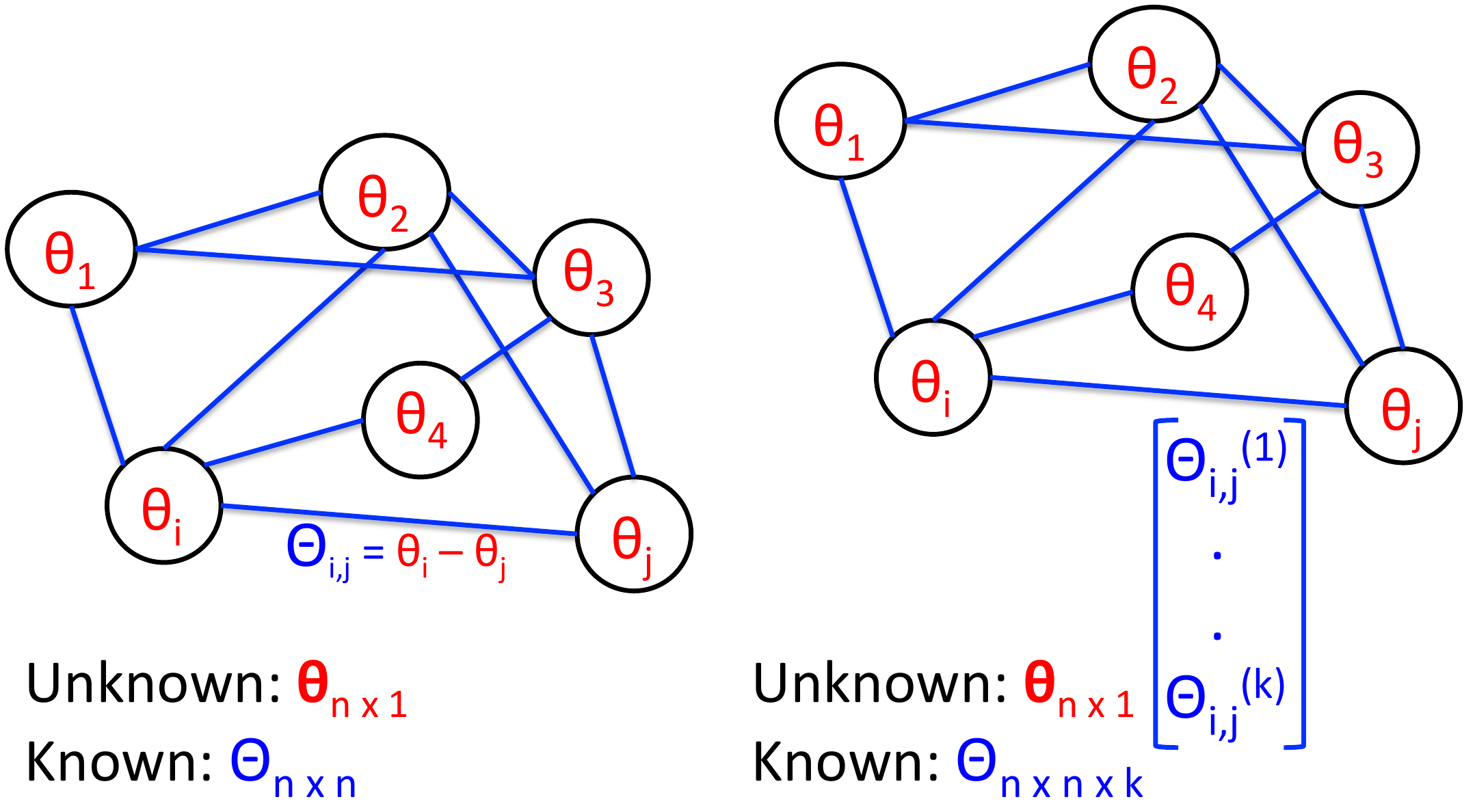}
\vspace{-2mm}
\captionsetup{width=1\linewidth}
\caption{Left: classical angular synchronization; Right: an instance of the $k$-list synchronization problem, where the goal is to recover a single group of $n$ angles, from a small subset of pairwise noisy offsets $\theta_i - \theta_j $; however, there are now multiple candidates for a given pairwise offset, only one of which is correct (or approximately correct).  }
\vspace{-2mm}
\label{fig:listSync}
\end{figure}

\paragraph{SDP relaxation} Extending the theoretical results obtained in this paper for the spectral relaxations to the setting of an SDP relaxation is another direction of interest. The advantage of the SDP relaxation stems from the fact that one can explicitly enforce the unit magnitude constraints in the formulation, by adding ones on the main diagonal of the Gram matrix in \eqref{eq:SDP_program_SYNC}. 
A potentially interesting approach, albeit computationally expensive,  is to consider a SDP which optimizes over $k$ complex-valued Hermitian positive semidefinite matrices $\Upsilon^{(1)}, \Upsilon^{(2)}, \ldots, \Upsilon^{(k)}$, while enforcing the unit magnitude diagonal constraints individually for each such matrix.

\vspace{0mm}
\paragraph{Sparse regime of the measurement graph} 
Another direction worth exploring pertains to improving existing algorithms for the sparse regimes. For example, in graph clustering problems, it is well known that spectral methods underperform in the very sparse regime, where edges exist with probability $ \Theta\left(\frac{1}{n}\right) $. Therein, graph regularization techniques have been proven to be effective both on the theoretical and experimental fronts; for eg., see regularization in the sparse regime for the signed clustering problem \cite{SPONGE2020regularized,chaudhuri12,Amini_2013}. Hence, a natural question to ask is what would be a corresponding graph regularization technique for very sparse measurement graphs in the group synchronization setting. For example, would a simple local augmentation step be effective to this end, where for each 2-path  $v_i, v_j, v_k$, with edges $\set{i,j}$ and  $\set{j,k}$, one can add an estimated measurement for edge $\set{i,k}$ stemming from the fact that summing up over triangles yields $0 \mod 2\pi$. Alternatively, for each missing edge $\set{i,k}$, one could consider all open triangles $i, j', k$ 
(with edges $\set{i,j'}$ and $\set{j',k}$), and each such triangle could cast a vote via the above procedure for the missing edge.

\paragraph{Cluster-synchronization}
Finally, we recall here the setup of cluster-synchronization, detailed in the last paragraph of Section \ref{sec:kSyncApp}, on the recovery of two groups of angles (of potentially different sizes $n_1$,$n_2$) from a noisy subset of pairwise offsets captured in a matrix of size $(n_1+n_2) \times (n_1+n_2)$. Therein, the noise model imparts a cluster structure on the measurement matrix, as depicted in 
\Cref{fig:clusterSync}. Proposing spectral and SDP algorithms for the recovery of both the latent clusters and the groups of angles themselves, is an interesting future direction worth pursuing, well motivated by applications within the ranking and recommender systems literature. 

Additional research directions include extending the $k$-synchronization pipeline to compact groups beyond SO$(2)$, and exploring the iterative (potentially re-weighted) scheme in the context of classical group synchronization ($k=1$), as detailed in Remark  \ref{rem:iterativeAlgo} and motivated by Figure   \ref{fig:iterations_detangle_k234}.




%% file: appendix.tex
\input{app_perturbation}

\input{app_concentration}
\input{app_bisync_proofs}

\input{app_ksync_proofs}

%% file: app_perturbation.tex
\section{Perturbation analysis} \label{app:sec_perturb_theory}
Let $A \in \mathbb{C}^{n \times n}$ be Hermitian with eigenvalues $\lambda_1 \geq \lambda_2 \geq \cdots \geq \lambda_n$ 
and corresponding eigenvectors $v_1,v_2,\dots,v_n \in \mathbb{C}^n$. 
Let $\widetilde{A} = A + W$ be a perturbed version of $A$, with the perturbation matrix 
$W \in \mathbb{C}^{n \times n}$ being Hermitian. Let us denote the eigenvalues of $\tilde{A}$ and $W$ by
$\tilde{\lambda}_1 \geq \cdots \geq \tilde{\lambda}_n$ and 
$\epsilon_1 \geq \epsilon_2 \geq \cdots \geq \epsilon_n$ respectively.

To begin with, we would like to quantify the perturbation of the eigenvalues of $\widetilde{A}$ with respect to the 
eigenvalues of $A$. Weyl's inequality \cite{Weyl1912} is a very useful result in this regard.
%
\begin{theorem} [Weyl's Inequality \cite{Weyl1912}] \label{thm:Weyl} 
For each $i = 1,\dots,n$, it holds that
\begin{equation}
 \lambda_i + \epsilon_n  \leq  \tilde{\lambda}_i \leq \lambda_i + \epsilon_1.
 \end{equation}
In particular, this implies that $\tilde{\lambda}_i \in [\lambda_i - \norm{W}_2, \lambda_i + \norm{W}_2]$.
\end{theorem} 
One can also quantify the perturbation of the subspace spanned by eigenvectors of $A$, this was established 
by Davis and Kahan \cite{daviskahan}. Before introducing the theorem, we need some definitions. 
Let $U,\widetilde{U} \in \mathbb{C}^{n \times k}$ (for $k \leq n$) have orthonormal columns respectively and 
let $\sigma_1 \geq \dots \geq \sigma_k$ denote the singular values of $U^{*}\widetilde{U}$. 
Then, the $k$ principal angles between $U, \widetilde{U}$ are defined as $\theta_i := \cos^{-1}(\sigma_i)$ for $1 \leq i \leq k$, 
with each $\theta_i \in [0,\pi/2]$. 
It is usual to denote $\Theta(U, \tilde{U}) \in \matR^{k \times k}$ to be a diagonal matrix with $\theta_i$ as its 
entries (arranged in decreasing order), and to define $\sin \Theta(U, \tilde{U}) \in \mathbb{C}^{k \times k}$ entrywise. 
Denoting $||| \cdot |||$ to be any unitarily invariant norm (Frobenius, spectral, etc.), 
the following relation holds (see for eg., \cite[Lemma 2.2]{li98} and ensuing discussion).
\begin{equation*} 
|||  \sin \Theta(U, \tilde{U})  |||  =  ||| (I - \tilde{U}  \tilde{U}^{*} ) U_1 |||.
\end{equation*}
With the above notation in mind, we now introduce a version of the Davis-Kahan theorem taken from \cite[Theorem 1]{dkuseful} 
(see also \cite[Theorem 3.2]{li98}).
%
\begin{theorem}[Davis-Kahan] \label{thm:DavisKahan} 
Fix $1 \leq r \leq s \leq n$, let $d = s-r+1$, and let 
$V = (v_r,\dots,v_s) \in \mathbb{C}^{n \times d}$ and 
$\widetilde{V} = (\widetilde{v}_r,\dots,\widetilde{v}_s) \in \mathbb{C}^{n \times d}$. Write
\begin{equation*}
 \varepsilon = \inf\set{\abs{\hat\lambda - \lambda}: \lambda \in [\lambda_s,\lambda_r], \hat\lambda \in (-\infty,\widetilde\lambda_{s+1}] \cup [\widetilde \lambda_{r-1},\infty)}, 
\end{equation*}
where we define $\widetilde\lambda_0 = \infty$ and $\widetilde\lambda_{n+1} = -\infty$ and assume that $\varepsilon > 0$. Then
\begin{equation*}
 ||| \sin \Theta(U, \widetilde U)|||  \leq  \frac{ ||| W ||| }{ \delta}.
 \end{equation*}
\end{theorem} 
For instance, if $r = s = j$, we obtain the following useful result
\begin{equation} \label{eq:dk_useful}
\sin \Theta(\widetilde{v}_j, v_j) = \norm{(I - v_j v_j^{*})\widetilde{v}_j}_2 \leq \frac{\norm{W}_2}{\min\set{\abs{\widetilde{\lambda}_{j-1}-\lambda_j},\abs{\widetilde\lambda_{j+1}-\lambda_j}}}.
\end{equation}
%

%% file: app_concentration.tex
\section{Useful concentration inequalities}
\subsection{Sum of sub-Gaussian random variables} \label{app:subsec_subgauss_conc}
Recall the following Hoeffding-type inequality for sums of independent sub-Gaussian random variables.
%
%
\begin{proposition} \cite[Proposition 5.10]{vershynin2012} \label{prop:hoeff_subgauss_conc}
Let $X_1,\dots,X_n$ be independent  centered sub-Gaussian random variables 
and let $K = \max_i \norm{X_i}_{\psi_2}$. Then for every $\veca \in \matR^n$, and 
every $t \geq 0$, we have
\begin{equation}
\prob(\abs{\sum_{i=1}^n a_i X_i} \geq t) 
\leq e \cdot \exp\left(-\frac{c^{\prime} t^2}{K^2 \norm{\veca}_2^2}\right), 
\end{equation}
where $c^{\prime} > 0$ is an absolute constant.
\end{proposition}

\subsection{Spectral norm of random matrices} \label{app:subsec_spec_rand_mat}
We will make use of the following result for bounding the spectral norm of 
symmetric matrices with independent, centered random variables.
\begin{theorem}[{\cite[Corollary 3.12, Remark 3.13]{bandeira2016}}] \label{app:thm_symm_rand}
Let $X$ be an $n \times n$ symmetric matrix whose entries $X_{ij}$ $(i \leq j)$ are 
independent, centered random variables. There exists, for any $0 < \varepsilon \leq 1/2$,  
a universal constant $c_{\varepsilon}$ such that for every $t \geq 0$, 
\begin{equation} \label{eq:afonso_conc}
\prob(\norm{X}_2 \geq (1+\varepsilon) 2\sqrt{2}\tilde{\sigma} + t) 
\leq n\exp\left(-\frac{t^2}{c_{\varepsilon}\tilde{\sigma}_{*}^2} \right), 
\end{equation}
where
\begin{equation*} 
\tilde{\sigma}:= \max_{i} \sqrt{\sum_{j} \expec[X_{ij}^2]}, 
\quad \tilde{\sigma}_{*}:= \max_{i,j} \norm{X_{ij}}_{\infty}.
\end{equation*}
\end{theorem}
Note that it suffices to employ upper bound estimates on $\tilde{\sigma},\tilde{\sigma}$ in
\eqref{eq:afonso_conc}.

%% file: app_bisync_proofs.tex
\section{Proofs from Section \ref{sec:bi_sync}}
\subsection{Proof of Proposition \ref{prop:bisync_approx_orth_whp}} \label{app:subsec_prop_approx_orth}
Recall that for a random variable $X$, its sub-Gaussian norm 
$\norm{X}_{\psi_2}$ is defined as 
\begin{equation}
\norm{X}_{\psi_2} := \sup_{p \geq 1} \frac{(\expec \abs{X}^p)^{1/p}}{\sqrt{p}}.
\end{equation}
Moreover, $X$ is a sub-Gaussian random variable if $\norm{X}_{\psi_2}$ is finite. 
For instance, a bounded random variable $X$ with $\abs{X} \leq M$ is 
sub-Gaussian with $\norm{X}_{\psi_2} \leq M$ (see \cite[Example 5.8]{vershynin2012}).
\begin{proof}[Proof of Proposition \ref{prop:bisync_approx_orth_whp}]
Observe that
\begin{align}
\dotprod{z_1}{z_2} 
= \frac{1}{n}\sum_{i=1}^n \exp(\iota (\beta_i - \alpha_i)) 
= \frac{1}{n}\left[\sum_{i=1}^n \underbrace{\cos(\beta_i - \alpha_i)}_{R_i} + 
\iota \sum_{i=1}^n \underbrace{\sin(\beta_i - \alpha_i)}_{X_i}\right].
\end{align}
$(R_i)_i$ are zero-mean iid sub-Gaussian random variables with $\norm{R_i}_{\psi_2} \leq 1$. 
Using standard concentration inequalities for sums of sub-Gaussian random variables (see Appendix \ref{app:subsec_subgauss_conc}), there exists a constant $c' > 0$ such that 
\begin{equation}
\prob\left(\abs{\frac{1}{n} \sum_{i=1}^n R_i} > t \right) \leq e\cdot\exp(-c' t^2 n), 
\end{equation}
for $t \geq 0$. An identical statement holds for the random variables $(X_i)_i$. Using the union bound, we then have with probability at least $1 - 2e\exp(-c' t^2 n)$, that 
$\abs{\dotprod{z_1}{z_2}} \leq 2t$ holds. Plugging $t = \delta/2$ yields the stated bound.
\end{proof}

\subsection{Proof of Lemma \ref{lem:eigperturn_approx_orth}} \label{app:subsec_eigperturn_approx_orth}
\begin{enumerate}
%
\item \textbf{(Bounds on $\lamtil_1, \lamtil_2$)}  
For this part, we note\footnote{This argument for computing 
the expressions for $\lamtil_1,\lamtil_2$ is based on a post by Davide Giraudo on Mathematics Stack Exchange, question $112186$ 
(\url{https://math.stackexchange.com/questions/112186/eigenvalues-of-a-sum-of-rank-one-matrices})} 
that
\begin{align}
(\expec[H])^2 
&= n^2p_1^2\lambda^2 z_1z_1^{*} + n^2p_2^2\lambda^2 z_2 z_2^{*} + 2n^2p_1p_2\lambda^2 z_1(z_1^{*}z_2) z_2^{*} \nonumber \\
\Rightarrow  
\tr((\expec[H])^2) 
&= n^2p_1^2\lambda^2 + n^2p_2^2\lambda^2 + 2n^2p_1p_2\lambda^2\abs{\dotprod{z_1}{z_2}}^2 
= \lamtil_1^2 + \lamtil_2^2, 
\label{eq:temp1}
\end{align}
where we used the linearity of the trace operator. Moreover, we also have that
$\tr(\expec[H]) = \lamtil_1 + \lamtil_2 = np_1\lambda + np_2\lambda$.
Using this with \eqref{eq:temp1}, we readily obtain 
$\lamtil_1\lamtil_2 = n^2p_1p_2\lambda^2(1-\abs{\dotprod{z_1}{z_2}}^2)$ leading to 
\begin{align}
(np_1\lambda + np_2\lambda - \lamtil_2)\lamtil_2
= n^2p_1p_2\lambda^2(1-\abs{\dotprod{z_1}{z_2}}^2) \nonumber \\
\Leftrightarrow 
\lamtil_2^2 - \lamtil_2(np_1\lambda + np_2\lambda) + n^2p_1p_2\lambda^2(1-\abs{\dotprod{z_1}{z_2}}^2) = 0 
\label{eq:temp4} \\
\Longleftrightarrow
\lamtil_2 
= \frac{np_1\lambda + np_2\lambda - 
\sqrt{(np_1\lambda + np_2\lambda)^2 - 4n^2p_1p_2\lambda^2(1-\abs{\dotprod{z_1}{z_2}}^2)}}{2} \nonumber \\
= \frac{np_1\lambda + np_2\lambda - 
\sqrt{(np_1\lambda - np_2\lambda)^2 + 4n^2p_1p_2\lambda^2\abs{\dotprod{z_1}{z_2}}^2}}{2}, \label{eq:lamtil2} \\
\text{and} \quad 
\lamtil_1 = \frac{np_1\lambda + np_2\lambda +
\sqrt{(np_1\lambda - np_2\lambda)^2 + 4n^2p_1p_2\lambda^2\abs{\dotprod{z_1}{z_2}}^2}}{2}. \label{eq:lamtil1}
\end{align} 
Note that we neglect the other root in \eqref{eq:temp4} since $\lamtil_1 \geq \lamtil_2$.
In fact, $\lamtil_1 > \lamtil_2$ since $p_1 > p_2$. The bounds 
$\lamtil_2 \leq np_2 \lambda$ and $\lamtil_1 \geq np_1\lambda$ follow from the fact 
that $\abs{\dotprod{z_1}{z_2}} \geq 0$. The bounds in \eqref{eq:bds_lamtil_2}, \eqref{eq:bds_lamtil_1} 
hold if $\abs{\dotprod{z_1}{z_2}} \leq \delta$ holds.

\item \textbf{(Bounds on $\abs{\dotprod{z_1}{\vtil_1}}^2, \abs{\dotprod{z_2}{\vtil_2}}^2$)} 
To bound $\abs{\dotprod{z_1}{\vtil_1}}^2$, we note that 
\begin{align}
z_1^{*} \expec[H] z_1 
&= np_1\lambda + np_2\lambda \abs{\dotprod{z_1}{z_2}}^2 \nonumber \\
\Rightarrow 
\lamtil_1 \abs{\dotprod{z_1}{\vtil_1}}^2 + \lamtil_2 \abs{\dotprod{z_1}{\vtil_2}}^2 
&= np_1\lambda + np_2\lambda \abs{\dotprod{z_1}{z_2}}^2 \nonumber \\
\Rightarrow \lamtil_1 \abs{\dotprod{z_1}{\vtil_1}}^2 + \lamtil_2 (1-\abs{\dotprod{z_1}{\vtil_1}}^2)  
&= np_1\lambda + np_2\lambda \abs{\dotprod{z_1}{z_2}}^2 \label{eq:temp5} \\ 
\Rightarrow (\lamtil_1 - \lamtil_2) \abs{\dotprod{z_1}{\vtil_1}}^2 
&= np_1\lambda -\lamtil_2 + np_2\lambda \abs{\dotprod{z_1}{z_2}}^2 \nonumber \\
&\geq np_1\lambda -np_2\lambda \label{eq:temp6} \\
\Rightarrow \abs{\dotprod{z_1}{\vtil_1}}^2 
&\geq \frac{p_1 - p_2}{\sqrt{(p_1-p_2)^2 + 4p_1p_2\delta^2}}. \label{eq:temp7}
\end{align}
In \eqref{eq:temp5}, we used the fact 
$1 = \norm{z_1}_2^2 = \abs{\dotprod{z_1}{\vtil_1}}^2 + \abs{\dotprod{z_1}{\vtil_2}}^2$, 
since $z_1$ lies in span($\vtil_1,\vtil_2$). 
In \eqref{eq:temp6}, we used $\abs{\dotprod{z_1}{z_2}}^2 \geq 0$ and 
also $\lamtil_2 \leq np_2\lambda$ (shown earlier). 
In \eqref{eq:temp7}, we used the bound 
$\lamtil_1 - \lamtil_2 \leq n\lambda\sqrt{(p_1-p_2)^2 + 4p_1p_2\delta^2}$, which follows easily from \eqref{eq:bds_lamtil_1} and  \eqref{eq:bds_lamtil_2}. 
Finally, the bound on $\abs{\dotprod{z_2}{\vtil_2}}^2$ follows in a similar fashion, 
but by considering $z_2^{*} \expec[H] z_2$ instead.
\end{enumerate}

%
\subsection{Proof of Lemma \ref{lem:mat_pert_dk}} \label{app:subsec_mat_pert_dk}
Recall from \eqref{eq:bisync_Hrank2Decomp} that 
$H = \mathbb{E}(H) + R = (n\lambda p_1 z_1 z_1^* + n\lambda p_2 z_2 z_2^*) + R$. 
As a consequence of Weyl's inequality for symmetric perturbation 
of symmetric matrices (see Theorem \ref{thm:Weyl} in Appendix \ref{app:sec_perturb_theory}), we have
\begin{align} \label{eq:eig_bds_tmp}
\lambda_i \in [\lamtil_i - \triangle, \lamtil_i + \triangle]; \quad i=1,\dots,n.
\end{align}
In order to obtain the stated bounds, we will now invoke the Davis-Kahan 
theorem (see Theorem \ref{thm:DavisKahan} in Appendix \ref{app:sec_perturb_theory}).  
\begin{enumerate}
\item \textbf{(Bound in \eqref{eq:vtilv_bds_1})}  
From Theorem \ref{thm:DavisKahan}, we see that if $\abs{\lamtil_1 - \lambda_2} > 0$, 
then it holds that
\begin{equation} \label{eq:vspace_bd1}
\norm{(I - v_1v_1^{*})\vtil_1}_2 \leq \frac{\norm{R}_2}{\abs{\lamtil_1 - \lambda_2}} 
\leq \frac{\triangle}{\abs{\lamtil_1 - \lambda_2}}. 
\end{equation}
Using \eqref{eq:eig_bds_tmp} and the bound $\lamtil_2 \leq np_2\lambda$ from 
Lemma \ref{lem:eigperturn_approx_orth}, we obtain $\lambda_2 \leq np_2\lambda + \triangle$.
Furthermore, using the bound $\lamtil_1 \geq np_1\lambda$ from Lemma \ref{lem:eigperturn_approx_orth}, we obtain 
$\lamtil_1 - \lambda_2 \geq np_1\lambda - np_2\lambda - \triangle$. Therefore if 
$\triangle < np_1\lambda - np_2\lambda$, it follows from \eqref{eq:vspace_bd1} that
\begin{align*}
\norm{(I - v_1v_1^{*})\vtil_1}_2 \leq \frac{\triangle}{np_1\lambda - np_2\lambda - \triangle}. 
\end{align*}
Since $\norm{(I - v_1v_1^{*})\vtil_1}_2^2 = 1 - \abs{\dotprod{\vtil_1}{v_1}}^2$, we thus 
obtain the stated bound in \eqref{eq:vtilv_bds_1}.

\item \textbf{(Bound in \eqref{eq:vtilv_bds_2})} 
We now apply Theorem \ref{thm:DavisKahan} to bound 
$\norm{(I - v_2v_2^{*})\vtil_2}_2$. If 
$$\min\set{\abs{\lambda_1-\lamtil_2},\abs{\lambda_3-\lamtil_2}} > 0,$$ 
it follows from Theorem \ref{thm:DavisKahan} that
\begin{equation} \label{eq:vspace_2_bd}
\norm{(I - v_2v_2^{*})\vtil_2}_2 \leq \frac{\norm{R}_2}{\min\set{\abs{\lambda_1-\lamtil_2},\abs{\lambda_3-\lamtil_2}}}.
\end{equation}
We saw earlier that $\lambda_1-\lamtil_2 \geq np_1\lambda - np_2\lambda - \triangle > 0$ if 
$\triangle < np_1\lambda - np_2\lambda$. Also, since $\lambda_3 \in [-\triangle,\triangle]$ (from \eqref{eq:eig_bds_tmp}), 
it follows that $\lamtil_2 - \lambda_3 \geq np_2\lambda - \triangle > 0$ if $\triangle < np_2\lambda$. 
Plugging these observations in \eqref{eq:vspace_2_bd}, and using the fact 
$\norm{(I - v_2v_2^{*})\vtil_2}_2^2 = 1 - \abs{\dotprod{\vtil_2}{v_2}}^2$, 
the stated bound in \eqref{eq:vtilv_bds_2} follows readily.
\end{enumerate}

%
\subsection{Proof of Lemma \ref{lem:rand_pert_bd}} \label{app:subsec_rand_pert_bd}
Since $\norm{R}_2 \leq \norm{\real(R)}_2 + \norm{\imag(R)}_2$, 
therefore we will bound the terms $\norm{\real(R)}_2, \norm{\imag(R)}_2$ individually w.h.p, which 
will then imply a bound on $\norm{R}_2$. 
\begin{enumerate}
\item (\textbf{Bounding $\norm{\real(R)}_2$.}) 
The entries of $\real(R)$ above and on the diagonal are independent centered random variables. 
Therefore we can bound $\norm{\real(R)}_2$ using a result of 
Bandeira et al. \cite[Corollary 3.12, Remark 3.13]{bandeira2016} for bounding the spectral 
norm of such symmetric random matrices (stated as Theorem \ref{app:thm_symm_rand} in 
Appendix \ref{app:subsec_spec_rand_mat}). To do so, it suffices to know upper bounds on the quantities
\begin{equation} \label{eq:sigma_sigtil_def}
\tilde{\sigma}:= \max_{i} \sqrt{\sum_{j} \expec[\real(R)_{ij}^2]}, 
\quad \tilde{\sigma}_{*}:= \max_{i,j} \norm{\real(R)_{ij}}_{\infty}.
\end{equation}
Starting with $\tilde{\sigma}$, we note that 
\begin{align*}
\expec[\real(R)_{ij}^2] 
&= p_1\lambda[\cos(\alpha_i-\alpha_j) -p_1\lambda\cos(\alpha_i-\alpha_j) - p_2\lambda\cos(\beta_i-\beta_j)]^2 \nonumber \\
&+ p_2\lambda[\cos(\beta_i-\beta_j) -p_1\lambda\cos(\alpha_i-\alpha_j) - p_2\lambda\cos(\beta_i-\beta_j)]^2 \nonumber \\
&+ (1-p_1-p_2)\lambda (\expec[\cos N_{ij} -p_1\lambda \cos(\alpha_i-\alpha_j) -p_2\lambda \cos(\beta_i-\beta_j)]^2) \nonumber \\
&+ (1-\lambda)[p_1\lambda \cos(\alpha_i-\alpha_j) + p_2\lambda \cos(\beta_i-\beta_j)]^2 \nonumber \\
&\leq 2p_1\lambda[(1-p_1\lambda)^2 + (p_2\lambda)^2] + 2p_2\lambda[(1-p_2\lambda)^2 + (p_1\lambda)^2] \nonumber \\
&+ (1-p_1-p_2)\lambda\left[\frac{1}{2} + (p_1\lambda + p_2\lambda)^2\right] 
+ (1-\lambda)[p_1\lambda + p_2\lambda]^2   \nonumber \\
%
%
&= C(\lambda,p_1,p_2).
\end{align*}
Hence we obtain $\tilde{\sigma} \leq \sqrt{C(\lambda,p_1,p_2) n}$. 
Using Theorem \ref{app:thm_symm_rand} with $t = 2\sqrt{2}\sqrt{C(\lambda,p_1,p_2) n}$, 
and with \eqref{eq:infty_norm_bd_1} in mind, we then obtain for any $\varepsilon \geq 0$,  that there 
exists a universal constant $c_{\varepsilon} > 0$ such that 
%
%
%
\begin{equation} \label{eq:symm_norm_bd_1}
\prob\left(\set{\norm{\real(R)}_2 \geq (2+\varepsilon)2\sqrt{2} \sqrt{C(\lambda,p_1,p_2) n}} \right) \leq 
n\exp\left(- \frac{8 C(\lambda,p_1,p_2) n}{\sigtil(\lambda,p_1,p_2)^2 c_{\varepsilon}} \right).
\end{equation}
%

\item (\textbf{Bounding $\norm{\imag(R)}_2$.}) 
We can write $\imag(R) = \immat - \immat^T$,  where $\immat_{ij} = 0$ for $i \geq j$. 
The random variables $(\immat_{ij})_{i < j}$ are independent and zero-mean. 
Since $\norm{\imag(R)}_2 \leq 2\norm{\immat}_2$, therefore we will focus on bounding $\norm{\immat}_2$. 
Now, while $\immat$ is not symmetric, one can easily verify that for the following symmetric matrix 
\begin{equation*}
\immattil =
\begin{bmatrix}
 0 & \immat \\
 \immat^{T} & 0
\end{bmatrix} \in \matR^{2n \times 2n},
\end{equation*}
we have $\norm{\immattil}_2 = \norm{\immat}_2$. Since $\immattil$ has independent entries on and above the 
main diagonal, we can again use the result in \cite[Corollary 3.12, Remark 3.13]{bandeira2016}
to bound its spectral norm. To this end, we need to bound the quantities 
\begin{equation*}
\tilde{\sigma}:= \max_{i} \sqrt{\sum_{j} \expec[\immattil_{ij}^2]}, 
\quad \tilde{\sigma}_{*}:= \max_{i,j} \norm{\immattil_{ij}}_{\infty}.
\end{equation*}
In an identical manner as before, one can readily verify that $\tilde{\sigma} \leq \sqrt{C(\lambda,p_1,p_2) n}$. 
Moreover, it holds that $\tilde{\sigma}_{*} \leq \sigtil(\lambda,p_1,p_2)$.
Applying Theorem \ref{app:thm_symm_rand} on $\immattil$ with $t = 2\sqrt{2}\sqrt{C(\lambda,p_1,p_2) n}$, 
and with the above observations in mind, we finally obtain for any $\varepsilon \geq 0$ that there exists 
a universal constant $c_{\varepsilon} > 0$ such that
\begin{equation} \label{eq:skew_symm_norm_bd_1}
\prob\left(\set{\norm{\imag(R)}_2 \geq (2+\varepsilon)4\sqrt{2} \sqrt{C(\lambda,p_1,p_2) n}} \right) \leq 
2n\exp\left(- \frac{16 C(\lambda,p_1,p_2) n}{\sigtil(\lambda,p_1,p_2)^2 c_{\varepsilon}} \right).
\end{equation}
\end{enumerate}
Finally, the statement in the Lemma follows by applying the union 
bound to \eqref{eq:symm_norm_bd_1}, \eqref{eq:skew_symm_norm_bd_1}. 
%

\subsection{Proof of Proposition \ref{prop:chain_bd_dotprod}} \label{app:subsec_chain_bd_dotprod}
We know that $1-\abs{\dotprod{x}{\xbar}}^2 = \norm{(I-\xbar\xbar^*)x}_2^2$. Moreover,
\begin{align}
\norm{(I-\xbar\xbar^*)x}_2
\leq 
\norm{x - y y^{*} x}_2 + \norm{(y y^{*} - \xbar\xbar^*) x}_2 
\leq 
\sqrt{\varep} + \norm{y y^{*} - \xbar\xbar^*}_2. \label{eq:lemma_trian_tmp1}
\end{align}
By observing that 
\begin{align*}
\norm{y y^{*} - \xbar\xbar^*}_2 
\leq 
\norm{y y^{*} - \xbar\xbar^*}_F = \sqrt{2-2\abs{\dotprod{\xbar}{y}}^2} 
\leq 
\sqrt{2\varepbar},
\end{align*}
and plugging this in \eqref{eq:lemma_trian_tmp1}, we obtain the stated bound.

%% file: app_ksync_proofs.tex
\section{Proofs from Section \ref{sec:k_sync}}
\subsection{Proof of Lemma \ref{lem:deflation_gen}} \label{app:subsec_deflation_gen}
\begin{enumerate}
%
\item Let us consider $j = 1$. Starting with the definition of $\lamtil_1$, we obtain 
\begin{align} \label{eq:lamtil1_low_bd}
\lamtil_1 := \max_{\norm{x}_2 = 1} x^* \expec[H] x 
\geq z_1^{*} \expec[H] z_1 
= np_1\lambda + \sum_{j=2}^{k} np_j\lambda\abs{\dotprod{z_1}{z_j}}^2 \geq np_1\lambda.
\end{align}
In order to lower bound $\abs{\dotprod{z_1}{\vtil_1}}^2$, let us first note that 
\begin{equation*} 
\lamtil_1 = \vtil^{*} \expec[H] \vtil = np_1\lambda \abs{\dotprod{z_1}{\vtil_1}}^2 
+ \sum_{j=2}^k np_j\lambda \abs{\dotprod{z_j}{\vtil_1}}^2.
\end{equation*}
Therefore, $\lamtil_1$ can be bounded as 
\begin{align} \label{eq:lamtil1_bd_1}
\lamtil_1 
\leq np_1\lambda \abs{\dotprod{z_1}{\vtil_1}}^2 + \sum_{j=2}^k np_j\lambda \abs{\dotprod{z_j}{\vtil_1}}^2 
= (np_1\lambda-np_2\lambda) \abs{\dotprod{z_1}{\vtil_1}}^2 + np_2\lambda \sum_{j=1}^k \abs{\dotprod{z_j}{\vtil_1}}^2.	 
\end{align}
We will now upper bound $\sum_{j=1}^k \abs{\dotprod{z_j}{\vtil_1}}^2$ as follows. Denoting $Z = [z_1 \ z_2 \ \dots \ z_k] \in \mathbb{C}^{n \times k}$, clearly
\begin{equation} \label{eq:useful_gers_bd}
\sum_{j=1}^k \abs{\dotprod{z_j}{\vtil_1}}^2 = \sum_{j=1}^k \vtil_1^* z_j z_j^* \vtil_1 = \vtil_1^{*} Z Z^* \vtil_1 \leq \norm{Z}_2^2 \leq 1 + (k-1) \delta,
\end{equation}
where in the last step we used\footnote{$\norm{Z}_2^2 = \lambda_{\max}(Z^*Z)$ where $\lambda_{\max}(Z^*Z)$ denotes the largest eigenvalue of $Z^*Z$. The magnitude of each off-diagonal entry of $Z^*Z$ is bounded by $\delta$, while each diagonal entry is equal to $1$. Hence $\lambda_{\max}(Z^*Z) \leq 1 + (k-1)\delta$ by Gershgorin's theorem.} Gershgorin's theorem along with the $\delta$-orthonormality 
of $z_i$'s.  Using \eqref{eq:useful_gers_bd} in \eqref{eq:lamtil1_bd_1} along 
with \eqref{eq:lamtil1_low_bd}, we get 
\begin{align}
np_1\lambda &\leq (np_1\lambda-np_2\lambda) \abs{\dotprod{z_1}{\vtil_1}}^2 + np_2\lambda (1 + (k-1)\delta) \label{eq:lamtil_bd_2} \\
\Longleftrightarrow \abs{\dotprod{z_1}{\vtil_1}}^2 &\geq \frac{np_1\lambda - np_2\lambda - np_2\lambda(k-1)\delta}{np_1\lambda - np_2 \lambda} \nonumber \\
&= 1 - \frac{p_2(k-1)\delta}{p_1 - p_2} = 1 - \ergen(\delta). \nonumber
\end{align}
Note that, $\abs{\dotprod{z_1}{\vtil_1}}^2 > 1/2$ if $\delta < \frac{p_1-p_2}{2p_2(k-1)}$. 
Finally, we obtain from \eqref{eq:lamtil_bd_2} the bound 
\begin{equation*}
 \lamtil_1 \leq 
(np_1\lambda-np_2\lambda) + np_2\lambda (1 + (k-1)\delta) = np_1\lambda + \underbrace{np_2\lambda (k-1)\delta}_{u_1(\delta)}. 
\end{equation*}
%
%

\item Let us now handle $\lamtil_{m}, \vtil_{m}$ for $m > 1$.
Assume that for $1 \leq j \leq m-1$, $\lamtil_j,\vtil_j$ satisfy 
\begin{equation} \label{eq:induc_lam_eig_bds}
\lamtil_j \leq np_j\lambda + u_j(\delta) , \quad \abs{\dotprod{\vtil_j}{z_j}}^2 \geq 1 - \ergen_j(\delta), 
\end{equation}
where $\delta \in [0,1]$ satisfies 
\begin{equation} \label{eq:induc_del_bds}
\delta \leq \sqrt{2\ergen_j(\delta)} \leq 1/2; \quad \ergen_1(\delta) \leq \dots \leq \ergen_{m-1}(\delta). 
\end{equation}
We will show that \eqref{eq:induc_lam_eig_bds},\eqref{eq:induc_del_bds} together imply the bounds
\begin{equation*}
np_m\lambda - l_m(\delta) \leq \lamtil_m \leq np_m \lambda + u_m(\delta) , \quad \abs{\dotprod{\vtil_m}{z_m}}^2 \geq 1 - \ergen_m(\delta).
\end{equation*}
To begin with, it is easy to see that
\begin{align}
\lamtil_{m} = \vtil_{m} \expec[H] \vtil_m
&= \max_{\norm{x}_2 = 1} x^*(\expec[H] - \sum_{j=1}^{m-1} \lamtil_j \vtil_j \vtil_j^*)x \nonumber \\
&\geq z_{m}^* (\expec[H] - \sum_{j=1}^{m-1} \lamtil_j \vtil_j \vtil_j^*) z_{m} \nonumber \\ 
&= z_{m}^* \expec[H] z_m - \sum_{j=1}^{m-1} \lamtil_j \abs{\dotprod{\vtil_j}{z_m}}^2. \label{eq:lamtil_m_bd_1}
\end{align}
One can also verify that $z_{m}^* \expec[H] z_m \geq np_m \lambda$. Moreover, $\lamtil_j \leq np_j\lambda + u_j(\delta)$ 
for $j = 1,\dots,m-1$ by assumption. In order to upper bound $\abs{\dotprod{\vtil_j}{z_m}}^2$, let us first 
note that $\abs{\dotprod{\vtil_j}{z_m}}^2 = 1 - \norm{z_m - \vtil_j\vtil_j^*z_m}_2^2$. We can lower bound $\norm{z_m - \vtil_j\vtil_j^*z_m}_2$ as follows. 
\begin{align*}
\norm{z_m - \vtil_j\vtil_j^*z_m}_2 
&\geq \norm{z_m - z_j z_j^*z_m}_2 - \norm{z_j z_j^*z_m - \vtil_j\vtil_j^*z_m}_2 \\
&\geq \sqrt{1 - \delta^2} - \sqrt{2\ergen_j(\delta)} \\
&\geq 1 - \delta - \sqrt{2\ergen_j(\delta)} \quad (\text{Since } \sqrt{1 - \delta^2} \geq 1 - \delta) \\
&\geq 1 - 2\sqrt{2\ergen_j(\delta)} \quad \left(\text{Since } \delta \leq \sqrt{2\ergen_j(\delta)}\right) \\
&\geq 0 \quad (\text{Since } \sqrt{2\ergen_j(\delta)} \leq 1/2).
\end{align*}
In the second inequality, we used $\norm{z_m - z_j z_j^*z_m}_2 = \sqrt{1-\abs{\dotprod{z_j}{z_m}}^2} \geq \sqrt{1-\delta^2}$. 
For each $j = 1,\dots,m-1$, this leads to 
\begin{align*}
\abs{\dotprod{\vtil_j}{z_m}}^2 
\leq 1 - \left(1 - 2\sqrt{2\ergen_j(\delta)} \right)^2 
= \left(2 \sqrt{2\ergen_j(\delta)}\right)\left(2 - 2\sqrt{2\ergen_j(\delta)}\right) 
\leq 4\sqrt{2\ergen_j(\delta)}.
\end{align*}
Using these bounds in \eqref{eq:lamtil_m_bd_1}, we can lower bound $\lamtil_m$ as
\begin{equation} \label{eq:lamtil_m_bd_2}
 \lamtil_m \geq n p_{m}\lambda - 4\sum_{j=1}^{m-1} (np_j\lambda + u_j(\delta))\sqrt{2\ergen_j(\delta)}.  
\end{equation}
Let us now proceed to upper bound $\lamtil_m$. To this end, we first note that
\begin{align}
\lamtil_m = \vtil_m^* \expec[H] \vtil_m 
&= \sum_{i=1}^{m-1} np_i\lambda\abs{\dotprod{z_i}{\vtil_m}}^2 + np_m\lambda\abs{\dotprod{z_m}{\vtil_m}}^2 
+ \sum_{j=m+1}^k np_j\lambda \abs{\dotprod{z_j}{\vtil_m}}^2 \nonumber \\
&\leq \sum_{i=1}^{m-1} np_i\lambda\abs{\dotprod{z_i}{\vtil_m}}^2 + np_m\lambda\abs{\dotprod{z_m}{\vtil_m}}^2 
+ np_{m+1} \lambda\sum_{j=m+1}^k \abs{\dotprod{z_j}{\vtil_m}}^2 \nonumber \\
&\leq \sum_{i=1}^{m-1} (np_i\lambda - np_{m+1} \lambda) \abs{\dotprod{z_i}{\vtil_m}}^2 + 
(np_m\lambda - np_{m+1}\lambda)\abs{\dotprod{z_m}{\vtil_m}}^2 \nonumber \\ 
&+ np_{m+1}\lambda(1+(k-1)\delta). \label{eq:lamtil_m_bd_3}
\end{align}
We can upper bound $\abs{\dotprod{z_i}{\vtil_m}}^2$ for each $i=1,\dots,m-1$ 
by noting that $\abs{\dotprod{z_i}{\vtil_m}}^2 = 1 - \norm{z_i - \vtil_m\vtil_m^* z_i}_2^2$, and also 
\begin{align}
\norm{z_i - \vtil_m\vtil_m^* z_i}_2 
&\geq \norm{\vtil_i\vtil_i^* z_i - \vtil_m\vtil_m^* z_i}_2 - \norm{z_i - \vtil_i\vtil_i^* z_i}_2 \nonumber \\
&= \sqrt{z_i^*(\vtil_i\vtil_i^* + \vtil_m\vtil_m^*) z_i} - \norm{z_i - \vtil_i\vtil_i^* z_i}_2 \nonumber \\
&= (\underbrace{\abs{\dotprod{z_i}{\vtil_i}}^2}_{\geq 1-\ergen_i(\delta)} + \underbrace{\abs{\dotprod{z_i}{\vtil_m}}^2}_{\geq 0})^{1/2} - (\underbrace{1 - \abs{\dotprod{z_i}{\vtil_i}}^2}_{\leq \ergen_i(\delta)})^{1/2} \nonumber \\
&\geq \sqrt{1-\ergen_i(\delta)} - \sqrt{\ergen_i(\delta)} \nonumber \\
&\geq 1-\ergen_i(\delta) - \sqrt{\ergen_i(\delta)} \quad (\text{Since } 1-\ergen_i(\delta) \in [0,1] ) \nonumber \\
&\geq 1 - 2\sqrt{\ergen_i(\delta)}  \quad (\text{ Since } \ergen_i(\delta) \leq \sqrt{\ergen_i(\delta)} ) \nonumber \\ 
&\geq 0 \quad (\text{ Since } \sqrt{2\ergen_i(\delta)} \leq 1/2). \label{eq:temp1_1}
\end{align}
Using \eqref{eq:temp1_1} we thus obtain 
\begin{equation} \label{eq:temp1_2}
\abs{\dotprod{z_i}{\vtil_m}}^2 = 1 - (1-2\sqrt{\ergen_i(\delta)})^2 = 2\sqrt{\ergen_i(\delta)}(2-2\sqrt{\ergen_i(\delta)}) 
\leq 4\sqrt{\ergen_i(\delta)}.
\end{equation} 
Plugging \eqref{eq:temp1_2} in \eqref{eq:lamtil_m_bd_3} then leads to
\begin{equation} \label{eq:lamtil_m_bd_4}
\lamtil_m \leq 
4\sum_{i=1}^{m-1} (np_i\lambda - np_{m+1} \lambda) \sqrt{\ergen_i(\delta)}  + 
(np_m\lambda - np_{m+1}\lambda)\abs{\dotprod{z_m}{\vtil_m}}^2 + np_{m+1}\lambda(1+(k-1)\delta).
\end{equation}
Using \eqref{eq:lamtil_m_bd_2},\eqref{eq:lamtil_m_bd_4} and by re-arranging terms, we obtain
\begin{align}
\abs{\dotprod{z_m}{\vtil_m}}^2 
&\geq \frac{\splitfrac{np_m\lambda-np_{m+1}\lambda(1+(k-1)\delta) - 4\sum_{i=1}^{m-1}[(np_i\lambda+u_i(\delta))\sqrt{2\ergen_i(\delta)}}{
+ (np_i\lambda-np_{m+1}\lambda)\sqrt{\ergen_i(\delta)}] }}{np_m\lambda - np_{m+1}\lambda} \nonumber \\
&= 1-\left(\frac{p_{m+1}(k-1)\delta + 4\sum_{i=1}^{m-1}\left[(p_i + \frac{u_i(\delta)}{n\lambda})\sqrt{2\ergen_i(\delta)} 
+ (p_i - p_{m+1})\sqrt{\ergen_i(\delta)}\right]}{p_m - p_{m+1}} \right) \nonumber \\
&\geq
1 - \left(\frac{p_{m+1}(k-1)\delta + 4\sqrt{2\ergen_{m-1}(\delta)}\sum_{i=1}^{m-1}\left[2p_i + \frac{u_i(\delta)}{n\lambda} 
 - p_{m+1} \right]}{p_m - p_{m+1}} \right), \label{eq:temp1_3}
\end{align}
where in the last inequality, we used our assumption $\ergen_1(\delta) \leq \cdots \leq \ergen_{m-1}(\delta)$. 
Now from the definition of $u_i(\delta)$, we have for each $i=1,\dots,m-1$ that 
\begin{align}
\frac{u_{i}(\delta)}{n\lambda} 
&= 4\sum_{j=1}^{i-1} (p_j - p_{i+1})\sqrt{\ergen_j(\delta)} + p_{i+1}(k-1)\delta \nonumber \\
&\leq \sqrt{2}\sum_{j=1}^{i-1}(p_j - p_{i+1}) + \frac{1}{2}p_{i+1}(k-1) \quad \left(\text{ since } \delta \leq \sqrt{2\ergen_j(\delta)} \leq 1/2 \right) \nonumber \\
&\leq \sqrt{2}(i-1)(p_1-p_{m}) + \frac{1}{2}p_2(k-1). \label{eq:ui_bd}
\end{align}
Using \eqref{eq:ui_bd} in \eqref{eq:temp1_3}, we obtain
\begin{align*}
&\abs{\dotprod{z_m}{\vtil_m}}^2 \\
&\geq 1 - \left(\frac{\splitfrac{p_{m+1}(k-1)\delta + 4\sqrt{2\ergen_{m-1}(\delta)}\sum_{i=1}^{m-1} [2p_i + 
\sqrt{2}(i-1)(p_1-p_{m})}{ + \frac{p_2(k-1)}{2} - p_{m+1} ]}}{p_m - p_{m+1}} \right) \\
&\geq 1 - \left(\frac{\splitfrac{p_{m+1}(k-1)\delta + 4\sqrt{2\ergen_{m-1}(\delta)}[2S_{m-1} + 
\sqrt{2}(m-1)(m-2)(p_1-p_{m})}{ + \frac{(m-1)}{2}(p_2(k-1) - 2p_{m+1})]}}{p_m - p_{m+1}} \right) \\
&\geq 1 - \left(\frac{\splitfrac{p_{m+1}(k-1)\sqrt{2} + 4\sqrt{2}[2S_{m-1} + 
\sqrt{2}(m-1)(m-2)(p_1-p_{m})}{ + \frac{(m-1)}{2}(p_2(k-1) - 2p_{m+1})]}}{p_m - p_{m+1}} \right)\sqrt{\ergen_{m-1}(\delta)} \\
&= 1 - \underbrace{C_m\sqrt{\ergen_{m-1}(\delta)}}_{\ergen_m(\delta)}, 
\end{align*}
where in the penultimate step, we used $\delta \leq \sqrt{2\ergen_{m-1}(\delta)}$. To conclude, 
we now establish the bounds on $\lamtil_m$ stated in the Lemma. The upper bound follows 
from \eqref{eq:lamtil_m_bd_4} by using $\abs{\dotprod{z_m}{\vtil_m}}^2 \leq 1$. Indeed,
\begin{align*}
\lamtil_m 
&\leq 
4\sum_{i=1}^{m-1} (np_i\lambda - np_{m+1} \lambda) \sqrt{\ergen_i(\delta)}  + 
(np_m\lambda - np_{m+1}\lambda) + np_{m+1}\lambda(1+(k-1)\delta) \\
&= np_m\lambda + \left(4\sum_{i=1}^{m-1} (np_i\lambda - np_{m+1} \lambda) \sqrt{\ergen_i(\delta)} + np_{m+1}\lambda(k-1)\delta \right) \\
&= np_m\lambda + u_m(\delta).
\end{align*}
To lower bound $\lamtil_m$, we start with \eqref{eq:lamtil_m_bd_2} and 
the assumption $\ergen_1(\delta) \leq \cdots \leq \ergen_{m-1}(\delta)$. This leads to
\begin{align*}  
\lamtil_m 
&\geq np_{m}\lambda - 4\sqrt{2\ergen_{m-1}(\delta)}\left(\sum_{j=1}^{m-1} (np_j\lambda + u_j(\delta)) \right) \\
&\geq np_{m}\lambda - 4\sqrt{2\ergen_{m-1}(\delta)} (n\lambda S_{m-1} + \sum_{j=1}^{m-1} u_j(\delta)).
\end{align*}
Using the definition of $u_j(\delta)$, we then obtain
{\withMaybeSmall
\begin{align*}
\lamtil_m 
&\geq np_{m}\lambda - 4\sqrt{2\ergen_{m-1}(\delta)} \biggl(n\lambda S_{m-1}  
+ \sum_{j=1}^{m-1} \biggl(2\sum_{i=1}^{j-1} (np_i - np_{j+1})\lambda\sqrt{\ergen_i(\delta)} 
 + np_{j+1}\lambda(k-1)\delta\biggr)\biggr) \\
%
%
&\geq np_{m}\lambda - 4\sqrt{2\ergen_{m-1}(\delta)} \biggl(n\lambda S_{m-1}  
+ \sum_{j=1}^{m-1} \biggl(\sqrt{2} (j-1) (np_1 - np_{j+1})\lambda + np_{j+1}\lambda(k-1) \biggr)\biggr) \\
&\geq np_{m}\lambda - 4\sqrt{2\ergen_{m-1}(\delta)} \biggl(n\lambda S_{m-1}  
+ \sqrt{2} (m-1)(m-2) (np_1 - np_{m})\lambda + np_{2}\lambda(k-1)(m-1) \biggr) \\
&= np_{m}\lambda - l_m(\delta).
\end{align*}
}
\end{enumerate} 

%
\subsection{Proof of Lemma \ref{lem:mat_pert_dk_gen}} \label{app:subsec_mat_pert_dk_gen}
Recall that $H = \mathbb{E}(H) + R = \sum_{i=1}^{k} n\lambda p_i z_i z_i^* + R$. 
Also recall that via Weyl's inequality  
(see Theorem \ref{thm:Weyl} in Appendix \ref{app:sec_perturb_theory}), it holds
\begin{align} \label{eq:eig_bds_tmp_gen}
\lambda_i \in [\lamtil_i - \triangle, \lamtil_i + \triangle]; \quad i=1,\dots,n.
\end{align}
As in Lemma \ref{lem:mat_pert_dk}, we will obtain the bounds via an application of 
the Davis-Kahan theorem (see Theorem \ref{thm:DavisKahan} in Appendix \ref{app:sec_perturb_theory}). 
From now, we will assume that $\delta$ satisfies the conditions stated in Lemma \ref{lem:deflation_gen}. 
Before proceeding to the analysis, it will be helpful to use the following upper bound on $u_j(\delta)$ 
for each $j=2,\dots,k$.
\begin{align}
u_j(\delta) 
&= 4\sum_{i=1}^{j-1} (np_i - np_{j+1})\lambda\sqrt{\ergen_i(\delta)} + np_{j+1}\lambda(k-1)\delta \nonumber \\
&\leq 4(j-1)(np_1 - np_{j+1})\lambda \sqrt{\ergen_{j-1}(\delta)} + np_{j+1}\lambda(k-1)\delta \nonumber \\
&\leq n\lambda \sqrt{\ergen_{j-1}(\delta)}[4(j-1)(p_1 - p_{j+1}) + p_{j+1}(k-1)]. \label{eq:bd_uj}
\end{align}
\begin{enumerate}
\item \textbf{(Bounding $\abs{\dotprod{\vtil_1}{v_1}}^2$)} 
From Theorem \ref{thm:DavisKahan}, we observe that 
\begin{equation} \label{eq:vspace_bd1_gen}
\norm{(I - v_1v_1^{*})\vtil_1}_2 \leq \frac{\triangle}{\abs{\lamtil_1 - \lambda_2}},
\end{equation}
if $\abs{\lamtil_1 - \lambda_2} > 0$. Since $\lamtil_1 \geq np_1\lambda$ (see Lemma \ref{lem:deflation_gen}), and $\lambda_2 \leq \lamtil_2 + \triangle \leq np_2\lambda + u_2(\delta) + \triangle$ (from \eqref{eq:eig_bds_tmp_gen} and Lemma \ref{lem:deflation_gen}), it follows that 
\begin{align*}  
  \lamtil_1 - \lambda_2 
	&\geq \lamtil_1 - \lamtil_2 - \triangle \\ 
	&\geq n\lambda(p_1-p_2) - u_2(\delta) - \triangle \\
	&\geq n\lambda(p_1-p_2) - n\lambda \sqrt{\ergen_{1}(\delta)}(4(p_1 - p_3) + p_{3}(k-1)) - \triangle.
\end{align*}
Hence for $\mu \in [0,1/2]$, we have that $\lamtil_1 - \lambda_2 \geq (1-\mu) n\lambda(p_1-p_2) > 0$ if  
\begin{align} \label{eq:ergen_j_0_cond}
\triangle \leq \mu\frac{n\lambda(p_1-p_2)}{2}, \quad \ergen_{1}(\delta) \leq \mu^2 \left(\frac{p_1-p_2}{8(p_1 - p_3) + 2p_{3}(k-1)}\right)^2.
\end{align} 
In particular, it then follows from \eqref{eq:vspace_bd1_gen}, 
that
\begin{equation*}
\norm{(I - v_1v_1^{*})\vtil_1}_2 \leq \frac{\mu}{2(1-\mu)}.
\end{equation*}
Finally, the stated bound on $\abs{\dotprod{\vtil_1}{v_1}}^2$ follows from the identity $\norm{(I - v_1v_1^{*})\vtil_1}_2^2 = 1 - \abs{\dotprod{\vtil_1}{v_1}}^2$.

\item \textbf{(Bounding $\abs{\dotprod{\vtil_j}{v_j}}^2$ for $1 < j < k$)} 
From Theorem \ref{thm:DavisKahan}, we observe that 
\begin{equation} \label{eq:vspace_bd2_gen}
\norm{(I - v_jv_j^{*})\vtil_j}_2 \leq \frac{\triangle}{\min\set{\abs{\lamtil_j - \lambda_{j+1}}, \abs{\lamtil_j - \lambda_{j-1}}}}, 
\end{equation}
if $\abs{\lamtil_j - \lambda_{j+1}}, \abs{\lamtil_j - \lambda_{j-1}} > 0$.

Let us first establish conditions under which $\abs{\lamtil_j - \lambda_{j-1}} > 0$ holds. Since $\lambda_{j-1} \geq \lamtil_{j-1} - \triangle$, it follows that 
\begin{align*}  
 & \lambda_{j-1} - \lamtil_j
	\geq \lamtil_{j-1} - \lamtil_{j} - \triangle \\ 
	&\geq n\lambda(p_{j-1}-p_{j}) - l_{j-1}(\delta) - u_{j}(\delta) - \triangle \quad (\text{From Lemma \ref{lem:deflation_gen}})\\
	&= n\lambda(p_{j-1}-p_{j}) 
	- 4n\lambda \sqrt{2\ergen_{j-2}(\delta)}(S_{j-2} + \sqrt{2} (j-2)(j-3) (p_1 - p_{j-1}) + p_{2}(k-1)(j-2)) \\
	&\hspace{10mm}  - n\lambda \sqrt{\ergen_{j-1}(\delta)}[4(j-1)(p_1 - p_{j+1}) + p_{j+1}(k-1)] - \triangle \quad (\text{From Lemma \ref{lem:deflation_gen}, \eqref{eq:bd_uj}}) \\
	&\geq n\lambda(p_{j-1}-p_{j}) - n\lambda \sqrt{\ergen_{j-1}(\delta)} E_j - \triangle \quad (\text{Since } \ergen_{j-2}(\delta) \leq \ergen_{j-1}(\delta)),
\end{align*}
where, 
Hence for $\mu \in [0,1/2]$, we have that $\lambda_{j-1} - \lamtil_j \geq (1-\mu) n\lambda(p_{j-1}-p_{j}) > 0$ if 
\begin{align} \label{eq:ergen_j_1_cond}
\triangle \leq \mu\frac{n\lambda(p_{j-1}-p_{j})}{2}, \quad \ergen_{j-1}(\delta) \leq \mu^2\left(\frac{p_{j-1} - p_j}{2E_j}\right)^2.
\end{align}

Now let us establish conditions under which $\abs{\lamtil_j - \lambda_{j+1}} > 0$ holds. Since $\lambda_{j+1} \leq \lamtil_{j+1} + \triangle$, it follows that 
\begin{align*}  
 & \lamtil_j - \lambda_{j+1} 
	\geq \lamtil_j - \lamtil_{j+1} - \triangle \\ 
	&\geq n\lambda(p_j-p_{j+1}) - l_j(\delta) - u_{j+1}(\delta) - \triangle \quad (\text{From Lemma \ref{lem:deflation_gen}})\\
	&= n\lambda(p_j-p_{j+1}) 
	- 4n\lambda \sqrt{2\ergen_{j-1}(\delta)}(S_{j-1} + \sqrt{2} (j-1)(j-2) (p_1 - p_{j}) + p_{2}(k-1)(j-1)) \\
	&\hspace{10mm}  - n\lambda \sqrt{\ergen_{j}(\delta)}[4j(p_1 - p_{j+2}) + p_{j+2}(k-1)] - \triangle \quad (\text{From Lemma \ref{lem:deflation_gen}, \eqref{eq:bd_uj}}) \\
	&\geq n\lambda(p_j-p_{j+1}) 
	- n\lambda \sqrt{\ergen_{j}(\delta)}\bigl(4\sqrt{2} S_{j-1} + 8 (j-1)(j-2) (p_1 - p_{j}) + 4\sqrt{2} p_{2}(k-1)(j-1) \\
	&\hspace{10mm}  + 4j(p_1 - p_{j+2}) + p_{j+2}(k-1)\bigr) - \triangle \quad (\text{Since } \ergen_{j-1}(\delta) \leq \ergen_{j}(\delta)) \\
	&= n\lambda(p_j-p_{j+1}) - n\lambda \sqrt{\ergen_j(\delta)} E_{j+1} - \triangle	
\end{align*}
%
%
%
Hence for $\mu \in [0,1/2]$ we have that $\lamtil_j - \lambda_{j+1} \geq (1-\mu)n\lambda(p_j-p_{j+1}) > 0$ if 
\begin{align} \label{eq:ergen_j_2_cond}
\triangle \leq \mu\frac{n\lambda(p_j-p_{j+1})}{2}, \quad \ergen_{j}(\delta) \leq \mu^2\left(\frac{p_j-p_{j+1}}{2E_{j+1}}\right)^2.
\end{align}
Therefore from \eqref{eq:ergen_j_1_cond}, \eqref{eq:ergen_j_2_cond}, it follows that 
$$\lambda_{j-1} - \lamtil_j, \lamtil_j - \lambda_{j+1} \geq n\lambda(1-\mu)\min\set{p_{j-1}-p_{j}, p_j-p_{j+1}} > 0$$ holds if 
\begin{equation*}
\withMaybeSmall  
\triangle < \mu\frac{n\lambda}{2} \min\set{(p_j-p_{j+1}), (p_{j-1}-p_{j})};  \;  \ergen_{j-1}(\delta) < \mu^2\left(\frac{p_{j-1} - p_j}{2E_j}\right)^2; \;  \ergen_{j}(\delta) < \mu^2\left(\frac{p_j-p_{j+1}}{2E_{j+1}}\right)^2.
\end{equation*}
In particular, it then follows from \eqref{eq:vspace_bd2_gen} that
{\withMaybeSmall
\begin{align*} 
\norm{(I - v_jv_j^{*})\vtil_j}_2 &\leq \frac{\triangle}{\min\set{ n\lambda(p_{j-1}-p_{j}) - n\lambda \sqrt{\ergen_{j-1}(\delta)} E_j - \triangle, n\lambda(p_j-p_{j+1}) - n\lambda \sqrt{\ergen_j(\delta)} E_{j+1} - \triangle }} \\
&\leq \frac{\mu}{2(1-\mu)}.
\end{align*}
} 
The bound on $\abs{\dotprod{\vtil_j}{v_j}}^2$ now follows from the identity $\norm{(I - v_jv_j^{*})\vtil_j}_2^2 = 1 - \abs{\dotprod{\vtil_j}{v_j}}^2$.

\item \textbf{(Bounding $\abs{\dotprod{\vtil_k}{v_k}}^2$)} From Theorem \ref{thm:DavisKahan}, we observe that 
\begin{equation} \label{eq:vspace_bd3_gen}
\norm{(I - v_kv_k^{*})\vtil_k}_2 \leq \frac{\triangle}{\min\set{\abs{\lamtil_k - \lambda_{k+1}}, \abs{\lamtil_k - \lambda_{k-1}}}}, 
\end{equation}
if $\abs{\lamtil_k - \lambda_{k+1}}, \abs{\lamtil_k - \lambda_{k-1}} > 0$.

Let us first see when $\abs{\lamtil_k - \lambda_{k-1}} > 0$ holds. By proceeding in an identical manner as before (for showing when $\lambda_{j-1} - \lamtil_j > 0$ holds for $1 < j < k$), 
we readily obtain $\lambda_{k-1} - \lamtil_k \geq n\lambda(p_{k-1}-p_{k}) - n\lambda \sqrt{\ergen_{k-1}(\delta)} E_k - \triangle$, where $E_k$ is as defined in \eqref{eq:E_j_def} by 
plugging $j = k$. 
Hence for $\mu \in [0,1/2]$, we have that 
$\lambda_{k-1} - \lamtil_k \geq (1-\mu)n\lambda(p_{k-1}-p_{k}) > 0$ holds if 
\begin{equation} \label{eq:ergen_j_3_cond}
\triangle \leq \mu\frac{n\lambda(p_{k-1}-p_{k})}{2}, \quad 
\ergen_{k-1}(\delta) \leq \mu^2\left(\frac{p_{k-1} - p_k}{2E_k}\right)^2.
\end{equation}

Now let us establish conditions under which $\abs{\lamtil_k - \lambda_{k+1}} > 0$ holds. Since $\lambda_{k+1} \leq \lamtil_{k+1} + \triangle = \triangle$, it follows that 
\begin{align*}  
  \lamtil_k - \lambda_{k+1} 
	&\geq \lamtil_k - \triangle \\ 
	&\geq n\lambda p_k - l_k(\delta) - \triangle \quad (\text{From Lemma \ref{lem:deflation_gen}})\\
	&= n\lambda p_k - 4n\lambda \sqrt{2\ergen_{k-1}(\delta)}(S_{k-1} + \sqrt{2} (k-1)(k-2) (p_1 - p_{k}) \\ 
	&\hspace{10mm} + p_{2}(k-1)^2) - \triangle \quad (\text{From Lemma \ref{lem:deflation_gen}, \eqref{eq:bd_uj}}) \\
&= n\lambda p_k - n\lambda \sqrt{\ergen_{k-1}(\delta)} \widetilde{E} - \triangle.
\end{align*}
Finally, for $\mu \in [0,1/2]$, we see that $\lamtil_k - \lambda_{k+1} \geq n\lambda(1-\mu)p_k > 0$ holds if 
\begin{equation} \label{eq:ergen_j_4_cond}
\triangle \leq \mu\frac{n\lambda p_{k}}{2}, \quad 
\ergen_{k-1}(\delta) \leq \mu^2\left(\frac{p_k}{2\widetilde{E}}\right)^2.
\end{equation}
Therefore from \eqref{eq:ergen_j_3_cond},\eqref{eq:ergen_j_4_cond}, it follows that 
$$\lambda_{k-1} - \lamtil_k, \lamtil_k - \lambda_{k+1} \geq (1-\mu)n\lambda\min\set{p_{k-1}-p_{k}, p_k} > 0,$$ 
provided that
\begin{equation*}  
\triangle \leq \mu\frac{n\lambda}{2} \min\set{p_{k}, (p_{k-1}-p_{k})} , \quad 
\ergen_{k-1}(\delta) \leq \mu^2 \min \set{\left(\frac{p_{k-1} - p_k}{2E_k}\right)^2, \left(\frac{p_k}{2\widetilde{E}}\right)^2}.
\end{equation*}
In particular, it then follows from \eqref{eq:vspace_bd3_gen} that 
\begin{align*}
\norm{(I - v_kv_k^{*})\vtil_k}_2 
&\leq 
\frac{\triangle}{\min\set{n\lambda(p_{k-1}-p_{k}) - n\lambda \sqrt{\ergen_{k-1}(\delta)} E_k - \triangle, n\lambda p_k - n\lambda \sqrt{\ergen_{k-1}(\delta)} \widetilde{E} - \triangle}} \\
&\leq \frac{\mu}{2(1-\mu)}.
\end{align*}
Finally, the bound on $\abs{\dotprod{\vtil_k}{v_k}}^2$ is obtained in the same manner as explained for previous cases.
\end{enumerate}
%
%

%
\subsection{Proof of Lemma \ref{lem:rand_pert_bd_gen}} \label{app:subsec_rand_pert_bd_gen}
As the proof follows the same arguments as that in the proof of 
Lemma \ref{lem:rand_pert_bd}, we only point out the main differences. 
\begin{enumerate}
\item (\textbf{Bounding $\norm{\real(R)}_2$.}) 
Recall the definition of $\tilde{\sigma}$, $\tilde{\sigma}_{*}$ as in \eqref{eq:sigma_sigtil_def}. 
From \eqref{eq:infty_norm_grn_bd_1}, we have that $\tilde{\sigma}_{*} \leq \sigtil(\lambda,p_1,\dots,p_k)$.
Moreover, note that  
\begin{align}
\expec[\real(R)_{ij}^2] 
&= \underbrace{\sum_{l=1}^k p_l\lambda[\cos(\theta_{l,i}-\theta_{l,j}) - \sum_{l'=1}^k p_{l'}\lambda\cos(\theta_{l',i}-\theta_{l',j})]^2}_{P_1} \nonumber \\
&+ \underbrace{(1-\sum_{l=1}^k p_l)\lambda \expec[\cos N_{ij} - \sum_{l'=1}^k p_{l'}\lambda \cos(\theta_{l',i}-\theta_{l',j})]^2}_{P_2} \nonumber \\
&+ \underbrace{(1-\lambda)[\sum_{l=1}^k p_l \lambda \cos(\theta_{l,i}-\theta_{l,j})]^2}_{P_3} \nonumber \\
&= P_1 + P_2 + P_3. \label{eq:p_i_terms}
\end{align}
One can easily verify the following inequalities. 
\begin{align*}
P_1 &\leq \sum_{l=1}^k 2p_l \lambda [(1-p_l \lambda)^2 + (\sum_{l' \neq l} p_{l'}\lambda)^2], \\
%
P_2 &\leq (1-\sum_{l=1}^k p_l)\lambda \left[\frac{1}{2} + (\sum_{l'=1}^k p_{l'} \lambda)^2 \right] \mbox{ and } 
P_3 \leq (1-\lambda) (\sum_{l'=1}^k p_{l'} \lambda)^2.
\end{align*}
Plugging these in \eqref{eq:p_i_terms} readily leads to $\tilde{\sigma} \leq \sqrt{C(\lambda,p_1,\dots,p_k) n}$. 
Therefore, using Theorem \ref{app:thm_symm_rand} with $t = 2\sqrt{2}\sqrt{C(\lambda,p_1,\dots,p_k) n}$, 
we obtain for constants $\varepsilon \geq 0, c_{\varepsilon} > 0$ that 
\begin{equation} \label{eq:symm_norm_bd_1_gen}
\prob\left(\set{\norm{\real(R)}_2 \geq (2+\varepsilon)2\sqrt{2} \sqrt{C(\lambda,p_1,\dots,p_k) n}} \right) \leq 
n\exp\left(- \frac{8 C(\lambda,p_1,\dots,p_k) n}{\sigtil(\lambda,p_1,\dots,p_k)^2 c_{\varepsilon}} \right).
\end{equation}

\item (\textbf{Bounding $\norm{\imag(R)}_2$.}) By proceeding analogously as 
in the proof of Lemma \ref{lem:rand_pert_bd}, we obtain for 
constants $\varepsilon \geq 0, c_{\varepsilon} > 0$ that
\begin{equation} \label{eq:skew_symm_norm_gen_bd_1}
\prob\left(\set{\norm{\imag(R)}_2 \geq (2+\varepsilon)4\sqrt{2} \sqrt{C(\lambda,p_1,\dots,p_k) n}} \right) \leq 
2n\exp\left(- \frac{16 C(\lambda,p_1,\dots,p_k) n}{\sigtil(\lambda,p_1,\dots,p_k)^2 c_{\varepsilon}} \right).
\end{equation}
\end{enumerate}
Finally, the statement in the Lemma follows by applying the union 
bound to \eqref{eq:symm_norm_bd_1_gen}, \eqref{eq:skew_symm_norm_gen_bd_1}.

%% file: arXiv.bbl
\begin{thebibliography}{10}

\bibitem{Amini_2013}
{\sc A.~A. Amini, A.~Chen, P.~J. Bickel, and E.~Levina}, {\em Pseudo-likelihood
  methods for community detection in large sparse networks}, The Annals of
  Statistics, 41 (2013), p.~2097–2122.

\bibitem{andersson2001new}
{\sc G.~Andersson, L.~Engebretsen, and J.~H{\aa}stad}, {\em A new way of using
  semidefinite programming with applications to linear equations mod p},
  Journal of Algorithms, 39 (2001), pp.~162--204.

\bibitem{structFromMotion_Amit}
{\sc M.~{Arie-Nachimson}, S.~Z. {Kovalsky}, I.~{Kemelmacher-Shlizerman},
  A.~{Singer}, and R.~{Basri}}, {\em Global motion estimation from point
  matches}, in 2012 Second International Conference on 3D Imaging, Modeling,
  Processing, Visualization Transmission, 2012, pp.~81--88.

\bibitem{bandeira2017tightness}
{\sc A.~S. Bandeira, N.~Boumal, and A.~Singer}, {\em Tightness of the maximum
  likelihood semidefinite relaxation for angular synchronization}, Mathematical
  Programming, 163 (2017), pp.~145--167.

\bibitem{nonUniqueGamesSync}
{\sc A.~S. Bandeira, Y.~Chen, R.~R. Lederman, and A.~Singer}, {\em Non-unique
  games over compact groups and orientation estimation in cryo-{EM}}, Inverse
  Problems, 36 (2020), p.~064002.

\bibitem{bandeira2018notes}
{\sc A.~S. Bandeira, A.~Perry, and A.~S. Wein}, {\em Notes on
  computational-to-statistical gaps: predictions using statistical physics},
  2018.

\bibitem{bandeira2016}
{\sc A.~S. Bandeira and R.~van Handel}, {\em Sharp nonasymptotic bounds on the
  norm of random matrices with independent entries}, Ann. Probab., 44 (2016),
  pp.~2479--2506.

\bibitem{Barabasi99emergenceScaling}
{\sc A.-L. Barabasi and R.~Albert}, {\em Emergence of scaling in random
  networks}, Science, 286 (1999), pp.~509--512.

\bibitem{Barvinok1995}
{\sc A.~I. Barvinok}, {\em Problems of distance geometry and convex properties
  of quadratic maps}, Discrete {\&} Computational Geometry, 13 (1995),
  pp.~189--202.

\bibitem{biswas_stress_sdp}
{\sc P.~Biswas, T.~C. Lian, T.~C. Wang, and Y.~Ye}, {\em Semidefinite
  programming based algorithms for sensor network localization}, ACM
  Transactions on Sensor Networks, 2 (2006), pp.~188--220.

\bibitem{boumal2016nonconvex}
{\sc N.~Boumal}, {\em Nonconvex phase synchronization}, SIAM Journal on
  Optimization, 26 (2016), pp.~2355--2377.

\bibitem{Burer2005}
{\sc S.~Burer and R.~Monteiro}, {\em Local minima and convergence in low-rank
  semidefinite programming}, Mathematical Programming, 103 (2005),
  pp.~427--444.

\bibitem{chaudhuri12}
{\sc K.~Chaudhuri, F.~Chung, and A.~Tsiatas}, {\em Spectral clustering of
  graphs with general degrees in the extended planted partition model}, in 25th
  Annual Conference on Learning Theory, vol.~23 of Proceedings of Machine
  Learning Research, Edinburgh, Scotland, 2012, JMLR Workshop and Conference
  Proceedings, pp.~35.1--35.23.

\bibitem{Connelly3}
{\sc R.~Connelly}, {\em On generic global rigidity}, Applied Geometry and
  Discrete Mathematics, 4 (1991), pp.~147--155.

\bibitem{Connelly}
\leavevmode\vrule height 2pt depth -1.6pt width 23pt, {\em Generic global
  rigidity}, Discrete Comput. Geom, 33 (2005), pp.~549--563.

\bibitem{cox}
{\sc T.~F. Cox and M.~A.~A. Cox}, {\em Multidimensional Scaling}, Monographs on
  Statistics and Applied Probability 88, Chapman \& Hall/CRC, Boca Raton, FL,
  2001.

\bibitem{syncRank}
{\sc M.~Cucuringu}, {\em {{Sync-Rank: Robust Ranking, Constrained Ranking and
  Rank Aggregation via Eigenvector and Semidefinite Programming
  Synchronization}}}, IEEE Transactions on Network Science and Engineering, 3
  (2016), pp.~58--79.

\bibitem{DirectedClustImbCuts}
{\sc M.~Cucuringu, H.~Li, H.~Sun, and L.~Zanetti}, {\em Hermitian matrices for
  clustering directed graphs: insights and applications}, AISTATS 2020.

\bibitem{asap2d}
{\sc M.~Cucuringu, Y.~Lipman, and A.~Singer}, {\em Sensor network localization
  by eigenvector synchronization over the {E}uclidean group}, ACM Trans. Sen.
  Netw., 8 (2012), pp.~19:1--19:42.

\bibitem{asap3d}
{\sc M.~Cucuringu, A.~Singer, and D.~Cowburn}, {\em Eigenvector
  synchronization, graph rigidity and the molecule problem}, Information and
  Inference, 1 (2012), pp.~21--67.

\bibitem{SPONGE2020regularized}
{\sc M.~Cucuringu, A.~V. Singh, D.~Sulem, and H.~Tyagi}, {\em Regularized
  spectral methods for clustering signed networks}, arXiv:2011.01737,  (2020).

\bibitem{SVDRank}
{\sc {d'Aspremont, A., Cucuringu, M., and Tyagi, H.}}, {\em {Ranking and
  synchronization from pairwise measurements via SVD}}, accepted to Journal of
  Machine Learning Research, arXiv:1906.02746,  (2020).

\bibitem{daviskahan}
{\sc C.~Davis and W.~M. Kahan}, {\em The rotation of eigenvectors by a
  perturbation. iii}, SIAM Journal on Numerical Analysis, 7 (1970), pp.~1--46.

\bibitem{FANUEL2017JACHA}
{\sc M.~Fanuel and J.~Suykens}, {\em Deformed laplacians and spectral ranking
  in directed networks}, Applied and Computational Harmonic Analysis,  (2017).

\bibitem{FeigeLovaszTwoProver}
{\sc U.~Feige and L.~Lov\'{a}sz}, {\em Two-prover one-round proof systems:
  Their power and their problems (extended abstract)}, in Proceedings of the
  Twenty-Fourth Annual ACM Symposium on Theory of Computing, STOC '92, New
  York, NY, USA, 1992, Association for Computing Machinery, p.~733–744.

\bibitem{Harvard1}
{\sc S.~J. Gortler, A.~D. Healy, and D.~P. Thurston}, {\em Characterizing
  generic global rigidity}, AMERICAN JOURNAL OF MATHEMATICS, 4 (2010), p.~897.

\bibitem{gotsman}
{\sc C.~Gotsman and Y.~Koren}, {\em Distributed graph layout for sensor
  networks}, in Proc. Intl. Symp. Graph Drawing, 2004, pp.~273--284.

\bibitem{molecule_problem}
{\sc B.~Hendrickson}, {\em The molecule problem: Exploiting structure in global
  optimization}, SIAM J Optimization, 5 (1995), pp.~835--857.

\bibitem{kunegis2010spectral}
{\sc J.~Kunegis, S.~Schmidt, A.~Lommatzsch, J.~Lerner, E.~W.~D. Luca, and
  S.~Albayrak}, {\em Spectral Analysis of Signed Graphs for Clustering,
  Prediction and Visualization}, SIAM, 2010, pp.~559--570.

\bibitem{li98}
{\sc R.-C. Li}, {\em Relative perturbation theory: Ii. eigenspace and singular
  subspace variations}, SIAM Journal on Matrix Analysis and Applications, 20
  (1998), pp.~471--492.

\bibitem{perry2016optimality}
{\sc A.~Perry, A.~S. Wein, A.~S. Bandeira, and A.~Moitra}, {\em {Optimality and
  sub-optimality of PCA for spiked random matrices and synchronization}}, arXiv
  preprint arXiv:1609.05573,  (2016).

\bibitem{perry2018message}
\leavevmode\vrule height 2pt depth -1.6pt width 23pt, {\em {Message-Passing
  Algorithms for Synchronization Problems over Compact Groups}}, Communications
  on Pure and Applied Mathematics, 71 (2018), pp.~2275--2322.

\bibitem{shkolnisky2012viewing}
{\sc Y.~Shkolnisky and A.~Singer}, {\em Viewing direction estimation in cryo-em
  using synchronization}, SIAM journal on imaging sciences, 5 (2012),
  pp.~1088--1110.

\bibitem{sync}
{\sc A.~Singer}, {\em Angular synchronization by eigenvectors and semidefinite
  programming}, Appl. Comput. Harmon. Anal., 30 (2011), pp.~20--36.

\bibitem{SingerVDM}
{\sc A.~Singer and H.~T. Wu}, {\em {Vector diffusion maps and the Connection
  Laplacian}}, Communications on Pure and Applied Mathematics,  (2012).

\bibitem{singer2011viewing}
{\sc A.~Singer, Z.~Zhao, Y.~Shkolnisky, and R.~Hadani}, {\em Viewing angle
  classification of cryo-electron microscopy images using eigenvectors}, SIAM
  Journal on Imaging Sciences, 4 (2011), pp.~723--759.

\bibitem{syncFrescoes}
{\sc E.~Sizikova and T.~Funkhouser}, {\em Wall painting reconstruction using a
  genetic algorithm}, in Proceedings of the 14th Eurographics Workshop on
  Graphics and Cultural Heritage, GCH '16, Eurographics Association, 2016,
  p.~83–91.

\bibitem{overview}
{\sc M.~Tubaishat and S.~Madria}, {\em Sensor networks: An overview}, in IEEE
  Potentials, vol.~22, 2003, pp.~20--23.

\bibitem{Tzeneva_Thesis}
{\sc T.~Tzeneva}, {\em Global alignment of multiple 3-d scans using eigenvector
  synchronization}, 2011.

\bibitem{vershynin2012}
{\sc R.~Vershynin}, {\em Introduction to the non-asymptotic analysis of random
  matrices}, Cambridge University Press, 2012, pp.~210--268.

\bibitem{Weyl1912}
{\sc H.~Weyl}, {\em Das asymptotische verteilungsgesetz der eigenwerte linearer
  partieller differentialgleichungen (mit einer anwendung auf die theorie der
  hohlraumstrahlung)}, Mathematische Annalen, 71 (1912), pp.~441--479.

\bibitem{Stella_2012}
{\sc S.~Yu}, {\em Angular embedding: {A} robust quadratic criterion}, {IEEE}
  Trans. Pattern Anal. Mach. Intell., 34 (2012), pp.~158--173.

\bibitem{Stella_2009}
{\sc S.~X. Yu}, {\em Angular embedding: From jarring intensity differences to
  perceived luminance}, in 2009 {IEEE} Computer Society Conference on Computer
  Vision and Pattern Recognition {(CVPR} 2009), 20-25 June 2009, Miami,
  Florida, {USA}, {IEEE} Computer Society, 2009, pp.~2302--2309.

\bibitem{dkuseful}
{\sc Y.~Yu, T.~Wang, and R.~J. Samworth}, {\em A useful variant of the
  davis–kahan theorem for statisticians}, Biometrika, 102 (2015),
  pp.~315--323.

\bibitem{zhong2018near}
{\sc Y.~Zhong and N.~Boumal}, {\em Near-optimal bounds for phase
  synchronization}, SIAM Journal on Optimization, 28 (2018), pp.~989--1016.

\end{thebibliography}
